\documentclass[10pt]{article}
\usepackage{preamble}

\usepackage[T1]{fontenc}

\usepackage[bitstream-charter]{mathdesign}

\usepackage{setspace}
\usepackage[framemethod=TikZ]{mdframed}
\usepackage{multicol}
\newmdtheoremenv{program}{Program}
\usepackage[utf8]{inputenc}

\newcommand{\softmax}{\mathsf{softmax}}
\newcommand{\bC}{{\bf C}}
\newcommand{\bD}{{\bf D}}
\newcommand{\bM}{{\bf M}}
\newcommand{\bN}{{\bf N}}
\newcommand{\bW}{{\bf W}}

\newcommand{\bB}{{\bf B}}

\newcommand{\bQ}{{\bf Q}}
\newcommand{\bX}{{\bf X}}

\newcommand{\bY}{{\bf Y}}
\newcommand{\bZ}{{\bf Z}}

\newcommand{\calY}{\mathcal{Y}}

\newcommand{\Sone}{{S_{[2]}}}
\newcommand{\Tone}{{T_{[2]}}}
\newcommand{\Uone}{{U_{[2]}}}
\newcommand{\Ttwo}{{T_{[3]}}}
\newcommand{\Utwo}{{U_{[3]}}}
\newcommand{\Bonecone}{{\Bone'_{[2]}}}
\newcommand{\Bonectwo}{{\Bone'_{[3]}}}

\newcommand{\Stwo}{S_{[3]}}

\newcommand{\Pone}{P_{[2]}}
\newcommand{\Ptwo}{P_{[3]}}

\newcommand{\dee}{d'}
\newcommand{\sfa}{{\sf a}}

\newcommand{\leftside}{s}
\newcommand{\tempd}{\upbeta}
\newcommand{\PiSig}{\Pi_{\sf att}}

\newcommand{\calC}{\mathcal{C}}

\newcommand{\Xdom}{\brc{\pm 1}^{k\times d}}

\renewcommand{\triangleq}{\coloneqq}

\newcommand{\bV}{\mathbf{V}}
\newcommand{\bA}{\mathbf{A}}
\newcommand{\op}{\mathsf{op}}

\newcommand{\Ent}{\mathrm{Ent}}
\newcommand{\cube}[1]{\brc{\pm 1}^{#1}}
\newcommand{\tal}{d_{\sf T}}
\newcommand{\med}{\mathrm{med}}

\newcommand{\unif}[1]{\mathcal{U}_{#1}}
\newcommand{\slice}[2]{\mathcal{S}_{#1,#2}}
\newcommand{\sldist}[2]{\mathcal{U}_{#1,#2}}

\newcommand{\Sig}{\boldsymbol{\Theta}}

\newcommand{\hulldifftemp}{\sigma}

\newcommand{\Bone}{\mathds{1}}
\newcommand{\one}{\vec{1}}
\newcommand{\partialone}{\one_{[2]}}
\newcommand{\partialtwo}{\one_{[3]}}
\renewcommand{\top}{\intercal}

\renewcommand{\epsilon}{\varepsilon}
\newcommand{\J}{\mathbf{J}}

\newcommand{\reffsig}{r_{\sf \Sig}}
\newcommand{\reffw}{r_{\sf \bW}}
\newcommand{\diag}{{\rm diag}}

\newcommand{\D}{\mathrm{d}}
\newcommand{\iu}{\mathrm{i}}

\renewcommand{\le}{\leqslant}

\newcommand{\affinehull}{\mathbf{K}_{\Sig}}

\usepackage[scr=rsfs]{mathalpha}

\usepackage{upgreek}

\newcommand{\Siglbd}{\uplambda}
\newcommand{\Siginc}{\kappa}
\newcommand{\Winc}{\kappa'}
\newcommand{\Wlbd}{\uplambda'}
\newcommand{\Siglightrows}{\upsilon}
\newcommand{\projtrace}{\chi}
\newcommand{\tempa}{\uprho}

\newcommand{\axwprob}{\delta^*}
\newcommand{\spanerr}{\mathscr{E}}

\title{Provably learning a multi-head attention layer}
\author{
    Sitan Chen\thanks{Email: \texttt{sitan@seas.harvard.edu}} \\
    Harvard SEAS
        \and 
    Yuanzhi Li\thanks{Email: \texttt{yuanzhil@andrew.cmu.edu}} \\
    Microsoft Research, CMU}

\begin{document}

\maketitle

\begin{abstract}
The multi-head attention layer is one of the key components of the transformer architecture that sets it apart from traditional feed-forward models. Given a sequence length $k$, attention matrices $\Sig_1,\ldots,\Sig_m\in\mathbb{R}^{d\times d}$, and projection matrices $\mathbf{W}_1,\ldots,\mathbf{W}_m\in\mathbb{R}^{d\times d}$, the corresponding multi-head attention layer $F: \mathbb{R}^{k\times d}\to \mathbb{R}^{k\times d}$ transforms length-$k$ sequences of $d$-dimensional tokens $\mathbf{X}\in\mathbb{R}^{k\times d}$ via
\begin{equation}
    F(\mathbf{X}) \triangleq \sum^m_{i=1} \mathrm{softmax}(\mathbf{X}\Sig_i\mathbf{X}^\top)\mathbf{X}\mathbf{W}_i\,.
\end{equation}
 
In this work, we initiate the study of provably learning a multi-head attention layer from random examples and give the first nontrivial upper and lower bounds for this problem.
\begin{itemize}
    \item Provided $\{\mathbf{W}_i, \Sig_i\}$ satisfy certain non-degeneracy conditions, we give a $(dk)^{O(m^3)}$-time algorithm that learns $F$ to small error given random labeled examples drawn uniformly from $\{\pm 1\}^{k\times d}$.
    \item We prove computational lower bounds showing that in the worst case, exponential dependence on the number of heads $m$ is unavoidable.
\end{itemize}

We chose to focus on Boolean $\mathbf{X}$ to mimic the discrete nature of tokens in large language models, though our techniques naturally extend to standard continuous settings, e.g. Gaussian. Our algorithm, which is centered around using examples to sculpt a convex body containing the unknown parameters, is a significant departure from existing provable algorithms for learning feed-forward networks, which predominantly exploit fine-grained algebraic and rotation invariance properties of the Gaussian distribution. In contrast, our analysis is more flexible as it primarily relies on various upper and lower tail bounds for the input distribution and ``slices'' thereof.
\end{abstract}

\newpage

\tableofcontents

\newpage


\section{Introduction}
The transformer architecture~\cite{vaswani2017attention} is a key component of many state-of-the-art approaches to natural language processing~\cite{devlin2019bert, brown2020language} and vision~\cite{dosovitskiy2021an}. The core layer in a transformer is the multi-head attention layer, which computes a weighted sum of a sequence of vectors based on their pairwise similarities after linear transformations.

Despite the widespread success of transformers in various AI domains~\cite{openai2023gpt4}, our theoretical understanding of transformer networks is still nascent. Empirically, it has been shown that large-scale transformer networks exhibit striking generalization abilities~\cite{bubeck2023sparks}, and even small-scale transformers~\cite{li2023textbooks,gunasekar2023textbooks} can perform quite well on standard coding and logical reasoning benchmarks. Remarkably, these models are trained by simply applying gradient descent to a simple next-token prediction objective. It is an outstanding open question to rigorously prove that transformers can be learned in this fashion, but presently it is even a mystery why such models can be efficiently learned at all.

In this work we isolate a clean theoretical sandbox where one can hope to prove end-to-end guarantees that shed light on such questions. Concretely, we consider the setting of \emph{realizable, distribution-specific PAC learning}. We assume access to training examples of the form $(\bX, F(\bX))$, where every $\bX$ is drawn independently from a ``benign'' input distribution and perfectly labeled by an unknown ground truth transformer $F$. The goal is to output a transformer which achieves small test loss. This question is poorly understood even for a single multi-head attention layer (in fact even for a \emph{single attention head}!). Such a setup is appealing on two counts. Firstly, the empirical success of gradient-based training on real-world problem instances that are non-realizable and ostensibly more challenging suggests \emph{some} algorithm should work in our setting. Secondly, even though a single multi-head attention layer is a far cry from the deeper networks used in practice, this function class is already surprisingly expressive in practice. Notably, it has been observed that even a single transformer block (containing only a single multi-head attention layer and a single two-layer feedforward network)~\cite{eldan2023tinystories} can generate fluent English. We thus ask:

\begin{center}
    {\em Are there natural conditions under which realizable PAC learning of multi-head attention layers in polynomial time is provably possible? What are structural properties of this function class that enable efficient learning?}
\end{center}

\paragraph{Problem specification.} Before stating our main finding, we review the relevant definitions and specify our model for data. Throughout, a \emph{multi-head attention layer} refers to a function $F:\R^{k\times d}\to \R^{k\times d}$ given by
\begin{equation}
    F(\bX) \triangleq \sum^m_{i=1} \softmax(\bX\Sig_i\bX^\top)\bX\bW_i\,. \label{eq:attention}
\end{equation}
where $\bX \in \mathbb{R}^{k \times d}$ is the input sequence, and $\bW_i, \Sig_i \in \mathbb{R}^{d \times d}$ for all $i\in[m]$. The parameter $k$ corresponds to the \emph{sequence length}, and each row of $\bX$ is a \emph{token}. We refer to each $\bW_i$ as a \emph{projection matrix}\footnote{Despite the name, $\bW_i$ need not actually be a projector in the usual linear algebraic sense.} and each $\Sig_i$ as an \emph{attention matrix}. In Eq.~\eqref{eq:attention}, the softmax is applied row-wise so that every row of $\softmax(\bX \Sig_i\bX^\top)$ has nonnegative entries summing to $1$. Each of the $m$ summands $\bX\mapsto \softmax(\bX\Sig_i\bX^\top)\bX\bW_i$ is called an \emph{attention head}. 

\begin{remark}
    In practice, multi-head attention is defined in a seemingly different way as
    \begin{equation}
        \mathrm{concat}(\softmax(\bX\mathbf{Q}_1\mathbf{K}_1^\top\bX^\top)\bX\bV_1\mid \cdots \mid \softmax(\bX\mathbf{Q}_m\mathbf{K}_m^\top\bX^\top)\bX\bV_m)\mathbf{O}\,.
    \end{equation}
    In this expression, in lieu of $\softmax(\bX\bQ_i\bX^\top)\bX\bW_i \in \R^{k\times d}$, each attention head computes a $k\times (d/m)$ matrix of the form $\softmax(\bX\mathbf{Q}_i\mathbf{K}_i^\top\bX^\top)\bX\bV_i$, where $\mathbf{Q}_i, \mathbf{K}_i, \mathbf{V}_i \in \R^{d\times d/m}$. Then the multi-head attention layer concatenates the matrices for the $m$ attention heads column-wise to produce a $k\times d$ matrix, and then right-multiplies this by a global linear transformation $\mathbf{O}\in\R^{d\times d}$ to produce the output.

    Note that we can implement such a function via our parametrization in Eq.~\eqref{eq:attention} by taking $\bQ_i = \mathbf{Q}_i\mathbf{K}_i^\top$ and $\bW_i$ given by $\mathbf{V}'_i \mathbf{O}$, where $\mathbf{V}'_i \in \R^{d\times d}$ is the matrix whose $i$-th block of $d/m$ columns is given by $\mathbf{V}_i$, and whose remaining entries are zero. Thus, our setting is strictly more general as we do not constrain the ground-truth parameters $\bQ_i, \bW_i$ to admit such low-rank decompositions.
\end{remark}

In this work, we assume we are given pairs $(\bX_1,F(\bX_1)),\ldots, (\bX_N, F(\bX_N))$ for some unknown multi-head attention layer $F$, where the sequences $\bX_1,\ldots,\bX_N$ are sampled independently from a distribution $\calD$. Concretely, in this work we consider $\calD$ given by the uniform distribution over $\brc{\pm 1}^{k\times d}$. The goal is to produce an estimate $\wh{F}$ for which $F(\bX)$ and $\wh{F}(\bX)$ are close on average over $\bX\sim\calD$. Our main result is the following algorithmic guarantee:

\begin{theorem}\label{thm:main}
	Let $F: \brc{\pm 1}^{k\times d}\to \R^{k\times d}$ be a multi-head attention layer whose attention and projection matrices $\brc{(\Sig_i,\bW_i)}^m_{i=1}$ are non-degenerate in the sense of Section~\ref{sec:assume}. Then given at least $N = (kd)^{\Theta(m)} + \poly(m,k,d)\cdot\sqrt{\log(1/\delta)}$ examples $(\bX^{(1)}, F(\bX^{(1)}),\ldots (\bX^{(N)}, F(\bX^{(N)}))$ for $\bX^{(1)},\ldots,\bX^{(N)}\sim\brc{\pm 1}^{k\times d}$, there is an algorithm that runs in time $(kd)^{O(m^3)}\cdot\log(1/\delta)$ and with probability $1 - \delta$ outputs estimates $\brc{(\wh{\bW}_i, \wh{\Sig}_i)}^m_{i=1}$ for which the resulting multi-head attention layer $\wh{F}$ with these projection and attention matrices satisfies $\E[\bX\sim\Xdom]{\norm{F(\bX) - \wh{F}(\bX)}^2_F} \le (kd)^{-\Omega(m)}$.
\end{theorem}

\noindent This is the first PAC learning guarantee for nonlinear multi-head (or even single-head) attention, in arguably the most natural setting possible (see Section~\ref{sec:related} for a discussion of the existing theoretical literature on learnability of transformers). In contrast, the analogous question for traditional \emph{feed-forward} architectures has been studied extensively~\cite{janzamin2015beating,sedghi2016provable,bakshi2019learning,ge2018learning,ge2018learning2,sewoong,diakonikolas2020algorithms,zhang2016l1,goel2017reliably,Daniely17, goel2019learning,zhong2017recovery,LiY17,vempala2019gradient,zgu,soltanolkotabi2017learning,zhangps17,diakonikolas2020approximation,LiMZ20,convotron,azll,chen2022learning,smallcovers,chen2023learning}. As we will see in this work however, the mechanisms that make efficient learning possible for the latter are quite different from the ones that are known for the former. We defer a detailed overview of these mechanisms to Section~\ref{sec:overview} and focus here on two key ways in which Theorem~\ref{thm:main} departs significantly from what is known in the feed-forward setting: the choice of distributional assumption and the ways in which the distributional assumption and the function class are (or rather, cannot be) exploited.

\paragraph{Choice of distributional assumption.} Traditionally, work on the learnability of feed-forward models has predominantly focused on the standard Gaussian distribution because of its various nice properties (e.g. rotation invariance) that make the analysis of learning algorithms more tractable. Our techniques carry over without much difficulty to the setting where every coordinate of $\bX$ is sampled independently from standard Gaussian. Indeed the Boolean case we consider, where $\calD$ is uniform over $\brc{\pm 1}^{k\times d}$, is strictly more challenging: throughout our analysis, we often use central limit theorem-style arguments to argue that the Boolean setting does not deviate too much from the Gaussian setting in terms of the various lower and upper tail bounds that we exploit.

Our motivation for considering this more challenging setting is twofold. From a conceptual standpoint, given that transformers have primarily been used in natural language contexts, it is important to understand their learnability over domains that are at least somewhat structurally reminiscent of the ones that arise in practice. The discrete nature of the support of $\calD$, namely $\brc{\pm 1}^{k\times d}$, serves as a proxy, admittedly a highly stylized one, for the discrete domain of tokens that arise in language. 

From a technical standpoint, working over the Boolean cube $\brc{\pm 1}^{k\times d}$ instead of Gaussian space poses a number of unique analysis challenges. For starters, traditionally in the PAC learning literature, guarantees over the Boolean cube were largely achieved by arguing that the functions one would like to learn are well-approximated by low-degree polynomials, so that polynomial regression suffices. As our guarantees apply even in settings where the norms of the attention matrices $\bQ_i$ can be quite large (see Assumption~\ref{assume:Q1notsmall}) so that the softmax in Eq.~\eqref{eq:attention} behaves qualitatively like hard-max, we conjecture that polynomial regression and other kernel methods are insufficient to achieve the guarantee of Theorem~\ref{thm:main}. On the other hand, as mentioned above, modern results on PAC learning feed-forward networks predominantly apply to Gaussian space. As we explain next, the tools driving these feed-forward results break down in the absence of Gaussianity and seem ill-suited to self-attention.




\paragraph{New challenges beyond the feed-forward setting.}

For feed-forward networks, one of the most popular and effective approaches in the theoretical literature has been the \emph{method of moments}. Concretely, this entails estimating correlations between the network output and certain polynomials in the input and setting up an appropriate tensor decomposition problem~\cite{janzamin2015beating,zhong2017recovery,bakshi2019learning,diakonikolas2020algorithms,diakonikolas2020small,chen2023learning,chen2023faster,diakonikolas2023efficiently}. Indeed, tensor decomposition is a natural tool to employ in such settings because of the \emph{linearity} of the output layer of a one-hidden-layer feed-forward neural network. Similar tools are also useful in the context of mixtures of linear regressions~\cite{sedghi2016provable,li2018learning,chen2020learningmix,diakonikolas2020small} where the independence of the choice of linear function from the value of the input leads to a similar linear structure that can be exploited.

Unfortunately, this kind of structure disappears in the Boolean setting, where one loses nice features of Gaussian space like rotation invariance. Worse yet, even under a Gaussian input distribution, multi-head attention does not appear to be amenable to this type of approach. While there is linear structure across the heads, the softmax attention mechanism combines the rows of the input in a highly nonlinear fashion, and it remains unclear whether the higher moments of the joint distribution over $(\bX,F(\bX))$ have any clean description even when $\bX$ is drawn from a nice distribution like Gaussian or uniform over $\brc{\pm 1}^{k\times d}$.

Another aspect of problem structure that has guided much of the literature on learnability of feed-forward networks has been the connection to \emph{multi-index model regression}. The observation is that if the network generating the labels has an input layer with bounded width, then it only depends on the projection of the input to an unknown bounded-dimensional subspace. For example, if the network is a two-layer ReLU network with $m$ neurons, then the function only depends on the projection of the input to the $m$-dimensional subspace spanned by the input weight vectors. This low intrinsic dimensionality is crucial both for establishing separations between different regimes of gradient-based training for this problem \cite{bietti2023learning,arous2021online,dandi2023learning,abbe2023sgd}, and for obtaining performance guarantees that are ``fixed-parameter tractable'' in the sense that the (possibly large) runtime dependence on the intrinsic dimension is decoupled from the much better dependence on the ambient dimension~\cite{chen2020learning,diakonikolas2020algorithms,diakonikolas2020small,chen2022learning}.

In the multi-head attention setting, in fact even in the single-head attention setting of $m = 1$, there is no such hidden $m$-dimensional structure to leverage. If the attention matrix has large rank, the function depends in some complicated way on all directions of the input.

In short, the self-attention setting seems to evade existing algorithmic approaches in the deep learning theory literature. Instead, we need to devise new arguments; as we make clear in the technical overview of Section~\ref{sec:overview}, these arguments will be rooted in \emph{geometric}, rather than algebraic, moment-based techniques.

\paragraph{Exponential dependence on number of heads.}

We complement this upper bound with statistical query and cryptographic lower bounds suggesting that, at least in the worst case, exponential dependence on the number of heads is unavoidable:

\begin{theorem}[Informal, see Theorem~\ref{thm:main_lbd}]\label{thm:lowerbound_informal}
    Any statistical query algorithm for PAC learning a multi-head attention layer with polynomially bounded parameters with respect to the uniform distribution over $\brc{\pm 1}^{k\times d}$, even over $k = 2$ tokens and constant target error, requires either $d^{\Omega(m)}$ queries or $d^{-\Omega(m)}$ tolerance, where $m$ is the number of heads and $d$ is the dimension.

    Additionally, under a variant of the \emph{learning with rounding} (LWR) assumption in lattice-based cryptography (see Conjecture~\ref{conj:lwr}), no polynomial-time algorithm, even a non-SQ one, can PAC learn a multi-head attention layer with polynomially bounded parameters and polynomially many heads with respect to the uniform distribution over $\brc{\pm 1}^{k\times d}$, even over $k = 2$ tokens and constant target error.
\end{theorem}

\noindent To our knowledge these are the first \emph{computational} lower bounds for PAC learning transformers over a benign input distribution. That said, an important caveat is that Theorem~\ref{thm:lowerbound_informal} applies to a somewhat different regime because the lower bound instance actually does not satisfy the ``non-degeneracy'' assumption of Theorem~\ref{thm:main}. In other words, even if exponential dependence on $m$ is necessary in the worst case, it could be that polynomial dependence is possible under the conditions of Theorem~\ref{thm:main}. Nevertheless, we hope that this work will open the door to further exploration of computation-statistical tradeoffs for learning transformers.

\paragraph{Future directions.} Our results leave open a number of interesting directions. Currently, our upper bound requires a number of non-degeneracy assumptions (see Section~\ref{sec:assume}), and it remains to be seen which of these is ultimately necessary to obtain comparable runtime. In fact, even for a single attention head, it is unclear what the most general set of assumptions are under which one can learn in $\poly(d,k,m,1/\epsilon)$ time. 

Additionally, while we give \emph{an} algorithm for this problem, it remains open to analyze why gradient descent succeeds for this learning problem. This is still a grand and largely open challenge even in simple settings for one-hidden-layer feed-forward networks, but given that the setting in this paper seems to exhibit new structural features unique to the self-attention setting, it would be interesting to see to what extent these structural features extend to the analysis of gradient descent. 

Finally, it is of immense interest to prove learnability, or computational hardness of learnability, for deeper architectures. For instance, can one PAC learn a single transformer block, that is, a composition of a multi-head attention layer with a feed-forward layer? What about transformers of bounded depth? Could modifications to the distributional assumption, or the additional assumption that the parameters of the transformer are generated somewhat randomly, aid in the analysis of these models?

\subsection{Related work}
\label{sec:related}

The empirical literature on the optimization behavior of transformers is extensive but out of scope for this work, so in this section we only review provable results about transformers, in addition to prior theoretical work on more traditional architectures.

\paragraph{Learnability of transformers.} 

Ignoring computational aspects, the sample complexity of learning transformers is well-understood~\cite{edelman2022inductive,wei2021statistically,zhang2023analysis,trauger2023sequence} and builds on existing techniques from statistical learning theory. Here we focus on learnability guarantees that factor in the computational complexity of optimization. 

For \emph{single-head} attention, \cite{tarzanagh2023transformers} studied the \emph{asymptotic} convergence of gradient descent to certain solutions of a hard-margin SVM problem in the regime where the dimension is much larger than the number of training examples. Other works in the single-head case provide bespoke analyses of training dynamics in various stylized data models. \cite{jelassi2022vision} showed that over a certain toy model of data with spatial structure, a single-head attention layer trained with gradient descent can learn. \cite{li2023theoretical} showed another such learning result for a different toy model of data involving label-relevant and label-irrelevant tokens, and by training a single-head attention layer composed with a two-layer feed-forward network. \cite{oymak2023role} studied a similar data model and also single-head attention but focused on prompt-tuning where the attention matrix is assumed to be fixed. \cite{tian2023scan} studied training dynamics for a single-head attention layer for a toy data model for next-token prediction, and \cite{li2023topic} studied training dynamics for the same for a data model inspired by topic modeling. Additionally, for in-context learning, \cite{zhang2023trained} studied the dynamics of a single \emph{linear} attention head trained on linear regression tasks.

While all of these results shed light on interesting features of the optimization behavior of single-head attention in simple, self-contained settings, they do not provide any guarantees in the realizable PAC setting of our work even for $m = 1$. The reason is that by design, the data models studied in these works are meant to be trivial if one wants to achieve low test loss using \emph{any} algorithm, and the ``hard'' aspect is to prove that training a single attention head can achieve the same. In contrast, in our work, the data model is such that, even for $m = 1$, it is challenging to prove that \emph{any} learning algorithm succeeds.

Finally, compared to the large number of recent works on training dynamics for single-head attention over toy data models, there is comparatively little that is known for multi-head setting that we consider. Existing guarantees here only apply to \emph{linearized} regimes: \cite{fu2023can} studied learnability in a random features model setting where the attention matrices are frozen and only the projection matrices are trained, and~\cite{deora2023optimization} studied learnability in an NTK-like regime where they prove convergence under the strong condition (see Assumption 2 and Definition 1 therein) that there is a setting of parameters close to initialization which achieves small loss.


\paragraph{Other theoretical work on transformers.}

The bulk of the existing theoretical work on transformers has focused on understanding the representational power of transformers, for example from a universal approximation standpoint~\cite{yun2020are,wei2021statistically} or from the perspective of proving transformers can implement various models of computation~\cite{dehghani2018universal,perez2021attention,wei2021statistically,giannou2023looped,edelman2022inductive,elhage2021mathematical,merrill2022saturated}. Among results of the latter flavor, there have been various works on focusing on formal language recognition and computation by finite-state automata~\cite{bhattamishra2020ability,yao2021self,hahn2020theoretical,hao2022formal,liu2022shortcuts}. Other representational results have focused on the ability of transformers to perform in-context learning (ICL)~\cite{xie2021explanation,zhang2023analysis,von2023transformers,akyurek2022learning,dai2022can,giannou2023looped,bai2023transformers}. The work of~\cite{sanford2023representational} also proves various representational separations between transformers and other architectures using tools from communication complexity.

While all of these results indicate that there \emph{exist} transformers that can implement all kinds of rich computation, it remains unknown how to provably \emph{train} a transformer that can do the same.

Orthogonal to this thrust and to the focus of our result is a series of works on understanding how tokens transforms across multiple attention layers~\cite{lu2020understanding,geshkovski2023emergence,geshkovski2023mathematical} from the perspective of dynamical systems. They regard this process as the discretization of a mean-field interacting particle system and analyze structural features like the clustering of tokens under this evolution.

\paragraph{Provably learning feed-forward neural networks.} 

Our work is the natural continuation of the long line of work on realizable PAC learning of feed-forward neural networks under benign input distributions. As mentioned in the introduction, essentially all of these works focus on the case where the inputs come from the standard Gaussian distribution~\cite{janzamin2015beating,sedghi2016provable,bakshi2019learning,ge2018learning,ge2018learning2,sewoong,diakonikolas2020algorithms,zhang2016l1,goel2017reliably,Daniely17, goel2019learning,zhong2017recovery,LiY17,vempala2019gradient,zgu,soltanolkotabi2017learning,zhangps17,diakonikolas2020approximation,LiMZ20,convotron,azll,chen2022learning,smallcovers,chen2023learning}. In this setting, the best known guarantees for general one-hidden-layer feed-forward networks, i.e. functions of the form $F(x) = \sum^k_{i=1} \lambda_i \sigma(\iprod{w_i,x})$, where $\sigma$ is a known activation, e.g. $\mathsf{relu}$, achieve runtime $\poly(d,1/\epsilon)^{\poly(k)}$~\cite{diakonikolas2023efficiently,chen2023faster}, whereas the best known guarantee for deeper networks~\cite{chen2022learning} achieves runtime $\exp(\poly(S,1/\epsilon))\cdot \poly(d)$, where $S$ is the size of the network. Furthermore, there is now mounting evidence, based on popular cryptographic assumptions, that in the worst case there is no fully polynomial-time algorithm for learning feed-forward networks with at least two hidden layers~\cite{chen2022hardness,daniely2021local}, and in fact with an extra layer, this hardness might persist even beyond the worst case~\cite{daniely2023computational}. 

Additionally, we clarify that while transformers are a key ingredient in language modeling~\cite{openai2023gpt4} and text-to-image models~\cite{chang2023muse}, the focus of the present work is about supervised learning rather than generative modeling. More relevant to the latter are recent works giving end-to-end guarantees or hardness results for the unsupervised learning problem of learning a parametric transformation of noise via a feed-forward neural network~\cite{chen2023polytransform,chen2021minimax,chen2022pushhard}.

\paragraph{Mixtures of linear regressions.} Apart from the vast literature on learning one-hidden-layer networks, another theoretical setting which is related to ours is that of learning mixtures of linear regressions. To see this connection, consider the case of $m = 1$ and $\norm{\Sig_1}_{\sf op} \to \infty$. In this case, the $j$-th token in the output of $F$ would be $F_j(\bX) = \bW_1\bX_{\alpha}$, where $\alpha = \arg\max_{i \in [n]}\{\bX_j^{\top}\Sig_1 \bX_i \}$. This resembles the classical mixture of linear regressions problem, which has similarly received widespread attention in the learning theory community (see e.g. the recent works~\cite{li2018learning,chen2020learningmix,diakonikolas2020small} and references therein). Indeed, if we view $\bW_1\bX_{\alpha} = T_{\alpha} (\bX)$ as a linear function $T_{\alpha}$, then we see that for single-head attention, the labeled examples $(\bX,\bY)$ are each perfectly fit by some linear function $T_\alpha$. The key distinction here however is that the particular choice of linear function to apply for a given input $\bX$ is not independently random but instead \emph{depends} on the value of $\bX$. In that sense, the problem of learning single-head attention is closer in spirit to works on \emph{mixtures of experts}~\cite{makkuva2019breaking,makkuva2020learning} and \emph{max-linear regression}~\cite{cherapanamjeri2023makes}, for which various spectral methods have been developed in toy settings.


\section{Overview of techniques}
\label{sec:overview}

In this section we provide an overview of the key ideas in our proof. Our algorithm operates in six phases. We briefly outline these phases below, before turning to the details of the proofs for each phase later in the overview.

\begin{enumerate}[leftmargin=0pt, label=(\Roman*)]
	\item {\bf Crude estimation of projection matrix sum}: The starting point for our algorithm is the observation that it is possible to produce a noisy estimate for $\sum_i \bW_i$ simply by looking at correlations between the input $\bX$ and the label $\bY$. Indeed, a key structural result that we show is that
	\begin{equation}
		\E{\frac{1}{k}\bX^\top \J \bY} \approx \sum_i \bW_i\,, \label{eq:sumvalue_overview}
	\end{equation}
    where $\J\in\R^{k\times k}$ is the all-ones matrix (see Theorem~\ref{thm:sumvalue}). 

    One helpful interpretation of Eq.~\eqref{eq:sumvalue_overview} is that if one trains a \emph{single attention head} with attention matrix and projection matrix initialized to the all-zeros matrix to try to match $F$, then a single step of gradient descent on the projection matrix with respect to the population square loss with the appropriate step size will result in an approximation of $\sum_i \bW_i$.
    
    The error bound that we achieve is some modest inverse polynomial in the effective ranks of the attention matrices (see Assumption~\ref{assume:effranksig}). The proof of Eq.~\eqref{eq:sumvalue_overview} is quite delicate, but because this step is modular and rather technical, we defer providing intuition for its proof to Section~\ref{sec:sumvalue_overview}. In any case, armed with Eq.~\eqref{eq:sumvalue_overview}, we can form a crude estimate $\wh{\bW}$ of $\sum_i \bW_i$ simply by empirically estimating the left-hand side of Eq.~\eqref{eq:sumvalue_overview}.

    Note that while on the surface this might resemble traditional moment-based analyses, the key conceptual difference from the method of moments here is that the relation in Eq.~\eqref{eq:sumvalue_overview} is approximate even for the \emph{true} expectation. Indeed, the expectation on the left-hand side does not admit any simple closed form, and instead we need to make various approximation arguments to establish Eq.~\eqref{eq:sumvalue_overview}. These approximation arguments come with advantages and disadvantages. On the positive side, these arguments do not heavily exploit the underlying distribution beyond certain upper and lower tail estimates and thus have the potential to extend to much broader ranges of distributional assumptions, unlike traditional moment-based analyses. On the negative side, because even the true expectation $\E{\frac{1}{k}\bX^\top \J \bY}$ incurs some approximation error, we cannot get an arbitrarily good estimate of $\sum_i \bW_i$ in the first phase and need several subsequent steps to refine the estimate obtained in this phase.

	\item {\bf Sculpting a feasible set for the attention matrices}: In the next phase of our algorithm, we leverage our crude estimate of $\sum_i \bW_i$ to construct a convex body that is guaranteed to be somewhat close to the affine hull of the attention matrices. Our key insight for this phase is the following. When we encounter an example $\bX\in\Xdom$ for which the \emph{attention patterns} $\softmax(\bX\Sig_i\bX^\top)$ are all quite similar across $i = 1,\ldots,m$ to a fixed attention pattern $S\in\R^{k\times k}$,\footnote{Actually, in our analysis, it suffices for the \emph{first row} of each attention pattern to be sufficiently similar, as we ultimately only make use of the first row of $\bY$.} then
	\begin{equation}
		F(\bX) \approx \sum^m_{i = 1} S\bX\bW_i = S\bX\cdot \sum^m_{i=1} \bW_i\,.
	\end{equation}
	In other words, the label $\bY = F(\bX)$ is well-approximated by a linear function of the sum of projection matrices, and because $\wh{\bW}$ is a nontrivial approximation to the latter, $\bY$ is well-approximated by a linear function of $\wh{\bW}$ (see the first part of Lemma~\ref{lem:complete_sound}). In fact, we show that the \emph{converse} also holds: if we run linear regression to try to fit $\bY$ to a linear function of $\wh{\bW}$ and find that we succeed, then this must be because the attention patterns across $i = 1,\ldots,m$ were similar (see the second part of Lemma~\ref{lem:complete_sound}). In other words, linear regression gives us a way to \emph{certify} that $\softmax(\bX\Sig_i\bX^\top) \approx \softmax(\bX\Sig_j\bX^\top)$ for all $i,j$, and moreover, it gives us an approximation to this attention pattern.

	If we pretend for the time being that all of these approximations were exact, i.e. that $\softmax(\bX\Sig_i\bX^\top)$ were exactly the same for all $i = 1,\ldots, m$ and that we had access to this attention pattern, then we could read off significant information about the attention matrices. Indeed, the log-ratio between any two entries in a given row of this attention pattern tells us the value of $\bX_{a:} \Sig_i (\bX_{b:} - \bX_{c:})^\top$, i.e. it gives us an affine linear constraint that is simultaneously satisfied by all of $\Sig_1,\ldots,\Sig_m$. In particular, this linear constraint is also satisfied by any affine combination of the attention matrices. One could then hope that after seeing enough examples $\bX$ with the property that the attention patterns are all the same, we have collected enough linear constraints that the resulting feasible region is approximately the affine hull of the attention matrices (see Algorithm~\ref{alg:lp}).

	Of course, we cannot hope for the attention patterns across heads to ever be exactly the same. Worse yet, we only have a very crude approximation to $\sum_i \bW_i$. This means that the linear regression-based certification above incurs significant error that we will have to pay in how well we can approximate the affine hull. These subtleties lie at the technical heart of this work, and we elaborate upon these and other key issues in Section~\ref{sec:sculpt_overview}.

	\item {\bf Refining estimate for the projection matrix sum}: Putting these subtleties aside for the moment, at this point in the proof it is possible to show we have access to a convex body $K$ which is somewhat close to the affine hull of $\Sig_1,\ldots,\Sig_m$.

	\begin{remark}[An approach that does not work using large-margin patterns]\label{remark:fail}
		At this juncture, we proceed to explain a natural procedure for using this convex body that fails, but is close in spirit to our final approach. 

		Given the convex body $K$, we could try to construct an approximation to the \emph{linear span} of the attention matrices and then brute-force over this span to find estimates for $\Sig_1,\ldots,\Sig_m$. We could then use those estimates to improve upon the certification procedure in the previous step. Indeed, rather than merely wait for $\bX$'s which induce similar attention patterns across heads, given any $\bX$ we can now use our approximate knowledge of $\Sig_1,\ldots,\Sig_m$ to estimate the individual attention patterns for each head induced by $\bX$.

		Unfortunately, in reality we would only have these patterns up to some error. But for attention patterns which are \emph{large-margin} in the sense that every row is $1$-sparse, even with approximate access to $\Sig_i$ we can get very good estimates for such patterns. The intuition is that for two vectors $v,v'$ which both have one entry which is much larger than the rest, $\softmax(v)$ and $\softmax(v')$ will be extremely close to each other even if $v,v'$ are only somewhat close. 

		In such large-margin cases, the example $(\bX, \bY)$, combined with our highly accurate knowledge of the attention patterns in each head, provides us with highly accurate linear constraints on the \emph{projection matrices}. By accumulating enough such linear constraints, we can get highly accurate estimates for $\bW_1,\ldots,\bW_m$, at which point there are a number of ways to back out good estimates for $\Sig_1,\ldots,\Sig_m$.
	\end{remark}

	Unfortunately, the approach in Remark~\ref{remark:fail} has a fatal flaw: for convex polytopes in high dimensions, if one has access to a pointwise approximation to the polytope, the approximation error actually needs to be quite small in general (e.g. inverse in the ambient dimension) for it to be possible to learn the span of its vertices to nontrivial error. But our error in approximating the affine hull of the attention matrices in the second stage is much too large for this, ultimately because our estimator $\wh{\bW}$ in the first stage was too crude.

	Nevertheless, Remark~\ref{remark:fail} does suggest a useful idea: large-margin attention patterns as a mechanism for error refinement. Currently we are bottlenecked by the error in estimating $\sum_i \bW_i$ in the first phase, so in the current phase, we will try to improve this estimate.

	Instead of trying to learn approximations to the individual $\Sig_1,\ldots,\Sig_m$ by estimating their span using the convex body $K$, we will instead use the \emph{minimum-norm point} in $K$ as a \emph{proxy} for the attention matrices.

	Concretely, we can show (see Part (II) of Theorem~\ref{thm:main_lp}) that the minimum-norm point $\wt{\Sig}$ in $K$ is a moderately good approximation to a certain convex combination of $\Sig_1,\ldots,\Sig_m$. Crucially, this point is more useful than a generic point in $K$ because it places \emph{comparable weight} on each of the $\Sig_i$'s.

	Our strategy is then the following. Even though we do not know $\Sig_1,\ldots,\Sig_m$, we can still mimic the approach of Remark~\ref{remark:fail} by waiting for examples $\bX$ for which the attention pattern \emph{induced by the attention mechanism given by the minimum-norm point $\wt{\Sig}$}, namely $\softmax(\bX\wt{\Sig}\bX^\top)$, is large-margin! Because $\wt{\Sig}$ is approximately a convex combination of $\Sig_1,\ldots,\Sig_m$ with comparable weight on each of the $\Sig_i$'s, if we know $\softmax(\bX\wt{\Sig}\bX^\top)$ is a large-margin pattern, then the conditional probability that $\softmax(\bX\Sig_i\bX^\top)$ is a large-margin pattern for every $i\in[m]$ is large. In other words, without knowing estimates for $\Sig_1,\ldots,\Sig_m$, we can detect when $\bX$ induces approximately the same large-margin pattern in each attention head by using $\wt{\Sig}$ as a proxy (see Algorithm~\ref{alg:waittorefine}).

	By collecting enough such examples $(\bX, \bY)$ and setting up an appropriate least-squares problem, we prove that we can produce at the end of this phase a significantly refined estimate for $\sum_i \bW_i$, e.g. one that is $d^{-\Omega(m)}$-accurate in Frobenius norm (see Lemma~\ref{lem:proxy}).

	\item {\bf Rerun sculpting algorithm}: Now that we have a much better estimate for $\sum_i W_i$, we can simply rerun the second phase of the algorithm to produce a new convex body $K^*$ which is a much better approximation to the affine hull of the attention matrices. In Part (I) of Theorem~\ref{thm:main_lp}, we show that it does not contain any point which is $d^{-\Omega(m)}$-far from the span of $\Sig_1,\ldots,\Sig_m$, and additionally contains a small neighborhood around each true attention matrix.

	\item {\bf Extracting the span of the attention matrices from the convex body}: Now that we have very accurate access to the affine hull of $\Sig_1,\ldots,\Sig_m$, we can use membership oracle access to the convex body we have produced to estimate the linear span of $\Sig_1,\ldots,\Sig_m$. Roughly speaking, the idea will be to gradually accumulate a set of candidate orthonormal bases for subspaces of this linear span. At each stage, for a valid such orthonormal basis, we can try recentering the convex body around points in the span of this basis and query the membership oracle for a point which is orthogonal to this basis and is somewhat far from zero (see Algorithm~\ref{alg:net}). In Lemma~\ref{lem:main_net} we prove that we can produce a suitable estimate for the linear span of $\Sig_1,\ldots,\Sig_m$ this way, from which we can then produce an epsilon-net over matrices that is guaranteed to contain a good approximation to each of the $\Sig_i$'s.

	\item {\bf Solve for projection matrices}: At this point, we are essentially done. For each $m$-tuple of matrices from the epsilon-net constructed in the previous phase, we run linear regression (see Algorithm~\ref{alg:linreg}) to produce estimates $\brc{\wh{\bW}_i}$ for the projection matrices. We use standard generalization bounds to argue that for the $m$-tuple corresponding to an accurate estimate of each of the attention matrices, the resulting estimates for the projection matrices result in a function which is close to the true multi-head attention layer $F$ in test loss (see Lemma~\ref{lem:linreg}).

	There is the minor drawback that \emph{a priori} we do not know when we have correctly guessed a tuple of matrices from the epsilon-net which is sufficiently accurate, but our (standard) workaround is simply to produce one estimated function for each tuple using linear regression and then pick the best function from this list by evaluating on a held-out validation set.
\end{enumerate}

In the next two sections, we provide additional details for some of the steps above, focusing on phases 1 and 2 which present some of the main technical obstacles that must be overcome in this work.

\subsection{Crude estimation of projection matrix sum}
\label{sec:sumvalue_overview}

Here we provide some additional details on the approximation in Eq.~\eqref{eq:sumvalue_overview}, which we further elaborate upon at the beginning of Section~\ref{sec:sumvalue}. Note that to show this, it suffices to show the approximation
\begin{equation}
    \frac{1}{k}\E{\bX^\top\softmax(\bX\Sig\bX^\top)\bX} \approx \Id \label{eq:diff}
\end{equation}
for any attention matrix $\Sig$ satisfying our non-degeneracy assumptions in Section~\ref{sec:assume}.

Denote the difference between the two sides of Eq.~\eqref{eq:diff} by $\Delta$. At a high level, the idea is to control $u^\top \Delta v$ for all Boolean test vectors $u, v\in\brc{0,1}^d$ which are either equal or have disjoint support, after which a standard bucketing argument (see Section~\ref{sec:bucket}) implies that the operator norm of $\Delta$ is bounded. Note that the expectation
\begin{equation}
    \E{u^\top \bX^\top \softmax(\bX\Sig\bX^\top)\bX v}
\end{equation}
can be decomposed into three sources of randomness: the randomness of $\bX u$, the randomness of $\bX v$, and the remaining randomness in $\bX$. For any fixing of the first two sources, note that the conditional distribution on $\bX$ is exactly the uniform distribution over the intersection of one or two \emph{slices} of the Boolean cube (depending on whether $u = v$). We want to show that for any two different ``typical'' conditionings of $\bX u$ and $\bX v$, the corresponding conditional expectations of $\softmax(\bX\Sig\bX^\top)$ are close to each other. To argue this, we formulate a coupling (see Definition~\ref{def:coupling}) between these conditional distributions based on rerandomizing certain small subsets of the coordinates of $\bX$ (see Section~\ref{sec:coupling}). Bounding the distance between the softmax matrices that arise from this coupling ultimately amounts to using certain concentration and anti-concentration bounds (see Section~\ref{sec:tail_slice_cube} and~\ref{sec:applytail}). The key technical difficulty however is that we need to reason about such bounds not over the uniform distribution over the hypercube, but over the uniform distribution over the intersection of slices of the hypercube. For this, we need to exploit and extend various central limit theorem-style results that have been proven for the Boolean slice to this non-standard setting (see Appendix~\ref{app:HW}). The proof that assembles the above ingredients is rather involved, and we give a more detailed exposition of the strategy at the beginning of Section~\ref{sec:sumvalue}.

\subsection{Sculpting the affine hull}
\label{sec:sculpt_overview}

The other technical core of our results is the analysis of the second (and fourth) phases in which we argue that we can produce a large number of diverse affine linear constraints that are simultaneously satisfied by $\Sig_1,\ldots,\Sig_m$, so as to sculpt out a feasible region $K$ that closely approximates the convex hull of the attention matrices.

Here we elaborate on the kinds of linear constraints that arise. Recall from the above discussion that if all of the attention patterns across heads were the same, we could simply take the log-ratio between two entries of the attention pattern in the same row and produce a constraint on $\Sig_i$. Because of our error in crudely estimating $\sum_i \bW_i$, which leads to substantial noise in the procedure for certifying when attention patterns across heads are all approximately the same, we can only afford to implement this strategy for moderately ``large-margin'' attention patterns that are resilient to this error. But unlike in the analysis of phase 3 of our algorithm, here the large-margin patterns we consider are patterns where the rows are \emph{2-sparse} instead of 1-sparse, so that we can take log-ratios.

The main technical result that we need to establish is that by collecting enough constraints arising from such 2-sparse patterns, we effectively rule out all matrices which are far from the span of the true attention matrices (see Lemma~\ref{lem:main_sculpt}). Concretely, we need to show that for any matrix $\Sig'$ with nontrivial component in the subspace orthogonal to $\Sig_1,\ldots,\Sig_m$, with non-negligible probability we will encounter an example $\bX$ which simultaneously 1) induces approximately the same 2-sparse attention pattern for every head, and 2) induces a very different attention pattern for $\softmax(\bX\Sig'\bX^\top)$. This is already fairly technical to show for Gaussian examples, but because we are working with Boolean $\bX$'s, the difficulty is multiplied substantially.

One of the main obstacles is that the event that we are waiting for has fairly small probability, e.g. $\xi^{\Theta(m)}$ where $\xi$ is some error parameter that, roughly speaking, scales with the final error in estimating the affine hull of the attention matrices. In contrast, most off-the-shelf bounds based on quantitative versions of the central limit theorem are unable to establish \emph{lower bounds} on events that are this rare, because one typically pays some fixed inverse polynomal probability bound to pass from Booleans to Gaussians, and this would need to be of lower order compared to $\xi^{O(m)}$. Indeed, in that case we would be able to take $\xi$ to be, at best, of order $d^{-O(1/m)}$, which is sufficient for the purposes of phase 2 but insufficient for when the sculpting algorithm is called again in phase 4 to refine our estimates.

Our workaround is to use a more powerful central limit theorem, Borovkov's \emph{integro-local central limit theorem}~\cite{borovkov2017generalization}, which allows us to reason about lower bounds on much smaller-probability events. Very roughly speaking, this tool allows us to relate probabilities under Boolean randomness to corresponding probabilities under Gaussian randomness not via a naive triangle inequality on cdf distance a la Berry-Esseen, but via directly comparing to the value of the Gaussian \emph{pdf} at a point in the neighborhood of interest. 

As with the proof of Theorem~\ref{thm:sumvalue}, the proof of Lemma~\ref{lem:main_sculpt} is quite involved, and we defer the details to Section~\ref{sec:helpful}.

\subsection{Lower bound construction}

In this section we briefly describe the proof of our hardness result. For the statistical query lower bound, we show that it is possible to perfectly simulate certain parity functions over the Boolean hypercube using simple multi-head attention layers on $k = 2$ tokens. It is \emph{a priori} unclear how to do this as the softmax function is quite tricky to manipulate. Our key ingredient is a gadget construction (see Proposition~\ref{prop:Gtauv}) which allows us to implement certain one-hidden-layer feed-forward networks with a certain nonstandard tanh-like activation function $\phi$. With this gadget construction, we show how to take combinations of these gadgets in such a way that we can interpolate not just parity functions, but in fact any Boolean function $f: \brc{\pm 1}^{2d}\to\brc{\pm 1}$ of the form 
\begin{equation}
    f(z_1, z_2) h = h\Bigl(\Bigl\langle\frac{1}{2}\cdot \one_S, z_1 - z_2\Bigr\rangle\Bigr)
\end{equation}
where $z_1,z_2\in\brc{\pm 1}^d$, $\one_S\in\brc{0,1}^d$ denotes the indicator vector for some subset $S\subseteq[d]$, and $h$ is an arbitrary $\{0,1\}$-valued function on $\brc{-d,-d+1,\ldots,d-1,d}$. We can easily realize certain parity functions in this way, for which one can prove statistical query lower bounds that are thus inherited by the class of multi-head attention layers that we consider. 

Our cryptographic lower bound proceeds along similar lines, except instead of parity functions, we simulate the functions arising from the \emph{learning with rounding} assumption.

Finally, we note that the main technical challenge in proving that we can interpolate such functions $f$ using one-hidden-layer feed-forward networks with the particular nonstandard activation $\phi$ considered above is establishing that a certain set of linear equations is full-rank, or more precisely, well-conditioned (see the proof of Lemma~\ref{lem:main_lbd_step}). This relies on a condition number bound for tridiagonal matrices~\cite{graham1978distance} (see Lemma~\ref{lem:grahamlovasz}).

\paragraph{Roadmap} In Section~\ref{sec:technical_prelims} we provide technical preliminaries for our proofs. In Section~\ref{sec:assume} we outline the non-degeneracy assumptions we need to make. In Section~\ref{sec:sumvalue} we prove our guarantee for phase one in which we obtain a crude estimate for the sum of the projection matrices. In Section~\ref{sec:sculpt} we describe our sculpting algorithm. In Section~\ref{sec:proxy} we show how to use the minimum-norm point in the convex body produced by the sculpting algorithm as a proxy for refining the estimate of the projection matrix sum. This refined estimate then allows us to rerun the sculpting algorithm to get a closer approximation to the affine hulls of the attention matrices. In Section~\ref{sec:brute} we show how to go from this closer approximation to an approximation to the linear span of the attention matrices. In Section~\ref{sec:linreg} we show how to produce estimates for the projection matrices given good estimates for the attention matrices. In Section~\ref{sec:validate} we combine all of our ingredients to prove Theorem~\ref{thm:main}. In Section~\ref{sec:lowerbound}, we prove our lower bound, Theorem~\ref{thm:lowerbound_informal}.


\section{Technical preliminaries}
\label{sec:technical_prelims}

\subsection{Notation}

Given $S\subseteq[d]$, define $\one_S$ as the vector with $1$'s in the entries indexed by $S$ and $0$'s elsewhere.

Given $S\subseteq[d]$, let $P_S \in \R^{d\times d}$ denote the projector to the coordinates indexed by $S$, and let $P^\perp_S$ denote the projector to the coordinates indexed by $[d]\backslash S$. Let $\Pi^\perp_S$ denote the projector to the orthogonal complement of $\one_S$.

Let $\PiSig: \R^{d\times d}\to\R^{d\times d}$ denote the projector to the span of $\Sig_1,\ldots,\Sig_m$, and $\PiSig^\perp$ the projector to its orthogonal complement.

\subsection{Concentration inequalities}

\begin{theorem}[Hanson-Wright] \label{thm:hanson_wright}
    For any $M\in\R^{d\times d}$, if $x\sim \unif{d}$ and $t > 0$,
    \begin{equation}
        \Pr{|x^\top M x - \Tr(M)| > t} \lesssim \exp\Bigl(-\Omega\Bigl(\frac{t}{\norm{M}_{\sf op}} \wedge \frac{t^2}{\norm{M}^2_F}\Bigr)\Bigr)\,.
    \end{equation}
\end{theorem}

\begin{lemma} \label{thm:conc_quad_boolean}
    Let $M\in\R^{d\times d}$ be symmetric, and let $x\sim\brc{\pm 1}^d$. Then
    \begin{equation}
        \Pr{|x^\top M x - \Tr(M)| > t\sqrt{2\sum_{i\neq j} M^2_{ij}}} \le \exp(-t/e) \,.
    \end{equation}
\end{lemma}

\begin{proof}
    By hypercontractivity, for any degree-2 polynomial $p: \brc{\pm 1}^d\to\R$, we have $\Pr[x\sim\brc{\pm 1}^d]{|p(x) - \E{p}| > t\sqrt{\Var{p}}} \le \exp(-t/e)$. To apply this to $p(x) = x^\top M x$, we must compute $\Var{p}$. For convenience, define $N \triangleq M - \Tr(M) \cdot \Id / d$. Because $p(x) - \E{p} = x^\top N x$, we conclude that
    \begin{equation}
        \Var{p} = \E{(x^\top N x)^2} = \mathbb{E}\sum_{i,j,k,\ell} N_{ij} N_{k\ell} \, x_i x_j x_k x_\ell = (\sum_i N_{ii})^2 + 2\sum_{i\neq j} N^2_{ij} = 2\sum_{i\neq j} M^2_{ij} 
    \end{equation} as desired.
\end{proof}

\subsection{Anti-concentration inequalities}

\noindent We also use the following result on anti-concentration of linear forms over the cube:

\begin{theorem}[Theorem 1.1 in~\cite{dvorak2022probability}]\label{thm:booleananti}
    For $w\in\mathbb{S}^{d-1}$ and uniformly random $z\in \brc{\pm 1}^d$, $\Pr{\iprod{w,z} \ge 1} > 3/32$.
\end{theorem}

\noindent The following result of Kolmogorov, a refined version of which we state below, gives a \emph{lower bound} on the tail probabilities for linear forms over the cube that nearly matches Hoeffding's:

\begin{theorem}[Corollary 20.1.4 from~\cite{nagaev2002lower}]\label{thm:kolmogorov}
    Let $X_1,\ldots,X_n$ be symmetric random variables such that $|X_i| \le M$. Define $\sigma^2 \triangleq \sum_i \E{X_i^2}$. If $1.7 < t \le 0.04\sigma/M$, then if $\Phi(z) \triangleq \Pr[g\sim\calN(0,1)]{g \le z}$, then
    \begin{equation}
        \mathbb{P}\Bigl[\sum_i X_i > t\sigma\Bigr] > (1 - \Phi(t))\cdot (1 - O(Mt/\sigma))\cdot \exp(-O(Mt^3/\sigma))\,.
    \end{equation}
\end{theorem}

\begin{lemma}[Berry-Esseen]\label{lem:berry}
    For any independent mean-zero random variables $Y_1,\ldots,Y_d$ and any interval $I\subset\R$, if $Y \triangleq \frac{1}{d}\sum_i Y_i$ and $Z \sim\calN(0,1)$, then
    \begin{equation}
        |\Pr{Y / \sqrt{\Var{Y}} \in I} - \Pr{Z\in I}| \lesssim \frac{\sum_i \mathbb{E}|Y_i|^3}{\Var{d\cdot Y}^{3/2}}\,.
    \end{equation}
    In particular if $|Y_i| \le \rho\Var{d\cdot Y}^{1/2}$ for all $i\in[d]$, then
    \begin{equation}
        |\Pr{Y / \sqrt{\Var{Y}} \in I} - \Pr{Z\in I}| \lesssim \rho\,.
    \end{equation}
\end{lemma}

\subsection{Softmax properties}
\label{sec:softmax}

\begin{lemma}\label{lem:softmaxmargin}
    Let $v\in\R^k$. If $a^* = \arg\max_a |v_a|$ and $v_{a^*} - v_a \ge \tau$ for all $a\neq a^*$, then
    \begin{equation}
        \norm{\softmax(v) - e_{a^*}}_1 \le \frac{k - 1}{e^\tau + k - 1}
    \end{equation}
\end{lemma}

\begin{proof}
    Let $w \triangleq \softmax(v)$. Then because $v_{a^*} - v_a \ge \tau$, we have $w_{a^*} / w_a \ge \exp(\tau)$. So
    \begin{equation}
        (k - 1) e^{-\tau} w_{a^*} \ge \sum_{a\neq a^*} w_a = 1 - w_{a^*}
    \end{equation}
    and thus $1\ge w_{a^*} \ge \frac{1}{1 + (k-1)e^{-\tau}} = 1 - \frac{k - 1}{e^\tau + k - 1}$. As $w$ and $e_{a^*}$ lie in the simplex, this implies the claimed bound on $\norm{w - e_{a^*}}_1$.
\end{proof}

\noindent We will also need a version of this for approximately 2-sparse softmax patterns:

\begin{lemma}\label{lem:softmaxtwo}
    Let $v\in \R^k$ and $\epsilon > 0$. Let $a_1, a_2\in[k]$ be the indices of the largest and second largest entries in $v$ respectively. If for all $a' \neq a_1,a_2$ we have that $v_{a_1} - v_{a'} \ge \log(2k/\epsilon)$, then
    \begin{equation}
        \Bigl\|\softmax(v) - \Bigl(\frac{e^{\beta}}{e^{\beta} + 1}\cdot e_{a_1} + \frac{1}{e^{\beta} + 1}\cdot e_{a_2}\Bigr)\Bigr\|_1 \le \epsilon \ \ \ \text{for} \ \ \ \beta = v_{a_1} - v_{a_2}\,.
    \end{equation}
\end{lemma}

\begin{proof}
    Let $w\triangleq \softmax(v)$. For any $a'\neq a_1,a_2$, note that because $v_{a_1} - v_{a'} \ge \log(2k/\epsilon)$, we have that $w_{a_1} / w_{a'} \ge 2k/\epsilon$. So
    \begin{equation}
        \frac{\epsilon}{2} w_{a_1} \ge \frac{(k-2)\epsilon}{2k} w_{a_1} \ge \sum_{a'\neq a_1,a_2} w_a = 1 - w_{a_1} - w_{a_2}\,.
    \end{equation}
    In particular, $w_{a_1} + w_{a_2} \ge 1 - \epsilon w_{a_1}/2 \ge 1 - \epsilon/2$, and the lemma follows by the fact that $w$ lies in the simplex.
\end{proof}

\noindent The following implies that if the softmaxes of two vectors are entrywise close, then the pairwise differences between entries of one vector are close to the corresponding pairwise differences for the other vector.

\begin{lemma}\label{lem:logratio}
    For $0 \le \epsilon \le 1/6$, let $v,v'\in\Delta^{k-1}$ be vectors satisfying $\norm{v - v'}_\infty \le \epsilon$. For $i,j\in[k]$, if $v_i, v_j \ge 1/3$, then
    \begin{equation}
        |\log(v_i/v_j) - \log(v'_i/v'_j)| \le 6\epsilon\,.
    \end{equation}
\end{lemma}

\begin{proof}
    Let $\epsilon_i \triangleq v'_i - v_i$ and $\epsilon_j \triangleq v'_j - v_j$ so that $|\epsilon_i|, |\epsilon_j| \le \epsilon$. Then
    \begin{equation}
        \Bigl|\frac{v_i}{v_j}\cdot \frac{v'_j}{v'_i} - 1\Bigr| = \Bigl|\frac{v_i \epsilon_j - v_j \epsilon_i}{v_j(v_i + \epsilon_i)}\Bigr|\le \frac{\epsilon}{(1/3)\cdot (1/3 + 1/6)} = 6\epsilon\,,
    \end{equation}
    so the lemma follows by the elementary inequality $\log(1 + z) \le z$.
\end{proof}

\noindent Lastly, we will need the following result about the stability of the softmax function under small multiplicative perturbations of its arguments. Concretely, we show that if the entrywise distance between two vectors is dominated by the pairwise separation among the entries of one of the vectors, then their softmaxes are close:

\begin{lemma}\label{lem:softmax}
    Let $R > 0$ and $0 < \eta, C \le 1$. Let $a_1,\ldots,a_k, b_1,\ldots,b_k\in\R$ be scalars satisfying the following properties for all distinct $i,j\in[k]$:
    \begin{enumerate}
        \item $|a_i - b_i| \le \eta\cdot R$ \label{item:mult_close}
        \item $|a_i - a_j| \ge C\cdot R$ \label{item:asep}
    \end{enumerate}
    If $\eta\log(1/\eta) \le C$, then
    \begin{equation}
        \norm{\softmax(a_1,\ldots,a_k) - \softmax(b_1,\ldots,b_k)}_\infty \lesssim k^2\eta\log(1/\eta) / C\,.
    \end{equation}
\end{lemma}

\begin{proof}
    By triangle inequality and symmetry of softmax, it suffices to show that \begin{equation}
        \norm{\softmax(a_1,a_2,\ldots,a_k) - \softmax(b_1,a_2,\ldots,a_k)}_\infty \lesssim k\eta\log(1/\eta) / C\,.
    \end{equation}
    We first claim it suffices to show that 
    \begin{equation}
        \norm{\softmax(a_1,\ldots,a_k) - \softmax(b_1,\ldots,b_k)}_\infty \le ke^{-CR}\cdot\min(1,e^{\eta R} - 1)\,. \label{eq:softmax_want}
    \end{equation}
    Indeed, if $R \le \log(1/\eta) / C$, then by hypothesis $\eta R \le 1$, so $e^{\eta R} - 1 \lesssim \eta R$, so \eqref{eq:softmax_want} implies a bound of $k\eta\log(1/\eta) / C$ as desired. If on the other hand $R > \log(1/\eta) / C$, then \eqref{eq:softmax_want} implies a bound of $ke^{-CR} \le k\eta \le k\eta\log(1/\eta)/C$ as desired.
    
    It remains to prove \eqref{eq:softmax_want}. We first compare the first entries of the softmaxes. Suppose without loss of generality that $a_1 \ge b_1$. Then
    \begin{equation}
        |\softmax(a_1,\ldots,a_k)_1 - \softmax(b_1,a_2,\ldots,a_k)_1| = \frac{(e^{a_1} - e^{b_1})\cdot \sum^k_{i=2} e^{a_i}}{(e^{a_1} + \sum^k_{i=2} e^{a_i}) (e^{b_1} + \sum^k_{i=2} e^{a_i})}\,. \label{eq:main_softmax_quantity}
    \end{equation}
    First suppose that $a_1 \neq \max_{i\in[k]} a_i$. Then by Item~\ref{item:asep} of the hypothesis, this implies that $e^{a_1} \le e^{-CR} \sum^k_{i=2} e^{a_i}$, so we may upper bound \eqref{eq:main_softmax_quantity} by $e^{-CR}$. As Item~\ref{item:mult_close} implies that $e^{a_1 - b_1} \le e^{\eta R}$, we can also upper bound \eqref{eq:main_softmax_quantity} by $\frac{e^{b_1}}{e^{a_1} + \sum^k_{i=2} e^{a_i}} \cdot (e^{\eta R} - 1) \le e^{-CR}\cdot (e^{\eta R} - 1)$, where in the last step we used that $b_1 \le a_1$.

    Next, suppose that $a_1 = \max_{i\in[k]} a_i$. Then by Item~\ref{item:asep} of the hypothesis, this implies that $\sum^k_{i=2} e^{a_i} \le ke^{-CR} \cdot e^{a_1}$. So we may upper bound \eqref{eq:main_softmax_quantity} by
    \begin{equation}
        \frac{ke^{-CR} (e^{a_1} - e^{b_1})}{e^{a_1} + \sum^k_{i=2} e^{a_i}}\,.
    \end{equation}
    Similarly to the previous case, we can bound
    \begin{equation}
        \frac{e^{a_1} - e^{b_1}}{e^{a_1} + \sum^k_{i=2} e^{a_i}} \le \min(1, e^{\eta R} - 1)\,.
    \end{equation}
    We conclude that in either of the two cases,
    \begin{equation}
        |\softmax(a_1,\ldots,a_k)_1 - \softmax(b_1,a_2,\ldots,a_k)_1| \le ke^{-CR}\cdot \min(1,e^{\eta R} - 1)\,.
    \end{equation}
    Next, we compare the remaining entries of the softmaxes. By symmetry, it suffices to compare the second entries. We have
    \begin{equation}
        |\softmax(a_1,\ldots,a_k)_2 - \softmax(b_1,a_2,\ldots,a_k)_2| = \frac{(e^{a_1} - e^{b_1})\cdot e^{a_2}}{(e^{a_1} + \sum^k_{i=2} e^{a_i})(e^{b_1} + \sum^k_{i=2} e^{a_i})}\,. \label{eq:main_softmax_quantity2}
    \end{equation}
    The analysis proceeds almost entirely analogously. Suppose that $a_1 \le a_2$. By Item~\ref{item:asep} this implies that $e^{a_1} \le e^{-CR} e^{a_2}$, so similar to the first case above, we can bound \eqref{eq:main_softmax_quantity2} by $e^{-CR} \cdot \min(1, e^{\eta R} - 1)$.

    Finally, suppose that $a_1 > a_2$. Then by Item~\ref{item:asep}, this implies that $e^{a_2} \le e^{-CR} e^{a_1}$, so similar to the second case above, we can bound \eqref{eq:main_softmax_quantity2} by $e^{-CR} \cdot \min(1,e^{\eta R} - 1)$.
\end{proof}


\section{Assumptions}
\label{sec:assume}

In this section we describe the non-degeneracy assumptions that we make on the attention and projection matrices. We note that while a caveat of our work is that we need to make a number of assumptions that might appear somewhat nonstandard in the context of prior work on learning traditional architectures like one-hidden-layer feed-forward networks, these assumptions are fairly generic; many of them are easily satisfied in natural smoothed analysis settings (see e.g. Appendix~\ref{app:arithmetic}). Furthermore, we emphasize that prior to the present work, no natural assumptions were known under which learning multi-head attention layers was possible, even for $m = 1$.

\subsection{Assumptions on the attention matrices.}
    We will assume that the attention matrices have comparable norm, and our bounds will ultimately scale exponentially in the gap between the smallest and largest norms:

    \begin{assumption}\label{assume:range}
        Without loss of generality, $\norm{\Sig_1}_F \ge \cdots \ge \norm{\Sig_m}_F$. Furthermore, there is a parameter $\Siglbd\in(0,1)$ such that for all $i\in[m]$, $\norm{\Sig_i}^2_F \ge \Siglbd\norm{\Sig_1}^2_F$. Furthermore, $\Siglbd$ satisfies
        \begin{equation}
            \Siglbd \gg \frac{m^{3/2}\log m}{\sqrt{\reffsig}}\,,  \label{eq:lambdabig}
        \end{equation}
        where $\reffsig$ is defined in Assumption~\ref{assume:effranksig} below.
    \end{assumption}

    \noindent We will assume that the attention matrices are somewhat incoherent. Note that, the incoherence parameter here only scales with $1/\mathrm{polylog}(d)$, instead of $1/\poly(d)$. In contrast, if the attention matrices were random matrices, we would get incoherence scaling even with $1/\poly(d)$. We are thus able to handle a much more general setting than if the $\Sig_i$'s were merely random.

    \begin{assumption}\label{assume:sig_orth}
        For all $i,i'\in[m]$, $|\iprod{\Sig_i, \Sig_{i'}}| \le \Siginc\,\norm{\Sig_i}_F\cdot\norm{\Sig_{i'}}_F$ for
        \begin{equation}
            \Siginc \triangleq \frac{c\Siglbd^{13/2}}{m^{13/2}\log^2 d}\,, \label{eq:lambound}
        \end{equation}
        where $c$ is an arbitrarily small constant.

    \end{assumption}

    \noindent In practice, attention matrices can potentially have moderately low rank. We only need to assume that their effective ranks are at least polylogarithmic in the dimension. For technical reasons, we will need to assume this also for any large constant fraction of the columns of $\Sig_i$. Note that this extra property is satisfied by generic matrices of sufficient effective rank.
    \begin{assumption}\label{assume:effranksig}
        For some $\reffsig$ satisfying
        \begin{equation}
            \frac{\log^c(dmk)}{\Siginc\Siglbd} \ll \reffsig \ll d^{1/4} \label{eq:reffsig_vs_logd}
        \end{equation}
        for sufficiently large absolute constant $c > 0$,
        we have $\norm{\Sig_i}^2_F \ge \reffsig\,\norm{\Sig_i}^2_{\sf op}$, and $\norm{\Sig_i\Pi^\perp_S}^2_F \ge \reffsig\,\norm{\Sig_i\Pi^\perp_S}^2_{\sf op}$ 
        for all $S\subseteq[d]$ and $i\in[m]$ satisfying $|S|\le cd$ for sufficiently small constant $c > 0$.
    \end{assumption}

    \noindent We will assume that each row of $\Sig_i$ is not too ``heavy'' relative to the overall Frobenius norm of the matrix. This is primarily a technical condition used to handle the Booleanity of the inputs: if there are a few heavy hitters among the columns of $\Sig_i$, the anticoncentration behavior of the quadratic form $x^\top \Sig_i y$ for random Boolean $x,y$ is very different than for Gaussian $x,y$.
    \begin{assumption}\label{assume:lightrows}
        There is a parameter $\Siglightrows\ge 1$ satisfying $\Siglightrows \ll \sqrt{d}/\reffsig^2$ such that for all $i\in[m]$ and $j\in[d]$, we have $\norm{(\Sig_i)_{:j}}, \norm{(\Sig_i)_{j:}} \le \frac{\Siglightrows}{\sqrt{d}}\norm{\Sig_1}_F$, and $\norm{(\Sig_i \Pi^\perp_S)_{:i}} \le \frac{\Siglightrows}{\sqrt{d}}\,\norm{\Sig_i \Pi^\perp_S}_F$ for all $S\subseteq[d]$.
    \end{assumption}


    \noindent We assume that the trace of any attention matrix is not too much larger than its Frobenius norm. This is needed to ensure that the diagonal entries of the attention patterns do not behave too differently from the off-diagonal entries. Indeed, the expected value of $x^\top \Sig_i x$ (resp. $x^\top \Sig_i y$) for $x,y \sim\brc{\pm 1}^d$ is $\Tr(\Sig_i)$ (resp. $0$), whereas the fluctuations of $x^\top \Sig_i x$ or $x^\top \Sig_i y$ are both on the order of $\norm{\Sig_i}_F$:
    \begin{assumption}\label{assume:trace}
         There is a parameter $1 \le \projtrace \le \sqrt{\log k}$ such that $|\Tr(\Sig_i)| \le \projtrace\cdot \norm{\Sig_i}_F$.
    \end{assumption}

    \noindent We will assume a modest lower bound on the maximum norm of any of the attention matrices. This is simply to ensure that with non-negligible probability, we encounter attention patterns that are sufficiently sparse. We will also make the very mild assumption that $\norm{\Sig_1}_F$ is polynomially bounded to ensure that the patterns are not always approximately \emph{$1$-sparse}:
    \begin{assumption}\label{assume:Q1notsmall}
        $\log\Bigl(kd/\sqrt{\Wlbd}\Bigr) \ll \norm{\Sig_1}_F \le (kd)^{O(m)}$.
    \end{assumption}

    \noindent Finally, for technical reasons, we will need the following ``non-arithmeticity'' assumption that, intuitively, ensures that the entries of the attention matrices are somewhat far from any lattice. While the assumption might appear somewhat unusual at first glance, it is actually a generic property in the sense that it holds in a natural smoothed analysis setting (see Appendix~\ref{app:arithmetic}):

    \begin{assumption}\label{assume:nonarithmetic}
        With probability at least $1 - \exp(-d^{\Theta(1)})$ over $x\in\brc{\pm 1}^d$, for all $T\subseteq[d]$ of size $(1 - o(1))d$, 
        \begin{equation}
            \prod^m_{i=1} \prod^d_{j\in T} |\cos(\lambda_i \langle x, \Sig_{:j}\rangle)| \le \exp(-\Omega(\sqrt{d}))
        \end{equation}
        for all $\lambda\in\R^m$ for which $\sqrt{m} \le \norm{\lambda} \le \exp(d^{o(1)})$.
    \end{assumption}

\subsection{Assumptions on the projection matrices.}

    As with Assumption~\ref{assume:sig_orth}, we must make some mild incoherence assumption on the projection matrices. Here the level of incoherence we need is even weaker:

    \begin{assumption}\label{assume:Wincoherent}
        There is a parameter $\Winc\in(0,1)$ such that for all $i,i'\in[m]$, $|\iprod{\bW_i, \bW_{i'}}| \le \Winc\,\norm{\bW_i}_F \cdot \norm{\bW_{i'}}_F$. Furthermore, $\Winc$ satisfies a
        \begin{equation}
            \Winc \ll 1/m\,. \label{eq:kappaprime_constraint}
        \end{equation}
    \end{assumption}

    \noindent As with Assumption~\ref{assume:effranksig}, we must assume some lower bound on the effective rank of the projection matrices. We only need a lower bound scaling logarithmically in the dimension:
    \begin{assumption}\label{assume:Weffr}
        For some $\reffw \gg m^3k\log(d)$ and for every $i\in[m]$, $\norm{\bW_i}^2_F \ge \reffw \norm{\bW_i}^2_\op$.
    \end{assumption}

    \noindent As with Assumption~\ref{assume:range}, we assume that the projection matrices have comparable norms. Unlike for the attention matrices, we will not depend exponentially on the gap between the norms of the projection matrices; in fact the gap can be as large as polynomial in the effective ranks of the attention matrices.
    \begin{assumption}\label{assume:Wnorm}
        $\norm{\bW_1}_F = 1$ (this is without loss of generality), and there is a parameter $\Wlbd \le 1/\reffsig$ for sufficiently small constant $c > 0$ such that for all $i>1$, $\norm{\bW_i}^2_F \ge \Wlbd$.
    \end{assumption}

\subsection{Further assumptions.}

    We will assume that 
    \begin{equation}
        m \lesssim k \lesssim \reffsig^{\Theta(1)}\,,    
    \end{equation}
    where the constant factor in $\Theta(1)$ is sufficiently small.
    It is natural in practice for the number of heads $m$ to be less than the sequence length $k$, and for the sequence length $k$ to be less than the dimension $d$. In this work, for technical reasons our proof requires a stronger bound on $k$, but if we think of $\reffsig$ as some small polynomial in $d$, then we just need that $k$ is polynomially bounded in terms of $d$. This assumption is needed primarily for the first phase of the algorithm in Section~\ref{sec:sumvalue}.

    For technical reasons, in Assumption~\ref{assume:fewheads} below, we will additionally need that the number of attention heads is at most roughly \emph{logarithmic} in the dimension $d$. The reason is that we will need to wait for certain events that would have probability roughly $\exp(-\Theta(m))$ over Gaussian inputs, but because we are working in the more challenging regime of Boolean inputs, when applying certain central limit theorems we incur $1/\poly(d)$ error in this probability bound. As long as $m$ is logarithmic, $1/\poly(d)$ is of lower order and our bounds carry through.

    In any case, even the case of constant $m$ is highly nontrivial, and the reader may find it helpful to think of $m$ as a constant upon a first reading of this work, in which case the following assumptions holds immediately.

    \begin{assumption}\label{assume:fewheads}
        There are quantities $\overline{m}_1, \overline{m}_2$ for which
        \begin{equation}
            \overline{m}_1 = \wt{\Theta}\Bigl(\frac{\Siglbd \log(d/\Siglightrows)}{\log k}\Bigr) \qquad\qquad
            \overline{m}_2 = \Theta\Bigl(\frac{\Siginc^9d^3}{\Siglightrows^6\log^3(d)\cdot k^{O(1/\Siginc)}}\Bigr)
        \end{equation}
        such that the number of attention heads $m$ satisfies $m\le \min(\overline{m}_1, \overline{m}_2)$.
    \end{assumption}


\section{Estimating the sum of projection matrices}
\label{sec:sumvalue}

    In this section we give a procedure for estimating $\sum_i \bW_i$ using correlations between the entries of the output $\bY = F(\bX)$ and entries of the input $\bX$. Concretely, we show the following:
    \begin{theorem}\label{thm:sumvalue}
        For $\bX\sim\brc{\pm 1}^{k\times d}$ and $\bY = F(\bX)$ for $F$ given by Eq.~\eqref{eq:attention}, there is a constant $0 < C < 1$ such that
        \begin{equation}
            \Bigl\|\E{\bX^\top \J \bY} - k\sum^m_{i=1} \bW_i\Bigr\|_F \le \wt{\Theta}\Bigl(\frac{mk^5}{\reffsig^{1/12} \wedge (d^C/\upsilon)}\Bigr)\,.
        \end{equation}
    \end{theorem}

    \paragraph{Overview of proof.} To prove Theorem~\ref{thm:sumvalue}, it will suffice to show that for any attention matrix $\Sig$ satisfying the Assumptions in Section~\ref{sec:assume}, we can bound
    \begin{equation}
        \bigl\|\E{\bX^\top \J\softmax(\bX\Sig\bX^\top)\bX} - k\cdot\Id\bigr\|_{\sf op}\,. \label{eq:diffop}
    \end{equation}
    By a bucketing argument, it suffices to bound the operator norm \emph{restricted to test vectors of the form $v/\norm{v}, w/\norm{w}$ for $v,w\in\brc{0,1}^d$}\--- the proof of this is standard, see Section~\ref{sec:bucket} and Lemma~\ref{lem:loselog} for the formal statement and proof. This will be convenient for us as we can then break down the expectation
    \begin{equation}
        \frac{1}{\norm{v}\norm{w}} v^\top \E{\bX^\top \J\softmax(\bX\Sig\bX^\top)\bX} w \label{eq:vEw}
    \end{equation}
    into an expectation over the value of $\bX^\top v$ and $\bX^\top w$, which follows a product of binomial distributions, followed by a conditional expectation over the remaining randomness in $\bX$. For typical $\bX^\top v$, the entries are concentrated in a band of radius $O(1/\sqrt{\norm{v}})$ around $0$, and similarly for $\bX^\top w$. We will argue that for all such $\bX^\top v$ and $\bX^\top w$, the conditional expectation over $\softmax(\bX\Sig\bX^\top)$ is roughly the same, allowing us to decouple the randomness of the inner softmax in Eq.~\eqref{eq:vEw} from the randomness of the outer $\bX$'s and conclude that for a deterministic row-stochastic matrix $\Delta$, Eq.~\eqref{eq:vEw} is close to
    \begin{equation}
        \frac{1}{\norm{v}\norm{w}}v^\top \E{\bX^\top \J\Delta \bX} = \Tr(\J\Delta)w = k\,,
    \end{equation}
    as desired.

    \paragraph{Overview of conditional expectation bound.} The bulk of the analysis is in actually proving that the conditional expectation over $\softmax(\bX\Sig\bX^\top)$ is roughly the same for any typical $\bX^\top v$ and $\bX^\top w$. Consider two independent samples $\bX$, $\bX'$ under two different but typical conditionings of $\bX^\top v, \bX^\top w$ and $\bX'^\top v, \bX'^\top w$ respectively. We want to show that for each row $i\in[k]$, the conditional expectations of $\softmax(\bX_{i:}\Sig\bX^\top)$ and $\softmax(\bX'_{i:}\Sig\bX'^\top)$ are close. Our main tool is Lemma~\ref{lem:softmax} from Section~\ref{sec:softmax} showing that the softmax function is robust to small multiplicative perturbations of the entries, provided the entries are well-separated. 

    Showing that $\bX_{i:}\Sig\bX^\top$ and $\bX'_{i:}\Sig\bX'^\top$ have well-separated entries amounts to anticoncentration of certain well-behaved linear functions, the catch being that the relevant distributions are rather nonstandard. Concretely, because we are conditioning on $\bX^\top v$ and $\bX^\top w$, this amounts to conditioning on $\bX\sim\brc{\pm 1}^{k\times d}$ coming from the intersection of the hypercube with two Boolean-weight hyperplanes. In Section~\ref{sec:tail_slice_cube}, we collect some tools, proved in Appendix~\ref{app:HW}, that we then use to establish the desired separation bounds.

    To apply Lemma~\ref{lem:softmax}, it remains to show that for every $j\in[k]$, $\bX_{i:}\Sig \bX_{j:}^\top$ is close to $\bX'_{i:}\Sig \bX_{j:}$, averaged over the randomness of the conditional distributions over $\bX, \bX'$. For this, we introduce a \emph{coupling} between the two conditional distributions (see Definition~\ref{def:coupling}) that allows us to relate any $\bX$ to a nearby $\bX'$ by re-randomizing a small number of coordinates. In Section~\ref{sec:coupling} we bound the probability under this coupling that $\bX_{i:}\Sig \bX_{j:}^\top$ is far from $\bX'_{i:}\Sig \bX'^\top_{j:}$. In Sections~\ref{sec:conclude1} to~\ref{sec:conclude2}, we combine all of these ingredients to complete the proof of Theorem~\ref{thm:sumvalue}.

    \subsection{Tail bounds for product of slice and cube}
        \label{sec:tail_slice_cube}

        Before we proceed to the proof of Lemma~\ref{thm:sumvalue}, we need to collect some technical preliminaries that are specific to this section.


        \begin{definition}
            Given $d\in\mathbb{N}$ and $\mu \in [-1,1]$ for which $\mu d$ is an integer, define $\slice{d}{\mu}\subset\brc{\pm 1}^d$ to be the set of strings $x$ for which $\frac{1}{d}\sum_i x_i = \mu$. Let $\sldist{d}{\mu}$ denote the uniform distribution over this set.
        \end{definition}

        \begin{fact}\label{fact:slice_corr}
            Given $x\sim \sldist{d}{\mu}$ and any distinct $i,j\in[d]$, $\E{x_i x_j} = \frac{\mu^2 d - 1}{d - 1}$.
        \end{fact}

        \begin{proof}
            $\sldist{d}{\mu}$ is the uniform distribution over strings in $\brc{\pm 1}^d$ with exactly $a\triangleq \frac{1 + \mu}{2}\cdot d$ positive entries and $d - a$ negative entries. So
            \begin{equation}
                \E{x_i x_j} = \binom{d}{a}^{-1}\biggl\{\binom{d - 2}{a - 2} + \binom{d - 2}{a} - 2\binom{d-2}{a-1}\biggr\} = \frac{(d - 2a)^2 - d}{d^2 - d} = \frac{\mu^2 d - 1}{d - 1}
            \end{equation} as claimed.
        \end{proof}

        Let $D_\mu$ denote the distribution over $\brc{\pm 1}$ with mean $\mu$. First, define the distributions
        \begin{align}
            \pi^c_{\mu,\nu;d_1,d_2,d_3} &\triangleq \unif{d_1}\otimes D_\mu^{\otimes d_2} \otimes D_\nu^{\otimes d_3}\\
            \pi^s_{\mu,\nu;d_1,d_2,d_3} &\triangleq \unif{d_1}\otimes \sldist{d_2}{\mu} \otimes \sldist{d_3}{\nu}\,.
        \end{align}
        When $d_1, d_2, d_3$ are clear from context, we will denote these by $\pi^c_{\mu,\nu}$ and $\pi^s_{\mu,\nu}$ respectively.

        It will also be convenient to define
        \begin{align}
            \partialone &\triangleq \sum^{d_1+d_2}_{i=d_1+1} e_i \label{eq:partialone} \\
            \partialtwo &\triangleq \sum^{d}_{i=d_1+d_2+1} e_i\,. \label{eq:partialtwo} \\
        \end{align}

        Define $S_{[1]}\triangleq \brc{1,\ldots,d_1}$, $\Sone\triangleq \brc{d_1+1,\ldots,d_1 + d_2}$, and $\Stwo\triangleq \brc{d_1+d_2+1,\ldots,d}$. For $a\in[3]$, 
        let $P_{[a]}\in \R^{d_a\times d}$ denote the projector to the coordinates indexed by $S_{[a]}$, and let $P_{[2,3]}\in \R^{(d_2+d_3)\times d}$ denote the projector to the coordinates indexed by $S_{[2]}\cup S_{[3]}$.
        Let $\Sig_{[a]} \triangleq \Sig P_{[a]}\in \R^{d\times d_a}$
        , and let $\Sig_{[2,3]} \triangleq \Sig P_{[2,3]} \in \R^{d\times (d_2+d_3)}$. Finally, let $\Pi^\perp_{[a]} \in \R^{d_a\times d_a}$ denote the matrix $\Id_{d_a} - \frac{1}{d_a}\one\one^\top$.

        \begin{theorem}[Special case of Corollary 5.11 from~\cite{filmus2018invariance}]\label{thm:inv_slice}
            Let $p(x) \triangleq \iprod{v,x}$ for $v \in \R^m$ orthogonal to the all-ones vector $\one$. Suppose $\norm{v}_\infty \le \tau\norm{v}_2$.\footnote{In~\cite{filmus2018invariance} they consider general polynomials of constant degree, which they need to assume are harmonic and have bounded influence. For degree-1 polynomials, harmonicity corresponds to the above condition that $v$ is orthogonal to the all-ones vector, and bounded influence corresponds to the bound on $\norm{v}_\infty$ in terms of $\norm{v}_2$.} There is an absolute constant $C > 0$ such that for any constant $\mu\in [-0.9,0.9]$, if $x\sim \sldist{m}{\mu}$ and $\gamma\sim \calN(\mu,1-\mu^2)^{\otimes m}$, then for all $r\in\R$,
            \begin{equation}
                |\Pr{p(x) \le r} - \Pr{p(\gamma) \le r}| \lesssim \tau^C\,.
            \end{equation}
            In particular, for any $t\in\R$ and $s > 0$, we have
            \begin{equation}
                \Pr{|p(x)-t|\le s\norm{v}_2} \lesssim \tau^C + s^{1/2}
            \end{equation}
        \end{theorem}

        \noindent We will also use the following concentration inequalities for Lipschitz functions and quadratic polynomials over $\pi^s_{\mu,\nu;d_1,d_2,d_3}$.


        \begin{restatable}{lemma}{conclinear}\label{lem:slice_lip_linear}
            For any vector $v\in\R^d$, $x\sim\pi^s_{\mu,\nu;d_1,d_2,d_3}$, and $t > 0$, we have
            \begin{equation}
                \Pr{|\iprod{v,x} - \iprod{v,\mu\cdot\partialone + \nu\cdot\partialtwo}| \ge t} \le 2\exp(-\Omega(t^2 / \norm{v}^2))\,.
            \end{equation}
            Note that when $d_2, d_3 = 0$, this is simply Hoeffding's inequality.
        \end{restatable}

        \begin{restatable}{theorem}{HW}\label{thm:main_HW}
            For $|\mu| \le c/\sqrt{d_2}, |\nu| \le c/\sqrt{d_3}$ for some absolute constant $c > 0$, there is an absolute constant $C > 0$ such that for any matrix $\bA\in\R^{d\times d}$, if $x\sim \pi^s_{\mu,\nu;d_1,d_2,d_3}$ and $t > 0$, then
            \begin{equation}
                \Pr{|x^\top \bA x - \E{x^\top \bA x}| > t} \lesssim \exp\Bigl(-C\, \min\Bigl(\frac{t}{\norm{\bA}_{\sf op}}, \frac{t^2}{\norm{\bA}^2_F}\Bigr)\Bigr)\,.
            \end{equation}
            Note that when $d_2, d_3 = 0$, this is simply the Hanson-Wright inequality (Theorem~\ref{thm:hanson_wright}).
        \end{restatable}

        \noindent The proofs of these follow from a standard application of the entropy method; we defer the details of proof to Appendix~\ref{app:HW}.

        We record here the following simple corollary of Lemma~\ref{lem:slice_lip_linear}.

        \begin{corollary}\label{cor:Xv}
            Let $v\in\S^{d-1}$. For any $\mu_1,\ldots,\mu_k\in[-c/\sqrt{d_2}, c/\sqrt{d_2}]$ and $\nu_1,\ldots,\nu_k\sim[-c/\sqrt{d_3},c/\sqrt{d_3}]$, if $\bX\in\brc{\pm 1}^{k\times d}$ is sampled according to $\bX_{i:}\sim \pi^s_{\mu_i,\nu_i;d_1,d_2,d_3}$, then for any $\delta > 0$,
            \begin{equation}
                \norm{\bX v} \le c\sqrt{k} + O(\sqrt{k\log(k/\delta)})
            \end{equation}
            with probability at least $1 - \delta$.
        \end{corollary}

        \begin{proof}
            It suffices to show that for any $\mu\in[-c/\sqrt{d_2}, c/\sqrt{d_2}]$, $\nu\in[-c/\sqrt{d_3}, c/\sqrt{d_3}]$, if $x\sim \pi^s_{\mu_i,\nu_i;d_1,d_2,d_3}$, then $|\iprod{x,v}| \le c + O(\sqrt{\log(k/\delta)})$ with probability at least $1 - \delta/k$. By Lemma~\ref{lem:slice_lip_linear}, $\Pr{|\iprod{v,x} - \mu\iprod{v,\partialone}| \ge t} \le 2\exp(-\Omega(t^2))$. Note that $|\mu\iprod{v,\partialone}| \le \mu\sqrt{d_2} \le c$, so $\Pr{|\iprod{v,x}| \ge t + c} \le 2\exp(-\Omega(t^2))$. The proof is complete upon taking $t = O(\sqrt{\log(k/\delta)})$.
        \end{proof}

    \subsection{Applying the tail bounds}
        \label{sec:applytail}

        In the sequel, let $\Sig$ denote any one of the attention matrices $\Sig_1,\ldots,\Sig_m$. Here we apply the anticoncentration and concentration results from Section~\ref{sec:tail_slice_cube} to prove that for $x,y$ sampled uniformly from cube-slice products, the quadratic form $x^\top \Sig y$ anticoncentrates (Lemmas~\ref{lem:xSigy_anticonc} and~\ref{lem:xSigy_anticonc_corner}). We begin by proving this when the slices are closer to the central slice.

        \begin{lemma}\label{lem:xSigy_anticonc}
            Let $\mu,\mu'\in[-1,1]$ be multiples of $1/d_2$, and let $\nu,\nu'\in[-1,1]$ be multiples of $1/d_3$, such that $|\mu|, |\mu'| \le c / \sqrt{d_2}$ for $1 \le c \le \sqrt{\min(\reffsig,d_2,d_3)}/2$, and at least one of the following holds:
            \begin{itemize}
                 \item $|\nu|, |\nu'| \le c/\sqrt{d_3}$
                 \item $|\nu|, |\nu'| \le 0.9$ and $d_3 \le \min(d/8\upsilon^2, \sqrt{d\reffsig}/8c\upsilon)$.
            \end{itemize} 
            Then there is an absolute constant $0 < C < 1$ such that if $x\sim \pi^s_{\mu,\nu;d_1,d_2,d_3}, y\sim \pi^s_{\mu',\nu';d_1,d_2,d_3}$ and $s,t > 0$, then
            \begin{equation}
                \Pr[x,y]{|x^\top\Sig y - t| \le s\norm{\Sig}_F} \lesssim s^{1/2} + \Bigl(\frac{c^2\upsilon\sqrt{\reffsig}}{\sqrt{d}}\Bigr)^C + \exp(-\Omega(\reffsig))\,. \label{eq:xQy_anti}
            \end{equation}
        \end{lemma}

        \begin{proof}
            Given $y\in\R^d$, let $y_{[a]}\in\R^{d_a}$ denote the vector $P_{[a]} y$. Note that for $y\sim \pi^s_{\mu',\nu';d_1,d_2,d_3}$, $y_{[1]}, y_{[2]}, y_{[3]}$ are independent draws from $\unif{d_1}$, $\sldist{d_2}{\mu'}$, $\sldist{d_3}{\nu'}$ respectively. We will also use $y_{[2,3]}$ to denote the concatenation of $y_{[2]}$ and $y_{[3]}$.

            Given $x\in\brc{\pm 1}^d$, define the linear function $p_x(y) \triangleq x^\top \Sig y$. Note that for any $y$ in the support of $\pi^s_{\mu',\nu';d_1,d_2,d_3}$,
            \begin{equation}
                p_x(y) = \iprod{y_{[1]}, \Sig^\top_{[1]} x} + \iprod{y_{[2]}, \Pi^\perp_{[2]}\Sig^\top_{[2]} x} + \iprod{y_{[3]}, \Pi^\perp_{[3]}\Sig^\top_{[3]} x} + c_x \label{eq:xQy}
            \end{equation}
            for some constant $c_x$ that does not depend on $y$. We first compute the expected variances of the three linear terms.

            Note that for any $\bA\in\R^{d\times d_a}$,
            \begin{align}
                \mathbb{E}\norm{\bA^\top x}^2 &= \norm{\bA}^2_F + \sum_{i,j\in \Sone: i\neq j} (\bA\bA^\top)_{ij} \cdot \frac{\mu^2 d_2 - 1}{d_2 - 1} + \sum_{i,j\in \Stwo: i\neq j} (\bA\bA^\top)_{ij} \cdot \frac{\nu^2 d_3 - 1}{d_3 - 1} \\
                &\quad\qquad\qquad + 2\sum_{i\in \Sone, j\in \Stwo} (\bA\bA^\top)_{ij}\cdot \mu\nu \,. \label{eq:eAx}
            \end{align}
            We have
            \begin{equation}
                \Bigl|\sum_{i,j\in \Sone: i\neq j}(\bA\bA^\top)_{ij} \cdot \frac{\mu^2 d_2 - 1}{d_2 - 1}\Bigr| \le \frac{c^2}{d_2}\partialone^\top (\bA\bA^\top) \partialone \le c^2\,\norm{\bA}^2_{\op}\,. \label{eq:diag2}
            \end{equation}
            If the first bullet point in the hypothesis is satisfied, then
            \begin{equation}
                \Bigl|\sum_{i,j\in \Stwo: i\neq j}(\bA\bA^\top)_{ij} \cdot \frac{\nu^2 d_3 - 1}{d_3 - 1}\Bigr| \le \frac{c^2}{d_3}\partialtwo^\top (\bA\bA^\top) \partialtwo \le c^2\,\norm{\bA}^2_{\op}\,. \label{eq:diag3-first}
            \end{equation}
            where the last step follows by Assumption~\ref{assume:effranksig}, and similarly
            \begin{equation}
                \Bigl|\sum_{i\in \Sone, j\in \Stwo} (\bA\bA^\top)_{ij} \cdot \mu\nu\Bigr| \le \frac{c^2}{\sqrt{d_2d_3}}\partialone^\top (\bA\bA^\top) \partialtwo \le c^2 \norm{\bA}^2_{\sf op}\,. \label{eq:cross23-first}
            \end{equation}
            In general, we can also bound the quantities on the left-hand side of Eq.~\eqref{eq:diag3-first} and~\eqref{eq:cross23-first} via
            \begin{align}
                \sum_{i,j\in \Stwo: i\neq j}(\bA\bA^\top)_{ij} \cdot \frac{\nu^2 d_3 - 1}{d_3 - 1} &= \frac{\nu^2 d_3 - 1}{d_3 - 1}\cdot \partialtwo^\top (\bA\bA^\top) \partialtwo - \frac{\nu^2 d_3 - 1}{d_3 - 1}\cdot\norm{P_{[3]}\bA}^2_F \\
                &\ge -2\norm{\bA}^2_{\sf op} - \norm{P_{[3]}\bA}^2_F\,, \label{eq:diag3-second}
            \end{align}
            and
            \begin{align}
                \sum_{i\in \Sone, j\in\Stwo} (\bA\bA^\top)_{ij} \cdot \mu\nu &= \mu\nu\cdot \partialone^\top (\bA\bA^\top) \partialtwo \\
                &\ge -c\cdot \norm{\bA^\top \partialtwo} \cdot \norm{\bA}_{\sf op} \label{eq:cross23-second}
            \end{align}
            In Eq.~\eqref{eq:diag3-second}, in the second step we used that $-\frac{1}{d_3 - 1} \le \frac{\nu^2d_3 - 1}{d_3 - 1} \le 1$ for any $\nu\in[-1,1]$, that $\partialtwo (\bA\bA^\top) \partialtwo \le d_3\norm{\bA}^2_{\sf op}$, and that $\frac{d_3}{d_3 - 1} \le 2$ for all $d_3 > 1$, noting that when $d_3 \le 1$, $\nu \in \brc{-1,1}$ so that $\frac{\nu^2 d_3 - 1}{d_3 - 1} = 1$.

            If the first bullet point in the hypothesis is satisfied, then substituting $\bA = \Sig_{[1]}$, $\bA = \Sig_{[2]}\Pi^\perp$, or $\bA = \Sig_{[3]}\Pi^\perp$ into Eqs.~\eqref{eq:diag2}, \eqref{eq:diag3-first}, and~\eqref{eq:cross23-first}, we conclude by Assumption~\ref{assume:effranksig} that
            \begin{align}
                \mathbb{E}\norm{\Sig^\top_{[1]}x}^2 &\ge \norm{\Sig_{[1]}}^2_F - 3c^2\norm{\Sig_{[1]}}^2_{\sf op} \\
                &\ge \norm{\Sig_{[1]}}^2_F - \frac{3c^2}{\reffsig}\norm{\Sig}^2_F  \label{eq:variance1-first} \\
                \mathbb{E}\norm{\Pi^\perp_{[2]}\Sig^\top_{[2]}x}^2 &\ge \norm{\Sig_{[2]}\Pi^\perp_{[2]}}^2_F - 3c^2\norm{\Sig_{[2]}\Pi^\perp_{[2]}}^2_{\sf op} \\ 
                &\ge \Bigl(\norm{\Sig_{[2]}}_F - \frac{1}{d_2}\norm{\Sig_{[2]}\one\one^\top}_{\sf op}\Bigr)^2 - 3c^2\Bigl(\norm{\Sig_{[2]}}_{\sf op} + \frac{1}{d_2}\norm{\Sig_{[2]}\one\one^\top}_{\sf op}\Bigr)^2 \label{eq:variance2-first} \\
                \mathbb{E}\norm{\Pi^\perp_{[3]}\Sig^\top_{[3]}x}^2 &\ge \Bigl(\norm{\Sig_{[3]}}_F - \frac{1}{d_3}\norm{\Sig_{[3]}\one\one^\top}_{\sf op}\Bigr)^2 - 3c^2\Bigl(\norm{\Sig_{[3]}}_{\sf op} + \frac{1}{d_3}\norm{\Sig_{[3]}\one\one^\top}_{\sf op}\Bigr)^2 \label{eq:variance3-first}\,. 
            \end{align}
            In general, we can also substitute $\bA = \Sig_{[1]}$, $\bA = \Sig_{[2]}\Pi^\perp$, or $\bA = \Sig_{[3]}\Pi^\perp$ into Eqs.~\eqref{eq:diag2}, \eqref{eq:diag3-second}, and~\eqref{eq:cross23-second} to get
            \begin{align}
                \mathbb{E}\norm{\Sig^\top_{[1]}x}^2 &\ge \norm{\Sig_{[1]}}^2_F - (c^2 + 2)\norm{\Sig_{[1]}}^2_{\sf op} - \norm{P_{[3]}\Sig_{[1]}}^2_F - c\cdot \norm{\Sig^\top_{[1]}\partialtwo}\cdot \norm{\Sig_{[1]}}_{\sf op} \\
                &\ge \Bigl(1 - \frac{d_3\upsilon^2}{d}\Bigr)\norm{\Sig_{[1]}}^2_F - \Bigl(\frac{c^2 + 2}{\reffsig} + \frac{c d_3\upsilon}{\sqrt{\reffsig d}}\Bigr)\norm{\Sig}^2_F \label{eq:variance1-second} \\
                \mathbb{E}\norm{\Pi^\perp_{[2]}\Sig^\top_{[2]}x}^2 &\ge \norm{\Sig_{[2]}\Pi^\perp_{[2]}}^2_F - (c^2 + 2)\norm{\Sig_{[2]}\Pi^\perp_{[2]}}^2_{\sf op} - \norm{P_{[3]}\Sig_{[2]}\Pi^\perp_{[2]}}^2_F - c\cdot \norm{\Pi^\perp_{[2]}\Sig^\top_{[2]}\partialtwo}\cdot \norm{\Sig_{[2]}\Pi^\perp_{[2]}}_{\sf op} \\
                &\ge \Bigl(\norm{\Sig_{[2]}}_F - \frac{1}{d_2}\norm{\Sig_{[2]}\one\one^\top}_{\sf op}\Bigr)^2 - (c^2 + 2)\Bigl(\norm{\Sig_{[2]}}_{\sf op} + \frac{1}{d_2}\norm{\Sig_{[2]}\one\one^\top}_{\sf op}\Bigr)^2 - \Bigl(\frac{cd_3\upsilon}{\sqrt{\reffsig d}} + \frac{d_3\upsilon^2}{d}\Bigr)\norm{\Sig}^2_F \label{eq:variance2-second} \\
                \mathbb{E}\norm{\Pi^\perp_{[3]}\Sig^\top_{[3]}x}^2 &\ge \Bigl(\norm{\Sig_{[3]}}_F - \frac{1}{d_3}\norm{\Sig_{[3]}\one\one^\top}_{\sf op}\Bigr)^2 - (c^2 + 2)\Bigl(\norm{\Sig_{[3]}}_{\sf op} + \frac{1}{d_3}\norm{\Sig_{[3]}\one\one^\top}_{\sf op}\Bigr)^2 - \Bigl(\frac{cd_3\upsilon}{\sqrt{\reffsig d}} + \frac{d_3\upsilon^2}{d}\Bigr)\norm{\Sig}^2_F\,. \label{eq:variance3-second} 
            \end{align}

            We now proceed by casework on the relative magnitudes of $\norm{\Sig_{[1]}}^2_F, \norm{\Sig_{[2]}}^2_F, \norm{\Sig_{[3]}}^2_F$.


            \paragraph{Case 1: $\norm{\Sig_{[1]}}^2_F \ge \frac{1}{3}\norm{\Sig}^2_F$} 

            If the first bullet of the hypothesis holds so that we have Eq.~\eqref{eq:variance1-first}, then because $\reffsig \gg 1$ by assumption, $\mathbb{E}\norm{\Sig^\top_{[1]} x}^2 \gtrsim \norm{\Sig_{[1]}}^2_F$. On the other hand, if the second bullet of the hypothesis holds, then we can apply the assumed bound on $d_3$ to Eq.~\eqref{eq:variance1-second} to also conclude that $\mathbb{E}\norm{\Sig^\top_{[1]} x}^2 \gtrsim \norm{\Sig_{[1]}}^2_F$. 

            So taking $t = \Theta(\norm{\Sig}^2_F)$ in Theorem~\ref{thm:main_HW} and noting that $\frac{t}{\norm{\Sig_{[1]}}^2_{\sf op}} \ge \reffsig$ and $\frac{t^2}{\norm{\Sig_{[1]}\Sig_{[1]}^\top}^2_F} \ge \reffsig$ by Assumption~\ref{assume:effranksig}, we conclude that
            \begin{equation}
                \Pr{\norm{\Sig_{[1]}^\top x}^2 \ge \frac{1}{3}\norm{\Sig_{[1]}}^2_F} \ge 1 - \exp(-\Omega(\reffsig))\,. \label{eq:varbound_1}
            \end{equation}
            To show that $p_x(y)$ has good anti-concentration, it suffices to condition on an arbitrary assignment to $y_{[2]}$ and $y_{[3]}$ and only consider the randomness in $y_{[1]}\sim\unif{d_1}$. To show that $\iprod{y_{[1]}, \Sig^\top_{[1]}x}$ has sufficient anti-concentration, we need to show that the squared entries of $\Sig^\top_{[1]}x$ are not too large with high probability. Note that by Lemma~\ref{lem:slice_lip_linear}, for any $i\in S_{[1]}$ and $t' > 0$,
            \begin{equation}
                \Pr{|(\Sig_{[1]}^\top x)_i - \iprod{(\Sig_{[1]})_{:i}, \mu\cdot \partialone + \nu\cdot \partialtwo}| \ge t'\norm{(\Sig_{[1]})_{:i}}} \lesssim \exp(-\Omega(t'^2))\,.
            \end{equation}
            As $|\mu| \le c/\sqrt{d_2}$, by Cauchy-Schwarz we have that for any such $i$, $|\iprod{(\Sig_{[1]})_{:i}, \mu\cdot \partialone}| \le c\norm{(\Sig_{[1]})_{:i}} \le \frac{c\upsilon}{\sqrt{d}}\norm{\Sig}_F$. If the first bullet point holds, then similarly we have $|\iprod{(\Sig_{[1]})_{:i}, \nu\cdot\partialtwo}| \le \frac{c\upsilon}{\sqrt{d}}\norm{\Sig}_F$. Otherwise, if the second bullet point holds, then $|\iprod{(\Sig_{[1]})_{:i}, \nu\cdot\partialtwo}| \le \frac{c\upsilon\sqrt{d_3}}{\sqrt{d}}\norm{\Sig}_F \le \frac{c\upsilon^{1/2}\reffsig^{1/4}}{d^{1/4}}\norm{\Sig}_F$. So under either bullet point, we have that
            \begin{equation}
                \Pr*{|(\Sig^\top_{[1]} x)_i| \lesssim \Bigl(\frac{(c + t')\upsilon}{\sqrt{d}} + \frac{c\upsilon^{1/2}\reffsig^{1/4}}{d^{1/4}}\Bigr)\norm{\Sig}_F \ \forall \ i\in S_{[1]}} \ge 1 - d_1\exp(-\Omega(t'^2))\,. \label{eq:infbound_1}
            \end{equation}
            Condition on $x$ satisfying the events of Eq.~\eqref{eq:varbound_1} and~\eqref{eq:infbound_1}. Then by combining Berry-Esseen (Lemma~\ref{lem:berry}) with standard Gaussian anti-concentration, we conclude that for any $t > 0$, $s > 0$, and any fixing of $y_{[2]}, y_{[3]}$,
            \begin{equation}
                \Pr[y_{[1]}]{|p_x(y) - t| \le s\norm{\Sig}_F} \lesssim s^{1/2} + \frac{(c + t')\upsilon}{\sqrt{d}} + \frac{c\upsilon^{1/2}\reffsig^{1/4}}{d^{1/4}}\,.
            \end{equation}
            By a union bound with the events of Eq.~\eqref{eq:varbound_1} and~\eqref{eq:infbound_1}, taking $t' = \Theta(\sqrt{\log d})$ sufficiently large that the failure probability $d\exp(-\Omega(t'^2))$ in Eq.~\eqref{eq:infbound_1} is of lower order compared to the other failure probabilities, and noting that $\frac{(c+t')\upsilon}{\sqrt{d}} \ll \frac{c\upsilon^{1/2}\reffsig^{1/4}}{d^{1/4}}$, we conclude by Eq.~\eqref{eq:xQy} that the claimed bound in Eq.~\eqref{eq:xQy_anti} holds.

            \paragraph{Case 2: $\norm{\Sig_{[2]}}^2_F \ge \frac{1}{3}\norm{\Sig}^2_F$ or $\norm{\Sig_{[2]}}^2_F \ge \frac{1}{3}\norm{\Sig}^2_F$.}

            The two subcases here can be handled in an identical fashion, so without loss of generality we consider the former subcase. If the first bullet point of the hypothesis holds, then we have Eq.~\eqref{eq:variance2-first}. Note that $\frac{1}{d_2}\norm{\Sig_{[2]}\one\one^\top}^2_{\sf op} \le \norm{\Sig_{[2]}}^2_{\sf op} \le \norm{\Sig}^2_F / \reffsig$. The bound in Eq.~\eqref{eq:variance2-first} thus simplifies to $\mathbb{E}\norm{\Pi^\perp_{[2]}\Sig^\top_{[2]}x}^2 \gtrsim \norm{\Sig_{[2]}}^2$. On the other hand, if the second bullet of the hypothesis holds, then we can apply the assumed bound on $d_3$ to Eq.~\eqref{eq:variance2-second} to also conclude that $\mathbb{E}\norm{\Pi^\perp_{[2]}\Sig^\top_{[2]}x}^2 \gtrsim \norm{\Sig_{[2]}}^2$.

            So taking $t = \Theta(\norm{\Sig}^2_F)$ in Theorem~\ref{thm:main_HW} and noting that $\frac{t}{\norm{\Sig_{[2]}}^2_{\sf op}} \ge \reffsig$ and $\frac{t^2}{\norm{\Sig_{[2]}\Sig^\top_{[2]}}^2_F} \ge \reffsig$ by Assumption~\ref{assume:effranksig}, we conclude that
            \begin{equation}
                \Pr{\norm{\Sig^\top_{[2]}x}^2 \ge \frac{1}{3}\norm{\Sig_{[2]}}^2_F} \ge 1 - \exp(-\Omega(\reffsig))\,. \label{eq:varbound_2}
            \end{equation}
            To show that $p_x(y)$ has good anti-concentration, it suffices to condition on an arbitrary assignment to $y_{[2]}$ and $y_{[3]}$ and only consider the randomness in $y_{[2]}\sim\sldist{d_2}{\mu'}$. To show that $\iprod{y_{[2]},\Sig^\top_{[2]}x}$ has sufficient anti-concentration, we need to show that the squared entries of $\Sig^\top_{[2]}x$ are not too large with high probability. The argument for this is identical to the one leading to Eq.~\eqref{eq:infbound_1}, so
            \begin{equation}
                \Pr*{|(\Sig^\top_{[2]}x)_i| \lesssim \Bigl(\frac{(c+t')\upsilon}{\sqrt{d}} + \frac{c\upsilon^{1/2}\reffsig^{1/4}}{d^{1/4}}\Bigr)\norm{\Sig}_F \ \forall \ i\in \Sone} \ge 1 - d_2\exp(-\Omega(t'^2))\,. \label{eq:infbound_2}
            \end{equation}
            Condition on $x$ satisfying the events of Eq.~\eqref{eq:varbound_2} and~\eqref{eq:infbound_2}. Then by Theorem~\ref{thm:inv_slice}, we conclude that for any $t > 0, s > 0$, and any fixing of $y_{[1]}, y_{[3]}$,
            \begin{equation}
                \Pr[y_{[2]}]{|p_x(y) - t| \le s\norm{\Sig}_F} \lesssim \Bigl(\frac{(c+t')\upsilon}{\sqrt{d}} + \frac{c\upsilon^{1/2}\reffsig^{1/4}}{d^{1/4}}\Bigr)^C + s^{1/2}\,.
            \end{equation}
            By a union bound with the events of Eq.~\eqref{eq:varbound_2} and~\eqref{eq:infbound_2}, by taking $t' = \Theta(\sqrt{\log d})$ sufficiently large that the failure probability $d\exp(-\Omega(t'^2))$ in Eq.~\eqref{eq:infbound_2} is of lower order compared to the other failure probabilities, and again noting that $\frac{(c+t')\upsilon}{\sqrt{d}} \ll \frac{c\upsilon^{1/2}\reffsig^{1/4}}{d^{1/4}}$, we conclude by Eq.~\eqref{eq:xQy} that the claimed bound in Eq.~\eqref{eq:xQy_anti} holds.
        \end{proof}

        \noindent We will also need a version of Lemma~\ref{lem:xSigy_anticonc} that holds for general $\mu,\mu'$. This version will only be meaningful when both $d_2, d_3$ are small.

        \begin{lemma}\label{lem:xSigy_anticonc_corner}
            Let $\mu,\mu'\in[-1,1]$ be any multiples of $1/d_2$, and let $\nu,\nu'\in[-1,1]$ be any multiples of $1/d_3$. If $d_2, d_3 \lesssim \reffsig^2$, $x\sim\pi^s_{\mu,\nu;d_1,d_2,d_3}$, $y\sim\pi^s_{\mu',\nu';d_1,d_2,d_3}$, and $s,t > 0$, then
            \begin{equation}
                \mathbb{P}\Bigl[|x^\top \Sig y - t| \le s\,\norm{\Sig}_F\Bigr] \lesssim s^{1/2} + \frac{\upsilon^3\reffsig^3}{\sqrt{d}} \,.
            \end{equation}
        \end{lemma}

        \begin{proof}
            By the assumed bound on $d_2,d_3$, 
            \begin{equation}
                \norm{\Sig_{[1]}}^2_F = \norm{\Sig}^2_F - \sum^{d_3}_{i = d_1 + 1} \norm{\Sig_{:i}}^2 \ge (1 - \frac{\upsilon^2 (d_2+d_3)}{d})\norm{\Sig}^2_F \gtrsim \norm{\Sig}^2_F\,, \label{eq:useproj_off_P}
            \end{equation}
            so $\norm{\Sig_{[1]}}^2_{\sf op} \lesssim \norm{\Sig_{[1]}}^2_F/\reffsig$.

            We first apply Theorem~\ref{thm:main_HW} to argue that $\Sig_{[1]}^\top x$ has large norm while $\Sig_{[2]}^\top x$ and $\Sig_{[3]}^\top x$ both have small norm with high probability.

            We can apply Eq.~\eqref{eq:eAx} to $\bA = \Sig_{[1]}$ to get
            \begin{align}
                \mathbb{E}\norm{\Sig_{[1]}^\top x}^2 &= \norm{\Sig_{[1]}}^2_F - \frac{\mu^2 d_2 - 1}{d_2 - 1} \cdot \norm{\Pone\Sig_{[1]}}^2_F + \frac{\mu^2 d_2 - 1}{d_2 - 1} \cdot \partialone^\top (\Sig_{[1]}\Sig_{[1]}^\top) \partialone \\
                &\quad\quad\quad\quad\quad\,\, - \frac{\nu^2 d_3 - 1}{d_3 - 1}\cdot \norm{\Ptwo \Sig_{[1]}}^2_F + \frac{\nu^2 d_3 - 1}{d_3 - 1}\cdot \partialtwo^\top (\Sig_{[1]}\Sig_{[1]}^\top)\partialtwo + 2\mu\nu\cdot \partialone^\top (\Sig_{[1]}\Sig_{[1]}^\top) \partialtwo\\
                &\ge \Bigl(1 - \frac{4}{\reffsig}\Bigr)\,\norm{\Sig_{[1]}}^2_F - \norm{\Pone\Sig_{[1]}}^2_F - \norm{\Ptwo\Sig_{[1]}}^2_F - \norm{\Sig^\top_{[1]}\partialone}\cdot\norm{\Sig^\top_{[2]}\partialtwo} \\
                &\ge \Bigl(1 - \frac{4}{\reffsig} - \frac{O(\upsilon^2 d_2d_3)}{d}\Bigr)\,\norm{\Sig_{[1]}}^2_F \ge \frac{1}{2}\norm{\Sig}^2_F\,,
            \end{align}
            where in the second step we used that $-\frac{1}{d_2 - 1} \le \frac{\mu^2 d_2 - 1}{d_2 - 1} \le 1$ for any $\mu\in[-1,1]$, that $\partialone^\top \Sig_{[1]}\Sig_{[1]}^\top \partialone \le d_2\norm{\Sig_{[1]}}^2_\op \le \frac{d_2}{\reffsig}\norm{\Sig_{[1]}}^2_F$ by Assumption~\ref{assume:effranksig}, and that $\frac{d_2}{d_2 - 1} \le 2$ for all $d_2 > 1$, noting that when $d_2 \le 1$, $\mu \in \brc{-1,1}$ so that $\frac{\mu^2 d_2 - 1}{d_2 - 1} = 1$. In the third step we used Assumption~\ref{assume:lightrows}, and in the last step we used the hypothesis that $d_2,d_3 \lesssim \reffsig^2$ and that $\reffsig^2 \le \sqrt{d}/\upsilon$ in Assumption~\ref{assume:effranksig}.

            So applying Theorem~\ref{thm:main_HW} with $t = \frac{1}{6}\norm{\Sig_{[1]}}^2_F$ to the quadratic form given by $\Sig_{[1]}\Sig_{[1]}^\top$, and noting that $\frac{t}{\norm{\Sig_{[1]}\Sig_{[1]}^\top}_{\op}} \ge \reffsig$ and $\frac{t^2}{\norm{\Sig_{[1]}\Sig_{[1]}^\top}^2_F}\ge \reffsig$, we conclude that
            \begin{equation}
                \Pr{\norm{ \Sig^\top_{[1]} x}^2 \le \frac{1}{3}\norm{\Sig}^2_F} \lesssim \exp(-\Omega(\reffsig))\,.
            \end{equation}

            Next, we show that with high probability, the first $d_1$ entries of $\Sig^\top x$ are not too large. To bound these, note that by Lemma~\ref{lem:slice_lip_linear}, for any $s'' > 0$ we have for any $i\in[d_1]$ that
            \begin{equation}
                \Pr{|(\Sig^\top x)_i - \iprod{\Sig_{:i}, \mu\cdot\partialone + \nu\cdot\partialtwo}|\ge s''\norm{\Sig_{:i}}} \lesssim \exp(-\Omega(s''^2))\,.
            \end{equation}
            Note that $|\iprod{\Sig_{:i},\mu\cdot\partialone + \nu\cdot\partialtwo}|\le \sqrt{\max(d_2,d_3)}\norm{\Sig_{:i}} \le \frac{\upsilon\sqrt{\max(d_2,d_3)}}{\sqrt{d}}\norm{\Sig}_F$, so using the assumed bound $d_2,d_3\lesssim \reffsig^2$ and taking $s'' = \reffsig$, we conclude that $\Pr{|(\Sig^\top x)_i| \lesssim \frac{\reffsig}{\sqrt{d}}\norm{\Sig}_F \ \forall \ i\in[d_1]}\ge 1 - d\exp(-\Omega(\reffsig^2))$.
            
            Henceforth condition on the event that
            \begin{itemize}
                \item $\norm{\Sig_{[1]}^\top x}^2 > \frac{1}{3}\norm{\Sig}^2_F$
                \item $|(\Sig^\top x)_i| \lesssim \frac{\reffsig}{\sqrt{d}}\norm{\Sig}_F$ for all $i\in[d_1]$.
            \end{itemize}
            This event happens with probability at least $1 - \exp(-\Omega(\min(\reffsig)))$.

            We can decompose $x^\top \Sig y$ into $\iprod{\Sig_{[2,3]}^\top x, y_{[2,3]}}$ and $\iprod{\Sig_{[1]}^\top x, y_{[1]}}$. 
            We will show anti-concentration of $\iprod{\Sig_{[1]}^\top x, y_{[1]}}$. Observe that $P_{[1]} y$ has independent random entries in its first $d_1$ coordinates and zeroes elsewhere. Note that
            \begin{equation}
                \Var[y]{\iprod{\Sig_{[1]}^\top x, y_{[1]}}} = \norm{\Sig_{[1]}^\top x}^2 > \frac{1}{3}\norm{\Sig}^2_F
            \end{equation}
            and
            \begin{equation}
                \sum^{d_1}_{i=1} |\Sig^\top x|^3_i \lesssim \frac{\upsilon^3 \reffsig^3}{\sqrt{d}}\norm{\Sig}^3_F\,,
            \end{equation}
            so by Berry-Esseen (Lemma~\ref{lem:berry}), for any $s,t > 0$ we have
            \begin{equation}
                \Pr[y]{|\iprod{\Sig_{[1]}^\top x, y_{[1]}} - t| \le s\norm{\Sig}_F} \le \Pr[\gamma\sim\calN(0,\frac{1}{3}\norm{\Sig}^2_F)]{|\gamma - t| \le u\norm{\Sig}_F} + O\Bigl(\frac{\upsilon^3 \reffsig^3}{\sqrt{d}}\Bigr) \lesssim s^{1/2} + \frac{\upsilon^3 \reffsig^3}{\sqrt{d}}\,. \label{eq:perplarge}
            \end{equation}
            As this holds for any $t$ and furthermore $y_{[1]}$ is independent of $y_{[2,3]}$, $x^\top \Sig y$ enjoys the same anti-concentration. The claimed bound follows upon noting that $\exp(-\Omega(\reffsig))$ is of lower order compared to the other terms in the probability bound.
        \end{proof}

    \subsection{Estimates via coupling}
        \label{sec:coupling}

        Recall that our goal is to show that conditioning on typical values of $\partialone^\top \bX$ and $\partialtwo^\top \bX$ does not affect the conditional expectation of $\softmax(\bX\Sig \bX^\top) \bX$. To that end, in this subsection we show that for two different conditionings on $(\partialone^\top \bX, \partialtwo^\top \bX)$, there is a coupling between the conditional distributions under which the deviation in the value of the matrix $\softmax(\bX\Sig\bX^\top)\bX$ in the two cases is small with high probability. We first define the coupling:

        \begin{definition}\label{def:coupling}
            For $\mu,\nu,\mu',\nu'\in[-1,1]$, consider the following coupling of $\pi^s_{\mu,\nu;d_1,d_2,d_3}$ and $\pi^s_{\mu',\nu';d_1,d_2,d_3}$. Given $x\in\cube{d_1}\times\slice{d_2}{\mu}\times \slice{d_3}{\nu}$ sampled from $\pi^s_{\mu,\nu;d_1,d_2,d_3}$, we define a sample $x'$ from $\pi^s_{\mu',\nu';d_1,d_2,d_3}$ under this coupling as follows.

            Let $T\subseteq \Sone$ (resp. $U\subseteq \Stwo$) denote the subset of $(\frac{1+\mu}{2})d_2$ (resp. $(\frac{1+\nu}{2})d_3$) coordinates of $x$ within those respective blocks that correspond to positive entries.

            If $\mu \le \mu'$ (resp. $\nu \le \nu'$), then sample a random subset of $\brc{d_1+1,\ldots,d_1+d_2}\backslash T$ (resp. $\brc{d_1+d_2+1,\ldots,d}\backslash U$) of size $(\frac{\mu' - \mu}{2})d_2$ (resp. $(\frac{\nu'-\nu}{2})d_3$) and define $x'$ to be given by setting the entries of $x$ indexed by this random subset to be $+1$. 

            Otherwise, if $\mu > \mu'$ (resp. $\nu > \nu'$), then sample a random subset of $T$ (resp. $U$) of size $(\frac{\mu - \mu'}{2})d_2$ (resp. $(\frac{\nu - \nu'}{2})d_3$) and define $x'$ to be given by setting the entries of $x$ indexed by this random subset to be $-1$.

            Henceforth, we will denote a draw from this coupling by $(x,x') \sim \calD_{\mu,\nu,\mu',\nu';d_1,d_2,d_3}$. When $d_1,d_2,d_3$ are clear from context, we denote the coupling by $\calD_{\mu,\nu,\mu',\nu'}$.
        \end{definition}

        \noindent Note that conditioned on $x$, under the coupling, $x' - x$ is a random string whose first $d_1$ entries are zero, and whose remaining entries satisfy the following. If $\mu \le \mu'$, then all entries indexed by $T$ are $0$, and the remaining entries in the block $\brc{d_1+1,\ldots,d_1+d_2}$ are a random string in $\brc{0,2}^{(1 - \mu)d_2/2}$ with exactly $(\frac{\mu' - \mu}{2})d_2$ nonzero entries. On the other hand, if $\mu > \mu'$, then a random subset of $T$ of size $(\frac{\mu - \mu'}{2})d_2$ are given by $-2$ entries, and the remaining entries in the block $\brc{d_1+1,\ldots,d_1+d_2}$ are all zero. The situation for the third block $\brc{d_1+d_2+1,d}$ is entirely analogous. Marginalizing over $x$, $x ' - x$ is thus a random string in 
        \begin{equation}
            \brc{0}^{d_1}\times \brc{0,(-1)^{\bone{\mu\le \mu'}}\cdot 2}^{d_2} \times \brc{0,(-1)^{\bone{\nu\le \nu'}}\cdot 2}^{d_3}
        \end{equation}
        with exactly $(\frac{|\mu' - \mu|}{2})d_2$ nonzero entries among the coordinates indexed by $\Sone$, and exactly $(\frac{|\nu' - \nu|}{2})d_3$ nonzero entries among the coordinates indexed by $\Stwo$.

        The following lemma shows that if $y,y'$ are drawn from the above coupling and $x$ is an independent sample, then $x^\top \Sig y$ and $x^\top \Sig y$ are close with high probability.

        \begin{lemma}\label{lem:offdiag_diff_small}
            Let $1 \le c \le \sqrt{\min(\reffsig,d_2,d_3)}/2$ and let $|\mu|, |\mu'|, |\mu''| \le c/\sqrt{d_2}$. Suppose one of the following holds:
            \begin{itemize}
                \item $|\nu| \le c/\sqrt{d_3}$ for $1 \le c \le \sqrt{\min(\reffsig,d_2,d_3)}/2$,
                \item $d_3 \le \min(d/8\upsilon^2, \sqrt{d\reffsig}/8c\upsilon)$,
            \end{itemize}
            If the former bullet point holds, then
            \begin{equation}
                \mathbb{P}_{\substack{x\sim\pi^s_{\mu,\nu;d_1,d_2,d_3}, \\ (y,y')\sim\calD_{\mu',\nu',\mu'',\nu''}}}\Bigl[|x^\top \Sig (y - y')| > \Omega\Bigl(\frac{\sqrt{c}\cdot \reffsig^{1/8}}{\min(d_2,d_3)^{1/4}} + \frac{c(c+\sqrt{\log\reffsig})}{\reffsig^{3/8}}\Bigr)\norm{\Sig}_F \Bigr] \lesssim 1/\reffsig^{1/4}
            \end{equation}
            If the latter bullet point holds, then
            \begin{multline}
                \mathbb{P}_{\substack{x\sim\pi^s_{\mu,\nu;d_1,d_2,d_3}, \\ (y,y')\sim\calD_{\mu',\nu',\mu'',\nu''}}}\Bigl[|x^\top \Sig (y - y')| > \Omega\Bigl(\frac{\sqrt{c}\cdot \reffsig^{1/8}}{d_2^{1/4}} + \frac{c(c+\sqrt{\log\reffsig})}{\reffsig^{3/8}} \\
                 + \frac{\reffsig^{1/8}\upsilon\sqrt{d_3}\cdot (d_3 + \sqrt{\log\reffsig})}{\sqrt{d}} + \frac{\reffsig^{1/8}(\sqrt{c\upsilon d_3} + \sqrt{cs}\sqrt[4]{\log\reffsig})}{d^{1/4}} \Bigr)\norm{\Sig}_F \Bigr] \lesssim 1/\reffsig^{1/4}
            \end{multline}
        \end{lemma}

        \begin{proof}
            We first show that $\Sig^\top x$ has norm comparable to $\norm{\Sig}_F$. The argument is essentially the same as what was employed in the proofs of Lemmas~\ref{lem:xSigy_anticonc} and~\ref{lem:xSigy_anticonc_corner}. Recall from \eqref{eq:eAx} that
            \begin{equation}
                \mathbb{E}\norm{\Sig^\top x}^2 = \norm{\Sig}^2_F +\sum_{i,j\in \Sone: i\neq j} (\Sig\Sig^\top)_{ij} \cdot \frac{\mu^2 d_2 - 1}{d_2 - 1} + \sum_{i,j\in \Stwo: i\neq j} (\Sig\Sig^\top)_{ij} \cdot \frac{\nu^2 d_3 - 1}{d_3 - 1} + 2\mu\nu \sum_{i\in S_{[2]}, j\in S_{[3]}} (\Sig\Sig^\top)_{ij} \,. \label{eq:analogous}
            \end{equation}
            By the assumed bound on $|\mu|$, we can bound the second term on the right-hand side via
            \begin{equation}
                \Bigl|\sum_{i,j\in \Sone: i\neq j} (\Sig\Sig^\top)_{ij}\Bigr| \cdot \frac{\mu^2 d_2 - 1}{d_2 - 1} \le \frac{c^2}{d_2} \partialone^\top(\Sig\Sig^\top)\partialone \le c^2\norm{\Sig}^2_\op \le \frac{c^2}{\reffsig}\norm{\Sig}^2_F\,, \label{eq:secondtermanalogous}
            \end{equation}
            where the last step follows by Assumption~\ref{assume:effranksig}.
            For the third term on the right-hand side of Eq.~\eqref{eq:analogous}, if the first bullet point holds, then we get the same bound as in Eq.~\eqref{eq:secondtermanalogous}. Otherwise, if the second bullet point holds, then 
            \begin{equation}
                 \Bigl|\sum_{i,j\in \Stwo: i\neq j} (\Sig\Sig^\top)_{ij}\Bigr| \cdot \frac{\nu^2 d_3 - 1}{d_3 - 1} \le \partialtwo^\top (\Sig\Sig^\top)\partialtwo \le \frac{d_3\upsilon^2}{d}\norm{\Sig}^2_F\,.
            \end{equation} 
            For the fourth term on the right-hand side of Eq.~\eqref{eq:analogous}, if the first bullet point holds, then we get the same bound as in Eq.~\eqref{eq:secondtermanalogous}. Otherwise, if the second bullet point holds, then
            \begin{equation}
                2\mu\nu \Bigl|\sum_{i\in\Sone, j\in \Stwo} (\Sig\Sig^\top)_{ij}\Bigr| \le \frac{2c}{\sqrt{d_2}} \partialone^\top(\Sig\Sig^\top)\partialtwo \le 2c\norm{\Sig}_{\sf op} \cdot \norm{\Sig^\top \partialtwo} \le \frac{2c d_3\upsilon}{\sqrt{\reffsig d}}\norm{\Sig}^2_F\,.
            \end{equation}
            In either case, we conclude that 
            \begin{equation}
                \mathbb{E}\norm{\Sig^\top x}^2 \lesssim \norm{\Sig}^2_F\,.
            \end{equation}
            So taking $t = \norm{\Sig}^2_F$ in Theorem~\ref{thm:main_HW}, and noting that $\frac{t}{\norm{\Sig\Sig^\top}_\op} \le \reffsig$ and $\frac{t^2}{\norm{\Sig\Sig^\top}^2} \le \reffsig$ by Assumption~\ref{assume:effranksig}, we conclude that
            \begin{equation}
                \Pr{\norm{\Sig^\top x}^2 > \Omega(\norm{\Sig}^2_F)} \lesssim \exp(-\Omega(\reffsig))\,.
            \end{equation}
            
            Additionally, we would like to show that $x^\top \Sig \partialone$ and $x^\top \Sig \partialtwo$ have magnitude comparable to $\sqrt{d_2}\norm{\Sig}_F$ and $\sqrt{d_3}\norm{\Sig}_F$. First observe that 
            \begin{equation}
                \mu\cdot |\partialone^\top \Sig \partialone| \le c\sqrt{d_2}\norm{\Sig}_{\op} \le c\sqrt{\frac{d_2}{\reffsig}}\norm{\Sig}_F
            \end{equation}
            If the first bullet point in the hypothesis holds, then
            \begin{equation}
                \nu\cdot |\partialtwo^\top \Sig \partialone| \le c\sqrt{d_2}\norm{\Sig}_{\op} \le c\sqrt{\frac{d_2}{\reffsig}}\norm{\Sig}_F\,.
            \end{equation}
            Otherwise if the second bullet point in the hypothesis holds, then
            \begin{equation}
                \nu\cdot |\partialtwo^\top \Sig \partialone| \le \sum^d_{i = d_1+d_2+1} \sqrt{d_2} \norm{\Sig_{i:}} \le \frac{\Siglightrows d_3\sqrt{d_2}}{\sqrt{d}}\,\norm{\Sig}_F\,. 
            \end{equation}
            By replacing $\Sig$ with its transpose in the above, we get an analogous bound for $\nu\cdot |\partialone^\top \Sig \partialtwo|$. Additionally, if the first bullet point in the hypothesis holds, then
            \begin{equation}
                \nu\cdot |\partialtwo^\top \Sig \partialtwo| \le c\sqrt{d_3}\norm{\Sig}_{\sf op} \le c\sqrt{\frac{d_3}{\reffsig}}\norm{\Sig}_F\,,
            \end{equation}
            whereas if the second bullet point holds, then
            \begin{equation}
                \nu\cdot |\partialtwo^\top \Sig \partialtwo| \le \sum^d_{i=d_1+d_2+1}\sqrt{d_3}\norm{\Sig_{i:}} \le \frac{\upsilon d^{3/2}_3}{\sqrt{d}} \norm{\Sig}_F\,. \label{eq:nuexp} 
            \end{equation}
            Additionally,
            \begin{equation}
                \norm{\Sig\partialone}^2 = \partialone^\top (\Sig^\top \Sig) \partialone \le d_2\norm{\Sig^\top \Sig}_\op \le \frac{d_2}{\reffsig}\norm{\Sig}^2_F\,,
            \end{equation}
            \begin{equation}
                \norm{\Sig\partialtwo}^2 = \partialtwo(\Sig^\top \Sig)\partialtwo \le d_3 \norm{\Sig^\top \Sig}_{\sf op} \le \frac{d_3}{\reffsig}\norm{\Sig}^2_F
            \end{equation}
            by Assumption~\ref{assume:effranksig}, and also
            \begin{equation}
                \norm{\Sig\partialtwo}^2 \le \frac{d_3\upsilon^2}{d}\norm{\Sig}^2_F\,, \label{eq:Q3}
            \end{equation}
            by Assumption~\ref{assume:lightrows}.

            Now suppose that the first bullet point holds. By Lemma~\ref{lem:slice_lip_linear} applied to $v = \Sig^\top \partialone$ and $t = s\sqrt{\frac{d_2}{\reffsig}}\norm{\Sig}_F$ for $s > 0$,
            \begin{equation}
                \Pr{|x^\top \Sig \partialone| > (2c + s)\sqrt{\frac{d_2}{\reffsig}}\norm{\Sig}_F} \lesssim \exp(-\Omega(s^2))\,. \label{eq:alwaysholds}
            \end{equation}
            Likewise, by Lemma~\ref{lem:slice_lip_linear} applied to $v = \Sig^\top \partialtwo$ and $t = s\sqrt{\frac{d_3}{\reffsig}}\norm{\Sig}_F$ for $s > 0$,
            \begin{equation}
                \Pr{|x^\top \Sig \partialtwo| > (2c + s)\sqrt{\frac{d_3}{\reffsig}}\norm{\Sig}_F} \lesssim \exp(-\Omega(s^2))\,.
            \end{equation}
            Henceforth, condition on the event that $\norm{\Sig^\top x}^2 \le 2\norm{\Sig}^2_F$, $(x^\top \Sig \partialone)^2 \le \frac{(2c + s)^2 d_2}{\reffsig}\norm{\Sig}^2_F$, and $(x^\top \Sig \partialtwo)^2 \le \frac{(2c + s)^2 d_3}{\reffsig}\norm{\Sig}^2_F$.
            
            Recall that $y - y'$ is distributed as a random string in $\brc{0,2}^d$ with exactly $(\frac{|\mu'' - \mu'|}{2})d_2$ nonzero entries among $\Sone$, exactly $(\frac{|\nu'' - \nu'|}{2})d_3$ nonzero entries among $\Stwo$, and zeroes elsewhere.

            So if the first bullet point holds, then for $w\triangleq \Sig^\top x$,
            \begin{align}
                \MoveEqLeft \E{\iprod{w,(y - y')}^2} \\
                &\lesssim \frac{\max(|\mu'' - \mu'|, |\nu'' - \nu'|)}{2}\norm{w}^2 + \Bigl(\frac{\mu''-\mu'}{2}\Bigr)^2 \Bigl|\sum_{i,j\in \Sone: i\neq j}  w_i w_j \Bigr| + \Bigl(\frac{\nu'' - \nu'}{2}\Bigr)^2 \Bigl|\sum_{i,j\in \Stwo: i\neq j}  w_i w_j\Bigr| \\
                &\quad\quad\quad + 2\Bigl|\frac{\mu'' - \mu'}{2}\Bigr|\cdot \Bigl|\frac{\nu'' - \nu'}{2}\Bigr| \cdot \Bigl|\sum_{i\in\Sone, j\in\Stwo} w_i w_j\Bigr| \\
                &\lesssim \frac{c}{\sqrt{\min(d_2,d_3)}}\norm{\Sig}^2_F + \frac{c^2}{d_2}\iprod{w,\partialone}^2 + \frac{c^2}{d_3}\iprod{w,\partialtwo}^2 + \frac{2c^2}{\sqrt{d_2d_3}}|\iprod{w,\partialone}\iprod{w,\partialtwo}| \\
                &\lesssim \Bigl(\frac{c}{\sqrt{\min(d_2,d_3)}} + \frac{c^2}{\reffsig}(c + s)^2\Bigr)\norm{\Sig}^2_F\label{eq:wyy}\,.
            \end{align}
            So by Markov's, for $s' > 0$ we have that
            \begin{equation}
                \mathbb{P}\Bigl[|x^\top \Sig (y - y')| > s'\Bigl(\frac{c}{\sqrt{\min(d_2,d_3)}} + \frac{c^2}{\reffsig}(c + s)^2\Bigr)^{1/2}\norm{\Sig}_F\Bigr] \lesssim 1/s'^2 + \exp(-c'\reffsig) + \exp(-\Omega(s^2))\,.
            \end{equation}
            The first part of the lemma then follows by taking $s = \sqrt{\log \reffsig}$ and $s' = \reffsig^{1/8}$.
            
            Now we prove the second part of the lemma. Suppose the second bullet point holds.  Then by another application of Lemma~\ref{lem:slice_lip_linear}, this time using the estimates Eq.~\eqref{eq:nuexp} and~\eqref{eq:Q3}, we get
            \begin{equation}
                \Pr{|x^\top \Sig\partialtwo| > \Bigl(\frac{\Siglightrows d^{3/2}_3}{\sqrt{d}} + s\frac{\sqrt{d_3}\upsilon}{\sqrt{d}}\Bigr)\norm{\Sig}_F} \lesssim \exp(-\Omega(s^2))\,.
            \end{equation}
            Also note that Eq.~\eqref{eq:alwaysholds} still holds. Henceforth, condition on the event that $\norm{\Sig^\top x}^2 \le 2\norm{\Sig}^2_F$, $(x^\top \Sig\partialone)^2 \le \frac{(2c+s)^2 d_2}{\reffsig}\norm{\Sig}^2_F$, and $(x^\top \Sig \partialtwo)^2 \lesssim (\frac{\Siglightrows^2 d^3_3 + s^2 d_3\upsilon^2}{d})\norm{\Sig}^2_F$.

            Next, we have
            \begin{align}
                \MoveEqLeft \E{\iprod{w,(y - y')}^2} \\
                &\le \frac{|\mu'' - \mu'|}{2}\norm{P_{[2]}w}^2 + \frac{|\nu'' - \nu'|}{2}\norm{P_{[3]}w}^2 + \Bigl(\frac{\mu''-\mu'}{2}\Bigr)^2 \Bigl|\sum_{i,j\in \Sone: i\neq j}  w_i w_j \Bigr| + \Bigl(\frac{\nu'' - \nu'}{2}\Bigr)^2 \Bigl|\sum_{i,j\in \Stwo: i\neq j}  w_i w_j\Bigr| \\
                &\quad\quad\quad + 2\Bigl|\frac{\mu'' - \mu'}{2}\Bigr|\cdot \Bigl|\frac{\nu'' - \nu'}{2}\Bigr| \cdot \Bigl|\sum_{i\in\Sone, j\in\Stwo} w_i w_j\Bigr| \\
                &\le \frac{c}{\sqrt{d_2}} \norm{\Sig}^2_F + \frac{c^2}{d_2} \iprod{w,\partialone}^2+ \iprod{w,\partialtwo}^2 + \frac{2c}{\sqrt{d_3}} \norm{w} \cdot |\iprod{w, \partialtwo}| \\
                &\lesssim \Bigl(\frac{c}{\sqrt{d_2}} + \frac{(c^2c+s)^2}{\reffsig} + \frac{\Siglightrows^2 d^3_3 + s^2 d_3\upsilon^2}{d} + \frac{c(\upsilon d_3 + s\upsilon)}{\sqrt{d}}\Bigr)\norm{\Sig}^2_F \,.
            \end{align}
            So by Markov's, 
            \begin{equation}
                \mathbb{P}\Bigl[|x^\top \Sig (y - y')| > s'\Bigl(\frac{c}{\sqrt{d_2}} + \frac{c^2}{\reffsig}(c + s)^2 +\frac{\Siglightrows^2 d^3_3 + s^2 d_3\upsilon^2}{d} + \frac{c(\upsilon d_3 + s\upsilon)}{\sqrt{d}}\Bigr)^{1/2}\norm{\Sig}_F\Bigr] \lesssim 1/s'^2 + \exp(-c'\reffsig) + \exp(-\Omega(s^2))\,.
            \end{equation}
            The second part of the lemma then follows by taking $s = \sqrt{\log \reffsig}$ and $s' = \reffsig^{1/8}$.
        \end{proof}

        \noindent We will also need a version of Lemma~\ref{lem:offdiag_diff_small} that holds for general $\mu,\mu'$. This version will only be meaningful for small $d_2$:

        \begin{lemma}\label{lem:offdiag_diff_small_alt}
            Let $\mu,\mu',\mu''\in[-1,1]$ be any multiples of $1/d_2$. If $x\sim\pi^s_{\mu,\nu;d_1,d_2,d_3}$, $y\sim\pi^s_{\mu',\nu';d_1,d_2,d_3}$, and $y'\sim\pi^s_{\mu'';d_1,d_2}$, then for any $s,s',s'' > 0$,
            \begin{equation}
                \mathbb{P}_{\substack{x\sim\pi^s_{\mu,\nu;d_1,d_2,d_3}, \\ (y,y')\sim\calD_{\mu',\nu',\mu'',\nu''}}}\Bigl[|x^\top\Sig y - x^\top \Sig y'| > \frac{s''\upsilon(d_2+d_3)}{d} \cdot (s + s'\sqrt{d_2+d_3})\cdot \norm{\Sig}_F\Bigr)\Bigr] \lesssim 1/s^2 + 1/s'^2 + 1/s''^2\,.
            \end{equation}
        \end{lemma}

        \begin{proof}
            Note that $P_{[1]}(y - y') = 0$, so $x^\top \Sig y - x^\top \Sig y' = x^\top \Sig_{[2,3]} (y - y')$. 
            We have that
            \begin{align}
                \mathbb{E}\norm{\Sig^\top_{[2,3]} x}^2 &\le \norm{\Sig_{[2,3]}}^2_F + \partialone^\top (\Sig_{[2,3]} \Sig_{[2,3]}^\top) \partialone + \partialtwo^\top(\Sig_{[2,3]} \Sig^\top_{[2,3]})\partialtwo + 2\cdot \partialone^\top (\Sig_{[2,3]} \Sig^\top_{[2,3]}) \partialtwo \\
                &\lesssim \frac{\upsilon^2 (d_2+d_3)^2}{d}\,\norm{\Sig}^2_F \,,
            \end{align}
            where in the second step we used Assumption~\ref{assume:lightrows}.
            By Markov's, we also have that for any $s>0$,
            \begin{equation}
                \mathbb{P}\Bigl[\norm{\Sig^\top_{[2,3]} x}^2 \ge \frac{s^2\upsilon^2(d_2+d_3)^2}{d}\norm{\Sig}^2_F\Bigr] \lesssim 1/s^2\,. \label{eq:PSig}
            \end{equation}
            Additionally,
            \begin{align}
                \MoveEqLeft \E{(x^\top \Sig_{[2,3]} \partialone)^2} \\
                &= \norm{\Sig_{[2,3]} \partialone}^2 + \sum_{i,j\in \Sone: i\neq j} (\Sig_{[2,3]} \partialone)_i (\Sig_{[2,3]}\partialone)_j\cdot \frac{\mu^2 d_2 - 1}{d_2 - 1} + \sum_{i,j\in \Stwo: i\neq j} (\Sig_{[2,3]} \partialone)_i (\Sig_{[2,3]}\partialone)_j\cdot \frac{\nu^2 d_3 - 1}{d_3 - 1} \\
                &\quad\quad\quad + 2\mu\nu\sum_{i\in S_{[2]}, j\in S_{[3]}} (\Sig_{[2,3]}\partialone)_i (\Sig_{[2,3]}\partialone)_j \le \frac{\upsilon^2 (d_2+d_3)^3}{d} \norm{\Sig}^2_F\,,
            \end{align}
            so by Markov's, for any $s' > 0$,
            \begin{equation}
                \mathbb{P}\Bigl[(x^\top \Sig_{[2,3]}\partialone)^2 > \frac{s'^2\upsilon^2(d_2+d_3)^3}{d} \norm{\Sig}^2_F\Bigr] \lesssim 1/s'^2\,, \label{eq:SigPOne}
            \end{equation}
            and we have an identical bound for $(x^\top \Sig_{[2,3]}\partialtwo)^2$.

            Henceforth condition on the event that $\norm{\Sig^\top_{[2,3]} x}^2 \le \frac{s^2\upsilon^2(d_2+d_3)^2}{d}\norm{\Sig}^2_F$, $(x^\top \Sig_{[2,3]}\partialone)^2 \le \frac{s'^2\upsilon^2(d_2+d_3)^3}{d} \norm{\Sig}^2_F$, and $(x^\top \Sig_{[2,3]}\partialtwo)^2 \le \frac{s'^2\upsilon^2(d_2+d_3)^3}{d} \norm{\Sig}^2_F$.
            
            Recall that $y - y'$ is a vector in $\brc{0,2}^d$ for which a random $\frac{|\mu'' - \mu'|}{2}$ fraction of the coordinates in $\Sone$ are equal to $2$, a random $\frac{|\nu'' - \nu'|}{2}$ fraction of the coordinates in $\Stwo$ are equal to $2$, and the remaining entries are zero. So by applying the same calculation that led to Eq.~\eqref{eq:wyy} above, except taking $w = \Sig^\top_{[2,3]} x$ therein, we have
            \begin{align}
                \MoveEqLeft\E[(y,y')]{(x^\top \Sig_{[2,3]} (y - y'))^2} \\
                &\lesssim \frac{\max(|\mu'' - \mu'|, |\nu'' - \nu'|}{2}\norm{\Sig^\top_{[2,3]} x}^2 + \Bigl(\frac{\mu'' - \mu'}{2}\Bigr)^2 (x^\top \Sig_{[2,3]} \partialone)^2 + \Bigl(\frac{\nu'' - \nu'}{2}\Bigr)^2 (x^\top \Sig_{[2,3]} \partialtwo)^2 \\
                &\lesssim \frac{\upsilon^2(d_2+d_3)^2}{d}\cdot (s^2 + s'^2(d_2+d_3))\cdot \norm{\Sig}^2_F\,,
            \end{align}
            so the claimed bound follows by Markov's and triangle inequality.
        \end{proof}

        \noindent Next, we prove that if $x,x'$ are drawn from the coupling, then the quadratic forms $x^\top \Sig x$ and $x'^\top \Sig x'$ are close with high probability.

        \begin{lemma}\label{lem:diag_diff_small}
            Let $1 \le c \le \sqrt{\min(\reffsig,d_2,d_3)}/2$, and let $|\mu|, |\mu'|\le c/\sqrt{d_2}$.

            If $|\nu|, |\nu'| \le c/\sqrt{d_3}$, then
            \begin{equation}
                \mathbb{P}_{(x,x') \sim \calD_{\mu,\nu,\mu',\nu'}}\Bigl[|x^\top \Sig x - x'^\top \Sig x'| > \Omega\Bigl(\frac{\sqrt{c}\cdot \reffsig^{1/4}}{\min(d_2,d_3)^{1/4}} + \frac{c}{\reffsig^{1/4}}\Bigr)\cdot \norm{\Sig}_F\Bigr] \lesssim 1/\reffsig^{1/4}\,.
            \end{equation}
            Otherwise, in general we have
            \begin{equation}
                \mathbb{P}_{(x,x') \sim \calD_{\mu,\nu,\mu',\nu'}}\Bigl[|x^\top \Sig x - x'^\top \Sig x'| > \Omega\Bigl(\frac{\sqrt{c}\cdot \reffsig^{1/4}}{d_2^{1/4}} + \frac{c}{\reffsig^{1/4}} + \frac{d^{3/2}_3 \upsilon}{\sqrt{d}}\Bigr)\cdot \norm{\Sig}_F\Bigr] \lesssim 1/\reffsig^{1/4}\,.
            \end{equation}
        \end{lemma}

        \begin{proof}
            Suppose without loss of generality $\mu \le \mu'$ and $\nu \le \nu'$ (the case of $\nu > \nu'$ follows along very similar lines, so we omit it here). Fix $x$ and consider the conditional distribution on $x'$ under the coupling $\calD_{\mu,\nu,\mu',\nu'}$. If $\Tone\subset \Sone$ denotes the set of coordinates in $\Sone$ corresponding to positive entries of $x$, and $\Ttwo\subset \Stwo$ denotes the same in $\Stwo$, then $x'$ is distributed as $x + 2\sum_{i\in \Uone} e_i + 2\sum_{i\in\Utwo} e_i$, where $\Uone$ is a random subset of $\Sone\backslash \Tone$ of size $(\frac{\mu' - \mu}{2})d_2$ and $\Utwo$ is a random subset of $\Stwo\backslash\Ttwo$. 

            Consider the string $y \triangleq x + \sum_{i\in \Uone}e_i + \sum_{i\in \Utwo} e_i$. Marginalizing over $x, x'$, observe that $y$ is distributed as follows: its first $d_1$ bits are a string uniform from $\brc{\pm 1}^{d_1}$; among its bits in $\Sone$ (resp. $\Stwo$), a random $\frac{1+\mu}{2}$ (resp. $\frac{1 + \nu}{2}$) fraction indexed by the random subset $\Tone$ (resp. $\Ttwo$) are $+1$, a random $\frac{\mu'-\mu}{2}$ (resp. $\frac{\nu'-\nu}{2}$) fraction indexed by the random subset $\Uone$ (resp. $\Utwo$) are $0$, and the remaining $\frac{1 - \mu'}{2}$ (resp. $\frac{1 - \nu'}{2}$) fraction indexed by $\Sone\backslash(\Tone\cup \Uone)$ (resp. $\Stwo\backslash(\Ttwo\cup \Utwo)$) are $-1$. Denote by $y_{[0]}$ the string which agrees with $y$ on the first $d_1$ bits and is zero elsewhere. Denote by $\Bone_{\Tone}, \Bone_{\Uone}, \Bonecone$ the indicator vectors for $\Tone$, $\Uone$, and $\Sone\backslash(\Tone\cup \Uone)$, and define $\Bone_{\Ttwo}, \Bone_{\Utwo}, \Bonectwo$ analogously. Observe that the randomness of $y_{[0]}$ is independent of the randomness of these indicator vectors. Then
            \begin{equation}
                y = y_{[0]} + \Bone_{\Tone} - \Bonecone + \Bone_{\Ttwo} - \Bonectwo\,, \label{eq:ydecomp}
            \end{equation}
            and thus
            \begin{align}
                x'x'^\top - xx^\top &= 2(y\cdot(\Bone_{\Uone} + \Bone_{\Utwo})^\top + (\Bone_{\Uone} + \Bone_{\Utwo}) \cdot y^\top) \\
                &= 2(y_{[0]} + \Bone_{\Tone} - \Bonecone + \Bone_{\Ttwo} - \Bonectwo)(\Bone_{\Uone} + \Bone_{\Utwo})^\top\\
                &\quad\quad\quad\quad + 2(\Bone_{\Uone} + \Bone_{\Utwo})(y_{[0]} + \Bone_{\Tone} - \Bonecone + \Bone_{\Ttwo} - \Bonectwo)^\top\,.\label{eq:outerprod_diff}
            \end{align}
            We first show that $y^\top_{[0]} \Sig (\Bone_\Uone + \Bone_\Utwo)$ is small with high probability. We have
            \begin{equation}
                \E{(y^\top_{[0]}\Sig(\Bone_\Uone + \Bone_\Utwo))^2} = \Tr(P_{[1]} \Sig \E{(\Bone_\Uone + \Bone_\Utwo)^{\otimes 2}} \Sig^\top) \le \Tr(\Sig^\top \Sig \cdot \E{(\Bone_\Uone + \Bone_\Utwo)^{\otimes 2}})\,. \label{eq:ip_with_BU}
            \end{equation}
            Note that
            \begin{align}
                \E{(\Bone_\Uone + \Bone_\Utwo)^{\otimes 2}} &\preceq \frac{\max(\mu' - \mu,\nu'-\nu)}{2}\cdot (P_{[2]} + P_{[3]}) + \frac{(\mu'-\mu)(\nu'-\nu)}{4}\cdot(\partialone\partialtwo^\top + \partialtwo\partialone^\top) \\
                &\quad\quad\quad+ \Bigl(\frac{\mu' - \mu}{2}\Bigr)^2\partialone\partialone^\top + \Bigl(\frac{\nu' - \nu}{2}\Bigr)^2\partialtwo\partialtwo^\top \,. \label{eq:BUBU}
            \end{align}
            So from Eq.~\eqref{eq:ip_with_BU}, if the extra condition in the hypothesis of the lemma holds, then
            \begin{align}
                \E{(y^\top_{[0]}\Sig(\Bone_\Uone + \Bone_\Utwo))^2} &\lesssim \max(\mu'-\mu,\nu'-\nu)\cdot \norm{\Sig}^2_F + \Bigl(d_2\Bigl(\frac{\mu' - \mu}{2}\Bigr)^2 + d_3\Bigl(\frac{\nu'-\nu}{2}\Bigr)^2\Bigr)\cdot \norm{\Sig}^2_{\sf op} \\
                &\le \Bigl(\frac{c}{\sqrt{\min(d_2,d_3)}} + \frac{c^2}{\reffsig}\Bigr)\norm{\Sig}^2_F\,. \label{eq:SigU}
            \end{align}
            Otherwise
            \begin{align}
                \E{(y^\top_{[0]}\Sig(\Bone_\Uone + \Bone_\Utwo))^2} &\lesssim \Bigl(\frac{c}{\sqrt{d_2}} + \frac{c^2}{\reffsig}\Bigr)\norm{\Sig}^2_F + \norm{\Sig_{[2]}}^2_F + \partialtwo^\top \Sig^\top \Sig \partialtwo \\
                &\le \Bigl(\frac{c}{\sqrt{d_2}} + \frac{c^2}{\reffsig}\Bigr) + \frac{2d_3\upsilon^2}{d}\norm{\Sig}^2_F \\
                &\lesssim \Bigl(\frac{c}{\sqrt{d_2}} + \frac{c^2}{\reffsig} + \frac{d_3\upsilon^2}{d}\Bigr)\norm{\Sig}^2_F\,. \label{eq:SigU2}
            \end{align}
            By Markov's, for any $t > 0$ we have
            \begin{equation}
                \Pr{|y_{[0]}\Sig(\Bone_\Uone + \Bone_\Utwo)| > t\Bigl(\frac{c}{\sqrt{\min(d_2,d_3)}} + \frac{c^2}{\reffsig}\Bigr)^{1/2}\norm{\Sig}_F} \le 1/t^2 \label{eq:markov1}
            \end{equation}
            and
            \begin{equation}
                \Pr{|y_{[0]}\Sig(\Bone_\Uone + \Bone_\Utwo)| > t\Bigl(\frac{c}{\sqrt{d_2}} + \frac{c^2}{\reffsig} + \frac{d_3\upsilon^2}{d}\Bigr)^{1/2}\norm{\Sig}_F} \le 1/t^2 \label{eq:markov1.5}
            \end{equation}
            in these two cases respectively.

            We next show that $(\Bone_\Tone - \Bonecone + \Bone_\Ttwo - \Bonectwo)^\top \Sig (\Bone_\Uone + \Bone_\Utwo)$ is small with high probability. The calculation in Eq.~\eqref{eq:SigU} (resp. Eq.~\eqref{eq:SigU2}) shows that 
            \begin{equation}
                \E{\norm{\Sig(\Bone_\Uone + \Bone_\Utwo)}^2} \lesssim \Bigl(\frac{c}{\sqrt{\min(d_2,d_3)}} + \frac{c^2}{\reffsig}\Bigr)\norm{\Sig}^2_F\,,
            \end{equation}
            if the extra condition in the lemma holds, respectively that
            \begin{equation}
                \E{\norm{\Sig(\Bone_\Uone + \Bone_\Utwo)}^2} \lesssim \Bigl(\frac{c}{\sqrt{d_2}} + \frac{c^2}{\reffsig} + \frac{d_3\upsilon^2}{d}\Bigr)\norm{\Sig}^2_F
            \end{equation}
            if not.
            So by Markov's, for $s > 0$ we have 
            \begin{equation}
                \Pr{\norm{\Sig(\Bone_\Uone + \Bone_\Utwo)} > s\Bigl(\frac{c}{\sqrt{\min(d_2,d_3)}} + \frac{c^2}{\reffsig}\Bigr)^{1/2}\norm{\Sig}_F} \le 1/s^2 \label{eq:wsmall}    
            \end{equation}
            if the extra condition holds, respectively
            \begin{equation}
                \Pr{\norm{\Sig(\Bone_\Uone + \Bone_\Utwo)} > s\Bigl(\frac{c}{\sqrt{d_2}} + \frac{c^2}{\reffsig} + \frac{d_3\upsilon^2}{d}\Bigr)^{1/2}\norm{\Sig}_F} \le 1/s^2 \label{eq:wsmall2}    
            \end{equation}
            if not.
            Henceforth condition on the respective event. Let $P_{S\backslash U}$ denote the projector to the coordinates indexed by $(\Sone\backslash \Uone)\cup (\Stwo\backslash \Utwo)$. Clearly $\norm{P_{S\backslash U}\Sig(\Bone_\Uone + \Bone_\Utwo)} \le \norm{\Sig(\Bone_\Uone + \Bone_\Utwo)}$.

            Conditioned on $\Uone\cup \Utwo$, note that the vector $(\Bone_\Tone + \Bone_{\Ttwo} - \Bonecone - \Bonectwo)$ restricted to the coordinates of $(\Sone\backslash \Uone)\cup(\Stwo\backslash \Utwo)$ is a draw from $\sldist{|\Sone\backslash \Uone|}{\frac{\mu - \mu'}{2 + \mu - \mu'}}\times \sldist{|\Stwo\backslash \Utwo|}{\frac{\nu - \nu'}{2 + \nu - \nu'}}$. So if $w \triangleq P_{S\backslash U}\Sig(\Bone_\Uone + \Bone_\Utwo)$, then
            \begin{align}
                \MoveEqLeft\E{\iprod{w,\Bone_\Tone + \Bone_\Ttwo - \Bonecone - \Bonectwo}^2 \mid \Uone\cup \Utwo} \\
                &\lesssim \E{\iprod{w|_{\Sone\backslash\Uone}, \Bone_\Tone - \Bonecone}^2\mid \Uone} + \E{\iprod{w|_{\Stwo\backslash\Utwo}, \Bone_\Ttwo - \Bonectwo}^2\mid \Utwo}\\
                &= \norm{w}^2 + \sum_{i,j\in \Sone\backslash \Uone: i\neq j} \frac{(\frac{\mu - \mu'}{2 + \mu - \mu'})^2 \cdot |\Sone\backslash \Uone| - 1}{|\Sone\backslash \Uone| - 1}\cdot w_i w_j + \sum_{i,j\in \Stwo\backslash \Utwo: i\neq j} \frac{(\frac{\nu - \nu'}{2 + \nu - \nu'})^2 \cdot |\Stwo\backslash \Utwo| - 1}{|\Stwo\backslash \Utwo| - 1}\cdot w_i w_j \\
                &\qquad\qquad\qquad + \sum_{i\in S_{[2]}\backslash U_{[2]}, j\in S_{[3]}\backslash U_{[3]}} \Bigl(\frac{\mu - \mu'}{2 + \mu - \mu'}\Bigr)\cdot \Bigl(\frac{\nu - \nu'}{2 + \nu - \nu'}\Bigr)\cdot w_iw_j\\ 
                &\lesssim \norm{w}^2 + \frac{c^2}{d_2}\cdot \iprod{\one,w|_{\Sone\backslash\Uone}}^2 + \Bigl|\frac{(\frac{\nu - \nu'}{2 + \nu - \nu'})^2 \cdot |\Stwo\backslash \Utwo| - 1}{|\Stwo\backslash \Utwo| - 1}\Bigr| \cdot \iprod{\one, w|_{\Stwo\backslash\Utwo}}^2 \,.
            \end{align}
            If the extra condition in the lemma holds, then the third term can be similarly bounded, and we conclude that
            \begin{equation}
                \E{\iprod{w,\Bone_\Tone + \Bone_\Ttwo - \Bonecone - \Bonectwo}^2 \mid \Uone\cup \Utwo} \lesssim \Bigl(1 + \frac{c^2}{\reffsig}\Bigr)\norm{w}^2\,.
            \end{equation}
            On the other hand, if the extra condition does not hold, then the third term can be bounded by
            \begin{align}
                \Bigl(\sum_{i\in \Stwo\backslash\Utwo} w_i\Bigr)^2 &\le d_3\sum_{i\in \Stwo\backslash\Utwo} \Sig_{i:}\E{(\Bone_{\Uone} + \Bone_\Utwo)^{\otimes 2}} \Sig_{i:}^\top \\
                &\le d_3\sum_{i\in \Stwo\backslash\Utwo} (\norm{\Sig_{i:}}^2 + \iprod{\Sig_{i:},\partialtwo}^2) \\
                &\le \frac{\upsilon^2 d^3_3}{d}\norm{\Sig}^2_F\,.
            \end{align}
            As we are conditioning on the event of Eq.~\eqref{eq:wsmall} (resp. Eq.~\eqref{eq:wsmall2}) under the extra condition (resp. not under the extra condition), which implies a bound on $w$, we conclude that the expectation of $((\Bone_\Tone + \Bone_\Ttwo - \Bonecone - \Bonectwo)^\top\Sig(\Bone_\Uone + \Bone_\Utwo))^2$ with respect to the randomness of $T$ conditioned on such a $U$ is upper bounded by
            \begin{equation}
                s^2\Bigl(1 + \frac{c^2}{\reffsig}\Bigr)\Bigl(\frac{c}{\sqrt{\min(d_2,d_3)} }+ \frac{c^2}{\reffsig}\Bigr)\norm{\Sig}^2_F \label{eq:markov2}
            \end{equation}
            under the extra condition, respectively by
            \begin{equation}
                s^2\Bigl(\frac{c}{\sqrt{d_2}} + \frac{c^2}{\reffsig} + \frac{\upsilon^2 d^3_3}{d}\Bigr)\norm{\Sig}^2_F\,. \label{eq:markov2.5}
            \end{equation}
            otherwise.

            By another application of Markov's inequality and a union bound, we can establish a tail bound for $|(\Bone_\Tone + \Bone_\Ttwo - \Bonecone - \Bonectwo)^\top\Sig(\Bone_\Uone + \Bone_\Utwo)|$, and combining Eq.~\eqref{eq:markov2} with Eq.~\eqref{eq:markov1}, we conclude that if the extra condition holds, then
            \begin{align}
                \MoveEqLeft \mathbb{P}\Bigl[|x^\top \Sig x - x'^\top \Sig x'| > 2\norm{\Sig}_F\, \Bigl\{t\Bigr(\frac{c}{\sqrt{\min(d_2,d_3)}} + \frac{c^2}{\reffsig}\Bigr)^{1/2} + ss'\Bigl(1 + \frac{c^2}{\reffsig}\Bigr)^{1/2}\Bigl(\frac{c}{\sqrt{\min(d_2,d_3)}}+\frac{c^2}{\reffsig}\Bigr)^{1/2}\Bigr\}\Bigr] \\
                &\le 1/t^2 + 1/s^2 + 1/s'^2\,,
            \end{align}
            and the first part of the lemma follows by taking $t = 1/\reffsig^{1/4}$ and $s = s' = \reffsig^{1/8}$.

            If the extra condition does not hold, then similarly combining Eq.~\eqref{eq:markov2.5} with Eq.~\eqref{eq:markov1.5} yields
            \begin{align}
                \MoveEqLeft\mathbb{P}\Bigl[|x^\top \Sig x - x'^\top \Sig x'| > 2\norm{\Sig}_F\, \Bigl\{t\Bigr(\frac{c}{\sqrt{d_2}} + \frac{c^2}{\reffsig} + \frac{d_3\upsilon^2}{d}\Bigr)^{1/2}  \\
                &+ ss'\Bigl(\frac{c}{\sqrt{d_2}} +\frac{c^2}{\reffsig} + \frac{d_3^3\upsilon^2}{d}\Bigr)^{1/2}\Bigr\}\Bigr] \le 1/t^2 + 1/s^2 + 1/s'^2\,,   
            \end{align}
            and the lemma follows again by taking $t = 1/\reffsig^{1/4}$ and $s = s' = \reffsig^{1/8}$.
        \end{proof}

        \noindent We will also need a version of Lemma~\ref{lem:diag_diff_small} that holds for general $\mu,\mu'$. This version will only be meaningful for small $d_2, d_3$:

        \begin{lemma}\label{lem:diag_diff_small_alt}
            Let $\mu,\mu'\in[-1,1]$ be any multiples of $1/d_2$, and let $\nu,\nu'\in[-1,1]$ be any multiples of $1/d_3$. If $x\sim\pi^s_{\mu,\nu;d_1,d_2,d_3}$, and $x'\sim\pi^s_{\mu',\nu';d_1,d_2,d_3}$, then
            \begin{equation}
                \mathbb{P}_{\substack{(x,x')\sim\calD_{\mu,\nu,\mu',\nu'}}}\Bigl[|x^\top\Sig x - x'^\top \Sig x'| > \Omega\Bigr(\frac{\Siglightrows(d_2+d_3)^{3/2}}{\sqrt{d}}\Bigr)\,\norm{\Sig}_F\Bigr] \lesssim 1/\reffsig^{1/2}
            \end{equation}
        \end{lemma}

        \begin{proof}
            The general outline of the proof is the same as that of Lemma~\ref{lem:diag_diff_small}. As before, we will assume that $\mu \le \mu'$ and $\nu \le \nu'$, though the argument can be easily extended to handle the remaining cases. Recall the definition of $\Tone, \Ttwo, \Uone, \Utwo, \Bonecone, \Bonectwo$, the decomposition $y = y_{[0]} + \Bone_\Tone - \Bonecone + \Bone_\Ttwo - \Bonectwo$, and recall \eqref{eq:outerprod_diff}.

            We first show that $y^\top_{[0]}\Sig(\Bone_\Uone + \Bone_\Utwo)$ is small with high probability. Recall from \eqref{eq:ip_with_BU} that $\E{(y^\top_{[0]}\Sig(\Bone_\Uone + \Bone_\Utwo))^2} \le \Tr(\Sig^\top\Sig \cdot \E{(\Bone_\Uone + \Bone_\Utwo)^{\otimes 2}})$, and recall Eq.~\eqref{eq:BUBU}. We have 
            \begin{equation}
                \E{(y^\top_{[0]}\Sig(\Bone_\Uone + \Bone_\Utwo)^2} \le \frac{\upsilon^2(d_2 + d_3)^2}{d}\norm{\Sig}^2_F\,. \label{eq:SigBU_alt}
            \end{equation}
            By Markov's, for any $t > 0$ we have
            \begin{equation}
                \mathbb{P}\Bigl[|y^\top_{[0]}\Sig(\Bone_\Uone + \Bone_\Utwo)| > \frac{t\upsilon(d_2+d_3)}{\sqrt{d}}\Bigr)\,\norm{\Sig}_F\Bigr] \lesssim 1/t^2\,. \label{eq:ySc_alt}
            \end{equation}
            We next show that $(\Bone_\Tone - \Bonecone + \Bone_\Ttwo - \Bonectwo)^\top \Sig (\Bone_\Uone + \Bone_\Utwo)$ is small with high probability. The calculation in \eqref{eq:SigBU_alt} shows that $\E{\norm{\Sig \Bone_U}^2} \le \frac{\upsilon^2 (d_2+d_3)^2}{d}\,\norm{\Sig}^2_F$, so by Markov's, for $s > 0$ we have
            \begin{equation}
                \Pr{\norm{\Sig \Bone_U} > \frac{s\upsilon(d_2 + d_3)}{\sqrt{d}}\,\norm{\Sig}_F} \lesssim 1/s^2\,.
            \end{equation}
            Because $\norm{\Bone_\Tone - \Bonecone + \Bone_\Ttwo - \Bonectwo}\le\sqrt{d_2 + d_3}$, this implies by Cauchy-Schwarz that
            \begin{equation}
                \mathbb{P}\Bigl[(\Bone_T - \Bone')^\top \Sig \Bone_U > \frac{s\upsilon(d_2 + d_3)^{3/2}}{\sqrt{d}}\Bigr)\,\norm{\Sig}_F\Bigr] \lesssim 1/s^2\,. \label{eq:BTprime_alt}
            \end{equation}
            Combining \eqref{eq:ySc_alt} and \eqref{eq:BTprime_alt}, we conclude by a union bound that
            \begin{equation}
                \mathbb{P}\Bigl[|x^\top\Sig x - x'^\top \Sig x'| > 2\norm{\Sig}_F\cdot (t + s\sqrt{d_2+d_3})\cdot\frac{\upsilon(d_2+d_3)}{\sqrt{d}}\Bigr] \lesssim 1/s^2 + 1/t^2\,.
            \end{equation}
            The lemma follows by taking $s = t = \reffsig^{1/4}$.
        \end{proof}

    \newcommand{\vecdiff}{\Updelta(\bX,\bX')}
    \subsection{Concluding the argument}
    \label{sec:conclude1}
        We are now ready to put the ingredients from the preceding subsections together. For convenience, throughout this section we use the notation
        \begin{equation}
            \vecdiff \triangleq \softmax(\bX\Sig\bX^\top) - \softmax(\bX'\Sig\bX'^\top)
        \end{equation}
        for any $v\in\S^{d-1}$ and $\bX,\bX'\in\brc{\pm 1}^{k\times d}$.
        

        \begin{lemma}\label{lem:similar_conds}
            Let $1 \le c \le \sqrt{\min(\reffsig,d_2,d_3)}/2$, and define
            \begin{multline}
                \delta^*_c \triangleq \wt{\Theta}\biggl(k^{8/3} \cdot \Bigl(\frac{\sqrt{c}\cdot \reffsig^{1/4}}{d_2^{1/4}} + \frac{c^2}{\reffsig^{1/4}} + \Bigl[\frac{\reffsig^{1/8}\upsilon d^{3/2}_3}{\sqrt{d}} + \frac{\reffsig^{1/8}\sqrt{c\upsilon d_3}}{d^{1/4}}\Bigr]\cdot \mathds{1}\Bigr)^{1/3}  \\
                + \mathds{1} \cdot k^{7/3}\cdot \min\Bigl(\frac{d_2^{1/12}\reffsig^{1/24}c^{1/3}}{d^{1/4}}, \frac{\reffsig^{5/24}}{c^{1/6} d^{1/4}}, \frac{\reffsig^{1/12}\sqrt{c}}{\upsilon^{1/3} \sqrt{d_3} d^{1/12}}, \frac{\reffsig^{1/12}c^{1/3}}{\upsilon^{1/6} d_3^{1/6}d^{1/6}}\Bigr) + k^3\cdot\Bigl(\frac{c^2\upsilon\sqrt{\reffsig}}{\sqrt{d}}\Bigr)^C\biggr)\,, \label{eq:betastarc}         
            \end{multline}
            where $\mathds{1} \triangleq \bone{|\nu|, |\nu'| \le c/\sqrt{d_3}}$.
            Let $\mu_1,\ldots,\mu_k, \mu'_1,\ldots,\mu'_k\in [-c/\sqrt{d_2},c/\sqrt{d_2}]$. If the rows of $\bX, \bX'\in\brc{\pm 1}^{k\times d}$ are sampled according to $\bX_{i:} \sim \pi^s_{\mu_i,\nu_i;d_1,d_2,d_3}$ and $\bX'_{i:} \sim \pi^s_{\mu'_i,\nu'_i;d_1,d_2,d_e}$ for $i\in[k]$, then if either $|\nu|, |\nu'| \le c/\sqrt{d_3}$ or $d_3 \le \min(d/8\upsilon^2, \sqrt{d\reffsig}/8c\upsilon)$, then we have
            \begin{equation}
                \| \E{\vecdiff} \|_{\max} \le \delta^*_c\,.
            \end{equation}
        \end{lemma}

        \begin{proof}
            Recall the definition of the coupling $\calD_{\mu_i, \nu_i, \mu'_i, \nu'_i}$ from Definition~\ref{def:coupling}. Let $s > 0$ be the free parameter in Lemma~\ref{lem:xSigy_anticonc}, to be tuned later.

            We first consider the case that $|\nu|, |\nu'| \le c/\sqrt{d_3}$. For $a,b,c\in[k]$ such that $b\neq c$, let $\calA_{a,b,c}$ denote the event that 
            \begin{equation}
                |\bX^\top_{a:}\Sig \bX_{b:} - \bX^\top_{a:}\Sig \bX_{c:}| \ge s\norm{\Sig}_F\,.
            \end{equation} 
            For a sufficiently large constant $C' > 0$, let $\calB_{a,b}$ denote the event that 
            \begin{equation}
                |\bX^\top_{a:}\Sig \bX_{b:} - \bX'^\top_{a:}\Sig \bX'_{b:}| \le C'\cdot\Bigl(\frac{\sqrt{c}\cdot \reffsig^{1/4}}{\min(d_2,d_3)^{1/4}} + \frac{c^2}{\reffsig^{1/4}}\Bigr) \,\norm{\Sig}_F\,. \label{eq:Bab}
            \end{equation} By Lemmas~\ref{lem:xSigy_anticonc}, \ref{lem:offdiag_diff_small}, and \ref{lem:diag_diff_small}, the event $\calE \triangleq (\bigcap_{a,b,c} \calA_{a,b,c}) \cap (\bigcap_{a,b} \calB_{a,b})$ happens with probability at least $1 - \delta$ for 
            \begin{equation}
                \delta \lesssim k^3 s^{1/2} + k^3\Bigl(\frac{c^2\upsilon\sqrt{\reffsig}}{\sqrt{d}}\Bigr)^C + k^3\exp(-\Omega(\reffsig)) + \frac{k^2}{\reffsig^{1/4}}\,.
            \end{equation}

            Then because $\norm{\softmax(\cdot)}_{\max} \le 1$ deterministically,
            \begin{align}
                \norm{\E{\vecdiff}}_{\max} &\le \delta + \norm{\E{\vecdiff \mid \calE}}_{\max}\,. \label{eq:conditionout}
            \end{align}
            It remains to bound the final conditional expectation. We will bound it \emph{pointwise} over any $\bX,\bX'$ satisfying $\calE$. It suffices to show that each row of $\E{\vecdiff \mid \calE}$ has small $L_\infty$ norm. Without loss of generality, consider the first row. We invoke Lemma~\ref{lem:softmax} with $a_1,\ldots,a_k$ given by the entries of $\softmax(\bX\Sig\bX^\top)_{1:}$ and $a'_1,\ldots,a'_k$ given by the entries of $\softmax(\bX'\Sig\bX'^\top)_{1:}$ for \emph{any} $\bX,\bX'$ satisfying the event $\calE$, and we take $R, \eta, C$ in that lemma to be given by the quantities $\norm{\Sig}_F$, $C'\cdot\Bigl(\frac{\sqrt{c}\cdot \reffsig^{1/4}}{\min(d_2,d_3)^{1/4}} + \frac{c^2}{\reffsig^{1/4}}\Bigr)$, and $s$ respectively.

            Then under event $\calE$, for any $i\in[k]$,
            \begin{equation}
                \bigl\|\E{\vecdiff_{i:} \mid \calE}\bigr\|_\infty \le \wt{O}\Bigl(k^2 / s \cdot \Bigl(\frac{\sqrt{c}\cdot \reffsig^{1/4}}{\min(d_2,d_3)^{1/4}} + \frac{c^2}{\reffsig^{1/4}}\Bigr)\Bigr)\,. \label{eq:rowbound} 
            \end{equation}
            If we take $s = \frac{1}{k^{1/3}}(\frac{\sqrt{c}\cdot \reffsig^{1/4}}{\min(d_2,d_3)^{1/4}} + \frac{c^2}{\reffsig^{1/4}})^{2/3}$, then combining Eqs.~\eqref{eq:conditionout} and \eqref{eq:rowbound}, we obtain the desired bound on $\norm{\E{\Delta(\bX,\bX')}}_{\max}$, noting that $k^3\exp(-\Omega(\reffsig))$ and $k^2/\reffsig^{1/4}$ are of lower order compared to some of the terms that appear in the definition of $\delta^*_c$.

            Next, we consider the case that $d_3 \le \min(d/8\upsilon^2, \sqrt{d\reffsig}/8c\upsilon)$. The argument is essentially the same but with different estimates applied from the previous section. In place of Eq.~\eqref{eq:Bab}, we have
            \begin{equation}
                |\bX^\top_{a:}\Sig \bX_{b:} - \bX'^\top_{a:}\Sig \bX'_{b:}| \le C'\cdot\Bigl(\frac{\sqrt{c}\cdot \reffsig^{1/4}}{d_2^{1/4}} + \frac{c^2}{\reffsig^{1/4}} + \frac{\reffsig^{1/8}\upsilon\sqrt{d_3}\cdot (d_3 + \sqrt{\log \reffsig})}{\sqrt{d}} + \frac{\reffsig^{1/8}(\sqrt{c\upsilon d_3} + \sqrt{cs}\sqrt[4]{\log \reffsig})}{d^{1/4}} \Bigr) \,\norm{\Sig}_F.  \label{eq:Bab2}
            \end{equation}
            If we redefine $\calE$ with respect to this event and apply the corresponding parts of Lemmas~\ref{lem:xSigy_anticonc}, \ref{lem:offdiag_diff_small}, and \ref{lem:diag_diff_small}, $\calE$ happens with probability at least $1 - \delta$ for $\delta$ defined as above. Eq.~\eqref{eq:conditionout} still holds. For the conditional expectation, we can invoke Lemma~\ref{lem:softmax} to get that for any $i\in[k]$,
            \begin{equation}
                \bigl\|\E{\vecdiff_{i:} \mid \calE}\bigr\|_\infty \le \wt{O}\Bigl(k^2/s \cdot \Bigl(\frac{\sqrt{c}\cdot \reffsig^{1/4}}{d_2^{1/4}} + \frac{c^2}{\reffsig^{1/4}} + \frac{\reffsig^{1/8}\upsilon\sqrt{d_3}\cdot (d_3 + \sqrt{\log \reffsig})}{\sqrt{d}} + \frac{\reffsig^{1/8}(\sqrt{c\upsilon d_3} + \sqrt{cs}\sqrt[4]{\log \reffsig})}{d^{1/4}} \Bigr)\Bigr)\,.
            \end{equation}
            So if we take $s = \frac{1}{k^{1/3}}\Bigl(\frac{\sqrt{c}\cdot \reffsig^{1/4}}{d_2^{1/4}} + \frac{c^2}{\reffsig^{1/4}} + \frac{\reffsig^{1/8}\upsilon\sqrt{d_3}\cdot (d_3 + \sqrt{\log \reffsig})}{\sqrt{d}} + \frac{\reffsig^{1/8}\sqrt{c\upsilon d_3}}{d^{1/4}}\Bigr)^{2/3}$, we obtain the desired bound.
        \end{proof}

        \noindent We now show that the matrix $\E{\bX^\top \softmax(\bX\Sig\bX^\top)\bX}$ is close to a multiple of the identity matrix as measured by test vectors of which at least one is sufficiently dense.

        \begin{lemma}\label{lem:largehamming}
            Let $v,w \in\{0,1\}^d$ be such that either $v = w$ or $v$ and $w$ have disjoint supports. In the former case, define $d_3 = 0$ and in the latter case, define $d_3 = \min(\norm{v}_1, \norm{w}_1)$. Define $d_2 = \max(\norm{v}_1,\norm{w}_1)$. Then
            \begin{equation}
                \frac{1}{\sqrt{d_2d_3}}\Bigl|v^\top \Bigl(\E[\bX\sim\brc{\pm 1}^{k\times d}]{\bX^\top \J\softmax(\bX\Sig\bX^\top) \bX} - k\cdot \Id\Bigr) \, w\Bigr| \le \delta^*
            \end{equation}
            for
            \begin{equation}
                \delta^* \triangleq \wt{\Theta}\biggl(k^{14/3} \cdot \Bigl(\frac{\reffsig^{1/4}}{d_2^{1/4}} + \frac{1}{\reffsig^{1/4}} + \frac{\reffsig^{1/8}\sqrt{\upsilon}}{d^{1/4}}\Bigr)^{1/3} + k^{5}\cdot\Bigl(\frac{\upsilon\sqrt{\reffsig}}{\sqrt{d}}\Bigr)^C\biggr)\,.
            \end{equation}
        \end{lemma}

        \begin{proof}
            Define the quantity
            \begin{multline}
                \delta' \triangleq \wt{\Theta}\biggl(k^{14/3} \cdot \Bigl(\frac{\reffsig^{1/4}}{d_2^{1/4}} + \frac{1}{\reffsig^{1/4}} + \Bigl[\frac{\reffsig^{1/8}\upsilon d^{3/2}_3}{\sqrt{d}} + \frac{\reffsig^{1/8}\sqrt{\upsilon d_3}}{d^{1/4}}\Bigr]\cdot \mathds{1}\Bigr)^{1/3}  \\
                + \mathds{1}\cdot  k^{13/3}\cdot \min\Bigl(\frac{d_2^{1/12}\reffsig^{1/24}}{d^{1/4}}, \frac{\reffsig^{5/24}}{d^{1/4}}, \frac{\reffsig^{1/12}}{\upsilon^{1/3} \sqrt{d_3} d^{1/12}}, \frac{\reffsig^{1/12}}{\upsilon^{1/6} d_3^{1/6}d^{1/6}}\Bigr) + k^{5}\cdot\Bigl(\frac{\upsilon\sqrt{\reffsig}}{\sqrt{d}}\Bigr)^C\biggr)\,,
            \end{multline}
            where $\mathds{1} \triangleq \bone{v \neq w \ \mathrm{and} \ d_3 \lesssim \log(d)}$. Observe that the quantity $k^{13/3} \cdot \frac{\reffsig^{1/12}}{\upsilon^{1/6}d^{1/6}_3 d^{1/6}}$ in the minimum is dominated by $(\frac{\reffsig^{1/8}\sqrt{\upsilon d_3}}{d^{1/4}})^{1/3}$. Furthermore, note that $([\frac{\reffsig^{1/8}\upsilon d^{3/2}_3}{\sqrt{d}} + \frac{\reffsig^{1/8}\sqrt{\upsilon d_3}}{d^{1/4}}]\cdot \mathds{1})^{1/3} = \wt{O}(\frac{\reffsig^{1/8}\sqrt{\upsilon}}{d^{1/4}})$ as $\upsilon \ll \sqrt{d}$ by Assumption~\ref{assume:lightrows}, and thus $\delta' = \wt{O}(\delta^*)$.
            
            We first consider the case of $v$ disjoint from $w$. Suppose $d_2 = \norm{v}_1$ and $d_3 = \norm{w}_1$; as we will see, the case of $d_2 = \norm{w}_1$ and $d_3 = \norm{v}_1$ will be entirely analogous. By permutational symmetry of our Assumptions, it is enough to show the above bound for $v = \partialone$ and $w = \partialtwo$. It will be convenient to define
            \begin{equation}
                \bM^* \triangleq \E[\bX\sim\brc{\pm 1}^{k\times d}]{\softmax(\bX\Sig \bX^\top)}\,.
            \end{equation}

            Given $\bX\sim\brc{\pm 1}^{k\times d}$, let $\vec{\mu} = (\mu_1,\ldots,\mu_k)$ and $\vec{\nu} = (\nu_1,\ldots,\nu_k)$ denote the vectors $\frac{1}{d_2}\cdot \bX \partialone$ and $\frac{1}{d_3}\cdot\bX\partialtwo$. By standard binomial tail bounds, $|\mu_i|, |\nu_i| \lesssim \sqrt{\log(k/\delta)/d_2}$ for all $i\in[k]$ with probability at least $1 - \delta$, where $\delta = 1/\poly(d)$. We will take this $\sqrt{\log(k/\delta)}$ to be the definition of $c$ in the preceding lemmas. Let $\calE$ denote this event.

            As long as $d_2,d_3,\reffsig \gtrsim \log(d)$, the hypothesis that $1 \le c \le \sqrt{\min(\reffsig,d_2,d_3)}/2$ in the preceding lemmas holds. Indeed, these bounds on $d_2$ and $\reffsig$ hold by assumption. If $d_3 \gtrsim \log(d)$, then this bound also holds for $d_3$. Alternatively, if $d_3 \lesssim \log(d)$, then we can invoke the second case of the preceding lemmas as those apply whenever $d_3 \le \min(d/8\upsilon^2, \sqrt{d\reffsig}/8c\upsilon)$. In either case, we can apply the bound in Lemma~\ref{lem:similar_conds}.

            Given $\vec{\mu}$, define 
            \begin{equation}
                \bM_{\vec{\mu},\vec{\nu}} \triangleq \E[\bX\sim\otimes_i \pi^s_{\mu_i,\nu_i; d_1,d_2,d_3}]{\softmax(\bX\Sig\bX^\top)}\,. \label{eq:Mmu}
            \end{equation}
            Lemma~\ref{lem:similar_conds} tells us that if $\vec{\mu},\vec{\mu}'$ are such that $|\mu_i|, |\mu'_i| \le c/\sqrt{d_2}$ for all $i$, then $\norm{\bM_{\vec{\mu}, \vec{\nu}} - \bM_{\vec{\mu}', \vec{\nu}'}}_{\max} \le \delta^*_c$, as $|\nu_i|, |\nu'_i| \le c/\sqrt{d_2}$. Then by Jensen's, this implies that 
            \begin{equation}
                \norm{\bM_{\vec{\mu},\vec{\nu}} - \bM}_{\max} \le \delta^*_c\,, \label{eq:muclose}
            \end{equation}
            where $\bM\triangleq \E{\softmax(\bX \Sig \bX^\top) \mid \calE}$. Furthermore, note that
            \begin{equation}
                \norm{\bM^* - \bM}_{\op} = \Pr{\calE^c} \cdot \norm{\E{\softmax(\bX\Sig\bX^\top)\mid \calE^c} - \bM}_{\op} \le 2\delta\sqrt{k}\,, \label{eq:muclose2}
            \end{equation}
            where in the last step we used that any $k\times k$ matrix whose rows are all elements of $\Delta^{k-1}$ has operator norm at most $\sqrt{k}$.
            
            We can write
            \begin{equation}
                \frac{1}{d_2 d_3}\partialone^\top\E[\bX\sim\brc{\pm 1}^{k\times d}]{\bX^\top \J\softmax(\bX\Sig\bX^\top) \bX} \partialtwo = \E[\vec{\mu}, \vec{\nu}]{\vec{\mu}^\top \, \J\bM_{\vec{\mu},\vec{\nu}}\, \vec{\nu}} \le \E{\vec{\mu}^\top \, \J\bM_{\vec{\mu},\vec{\nu}} \, \vec{\nu}\mid \calE} + \delta \cdot \E{\vec{\mu}^\top \, \J\bM_{\vec{\mu},\vec{\nu}} \, \vec{\nu}\mid \calE^c}\,. \label{eq:oneip_decomp}
            \end{equation}
            The latter term can be bounded by using the fact that $\norm{\J\bM_{\vec{\mu},\vec{\nu}}}_{\op} \le k^{3/2}$ and $\norm{\vec{\mu}} \le \sqrt{k\log(k/\delta)/d_2}$ and $\norm{\vec{\nu}} \le \sqrt{k\log(k/\delta)/d_3}$, so that
            \begin{equation}
                \delta\cdot\E{\vec{\mu}^\top \, \J\bM_{\vec{\mu},\vec{\nu}} \, \vec{\nu}\mid \calE^c} \le \delta k^{5/2}\log(k/\delta)/\sqrt{d_2d_3}\,.
            \end{equation} The former term in \eqref{eq:oneip_decomp} can be bounded using~\eqref{eq:muclose} and \eqref{eq:muclose2}. Concretely, we have
            \begin{align}
                \E{\vec{\mu}^\top \J\bM_{\vec{\mu},\vec{\nu}} \vec{\nu}\mid \calE} &= \E{\vec{\mu}^\top \J\bM^* \vec{\nu}\mid \calE} + \E{\vec{\mu}^\top \J(\bM_{\vec{\mu},\vec{\nu}} - \bM^*) \vec{\nu}\mid\calE} \\
                &\le \frac{1}{d_2d_3}\partialone^\top \E[\bX\sim\brc{\pm 1}^{k\times d}]{\bX^\top \J\bM^* \bX}\partialtwo + (\delta^*_{\sqrt{\log(k/\delta)}} k^2 / \sqrt{d_2d_3} + 2\delta k^{5/2} \log(k/\delta)/\sqrt{d_2d_3}) \\
                &= \frac{1}{d_2d_3}\partialone^\top (\Tr(\J\bM^*)\cdot \Id) \partialtwo + (\delta^*_{\sqrt{\log(k/\delta)}} k^2 / \sqrt{d_2d_3} + 2\delta k^{5/2}\log(k/\delta)/\sqrt{d_2d_3})\,. \label{eq:muJMnu}
            \end{align}
            Note that $\Tr(\J\bM^*) = k$ as the sum of the entries in any row of $\bM^*$ is $1$.
            We conclude that for any $\delta > 0$,
            \begin{equation}
                \frac{1}{\sqrt{d_2d_3}}\Bigl|\partialone^\top \Bigl(\E[\bX\sim\brc{\pm 1}^{k\times d}]{\bX^\top \softmax(\bX\Sig\bX^\top) \bX} - k\cdot \Id\Bigr) \, \partialtwo\Bigr| \le \delta^*_{\sqrt{\log(k/\delta)}}k^2 + 3\delta k^{5/2}\log(k/\delta)\,.
            \end{equation}
            We will take $\delta = 1/\poly(d)$ sufficiently small that the latter term on the right-hand side is negligible compared to the former, yielding a final bound of $\delta'$. Recall at the outset that we bounded $\delta'$ by $\delta^*$ up to log factors, thus completing the proof when $v\neq w$.

            It remains to consider the case of $v = w$. As before, by permutational symmetry, we can assume that $v = w = \partialone$. The argument will be identical, in fact strictly simpler as we do not need to condition on any assignment to the coordinates indexed by $\Stwo$ as $\Stwo = \emptyset$. Given $\bX\sim\brc{\pm 1}^{k\times d}$, let $\vec{\mu} = (\mu_1,\ldots,\mu_k)$ denote the vector $\frac{1}{d_2}\cdot \bX\partialone$. Let $\calE$ denote the event that $|\mu_i| \le \sqrt{\log(k/\delta)/d_2}$ for all $i\in[k]$, so that $\Pr{\calE}\ge 1 - \delta$; take $\delta = 1/\poly(d)$ and take $c$ in the preceding lemmas to be $\sqrt{\log(k/\delta)}$. Because $d_2 \gtrsim\log(d)$ by hypothesis, the hypothesis on $c$ in the preceding lemmas holds, and we may apply the bound in Lemma~\ref{lem:similar_conds}.

            Given $\vec{\mu}$, define
            \begin{equation}
                \bM_{\vec{\mu}} \triangleq \mathbb{E}_{\bX\sim\otimes_i \pi^s_{\mu_i,0;d_1,d_2;0}}[\softmax(\bX\Sig\bX^\top)]\,.
            \end{equation}
            Lemma~\ref{lem:similar_conds} tells us that for any $\vec{\mu}, \vec{\mu}'$ such that $|\mu_i|, |\mu'_i| \le c/\sqrt{d_2}$, then $\norm{\bM_{\vec{\mu}} - \bM_{\vec{\mu}'}}_{\mathrm{max}} \le \delta^*_c$, as $d_3 = 0$. Then by Jensen's, this implies that $\norm{\bM_{\vec{\mu}} - \bM}_{\max} \le \delta^*_c$, where $\bM \triangleq \E{\softmax(\bX\Sig\bX^\top)\mid \calE}$. Eq.~\eqref{eq:muclose2} still holds. Analogous to Eq.~\eqref{eq:oneip_decomp}, we can write
            \begin{equation}
                \frac{1}{d_2^2}\partialone^\top \mathbb{E}_{\bX\sim\brc{\pm 1}^{k\times d}}[\bX^\top \J \softmax(\bX\Sig\bX^\top)\bX\partialone] = \mathbb[\vec{\mu}]{\vec{\mu}^\top \J \bM_{\vec{\mu}}\vec{\mu}} \le \E{\vec{\mu}^\top \J\bM_{\vec{\mu}}\vec{\mu}\mid \calE} + \delta\cdot \E{\vec{\mu}^\top \J\bM_{\vec{\mu}}\vec{\mu} \mid \calE^c}\,. \label{eq:oneip_decomp_2}
            \end{equation}
            The latter term can be bounded by using the fact that $\norm{\J\bM_{\vec{\mu}}}_{\sf op} \le k^{3/2}$ and $\norm{\vec{\mu}} \le \sqrt{k\log(k/\delta)/d_2}$, so that
            \begin{equation}
                \delta\cdot\E{\vec{\mu}^\top\J\bM_{\vec{\mu}}\vec{\mu}\mid \calE^c} \le \delta k^{5/2}\log(k/\delta)/d_2\,.
            \end{equation}
            The former term in Eq.~\eqref{eq:oneip_decomp_2} can be bounded using our bounds $\norm{\bM_{\vec{\mu}} - \bM}_{\max} \le \delta^*_c$ and Eq.~\eqref{eq:muclose2}. Analogous to the derivation of Eq.~\eqref{eq:muJMnu}, we obtain
            \begin{equation}
                \E{\vec{\mu}^\top \J\bM_{\vec{\mu}}\vec{\mu}\mid\calE} \le \frac{1}{d^2_2}\partialone^\top(\Tr(\J\bM^*)\cdot\Id)\partialtwo + (\delta^*_{\sqrt{\log(k/\delta)}}k^2/d_2 + 2\delta k^{5/2}\log(k/\delta)/d_2)\,.
            \end{equation}
            Recalling that $\Tr(\J \bM^*) = k$, we conclude that
            \begin{equation}
                \frac{1}{d_2}\Bigl|\partialone^\top\Bigl(\E[\bX\sim\brc{\pm 1}^{k\times d}]{\bX^\top \softmax(\bX\Sig\bX^\top)\bX} - k\cdot \Id\Bigr)\partialone\Bigr| \le \delta^*_{\sqrt{\log(k/\delta)}}k^2 + 3\delta k^{5/2}\log(k/\delta)\,,
            \end{equation}
            and for $\delta = 1/\poly(d)$ sufficiently small, this gives the desired bound.
        \end{proof}

        \noindent The above bound is good enough provided $d_2$ is large enough that the hypothesis $d_2 \gtrsim \log d$ holds and also the term $\frac{k^5\reffsig^{1/4}}{d_2^{1/4}}$ in the definition of $\delta^*$ is sufficiently small. To show that $\E{\bX^\top \softmax(\bX\Sig\bX^\top)\bX}$ is sufficiently close to a multiple of the identity as measured by any Boolean test vector of smaller Hamming weight, we need a slightly modified argument. First, we prove an alternative version of Lemma~\ref{lem:similar_conds}:

        \begin{lemma}\label{lem:similar_conds_alt}
            Define
            \begin{equation}
                \delta^{**} \triangleq \Theta\Bigl(\frac{k^{8/3}\upsilon^{1/3}\sqrt{d_2 + d_3}}{d^{1/6}} + \frac{k^2}{\sqrt{\reffsig}} + \frac{k^3\upsilon^3\reffsig^3}{\sqrt{d}}\Bigr)\,. \label{eq:deltastarstar}
            \end{equation}
            Suppose $d_2, d_3 \lesssim \reffsig^2$. Then for any $\mu_1,\ldots,\mu_k,\mu'_1,\ldots,\mu'_k \in [-1,1]$, if the rows of $\bX, \bX'\sim\brc{\pm 1}^{k\times d}$ are sampled according to $\bX_{i:}\sim\pi^s_{\mu_i;d_1,d_2}$ and $\bX'_{i:}\sim \pi^s_{\mu'_i;d_1,d_2}$ for $i\in[k]$, then
            \begin{equation}
                \bigl\| \E{\softmax(\bX\Sig\bX^\top)} - \E{\softmax(\bX'\Sig\bX'^\top)}\bigr\|_{\max} \le \delta^{**}\,.
            \end{equation}
        \end{lemma}

        \begin{proof}
            For $a,b,c\in[k]$ such that $b\neq c$, let $\calA_{a,b,c}$ denote the event that
            \begin{equation}
                |\bX^\top_{a:}\Sig \bX_{b:} - \bX^\top_{a:}\Sig \bX_{c:}| \ge s\,\norm{\Sig}_F\,,
            \end{equation}
            for $s > 0$ a parameter to be tuned.
            Let $\calB_{a,b}$ denote the event that
            \begin{equation}
                |\bX^\top_{a:}\Sig\bX_{b:} - \bX'^\top_{a:}\Sig \bX'_{b:}| \lesssim \frac{\upsilon (d_2+d_3)^{3/2}}{\sqrt{d}}\,\norm{\Sig}_F\,.
            \end{equation}
            By Lemmas~\ref{lem:xSigy_anticonc_corner}, \ref{lem:offdiag_diff_small_alt} (with $s,s',s''$ taken to be $\Theta(d^{1/4})$, and \ref{lem:diag_diff_small_alt}, the event $\calE \triangleq (\bigcap_{a,b,c} \calA_{a,b,c}) \cap (\bigcap_{a,b} \calB_{a,b})$ happens with probability at least $1 - \delta$ for $\delta \triangleq O(k^2/\reffsig^{1/2} + k^3(s^{1/2} + \frac{\upsilon^3\reffsig^3}{\sqrt{d}})$.
            
            As in the proof of Lemma~\ref{lem:similar_conds}, \eqref{eq:conditionout} holds, and it suffices to bound the $L_\infty$ norm of the first row of $\E{\softmax(\bX\Sig\bX^\top) - \softmax(\bX'\Sig\bX'^\top)\mid \calE}$. We invoke Lemma~\ref{lem:softmax} with $a_1,\ldots, a_k$ given by the entries of $\softmax(\bX\Sig\bX^\top)_{1:}$, and $a'_1,\ldots,a'_k$ given by the entries of $\softmax(\bX'\Sig\bX'^\top)_{1:}$ for \emph{any} $\bX, \bX'$ satisfying the event $\calE$, and we take $R, \eta, C$ in that lemma to be given by the quantities $\norm{\Sig}_F$, $\Theta(\frac{\upsilon (d_2+d_3)^{3/2}}{\sqrt{d}})$, and $s$ respectively.
            
            Then under event $\calE$, 
            \begin{equation}
                \bigl\|\E{\softmax(\bX\Sig\bX^\top) - \softmax(\bX'\Sig\bX'^\top) \mid \calE}\bigr\|_\infty \le \wt{O}\Bigl(k^2/s\cdot \Bigl(\frac{\upsilon (d_2 + d_3)^{3/2}}{\sqrt{d}}\Bigr)\Bigr)\,,\label{eq:rowbound_alt} 
            \end{equation}
            and the operator norm is at most $k$ times this.
            The lemma follows by combining \eqref{eq:conditionout} for our choice of $\delta$ with \eqref{eq:rowbound_alt}.
        \end{proof}

        \begin{lemma}\label{lem:smallhamming}
            Suppose $d_2, d_3 \lesssim \reffsig^2$. Let $v,w \in\{0,1\}^d$ be such that either $v = w$ or $v$ and $w$ have disjoint supports. In the former case, define $d_3 = 0$ and in the latter case, define $d_3 = \min(\norm{v}_1, \norm{w}_1)$. Define $d_2 = \max(\norm{v}_1,\norm{w}_1)$. Then
            \begin{equation}
                \frac{1}{\sqrt{d_2d_3}} \Bigl|v^\top \Bigl(\E[\bX\sim\brc{\pm 1}^{k\times d}]{\bX^\top \J\softmax(\bX\Sig\bX^\top) \bX} - k\cdot \Id\Bigr) \, w\Bigr| \le \delta^{**}k^2
            \end{equation}
            for $\delta^{**}$ defined in \eqref{eq:deltastarstar}.
        \end{lemma}

        \begin{proof}
            The proof is very similar to that of Lemma~\ref{lem:largehamming}, the main difference being that we don't need to condition on the event that the entries of $\frac{1}{d_2}\cdot \bX\partialone$ and $\frac{1}{d_3}\cdot\bX\partialtwo$ are bounded in magnitude. Here we write out the details for the $v \neq w$ case; the case of $v = w$ follows analogously in the same way that it follows analogously in the proof of Lemma~\ref{lem:largehamming}.

            As before, it suffices to show the bound for $v = \partialone$ and $w = \partialtwo$. Given $\bX\sim\brc{\pm 1}^{k\times d}$, let $\vec{\mu} = (\mu_1,\ldots,\mu_k)$ denote the vector $\frac{1}{d_2}\cdot \bX\partialone$, and recall the definition of $\bM_{\vec{\mu},\vec{\nu}}$ in \eqref{eq:Mmu}. Lemma~\ref{lem:similar_conds_alt} tells us that for any $\vec{\mu}, \vec{\mu}', \vec{\nu}, \vec{\nu'}$, $\norm{\bM_{\vec{\mu}, \vec{\nu}} - \bM_{\vec{\mu}', \vec{\nu}'}}_{\max} \le \delta^{**}$. By Jensen's, this implies that
            \begin{equation}
                \norm{\bM_{\vec{\mu},\vec{\nu}} - \bM^*}_{\max} \le \delta^{**}\,.
            \end{equation}
            We can write
            \begin{align}
                \frac{1}{d_2d_3} \partialone^\top \E[\bX\sim\brc{\pm 1}^{k\times d}]{\bX^\top \softmax(\bX\Sig\bX^\top)\bX}\partialtwo &= \E[\vec{\mu},\vec{\nu}]{\vec{\mu}^\top \J\bM_{\vec{\mu},\vec{\nu}}\vec{\nu}} \\
                &= \E{\vec{\mu}^\top \J\bM^*\vec{\nu}} + \E{\vec{\mu}^\top (\J\bM_{\vec{\mu},\vec{\nu}} - \J\bM^*) \vec{\nu}} \\
                &= \frac{1}{d_2d_3}\partialone^\top(k\cdot \Id)\partialone \pm \delta^{**}k^2/\sqrt{d_2d_3}\,.
            \end{align}
            so the lemma follows.
        \end{proof}

        \subsection{Relating restricted norm to spectral norm}
            \label{sec:bucket}

        \begin{definition}
            We say that a vector $v$ is \emph{$s$-flat-decomposable} if it can be written as $v = \sum^m_{i=1} \lambda_i \Bone_{S_i}$, where $S_1,\ldots,S_m$ are disjoint subsets of $[d]$, and $m \le s$.
        \end{definition}
        
        \begin{lemma}\label{lem:flatdecompose}
            Let $v\in\R^d$ have norm at most $1$. For any $\epsilon > 0$, there exist $O(\log(d/\epsilon))$-flat-decomposable vectors $w_1,\ldots,w_n$ for $n \le O(\log(d/\epsilon))$ such that $\norm{v - (w_1+\cdots + w_n)} \le \epsilon$ and $\norm{w_j} \le 0.9^{j-1}$ for all $j \in[n]$.
        \end{lemma}

        \noindent The proof of this is based on the following inductive step:

        \begin{lemma}\label{lem:recursiveflat}
            Let $v\in\R^d$ have norm at most $1$. For any $\epsilon > 0$, there exists a $O(\log(d/\epsilon))$-flat-decomposable vector $w\in\R^d$ of norm at most $\norm{v}$ and a vector $\delta\in\R^d$ of norm at most $\epsilon$ such that $\norm{v - w - \delta} \le 0.9\norm{v}$. 
        \end{lemma}

        \begin{proof}
            Let $\delta$ be the vector consisting of all the entries of $v$ with value at most $\epsilon/\sqrt{d}$. Let $v^+$ (resp. $v^-$) denote the vector consisting of the positive (resp. negative) entries of $v - \delta$ respectively.

            For $i = 1,\ldots,O(\log(d/\epsilon))$, define disjoint subsets $I_i\subseteq[d]$ as follows. Let $I_1$ consist of the indices for the largest entry of $v^+$ and all other entries which are at least $1/2$ times this. For $i > 1$, let $I_i$ consist of the indices for the largest positive entry of $v^+|_{[d]\backslash \cup_{i'<i} I_{i'}}$ and all other positive entries which are at least $1/2$ times this.
            
            Define $w^+$ to be the vector constructed as follows. For all $j\in I_i$, let $w^+_j = \frac{1}{|I_i|} \sum_{j'\in I_i} v^+_{j'}$. Note that $\norm{w^+|_{I_i}}_2 \le \norm{v^+|_{I_i}}_2$ for all $i$, so $\norm{w^+}_2 \le \norm{v^+}^2$. Additionally, note that for all $I_i$,
            \begin{equation}
                \frac{\norm{(v^+ - w^+_i)|_{I_i}}^2_2}{\norm{v^+|_{I_i}}^2} = \frac{\sum_{j\in I_i} \Bigl(v^+_j - \frac{1}{|I_i|}\sum_{j'\in I_i}v^+_{j'}\Bigr)^2}{\sum_{j\in I_i} (v^+_j)^2} = 1 - \frac{(\sum_{j\in I_i} v^+_{j})^2}{|I_i| \sum_{j\in I_i} (v^+_j)^2} \le 1 - \frac{\min_{j\in I_i} (v^+_j)^2}{\max_{j\in I_i} (v^+_j)^2} \le \frac{3}{4}\,,
            \end{equation}
            We conclude that $\norm{v^+ - w^+_i}^2 \le \frac{3}{4}\norm{v^+}^2$. Furthermore, $w^+$ is by construction an $O(\log(d/\epsilon))$-flat-decomposable vector.

            In an entirely analogous fashion, we can obtain $O(\log(d/\epsilon))$-flat-decomposable vector $w^-$ satisfying $\norm{v^- - w^-}^2 \le \frac{3}{4}\norm{v^-}^2$ and $\norm{w^-}^2 \le \norm{v^-}^2$. The supports of $w^+$ and $w^-$ will be disjoint as they are supported over the positive and negative entries of $v - \delta$ respectively, so $\norm{w^+ + w^-}^2 \le \norm{v+}^2+\norm{v^-}^2$. As $\norm{v^+}^2 + \norm{v^-}^2 \le \norm{v}^2$ and $\sqrt{3/4} \le 0.9$, this completes the proof.
        \end{proof}

        \begin{proof}[Proof of Lemma~\ref{lem:flatdecompose}]
            We can now apply the construction in Lemma~\ref{lem:recursiveflat} recursively to produce a sequence of triples of vectors $(w_1,\delta_1), (w_2,\delta_2), \ldots$ such that for any index $i$ in this sequence, 
            \begin{equation}
                \norm{v - \sum_{j\le i} (w_j + \delta_j)} \le 0.9\norm{v - \sum_{j < i} (w_j + \delta_j)}    
            \end{equation}
            and furthermore $w_j$'s are $O(\log(d/\epsilon))$-flat-decomposable,
            \begin{equation}
                \norm{w_i} \le \norm{v - \sum_{j < i} (w_j + \delta_j)} \le 0.9^{i-1}\norm{v}\,,
            \end{equation}
            and $\norm{\delta_i} \le \epsilon$.
            
            There is some $n = O(\log(d/\epsilon))$ for which $\norm{v - \sum_{j\le n} (w_j + \delta_j)} \le \epsilon$, at which point we can define $z \triangleq v - \sum_{j\le n} (w_j + \delta_j)$ and conclude that we have a decomposition
            \begin{equation}
                v = \sum^n_{i=1} w_i + \Bigl(z + \sum^{n}_{i=1} \delta_i\Bigr)
            \end{equation}
            such that $v - \sum_i w_i$ has norm at most $O(\epsilon\sqrt{\log(d/\epsilon)})$. By replacing $\epsilon$ in the above with $\epsilon/\sqrt{\log(d/\epsilon)}$, each of the $w_j$'s remains $O(\log(d/\epsilon)$-flat-decomposable, and the number of components $n$ remains $O(\log(d/\epsilon))$ (all with larger constant factors).
        \end{proof}

        \begin{lemma}\label{lem:loselog}
            If for all equal or disjoint subsets $S,T\subseteq[d]$, the matrix $M\in\R^{d\times d}$ satisfies the bound
            \begin{equation}
                |\Bone_S^\top M \Bone_T| \le \xi\sqrt{|S|\cdot |T|}\,, \label{eq:booleantest}
            \end{equation}
            then $\norm{M}_{\sf op} \lesssim \xi\log d$
        \end{lemma}

        \begin{proof}
            Given $S, T\subseteq[d]$ which are not necessarily equal or disjoint, note that we can write
            \begin{equation}
                \Bone_S = \Bone_{S\backslash T} + \Bone_{S\cap T} \qquad \text{and} \qquad \Bone_T = \Bone_{T\backslash S} + \Bone_{S\cap T}
            \end{equation}
            and get by triangle inequality that
            \begin{equation}
                |\Bone_S^\top M \Bone_T| \le \xi\cdot (\sqrt{|S\backslash T|} + \sqrt{|S\cap T|})\cdot(\sqrt{|T\backslash S|} + \sqrt{|S\cap T|}) \lesssim \xi\sqrt{|S|\cdot |T|}\,.
            \end{equation}
            So in the rest of the proof, we assume, up to constant factor loss, that Eq.~\eqref{eq:booleantest} holds for all $S,T$.
        
            First note that for any $s$-flat-decomposable vectors $u = \sum_i \lambda_i \Bone_{S_i}$ and $v = \sum_j \mu_j \Bone_{T_j}$ of unit norm,
            \begin{equation}
                |u^\top M v| \le \sum_{i,j} \lambda_i \mu_j |\Bone_{S_i}^\top M \Bone_{T_j}| \le \xi\cdot \Bigl(\sum_i \lambda_i \sqrt{|S_i|}\Bigr)\Bigl(\sum_i \mu_j \sqrt{|T_j|}\Bigr)\,.
            \end{equation}
            Note that
            \begin{equation}
                1 = \norm{u}^2 = \sum_i \lambda^2_i |S_i| \ge \frac{1}{s}\Bigl(\sum_i \lambda_i \sqrt{|S_i|}\Bigr)^2\,,
            \end{equation}
            so $|u^\top M v| \le \xi s$.
        
            Let $\epsilon = 1/d$. Given any $u,u'\in\S^{d-1}$, apply Lemma~\ref{lem:flatdecompose} to produce $O(\log(d/\epsilon))$-flat-decomposable vectors $w_1,\ldots,w_n$ and $w'_1,\ldots,w'_{n'}$ for $n,n' = O(\log(d/\epsilon))$ such that $\norm{w_j}, \norm{w'_j} \le 0.9^{j-1}$ and $\norm{\delta}, \norm{\delta'} \le \epsilon$ for $\delta = u - \sum_j w_j$ and $\delta' = u' - \sum_j w'_j$.
            Then
            \begin{equation}
                u^\top M u' = \Bigl(\sum^n_{i=1} \sum^{n'}_{j=1} w_i^\top M w'_j\Bigr) +  \delta^\top \Bigl(\sum^{n'}_{j=1} M w'_j) + \Bigl(\sum^n_{i=1} w^\top_i M\Bigr)\delta' + \delta^\top M \delta'\,. \label{eq:decompose_matrix_using_flat}
            \end{equation}
            Note that $|w_i^\top M w'_j| \le O(\xi\log(d/\epsilon))\cdot \norm{w_i}\norm{w'_j} \le O(\xi\log(d/\epsilon)) \cdot 0.9^{i+j-2}$ by the argument at the start of the proof, so the first sum on the right-hand side is at most $O(\xi \log(d/\epsilon)) = O(\xi \log d)$.

            It remains to control the last three terms in Eq.\eqref{eq:decompose_matrix_using_flat}. Note that because every entry of $M$ is at most $\xi$ in magnitude, we naively have $\norm{M}_{\sf op} \le d\xi$. Additionally, $\norm{\sum_j w_j} \lesssim 1$ and $\norm{\sum_j w'_j} \lesssim 1$, so
            \begin{equation}
                 \Bigl|\delta^\top \Bigl(\sum^{n'}_{j=1} M w'_j) + \Bigl(\sum^n_{i=1} w^\top_i M\Bigr)\delta' + \delta^\top M \delta'\Bigr| \lesssim \epsilon \xi d \le \xi\,,
            \end{equation}
            so we obtain the claimed bound on $\norm{M}_{\sf op}$.
        \end{proof}

    \subsection{Combining the bounds}
    \label{sec:conclude2}

        \noindent We can combine the bounds in Lemma~\ref{lem:largehamming} and~\ref{lem:smallhamming} with Lemma~\ref{lem:loselog} to conclude the proof of the main result of this section, Theorem~\ref{thm:sumvalue}:

        \begin{proof}[Proof of Theorem~\ref{thm:sumvalue}]
            In this proof, let $\Delta_{\Sig} \triangleq \E{\bX^\top \J \softmax(\bX\Sig\bX)^\top \bX} - k\cdot \Id$. We first show a bound on
            \begin{equation}
                \norm{\Delta_{\Sig}}_{\sf op}
            \end{equation}
            for any $\Sig \in \brc{\Sig_1,\ldots,\Sig_m}$.

            When at least one of $d_2, d_3$ is $\Omega(\reffsig^2)$, Lemma~\ref{lem:largehamming} implies that for $v, w\in\brc{0,1}^d$ which are $d_2$- and $d_3$-sparse and have disjoint supports,
            \begin{equation}
                \frac{1}{\sqrt{d_2d_3}} |v^\top \Delta_{\Sig} w| \le \wt{\Theta}\Bigl(k^5\cdot \Bigl[\frac{1}{\reffsig^{1/12}} + \frac{\reffsig^{1/24}\upsilon^{1/6}}{d^{1/12}} + \frac{\upsilon^C \reffsig^{C/2}}{d^{C/2}}\Bigr]\Bigr)\,.
            \end{equation}
            Likewise, Lemma~\ref{lem:largehamming} implies that if $v = w$,
            \begin{equation}
                \frac{1}{d_2} |v^\top \Delta_{\Sig} v| \le \wt{\Theta}\Bigl(k^5\cdot \Bigl[\frac{1}{\reffsig^{1/12}} + \frac{\reffsig^{1/24}\upsilon^{1/6}}{d^{1/12}} + \frac{\upsilon^C \reffsig^{C/2}}{d^{C/2}}\Bigr]\Bigr)\,.
            \end{equation}
            On the other hand, when $d_2,d_3 \lesssim \reffsig^2$, for $v, w\in \brc{0,1}^d$ which are $d_2$- and $d_3$-sparse and have disjoint supports, Lemma~\ref{lem:smallhamming} implies that
            \begin{equation}
                \frac{1}{\sqrt{d_2d_3}} |v^\top \Delta_{\Sig} w| \le \wt{\Theta}\Bigl(\frac{k^4}{\sqrt{\reffsig}} + \frac{k^5\upsilon\reffsig}{d^{1/6}}\Bigr)\,.
            \end{equation}
            Likewise, Lemma~\ref{lem:smallhamming} implies that
            \begin{equation}
                 \frac{1}{d_2} |v^\top \Delta_{\Sig} v| \le \wt{\Theta}\Bigl(\frac{k^4}{\sqrt{\reffsig}} + \frac{k^5\upsilon\reffsig}{d^{1/6}}\Bigr)\,.
             \end{equation} 
            Combining with Lemma~\ref{lem:loselog}, we conclude that
            \begin{equation}
                \norm{\Delta_{\Sig}}_{\sf op} \le \wt{\Theta}\Bigl(\frac{k^5}{\reffsig^{1/12}} + \frac{k^5\upsilon\reffsig}{d^{C/2}}\Bigr)\,.
            \end{equation}
            Now observe that
            \begin{equation}
                \E{\bX^\top\J\bY} - k\sum^m_{i=1}\bW_i = \sum_i \Delta_{\Sig} \bW_i\,,
            \end{equation}
            so the lemma follows by triangle inequality.
        \end{proof}


\newcommand{\tempvar}{}
\newcommand{\cspecial}{c_*}

\section{Sculpting the affine hull}
\label{sec:sculpt}

    Henceforth, define $\affinehull$ to be the \emph{affine hull} of the attention matrices $\Sig_1,\ldots,\Sig_m$, that is,
    \begin{equation}
        \affinehull \triangleq \Bigl\{\sum_i \lambda_i \Sig_i: \lambda\in\R^d \ \text{s.t.} \ \sum_i \lambda_i = 1\Bigr\}
    \end{equation}
    Let $\epsilon > 0$ and $0 < \delta < 1/2$ be free parameters, and let $\xi, \epsilon^*, \delta^* > 0$ be parameters to be specified later. Throughout this section we assume we have access to a matrix $\wh{\bW}\in\R^{d\times d}$ satisfying 
    \begin{equation}
        \Bigl\|\wh{\bW} - \sum^m_{i=1} \bW_i\Bigr\|_F \le \frac{c\epsilon}{\sqrt{\log(k/\axwprob)}}\min_i\norm{\bW_i}_F\,, \label{eq:hatWbound}
    \end{equation}
    for sufficiently small absolute constant $c > 0$.
    Also define
    \begin{equation}
        \Delta \triangleq \wh{\bW} - \sum^m_{i=1} \bW_i\,.
    \end{equation}

    The main claim of this section is that using $\wh{\bW}$, we can produce a convex body $K$ (see Eq.~\eqref{eq:Kdef}) which is a close approximation to the convex hull of $\Sig_1,\ldots,\Sig_m$ (see Definition~\ref{def:tighthull} and whose minimum norm point is close to the matrix
    \begin{equation}
        \frac{1}{Z}\sum^m_{i=1} \frac{1}{\norm{\Sig_i}^2_F} \cdot \Sig_i  \ \ \ \text{for} \ \ \ Z \defeq \sum^m_{i=1} \frac{1}{\norm{\Sig_i}^2_F}\,. \label{eq:wtQdef}
    \end{equation}
    These two guarantees are stated formally in Theorem~\ref{thm:main_lp} at the end.

    \subsection{Success conditions}

        Here we record some tail bounds that will be useful in the sequel. In our analysis, we will condition on only encountering random examples for which these high-probability bounds hold.

        \begin{proposition}\label{prop:aXW}
            For $\bX\sim\brc{\pm 1}^{k\times d}$, with probability at least $1 - \axwprob$ we have that
            \begin{equation}
                \sup_{\alpha\in\Delta^{k-1}} \norm{\alpha^\top \bX \Delta} \le \frac{\epsilon}{4}\min_i\norm{\bW_i}_F\,.
            \end{equation}
        \end{proposition}

        \begin{proof}
            By convexity, it suffices to consider $\alpha$ given by the standard basis vectors. Note that for any $j\in[k]$, $\E{\norm{\bX_{j:} \Delta}^2} = \norm{\Delta}^2_F$. So by Theorem~\ref{thm:hanson_wright} and a union bound over $j$, for any $\delta > 0$ we have
            \begin{equation}
                \Pr{\norm{\bX_{j:} \Delta} > \Omega(\norm{\Delta}_F \sqrt{\log(k/\delta))} \ \forall \ j\in[k]} \lesssim \delta\,.
            \end{equation}
            By taking $\delta = \delta^*$ and applying the assumed bound on $\norm{\Delta}_F$ in Eq.~\eqref{eq:hatWbound}, we obtain the claimed bound.
        \end{proof}

        \begin{lemma}\label{lem:nospuriouscombo}
            With probability at least
            \begin{equation}
                1 - O(\sqrt{mkd\Wlbd}/\epsilon)^{km}\exp(-\Omega(\reffw/m))\,,
            \end{equation}
            over $\bX$, the following holds: for any $\lambda_1,\ldots,\lambda_m\in\Delta^{k-1} - \Delta^{k-1}$,
            \begin{equation}
                \Bigl\|\sum_i \lambda^\top_i \bX \bW_i\Bigr\|^2 \in [0.9,1.1]\cdot\sum^m_{i=1}\norm{\lambda_i}^2\cdot\norm{\bW_i}^2_F\,.
            \end{equation}
        \end{lemma}

        \begin{proof}
            Define $\bM_j \triangleq \sum^m_{i=1} (\lambda_i)_j \cdot \bW_i$. Then
            \begin{equation}
                \sum_i \lambda^\top_i \bX\bW_i = \sum^k_{j=1} \bX_{j:}\bM_j\,.
            \end{equation}
            Denote the squared norm of this random vector by $Z$. Then 
            \begin{equation}
                \E{Z} = \sum_{j,j'\in[k]} \E{\bX_{j:} \bM_j \bM^\top_{j'} \bX_{j':}^\top} = \sum^k_{j=1} \E{\bX_{j:} \bM_j \bM^\top_j \bX_{j:}^\top} = \sum^k_{j=1} \norm{\bM_j}^2_F\,.
            \end{equation}
            We would like to apply Theorem~\ref{thm:hanson_wright} to show that $Z$ is concentrated around this expectation. To that end, define the vectorization $x\in\brc{\pm 1}^{kd}$ of $\bX$. Then for $\Sig\in\R^{kd\times kd}$ defined by
            \begin{equation}
                \Sig \triangleq \begin{pmatrix}
                    \bM_1 \\
                    \vdots \\
                    \bM_k
                \end{pmatrix}
                \begin{pmatrix}
                    \bM_1^\top & \cdots & \bM_k^\top
                \end{pmatrix}\,,
            \end{equation}
            we have $Z = x^\top \Sig x$.
            Note that
            \begin{equation}
                \norm{\bM_j}^2_\op \le m\sum^m_{i=1} (\lambda_i)^2_j \norm{\bW}^2_\op \le \frac{m}{\reffw} \sum^m_{i=1} (\lambda_i)^2_j \,\norm{\bW_i}^2_F \label{eq:Mjop1}
            \end{equation} and
            \begin{align}
                \Bigl|\norm{\bM_j}^2_F - \sum^m_{i=1} (\lambda_i)^2_j \, \norm{\bW_i}^2_F\Bigr| &\le \Winc\sum_{i\neq i'} |(\lambda_i)_j (\lambda_{i'})_j|\cdot \norm{\bW_i}_F\norm{\bW_{i'}}_F \\
                &\le \Winc\Bigl(\sum_i (\lambda_i)_j \norm{\bW_i}_F\Bigr)^2 \le m\Winc \sum^m_{i=1} (\lambda_i)^2_j \,\norm{\bW_i}^2_F\,, \label{eq:Mjop2}
            \end{align}
            so combining \eqref{eq:Mjop1} and \eqref{eq:Mjop2}, we have
            \begin{equation}
                \norm{\bM_j}^2_\op \le \frac{m}{\reffw}(1 + m\Winc) \norm{\bM_j}^2_F \lesssim \frac{m}{\reffw}\,\norm{\bM_j}^2_F\,,
            \end{equation}
            where the last step follows by \eqref{eq:kappaprime_constraint}. Therefore,
            \begin{equation}
                \norm{\Sig}_\op \le \sum_j \norm{\bM_j}^2_\op \lesssim \frac{m}{\reffw}\sum_j \norm{\bM_j}^2_F = \frac{m}{\reffw}\E{Z}
            \end{equation}
            and
            \begin{equation}
                \norm{\Sig}_F = \norm{\sum_j \bM_j^\top \bM_j}_F \le \sum_j \norm{\bM_j}_\op \norm{\bM_j}_F \lesssim \frac{m}{\reffw}\sum_j \norm{\bM_j}^2_F\,,
            \end{equation}
            so because $\reffw \gg m$, Theorem~\ref{thm:hanson_wright} implies that
            \begin{equation}
                \mathbb{P}\Bigl[Z \in [0.99,1.01]\cdot\sum^k_{j=1}\norm{\bM_j}^2_F\Bigr] \ge 1 - \exp(-\Omega(\reffw/m))\,.
            \end{equation}
            It remains to lower bound $\sum_j \norm{\bM_j}^2_F$. Eq.~\eqref{eq:Mjop2} and the assumption that $\Winc \ll 1/m$ imply that
            \begin{equation}
                \sum_j\norm{\bM_j}^2_F = (1 + o(1)) \sum^m_{i=1} \norm{\lambda_i}^2_2 \,\norm{\bW_i}^2_F \,.
            \end{equation}
            We have concluded that for any fixed vectors $\lambda_1,\ldots,\lambda_m$,
            \begin{equation}
                \mathbb{P}_{\bX}\Bigl[\Bigl\|\sum_i \lambda^\top_i \bX \bW_i\Bigr\|^2 = [0.95,1.05]\cdot\sum^m_{i=1} \norm{\lambda_i}^2_2 \cdot \norm{\bW_i}^2_F\Bigl] \ge 1 - \exp(-\Omega(\reffw/m))\,. \label{eq:noshortcombo}
            \end{equation}
            For $\delta \asymp \frac{\epsilon}{\sqrt{mkd\Wlbd}}$, let $\calS$ be a $\delta$-net over the set of vectors in $\R^{km}$ of norm between $\epsilon$ and $2$. By standard bounds, we can take $|\calS| \le O(1/\delta)^{km}$. Then for any $(\lambda_1,\ldots,\lambda_m)$ of norm between $\epsilon$ and $2$, let $(\lambda'_1,\ldots,\lambda'_m)$ denote its nearest neighbor in $\calS$. If we define the vector $\nu_i\triangleq \lambda_i - \lambda'_i$, then
            \begin{align}
                \Bigl\|\sum_i \nu_i^\top \bX \bW_i\Bigr\| &\le \sum_i \norm{\nu_i} \cdot \sqrt{kd}\cdot \norm{\bW_i}_F \le \delta \cdot \Bigl(kd\sum_i \norm{\bW_i}^2_F\Bigr)^{1/2} \\
                &\lesssim \epsilon\min_i\norm{\bW_i} \le \Bigl(\sum^m_{i=1} \norm{\lambda_i}^2_2 \cdot \norm{\bW_i}^2_F\Bigr)^{1/2}\,.
            \end{align}
            By taking the constant factor in the definition of $\delta$ sufficiently small, we conclude that if the event of Eq.~\eqref{eq:noshortcombo} happens for every vector in the net $\calS$, then it holds with a slightly wider range (i.e. $[0.9,1.1]$ instead of $[0.95,1.05]$) for \emph{every} vector in $\R^{km}$ with norm between $\epsilon$ and $2$. The claimed bound follows.
        \end{proof}

        In this section, define the event
        \begin{equation}
            \calE\triangleq \brc{\bX\in\brc{\pm 1}^{k\times d} \ \text{satisfies the bounds in Lemmas~\ref{prop:aXW} and~\ref{lem:nospuriouscombo}}}\,, \label{eq:goodEdef}
        \end{equation}
        so that by the above Lemmas,
        \begin{equation}
            \Pr{\bX\in\calE} \le O(\sqrt{mkd\Wlbd}/\epsilon)^{km}\exp(-\Omega(\reffw/m)) + \axwprob\,. \label{eq:Eprob}
        \end{equation}

    \subsection{LP-based certification}
        \label{sec:cert}

        \begin{algorithm2e}
        \DontPrintSemicolon
        \caption{\textsc{LPCertify}($\wh{\bW}, \epsilon, \delta$)}
        \label{alg:lp}
            \KwIn{Estimate $\wh{\bW}\in\R^{d\times d}$ for $\sum_i \bW_i$; error $\epsilon > 0$, failure probability $\epsilon > 0$}
            \KwOut{Set of linear constraints $\calL$ in the variable $\Sig$}
                $\calL \gets \emptyset$\;
                $\tempd \gets \Theta(\frac{1}{\norm{\Sig_1}_F}\cdot \log(\frac{k}{\epsilon\sqrt{\Wlbd}}))$\;
                $R \gets ke^{\tempd^2}$\;
                $\xi\gets R^{\Theta(1/\Siglbd^2)} \cdot \max(e^{\Siglbd/(m\Siginc)},\epsilon\log (m\norm{\Sig_1}_F))$\;
                $T\gets \Theta(1/\xi)^m \cdot R^{\Theta(m/\Siglbd)}\cdot (d^2\log(d\norm{\Sig_1}_F/\epsilon) + \log(1/\delta))$\;
                \For{$i\in[T]$}{
                    Draw random example $(\bX, \bY)$\;\label{step:drawexample}
                    $\alpha^* \gets \arg\min_{\alpha\in\Delta^{k-1}} \norm{\alpha \bX \wh{\bW} - \bY_{1:}}$\; \label{step:solvelinear}
                    \If{$\norm{\alpha^*\bX\wh{\bW} - \bY_{1:}} < \frac{\epsilon\sqrt{\Wlbd}}{2\tempvar}$}{\label{step:checknorm} 
                        \If{$\alpha^*_2, \alpha^*_3 \ge 1/3$}{\label{step:twosparse}
                            $s \gets \log(\alpha^*_2 / \alpha^*_3)$\;\label{step:sdef}
                            Add to $\calL$ the constraint $\{|\bX_{1:}\Sig(\bX_{2:} - \bX_{3:})^\top - s|\le 7\epsilon\}$\;\label{step:addconstraint}
                        }
                    }
                }
                \Return{$\calL$}
        \end{algorithm2e} 

        The idea behind the main algorithm in this section, {\sc LPCertify}, is to wait for examples $(\bX, \bY)$ for which the attention patterns for each head are similar. Indeed, if the first row of each pattern were close to the same convex combination $\alpha$, then we would have that the first row of $\bY$ satisfies
        \begin{equation}
            \bY_{1:} \approx \alpha\bX\sum_i \bW_i\,.
        \end{equation}
        Provided that $\wh{\bW}$ is sufficiently close to $\sum_i \bW_i$, a necessary condition for this would be for there to exist some convex combination of the rows of $\bX\wh{\bW}$ which is close to $\bY_{1:}$. As we show below (see Lemma~\ref{lem:complete_sound}), this also turns out to be a \emph{sufficient} condition, which we check for in Step~\ref{step:checknorm} of Algorithm~\ref{alg:lp} and whenever such a convex combination exists, we can read off from it information about $\Sig_1,\ldots,\Sig_m$. Indeed, the coefficients $\alpha^*$ of this convex combination will be close to the entries of
        \begin{equation}
            \softmax(\bX_{1:}\Sig_i\bX^\top)
        \end{equation}
        for all $i = 1,\ldots,m$, so by Lemma~\ref{lem:logratio}, we can estimate the difference between entries of $\bX_{1:}\Sig_i \bX^\top$ by taking the log-ratio between entries of $\alpha^*$ (see Steps~\ref{step:twosparse} onwards in Algorithm~\ref{alg:lp}). Every time we do this, we obtain one new linear constraint on the entries of $\Sig_1,\ldots,\Sig_m$, and the goal of our analysis in Sections~\ref{sec:helpful} and~\ref{sec:affinehullprops} will be to show that provided we draw enough examples and generate enough such constraints, the resulting convex body cut out by these constraints is a sufficiently good approximation to the affine hull $\affinehull$.

    \subsection{Completeness and soundness of certification}
        \label{sec:completeness}

        Here we show that if the attention patterns for the different heads are all close to some convex combination $\alpha$ on some input $\bX$, then the LP-based certification correctly identifies this is the case (completeness), and conversely, if the LP-based ceritification returns some convex combination $\alpha$, then the attention patterns are all close to $\alpha$ (soundness).

        \begin{lemma}\label{lem:complete_sound}
            For any $\epsilon > 0$, let $\wh{\bW}\in\R^{d\times d}$ be a matrix satisfying Eq.~\eqref{eq:hatWbound}.
            Suppose $\bX\in\brc{\pm 1}^{k\times d}$ satisfies the event $\calE$ in Eq.~\eqref{eq:goodEdef}, and let $\bY = F(\bX)$ for $F$ defined in Eq.~\eqref{eq:attention}. Then the following holds:
            \begin{itemize}[leftmargin=*,itemsep=0pt]
                \item \underline{Completeness}: If $\alpha\in\Delta^{k-1}$ is such that for all $i\in[m]$,
                    \begin{equation}
                        \norm{\alpha - \softmax(\bX_{1:}\Sig_i \bX^\top)} \le \frac{\epsilon\sqrt{\Wlbd}}{5\sqrt{m}}\min_i\norm{\bW_i}_F\,,
                    \end{equation}
                    then
                    \begin{equation}
                        \norm{\alpha \bX \wh{\bW} - \bY_{1:}} \le \frac{\epsilon}{2}\min_i\norm{\bW_i}_F\,.
                    \end{equation}
                \item \underline{Soundness}: If there exists $\alpha\in\Delta^{k-1}$ such that
                    \begin{equation}
                        \norm{\alpha^\top \bX \wh{\bW} - \bY_{1:}} < \frac{\epsilon}{2}\min_i \norm{\bW_i}_F \,,  \label{eq:existsalpha}
                    \end{equation}
                    then for all $i\in[m]$,
                    \begin{equation}
                        \norm{\alpha - \softmax(\bX_{1:} \Sig_i \bX^\top)} \le \epsilon \,. \label{eq:common_alpha}
                    \end{equation}
            \end{itemize}
        \end{lemma}

        \begin{proof}
            (Proof of completeness) Let $\delta_i \triangleq \alpha - \softmax(\bX_{1:}\Sig_i \bX^\top)$ so that $\norm{\delta_i} \le \epsilon$ for all $i\in[m]$. Then
            \begin{align}
                \alpha \bX \wh{\bW} &= \sum_i \alpha \bX \bW_i + \alpha \bX \Delta \\
                &= \sum_i \softmax(\bX_{1:}\Sig_i \bX^\top)\bX \bW_i + \sum_i \delta_i \bX \bW_i + \alpha \bX \Delta \\
                &= \bY_{1:} + \sum_i \delta_i \bX \bW_i + \alpha\bX\Delta\,.
            \end{align}
            As we are conditioning on the event of Lemma~\ref{prop:aXW},
            \begin{equation}
                \norm{\alpha\bX \Delta} \le \frac{\epsilon}{4}\min_i\norm{\bW_i}_F\,.
            \end{equation}
            Additionally, as we are conditioning on the event of Lemma~\ref{lem:nospuriouscombo},
            \begin{equation}
                \Bigl\|\sum_i \delta_i \bX\bW_i\Bigr\| \le 1.1\Bigl(\sum^m_{i=1}\norm{\delta_i}^2\cdot\norm{\bW_i}^2_F\Bigr)^{1/2} \le \frac{\epsilon}{4}\min_i \norm{\bW_i}_F\,,
            \end{equation}
            so completeness follows.

            (Proof of soundness) Suppose to the contrary that there exists $i\in[m]$ such that \eqref{eq:common_alpha} is violated yet \eqref{eq:existsalpha} holds. Then because we can write
            \begin{equation}
                \alpha^\top \bX\wh{\bW} - \bY_{1:} = \alpha^\top \bX (\wh{\bW} - \sum^m_{i=1} \bW_i) + \sum^m_{i=1}(\alpha - \softmax(\bX_{1:}\Sig_i \bX^\top))^\top \bX \bW_i\,,
            \end{equation}
            by triangle inequality, the assumed bound on $\norm{\wh{\bW} - \sum^m_{i=1} \bW_i}_\op$, and the assumption that $\bX$ satisfies the event of Lemma~\ref{prop:aXW}, we have
            \begin{equation}
                \Bigl\|\sum^m_{i=1} (\alpha - \softmax(\bX_{1:}\Sig_i \bX^\top))^\top \bX \bW_i \Bigr\| \le \frac{3\epsilon}{4}\min_i \norm{\bW_i}_F \,.
            \end{equation}
            If we define $\lambda_i \triangleq \alpha - \softmax(\bX_{1:}\Sig_i \bX^\top)$, then $\sum_i \norm{\lambda_i}^2 \in [\epsilon^2, 4]$ by assumption, so $\sum_i \norm{\lambda_i}^2 \cdot\norm{\bW_i}^2_F \ge \epsilon^2\cdot\min_i\norm{\bW_i}^2_F$. Then the above bound contradicts the hypothesis that the event of Lemma~\ref{lem:nospuriouscombo} holds. So soundness holds.
        \end{proof}

        
   \subsection{Helpful attention patterns}
        \label{sec:helpful}

        Here we show that with non-negligible probability over $\bX\sim\brc{\pm 1}^{k\times d}$, the first row of the attention pattern for each head is primarily supported on the second and third coordinates. Additionally, if $\Sig'$ is sufficiently far from the affine hull of $\Sig_1,\ldots,\Sig_m$, then the attention pattern induced by $\Sig'$ on input $\bX$ will be noticeably different.

        \begin{lemma}\label{lem:main_sculpt}
            Given $\Sig' \in \R^{d\times d}$, write it as $\Sig' = \Sig^\parallel + \Sig^\perp$ where $\Sig^\parallel$ is the projection of $\Sig'$ to $\affinehull$. Suppose $\norm{\Sig'}_F \le 2\norm{\Sig_1}_F$. Let $\tempd \ll \min(d,\norm{\Sig_1}_F)$, and let $\xi > 0$ be a parameter satisfying
            \begin{equation}
                \xi \ll k^{-\Theta(1/\Siglbd)}\cdot \exp(-O(\tempd^2/\Siglbd)) \label{eq:Deltabound}
            \end{equation}
            and
            \begin{equation}
                \xi \gg (ke^{\tempd^2})^{\Theta(1/\Siglbd)} / e^{\Siglbd/(m\Siginc)}\,. \label{eq:Deltabound2}
            \end{equation}
            If $\norm{\Sig^\perp}_F \ge \epsilon^*$ for 
            \begin{equation}
                \epsilon^* \ge m\log((ke^{\tempd^2})^{\Theta(1/\Siglbd)}/\xi)\cdot \Bigl(\frac{1}{2} + \frac{1}{\Siglbd}\Bigr)\cdot \xi\norm{\Sig_1}_F\,, \label{eq:epsbound}
            \end{equation}
            then with probability at least
            \begin{equation}
                \Omega(\xi)^m\cdot (ke^{\tempd^2})^{-O(m/\Siglbd)}
            \end{equation}
            over $\bX\sim\brc{\pm 1}^{k\times d}$, the following holds for all $i\in[m]$. Let $\leftside = c\min(1,\norm{\Sig_1}_F)$ for sufficiently small absolute constant $c > 0$. Then:
            \begin{enumerate}
                \item $\bX_{1:}\Sig_i(\bX_{2:} - \bX_{3:})^\top \in [\leftside, \leftside + \xi\norm{\Sig_1}_F]$ \label{item:notdiff}
                \item $\bX_{1:}\Sig'(\bX_{2:} - \bX_{3:})^\top \not\in [\leftside - \Theta(\epsilon^*), \leftside + \xi\norm{\Sig_1}_F + \Theta(\epsilon^*)]$ \label{item:mismatch}
                \item $\bX_{1:}\Sig_i(\bX_{2:} - \bX_{a:})^\top \ge \tempd \norm{\Sig_1}_F$ for all $i\in[m]$ and $a \in\brc{1,4,\ldots,m}$ \label{item:smalloff}   
            \end{enumerate}    
        \end{lemma}

        \noindent We briefly interpret the three Items in Lemma~\ref{lem:main_sculpt}. Item~\ref{item:smalloff} ensures that for every head $i\in[m]$, the attention pattern has very little mass on entries outside of the second and third coordinates. Item~\ref{item:notdiff} consists of an upper bound by $\leftside + \xi\norm{\Sig_1}_F$ and a lower bound by $\leftside$. The upper and lower bound simultaneously ensure that the second and third entries of all the attention patterns are not too different, which together with Item~\ref{item:smalloff} ensures that the attention patterns across heads are similar. Additionally, the lower bound ensures that in each of these patterns, the second entry is larger than the third entry by some non-negligible margin. Finally, Item~\ref{item:mismatch} ensures that for any $\Sig'$ which is far from the span of $\Sig_1,\ldots,\Sig_m$, in the attention pattern induced by $\Sig'$, the ratio between the second and third entries is noticeably different from the same for the attention pattern induced by any $\Sig_i$.

        \subsubsection{Proof preliminaries}
            For convenience, denote $(\bX_{1:})^\top$ by $x$. Write $\Sig' = \sum_i \alpha_i \Sig_i + \spanerr$ such that $\spanerr$ is orthogonal to $\Sig_1,\ldots,\Sig_m$. Define $\spanerr' \triangleq \frac{\norm{\Sig_1}_F}{\norm{\spanerr}_F} \cdot \spanerr$, i.e. the scaling of $\spanerr$ which satisfies $\norm{\spanerr'}_F = \norm{\Sig_1}_F$. 

            For any $i\in[m]$, note that 
            \begin{equation}
                \norm{\Sig'}_F \norm{\Sig_i}_F \ge |\iprod{\Sig', \Sig_i}| \ge |\alpha_i|\cdot \norm{\Sig_i}^2_F - \sum_{i'\neq i} |\alpha_{i'}| \cdot |\iprod{\Sig_i, \Sig_{i'}}| \ge \Siglbd|\alpha_i|\cdot \norm{\Sig_1}^2_F - \Siginc\cdot\norm{\Sig_1}^2_F \sum_{i'\neq i} |\alpha_{i'}|\,.
            \end{equation}
            Summing this over $i$, rearranging, and recalling \eqref{eq:lambound}, we conclude that
            \begin{equation}{}
                \sum^m_{i=1} |\alpha_i| \lesssim \frac{1}{\Siglbd}\cdot \frac{\norm{\Sig'}_F}{\norm{\Sig_1}_F}\,. \label{eq:sumalphas}
            \end{equation}
            Let
            \begin{equation}
                \hulldifftemp = \frac{\cspecial\epsilon^*}{2\leftside + \xi\norm{\Sig_1}_F}
            \end{equation}
            where
            \begin{equation}
                \cspecial \triangleq \Theta(m\log((ke^{\tempd^2})^{\Theta(1/\Siglbd)}/\xi))^{-1} \label{eq:cspecialdef}
            \end{equation}

            \begin{lemma}\label{lem:farfromhull}
                If $\sum_i \alpha_i \ge 1 + \hulldifftemp$ (resp. $\sum_i \alpha_i \le 1 - \hulldifftemp$), then a sufficient condition for Item~\ref{item:mismatch} to hold is that Item~\ref{item:notdiff} and the inequality $x^\top\spanerr(\bX_{2:} - \bX_{3:})^\top \ge 0$ (resp. $x^\top\spanerr(\bX_{2:} - \bX_{3:})^\top \le 0$) hold.
            \end{lemma}

            \begin{proof}
                Suppose that $\sum_i \alpha_i > 1 + \hulldifftemp$, and suppose that Item~\ref{item:notdiff} and $x^\top\spanerr(\bX_{2:} - \bX_{3:})^\top \ge 0$ hold. Let $\eta_i \triangleq x^\top\Sig_i (\bX_{2:} - \bX_{3:})^\top - (\leftside + \frac{\xi}{2}\norm{\Sig_1}_F)$ so that $|\eta_i| \le \frac{\xi}{2}\norm{\Sig_1}_F$. Then
                \begin{align}
                    x^\top \Sig'(\bX_{2:} - \bX_{3:})^\top - (\leftside + \frac{\xi}{2}\norm{\Sig_1}_F) &\ge \sum_i \alpha_i \Bigl(\eta_i + (\leftside + \frac{\xi}{2}\norm{\Sig_1}_F)\Bigr) - (\leftside + \frac{\xi}{2}\norm{\Sig_1}_F) \\
                    &> \sum_i \alpha_i \eta_i + \hulldifftemp\cdot(\leftside + \frac{\xi}{2}\norm{\Sig_1}_F) \\
                    &\ge \hulldifftemp \cdot(\leftside + \frac{\xi}{2}\norm{\Sig_1}_F) - \frac{\xi\norm{\Sig_1}_F}{\Siglbd}\,,
                \end{align}
                where in the last step we used Eq.~\eqref{eq:sumalphas} and the assumption that $\norm{\Sig'}_F \le 2\norm{\Sig_1}_F$. So Item~\ref{item:mismatch} would thus hold provided that 
                \begin{equation}
                    \hulldifftemp\cdot (\leftside + \frac{\xi}{2}\norm{\Sig_1}_F) \ge \Bigl(\frac{1}{2} + \frac{1}{\Siglbd}\Bigr)\cdot \xi\norm{\Sig_1}_F\,, \label{eq:sufficient_mismatch}
                \end{equation}
                which indeed holds by the assumed bound on $\epsilon^*$ in Eq.~\eqref{eq:epsbound} and our choice of $\sigma$. The proof in the case that $\sum_i \alpha_i < 1 - \hulldifftemp$ is entirely analogous.
            \end{proof}

            \begin{lemma}\label{lem:farfromspan}
                If $|1 - \sum_i \alpha_i| \le \hulldifftemp$, then a sufficient condition for Item~\ref{item:mismatch} to hold is that Item~\ref{item:notdiff} holds and additionally
                \begin{equation}
                    |x^\top \spanerr(\bX_{2:} - \bX_{3:})| \ge \hulldifftemp\cdot (2\leftside + \xi\norm{\Sig_1}_F)\,.
                \end{equation}
            \end{lemma}

            \begin{proof}
                Writing $\eta_i \triangleq x^\top \Sig_i (\bX_{2:} - \bX_{3:})^\top$ as in the proof of Lemma~\ref{lem:farfromhull}, we have
                \begin{align}
                    \MoveEqLeft\Bigl|x^\top \Sig'(\bX_{2:} - \bX_{3:})^\top - (\leftside + \frac{\xi}{2}\norm{\Sig_1}_F)\Bigr| \\
                    &= \Bigl|x^\top \spanerr (\bX_{2:} - \bX_{3:})^\top + \sum_i \alpha_i \Bigl(\eta_i + (\leftside + \frac{\xi}{2}\norm{\Sig_1}_F)\Bigr) - (\leftside + \frac{\xi}{2}\norm{\Sig_1}_F)\Bigr| \\
                    &\ge \Bigl|x^\top \spanerr'(\bX_{2:} - \bX_{3:})^\top\Bigr| - \hulldifftemp \cdot (\leftside + \frac{\xi}{2}\norm{\Sig_1}_F) - \frac{\xi\norm{\Sig_1}_F}{\Siglbd} \\
                    &\ge \hulldifftemp \cdot (\leftside + \frac{\xi}{2}\norm{\Sig_1}_F) - \frac{\xi\norm{\Sig_1}_F}{\Siglbd}\,,
                \end{align}
                where the penultimate step again follows by using Eq.~\eqref{eq:sumalphas} and the assumption that $\norm{\Sig'}_F \le 2\norm{\Sig_1}_F$ to bound the magnitude of $\sum_i \alpha_i \eta_i$ by $\frac{\xi\norm{\Sig_1}_F}{\Siglbd}$. This implies Item~\ref{item:mismatch} provided Eq.~\eqref{eq:sufficient_mismatch}, so the proof is complete.
            \end{proof}

            \noindent In light of Lemmas~\ref{lem:farfromhull} and~\ref{lem:farfromspan}, our goal is thus to lower bound the probability that Items~\ref{item:notdiff} and~\ref{item:smalloff} and the inequality
            \begin{equation}
                x^\top \spanerr(\bX_{2:} - \bX_{3:}) \cdot \mathrm{sgn}\Bigl(\sum_i \alpha_i - 1\Bigr) \ge \hulldifftemp\cdot (2\leftside + \xi\norm{\Sig_1}_F) \label{eq:Esuffice}
            \end{equation}
            hold.
        
        \subsubsection{Events on first row of $\bX$}

            The following fact establishes that for any unit vector $v$, there is a subset $S$ of coordinates of prescribed size such that either most of the mass in $v$ lies in $S$, or if not, then the remaining entries of $v$ are dense.


            \begin{fact}\label{fact:truncate}
                For any vector $v\in\R^{d}$ and any $1\le d'\le d$ and $\upnu$, there exists a subset $S\subset[d]$ of size at most $\dee$ such that at least one of the following holds for the vector $v|_{S^c}\in\R^{|S^c|}$ given by restricting $v$ to the coordinates indexed by $S^c$:
                \begin{itemize}
                    \item $\norm{v|_{S^c}}^2_2 \le \exp(-1/\upnu) \norm{v}^2_2$
                    \item $\norm{v|_{S^c}}^2_\infty \le \frac{1}{\upnu\dee}\norm{v|_{S^c}}^2_2$.
                \end{itemize}
            \end{fact}

            \begin{proof}
                We can assume without loss of generality that $v$ has unit norm and that $v^2_1 \ge \cdots \ge v^2_d$. Define $\gamma_j \triangleq \frac{v^2_j}{\sum_{i \ge j} v^2_i}$ and suppose that for all $1\le j \le \dee$, $\gamma_j \ge 1/(\upnu\dee)$. Then 
                \begin{equation}
                    \sum_{j > \dee} v^2_j = 1 - \sum_{j \le \dee} v^2_j = \prod^{\dee}_{j=1} (1 - \gamma_j) \le (1 - 1/(\upnu\dee))^{\dee} \le \exp(-1/\upnu)
                \end{equation}
                as claimed.
            \end{proof}


            \noindent In the sequel will take $\upnu$ and $\dee$ in Fact~\ref{fact:truncate} to be given by
            \begin{equation}
                \dee \triangleq \frac{\Siginc\sqrt{d}}{\Siglightrows\sqrt{\log d}} \qquad \text{and} \qquad \upnu \triangleq \Theta(\log m + \log\log((ke^{\tempd^2})^{\Theta(1/\Siglbd)}/\xi))^{-1}\,.\label{eq:deenudef}
            \end{equation}
            Note that by our choice of $\upnu$ and $\cspecial$ in Eq.~\eqref{eq:cspecialdef}, we have the relation
            \begin{equation}
                \cspecial = \exp(-1/2\upnu) \label{eq:cspecialrelation}
            \end{equation}
            if we choose constant factors correctly. This relation will be crucial at the end of the proof.

            Observe that these satisfy
            \begin{equation}
                m\Siginc \ll \Siglbd \cdot \exp(-1/\upnu) \qquad \text{and} \qquad \Siglightrows \dee\sqrt{\log(d) / d} \le \Siginc\,, \label{eq:sfab_want}
            \end{equation}
            where the former inequality follows by Eq.~\eqref{eq:Deltabound2}.
            
            For any choice of $x\in \brc{\pm 1}^d$, by Fact~\ref{fact:truncate} applied to the vector $x^\top \spanerr'$, there is some $S\subseteq[d]$ of size at most $\dee$ such that, if $\spanerr'_S$ and $\spanerr'_{S^c}$ denote the $d\times d$ matrices given by zeroing out the columns of $\spanerr'$ outside of $S$ and $S^c$ respectively, at least one of the following holds:
            \begin{enumerate}[label=(\Alph*)]
                \item $\norm{x^\top \spanerr'_{S^c}}^2_F \le \exp(-1/\upnu)\norm{x^\top \spanerr'}^2$.
                \item $\iprod{x,\spanerr'_{:j}}^2 \le \frac{1}{\upnu\dee}\norm{x^\top \spanerr'_{S^c}}^2$ for all $j\not\in S$ and $\norm{x^\top \spanerr'_{S^c}}^2 \ge \exp(-1/\upnu)\norm{x^\top \spanerr'}^2$.
            \end{enumerate}
            We observe that in both cases, by Assumption~\ref{assume:lightrows}, for all $i\in[m]$ we have
            \begin{equation}
                |\iprod{\Sig_i, \spanerr'_{S^c}}| = |\iprod{\Sig_i, \spanerr'_S}| = \Bigl|\sum_{j\in S} \iprod{(\Sig_i)_{:j}, \spanerr'_{:j}}\Bigr| \le \frac{\Siglightrows \dee}{\sqrt{d}} \norm{\Sig_1}^2_F \le \Siginc\norm{\Sig_1}^2_F \,, \label{eq:caseB1}
            \end{equation}
            where in the first step we used that $\iprod{\Sig_i, \spanerr'} = 0$.

            \begin{lemma}\label{lem:basic_highprob}
                The following events hold with all but arbitrarily small constant failure probability:
                \begin{enumerate}
                    \item $|x^\top \Sig_i x| \le (\projtrace + O(\sqrt{\log m}))\norm{\Sig_1}_F$ for all $i\in[m]$. \label{eq:xSx}
                    \item $\norm{\Sig^\top_i x}_\infty \lesssim \frac{\Siglightrows\sqrt{\log d}\norm{\Sig_i}_F}{\sqrt{d}}$ for all $i\in[m]$. \label{eq:Sigdiffinfluence}
                    \item $\norm{\spanerr'^\top x}^2 \lesssim \norm{\Sig_1}^2_F$. \label{eq:trivialEbound}
                    \item $|\norm{((\Sig_i)_{:, S^c})^\top x}^2 - \norm{\Sig_i}^2_F| \le 2\Siginc\,\norm{\Sig_1}^2_F$ for all $i\in[m]$ and all subsets $S\subseteq[d]$ of size at most $\dee$. \label{eq:Signorm}
                    \item $|x^\top \Sig_i\Sig_{i'}^\top x| \le \Siginc\,\norm{\Sig_1}^2_F$ for all $i\neq i' \in[m]$. \label{eq:mainblock_off}
                \end{enumerate}
            \end{lemma}

            \begin{proof}
                For Item~\ref{eq:xSx}, note that $|\E{x^\top \Sig_i x}| = |\Tr(\Sig_i)| \le \projtrace\cdot \norm{\Sig_1}_F$ by Assumption~\ref{assume:trace}. So by Theorem~\ref{thm:hanson_wright} and Assumption~\ref{assume:effranksig}, Item~\ref{eq:xSx} holds with high probability.

                For Item~\ref{eq:Sigdiffinfluence}, take any $j\in[d]$ and consider $|x^\top \Sig_i e_j|$. By Assumption~\ref{assume:lightrows}, the vector $\Sig_i^\top e_j$ has norm at most $\Siglightrows \norm{\Sig_i}_F / \sqrt{d}$, so we can apply Hoeffding's inequality to conclude that Item~\ref{eq:Sigdiffinfluence} holds with high probability.

                For Item~\ref{eq:trivialEbound}, note that $\E{x^\top \spanerr' \spanerr'^\top x} = \norm{\spanerr'}^2_F = \norm{\Sig_1}^2_F$. So Item~\ref{eq:trivialEbound} follows by by Markov's inequality.

                For Item~\ref{eq:Signorm}, by Theorem~\ref{thm:hanson_wright}, we have that for any $t > 0$,
                \begin{equation}
                    \Pr{|\norm{x^\top \Sig_i}^2 - \norm{\Sig_i}^2_F| > t} \lesssim \exp\Bigl(-\Omega\Bigl(\frac{t}{\norm{\Sig_i}^2_{\op}} \wedge \frac{t^2}{\norm{\Sig_i}^4_4}\Bigr)\Bigr)\,.
                \end{equation}
                Take $t = \Siginc \norm{\Sig_1}^2_F$ so that by Assumption~\ref{assume:effranksig}, the right-hand side is bounded by $\exp(-\Omega(\Siginc\reffsig))$. For any $S\subseteq[d]$ for which $|S|\le \dee$, if we condition on this event as well as the event of Item~\ref{eq:Sigdiffinfluence}, we have
                \begin{equation}
                    \norm{x^\top (\Sig_i)_{:, S^c}}^2 = \norm{x^\top \Sig_i}^2 - \sum_{j\in S} \iprod{(\Sig_i)_{j:}, x}^2 = \norm{\Sig_i}^2_F \cdot (1 \pm (\Siginc + \frac{\dee\log d}{d} \Siglightrows^2)\Bigr) = \norm{\Sig_i}^2_F \cdot (1 \pm 2\Siginc)\,,
                \end{equation}
                where we used the fact that $\Siglightrows \dee\log(d)/ \sqrt{d} \ll \Siginc$ by \eqref{eq:sfab_want}. This establishes Item~\ref{eq:Signorm}.

                Finally, for Item~\ref{eq:mainblock_off}, for any $1\le i,i' \le m$,
                \begin{equation}
                    \norm{\Sig_i\Sig_{i'}^\top}_\op \le \norm{\Sig_i}_{\op} \cdot \norm{\Sig_{i'}}_\op \le \frac{1}{\reffsig}\norm{\Sig_i}_F\cdot\norm{\Sig_{i'}}_F \le \frac{1}{\reffsig}\,\norm{\Sig_1}^2_F
                \end{equation}
                and 
                \begin{equation}
                    \norm{\Sig_i \Sig^\top_{i'}}^2_F \le \norm{\Sig_i}^2_\op \cdot \norm{\Sig_{i'}}^2_F \le \frac{1}{\reffsig}\,\norm{\Sig_i}^2_F \cdot \norm{\Sig_{i'}}^2_F \le \frac{1}{\reffsig}\,\norm{\Sig_1}^4_F\,, 
                \end{equation}
                so by Assumption~\ref{assume:sig_orth}, we can apply Theorem~\ref{thm:hanson_wright} with $t = \Siginc\norm{\Sig_1}^2_F$ to get
                \begin{equation}
                    \Pr{|x^\top \Sig_i\Sig_{i'}^\top x| \le \Siginc\,\norm{\Sig_1}^2_F \ \forall \ i\neq i' \in[m]} \ge 1 - m^2\exp(-\Omega(\Siginc \reffsig))\,,\label{eq:mainblock_offproof}
                \end{equation}
                thus establishing Item~\ref{eq:mainblock_off}.
            \end{proof}

            \begin{lemma}\label{lem:paley}
                $\Pr{\norm{\spanerr'^\top x}^2 \ge \norm{\Sig_1}^2_F / 2} \ge \Omega(1)$.
            \end{lemma}

            \begin{proof}
                Note that $\E{\norm{\spanerr'^\top x}^2} = \norm{\spanerr'}^2_F$ and
                \begin{align}
                    \Var{\norm{\spanerr'^\top x}^2} &= \E{(x^\top (\spanerr'\spanerr'^\top - \diag(\spanerr'\spanerr'^\top)) x)^2} \\
                    &= \sum_{\substack{i,j,k,\ell: \\ i\neq j; k\neq \ell}} (\spanerr'\spanerr'^\top)_{ij} (\spanerr'\spanerr'^\top)_{k\ell} \E{x_ix_jx_kx_\ell} \\
                    &= 2\sum_{i\neq j} (\spanerr'\spanerr'^\top)_{ij}^2 \le 2\norm{\spanerr'\spanerr'^\top}^2_F \le 2\norm{\spanerr'}^4_F\,. \label{eq:varxESc}
                \end{align}
                So by Paley-Zygmund,
                \begin{equation}
                    \Pr{\norm{\spanerr'^\top x}^2 \ge \norm{\spanerr'}^2_F / 2} \ge \Omega(1)\,. \label{eq:bottomright}
                \end{equation}        
            \end{proof}

            Denote the intersection of the events of Items~\ref{eq:xSx}, \ref{eq:Sigdiffinfluence}, \ref{eq:trivialEbound}, \ref{eq:Signorm}, and \ref{eq:mainblock_off} in Lemma~\ref{lem:basic_highprob} by $E_1$, so that $\Pr{E_1} = 1 - o(1)$. Denote the event of Lemma~\ref{lem:paley} by $E_2$ so that $\Pr{E_2} \ge \Omega(1)$. Henceforth condition on the events of $E_1$ and $E_2$, which happens with probability $\Omega(1)$.

            \begin{corollary}\label{cor:junk1}
                Under $E_1$, for any $z\in\brc{\pm 1}^d$,
                \begin{equation}
                    |x^\top (\Sig_i)_{:, S} z_S| \lesssim \frac{\Siglightrows\dee\sqrt{\log d}}{\sqrt{d}}\norm{\Sig_1}_F\,. \label{eq:junk1}
                \end{equation}
            \end{corollary}

            \begin{proof}
                This is immediate from Item~\ref{eq:Sigdiffinfluence} of Lemma~\ref{lem:basic_highprob} and H\"{o}lder's inequality.
            \end{proof}

            We use the following shorthand: given a matrix $\bM \in \R^{d\times d}$, we use the notation $\bM_{:, S}$ to denote the $d\times d$ matrix given by zeroing out all columns of $\bM$ outside of $S$. We next establish that the directions $(\Sig_1)_{:,S^c}^\top x, \ldots, (\Sig_m)_{:,S^c}^\top x$ are sufficiently incoherent under the event $E_1$.

            \begin{lemma}\label{lem:Vsigbound_A}
                Consider the matrix $\bV\in \R^{m\times d}$ with rows consisting of 
                \begin{equation}
                    ((\Sig_1)_{:,S^c})^\top x, ((\Sig_2)_{:, S^c})^\top x, \ldots, ((\Sig_m)_{:, S^c})^\top x\,.
                \end{equation}           
                Then under $E_1$, we have
                \begin{align}
                    \norm{\bV\bV^\top}_\op &\lesssim \norm{\Sig_1}^2_F \label{eq:maxsigbound1} \\
                    \sigma_{\min}(\bV\bV^\top) &\gtrsim \Siglbd\norm{\Sig_1}^2_F \label{eq:minsigbound1}
                \end{align}
                and in particular \begin{equation}
                    \frac{\norm{\bV}^2_F}{\sigma_{\min}(\bV\bV^\top)} \lesssim m/\Siglbd\,. \label{eq:frobvssigA}
                \end{equation}
            \end{lemma}

            \begin{proof}
                We will argue that $\bV\bV^\top$ is spectrally close to $\bZ\in\R^{m\times m}$ given by
                \begin{equation}
                    \bZ \triangleq 
                    \diag(\norm{\Sig_1}^2_F \ldots, \norm{\Sig_m}^2_F)\,. \label{eq:Zdef}
                \end{equation}
                Under $E_1$,
                \begin{equation}
                    \norm{\bV\bV^\top - \bZ}_\op \le m\,\norm{\bV \bV^\top - \bZ}_{\max} \lesssim m\Siginc\,\norm{\Sig_1}^2_F\,,
                    \label{eq:wellcond}
                \end{equation}
                so by triangle inequality and the fact that $\Siglbd \gg m\Siginc$ by \eqref{eq:lambound}, the bounds Eq.~\eqref{eq:maxsigbound1} and~\eqref{eq:minsigbound1} follow. Eq.~\eqref{eq:frobvssigA} is immediate from these.
            \end{proof}

            \noindent We will sometimes need a slight strengthening of Lemma~\ref{lem:Vsigbound_A}, namely that the matrix $\bV$ defined therein is still well-conditioned even after appending an additional row given by $\frac{\norm{\Sig_1}_F}{\norm{(\spanerr'_{S^c})^\top x}}\cdot (\spanerr'_{S^c})^\top x$.

            \begin{lemma}\label{lem:Vsigbound_B}
                Let $\bV'$ be the matrix whose rows consist of 
                \begin{equation}
                    ((\Sig_1)_{:,S^c})^\top x, ((\Sig_2)_{:, S^c})^\top x, \ldots, ((\Sig_m)_{:, S^c})^\top x, \frac{\norm{\Sig_1}_F}{\norm{(\spanerr'_{S^c})^\top x}}\cdot (\spanerr'_{S^c})^\top x\,.
                \end{equation}   
                Under $E_1$ and $E_2$, if we are in Case (B), then
                \begin{align}
                    \norm{\bV'\bV'^\top}_\op &\lesssim \norm{\Sig_1}^2_F \label{eq:maxsigbound2} \\
                    \sigma_{\min}(\bV'\bV'^\top) &\gtrsim \Siglbd \norm{\Sig_1}^2_F\,, \label{eq:minsigbound2}
                \end{align}
                and in particular
                \begin{equation}
                    \frac{\norm{\bV'}^2_F}{\sigma_{\min}(\bV'\bV'^\top)} \lesssim m/\Siglbd\,. \label{eq:frobvssigB}
                \end{equation}
            \end{lemma}

            \begin{proof}
                We will argue that $\bV'\bV'^\top$ is spectrally close to a matrix
                $\bZ'\in\R^{(m+1)\times (m+1)}$ given by
                \begin{equation}
                    \bZ' \triangleq 
                    \diag(\norm{\Sig_1}^2_F,\norm{\Sig_2}^2_F, \ldots, \norm{\Sig_m}^2_F, \norm{\Sig_1}^2_F) \label{eq:Zdef2}
                \end{equation}
                By Lemma~\ref{lem:Vsigbound_A}, the top-left $m\times m$ block of $\bV'\bV'^\top$ is close to that of $\bZ'$ with high probability, namely when event $E_1$ defined above holds.
                
                Next, we consider the last row and column of $\bV'\bV'^\top$. By \eqref{eq:caseB1}, $|\iprod{\Sig_i, \spanerr'_{S^c}}| \le \Siginc\norm{\Sig}^2_F$ for all $i$, and $\norm{\Sig_i \spanerr'^\top_{S^c}}_\op \le \norm{\Sig_i \spanerr'^\top_{S^c}} \le \norm{\spanerr'_{S^c}}_F \cdot \norm{\Sig_i}_\op \le \frac{1}{\sqrt{\reffsig}}\norm{\Sig_1}^2_F$, so by Theorem~\ref{thm:hanson_wright}, there is an absolute constant $c > 0$ such that for any $c' > 0$,
                \begin{equation}
                    \Pr{|x^\top \Sig_i \spanerr'^\top_{S^c} x| \le (\Siginc + \frac{c\log(m/c')}{\sqrt{\reffsig}})\norm{\Sig_1}^2_F \ \forall \ i\in[m]} \ge 1 - c'\,. \label{eq:lastrowcolA}
                \end{equation}
                We will take $c'$ to be some small constant. Note that $\frac{c\log(m/c')}{\sqrt{\reffsig}} \ll \Siginc$ holds (by some margin). Under this event, we have 
                \begin{equation}
                    \frac{\norm{\Sig_1}_F}{\norm{\spanerr'^\top_{S^c} x}} \cdot x^\top \Sig_i \spanerr'^\top_{S^c} x \lesssim  \exp(1/\upnu)\cdot  x^\top \Sig_i \spanerr'^\top_{S^c} x \lesssim \Siginc \exp(1/\upnu)\cdot \norm{\Sig_1}^2_F \ll \frac{\Siglbd}{m} \norm{\Sig_1}^2_F\,,
                \end{equation}
                where we have used the assumption that we are in Case (B) together with $E_2$ to deduce the first step, and in the last step we used the first part of Eq.~\eqref{eq:sfab_want}.
                
                Recall the definition of the matrix $\bZ'$ in \eqref{eq:Zdef2}. We have
                \begin{equation}
                    \norm{\bV'\bV'^\top - \bZ'}_\op \le m\,\norm{\bV' \bV'^\top - \bZ'}_{\max} \ll \Siglbd \norm{\Sig_1}^2_F\,. \label{eq:wellcond2}
                \end{equation}
                Then because $\sigma_{\min}(\bV'\bV'^\top) \ge \Siglbd\norm{\Sig_1}^2_F$, by triangle inequality we obtain Eq.~\eqref{eq:maxsigbound2} and~\eqref{eq:minsigbound2}. This, together with the fact that Item~\ref{eq:Signorm} holds under $E_1$, implies Eq.~\eqref{eq:frobvssigB}.
            \end{proof}

        \subsubsection{Sufficient events for remaining rows of $\bX$}
        \label{sec:bulletpoints}

            We will take $\eta > 0$ to be a parameter to be tuned depending on which case we are in. We will lower bound the probability that the following events happen
            \begin{itemize}
                \item For $\bX_{2:}$:
                \begin{align}
                    x^\top\Sig_i \bX_{2:} &\in \bigl[\norm{\Sig_1}_F (\sqrt{10\log k} + \tempd), \norm{\Sig_1}_F (\sqrt{10\log k} + \tempd) + \xi\cdot\norm{\Sig_1}_F\bigr] \label{eq:X2point1}\\
                    x^\top\spanerr' \bX_{2:} &\begin{cases}
                        \le \eta\,\norm{\Sig_1}_F & \text{if} \ \ \ \sum_i \alpha_i \le 1 \\
                        \ge -\eta\,\norm{\Sig_1}_F & \text{if} \ \ \ \sum_i \alpha_i > 1
                    \end{cases} \label{eq:X2point2}
                \end{align}
                \item For $\bX_{3:}$:
                \begin{align}
                    x^\top \Sig_i \bX_{3:} &\in [x^\top \Sig_i \bX_{2:} - \leftside - \xi\norm{\Sig_1}_F, x^\top \Sig_i \bX_{2:} - \leftside] \\
                    x^\top \spanerr' \bX_{3:} &\begin{cases}
                        \ge \eta\,\norm{\Sig_1}_F + \frac{\hulldifftemp\norm{\Sig_1}_F}{\epsilon^*}\cdot (2\leftside + \xi\norm{\Sig_1}_F) & \text{if} \ \ \  \sum_i \alpha_i \le 1 \\
                        \le -\eta\,\norm{\Sig_1}_F - \frac{\hulldifftemp\norm{\Sig_1}_F}{\epsilon^*}\cdot (2\leftside + \xi\norm{\Sig_1}_F) & \text{if} \ \ \  \sum_i \alpha_i > 1 \\
                    \end{cases}
                \end{align}
                \item For $\bX_{\ell:}$ for $\ell > 3$:
                \begin{align}
                    |x^\top \Sig_i \bX_{\ell:}| \le \norm{\Sig_1}_F \sqrt{10\log k}\,.
                \end{align}
            \end{itemize}
            Observe that if all of these events happen, then because we are conditioning on Item~\ref{eq:xSx} of Lemma~\ref{lem:basic_highprob}, Items 1 and 3 of Lemma~\ref{lem:main_sculpt} would hold, as would \eqref{eq:Esuffice}, thus also implying Item 2 of Lemma~\ref{lem:main_sculpt} would hold.

        \subsubsection{Establishing sufficient conditions in Case (A)}
        \label{sec:applyintegrolocalA}

            In Case (A), define
            \begin{equation}
                \eta \triangleq m\log((ke^{\tempd^2})^{\Theta(1/\Siglbd)}/\xi) \cdot \exp(-1/2\upnu) \,. \label{eq:etadef}
            \end{equation}
            Note that $\eta \ll 1$ by our choice of $\upnu$ in Eq.~\eqref{eq:deenudef}.

            \paragraph{Step 1: $\bX_{2:}$} \hypertarget{A1}{}

            First note that by symmetry, 
            \begin{equation}
                \mathbb{P}\Bigl[\mathrm{sgn}(x^\top \spanerr'_S \bX_{2,S}) \cdot \mathrm{sgn}\Bigl(\sum_i \alpha_i - 1\Bigr) \ge 0\Bigr] = 1/2 \label{eq:condition_sym}
            \end{equation}
            so condition on this event. We show that with all but inverse polynomially small probability over $\bX_{2,S^c}$, the quantity $|x^\top \spanerr'_{S^c} \bX_{2,S^c}|$ is $o(1)$. Because we are in Case (A), by Item~\ref{eq:trivialEbound} in Lemma~\ref{lem:basic_highprob} we have
            \begin{equation}
                \norm{(\spanerr'_{S^c})^\top x}^2 \le \exp(-1/\upnu)\norm{\spanerr'^\top x}^2 \lesssim \exp(-1/\upnu)\norm{\Sig_1}^2_F\,.
            \end{equation}
            By Hoeffding's inequality and our choice of $\eta$ in Eq.~\eqref{eq:etadef},
            \begin{equation}
                \Pr[\bX_{2,S^c}]{|x^\top \spanerr'_{S^c}\bX_{2,S^c}| \le \eta \norm{\Sig_1}_F} \ge 1 - O(\xi/k^{O(1/\Siglbd)})^m
                \,, \label{eq:xEXSc_a}
            \end{equation}
            which will be negligible compared to our final bound on the probability of the good event.

            So if the event of Eq.~\eqref{eq:xEXSc_a} held, then combined with the event of~\eqref{eq:condition_sym} which we are conditioning on, this would imply the second bullet point for $\bX_{2:}$ above. We will take the constant factor in the definition of $\eta$ large enough that the failure probability of \eqref{eq:xEXSc_a} is negligible relative to the lower bound on the probability (in the remaining randomness of $\bX_{2,S^c}$) that the first condition in the bullet point for $\bX_{2:}$ above holds.

            To show this lower bound, we wish to apply Theorem~\ref{thm:borovkov} to the vectors $v_1 \triangleq x^\top (\Sig_1)_{:,S^c}, \ldots, v_m \triangleq x^\top (\Sig_m)_{:,S^c}$, conditioned on any assignment to $\bX_{2,S}$, to bound the probability that 
            \begin{equation}
                (\iprod{v_1,\bX_{2,S^c}}, \ldots, \iprod{v_m, \bX_{2,S^c}}) \in A
            \end{equation}
            for
            \begin{equation}
                A\triangleq \prod^m_{i=1} [a_i + \xi_i] \ \ \ \text{for} \ \ \ a_i \triangleq\norm{\Sig_1}_F (\sqrt{10\log k} + \tempd) - x^\top (\Sig_i)_S \bX_{2,S} \ \ \ \text{and} \ \ \ \xi_i \triangleq \xi\norm{\Sig_1}_F\,.
            \end{equation}

            In Theorem~\ref{thm:borovkov}, we have $\underline{\xi} = \overline{\xi} = \xi\norm{\Sig_1}_F$. And recalling Eq.~\eqref{eq:junk1} from Corollary~\ref{cor:junk1} and the second bound in Eq.~\eqref{eq:sfab_want}, we have 
            \begin{equation}
                \norm{a} = \Theta(\sqrt{m}\cdot \norm{\Sig_1}_F (\sqrt{\log k} + \tempd))\,. 
            \end{equation}

            Recall that because we are conditioning on $E_1$, Item~\ref{eq:Sigdiffinfluence} and Item~\ref{eq:Signorm} from Lemma~\ref{lem:basic_highprob} imply that we can take $\rho$ in Eq.~\eqref{eq:rhocor} to be $O(\Siglightrows\sqrt{\log d})$. And by Item~\ref{eq:Signorm} from Lemma~\ref{lem:basic_highprob}, we can also take $\underline{r} = \sqrt{(1 - 2\Siginc)\Siglbd} \cdot \norm{\Sig_1}_F$ and $\overline{r} = \sqrt{1 + 2\Siginc}\cdot \norm{\Sig_1}_F$ in Eq.~\eqref{eq:vbound}. We can take $\upkappa$ in Eq.~\eqref{eq:angles} to be $\Siginc / \Siglbd$, so by our choice of $\Siginc$ in Assumption~\ref{assume:sig_orth}, Eq.~\eqref{eq:angles} is satisfied, and the condition that $\upkappa \ll 1/m$ in Theorem~\ref{thm:borovkov} is satisfied by the first part of Eq.~\eqref{eq:sfab_want}.

            Additionally, letting $\bV$ be the matrix whose columns consist of $v_1,\ldots,v_m$, by Lemma~\ref{lem:Vsigbound_A},
            \begin{equation}
                \Siglbd^m \cdot \Omega(\norm{\Sig_1}_F)^{2m} \le \det(\bV\bV^\top) \le O(\norm{\Sig_1}_F)^{2m}\ \label{eq:detbound}
            \end{equation}
            and, because $\norm{(\bV\bV^\top)^{-1}}_{\sf op} = \sigma_{\min}(\bV\bV^\top)^{-1} \lesssim \Siglbd^{-1}\norm{\Sig_1}^{-2}_F$,
            \begin{equation}
                \exp(-a^\top (\bV\bV^\top)^{-1} a) \ge \exp(-O(\norm{a}^2\cdot \Siglbd^{-1}\norm{\Sig_1}^{-2}_F)) \ge k^{-O(m/\Siglbd)} \cdot \exp(-O(m\tempd^2/\Siglbd))\,. \label{eq:expavva}
            \end{equation}
            We verify the three conditions in Eqs.~\eqref{eq:sufficient_lam}, \eqref{eq:sufficient_penultimate}, and \eqref{eq:sufficient_last}. First note that
            \begin{equation}
                \exp(-\Theta(\norm{a}^2/\underline{r}^2)) = k^{-\Theta(m/\Siglbd)}\cdot \exp(-O(m\tempd^2/\Siglbd)) \,. \label{eq:expar}
            \end{equation}

            \begin{itemize}[leftmargin=*]
                \item \underline{Eq.~\eqref{eq:sufficient_lam}}:
                    Note that 
                    \begin{equation}
                        \sqrt{m}\norm{a}/\underline{r}^2 \asymp \frac{m(\sqrt{\log k} + \tempd)}{\Siglbd\norm{\Sig_1}_F}\,.
                    \end{equation}
                    and 
                    \begin{equation}
                        \underline{\xi}^{-1} \exp(O(\norm{a}^2/\underline{r}^2))\prod_j \norm{v_j}/\xi_j = \frac{1}{\norm{\Sig_1}_F}\cdot k^{\Theta(m/\Siglbd)}\cdot \exp(-O(m\tempd^2/\Siglbd))\cdot \Theta(1/\xi)^{m+1}\,.
                    \end{equation}
                    So the left-hand side of Eq.~\eqref{eq:sufficient_lam} is bounded by Assumption~\ref{assume:nonarithmetic} (see Section~\ref{sec:charfunction})

                    On the other hand, the right-hand side of Eq.~\eqref{eq:sufficient_lam} is given by
                    \begin{equation}
                        \Bigl(\norm{\Sig_1}_F\cdot k^{\Theta(m/\Siglbd)}\cdot \exp(-\Theta(m\tempd^2/\Siglbd)) \cdot \xi^{\Theta(m)}\Bigr)^m\,.
                    \end{equation}

                \item \underline{Eq.~\eqref{eq:sufficient_penultimate}}: 
                    The left-hand side is given by
                    \begin{equation}
                        \frac{\rho^3 m^3 \overline{r}^3\norm{a}^3}{\underline{r}^6 d^{3/2}} = \wt{O}\Bigl(\frac{\Siglightrows^3 \poly(m)\tempd^3}{\Siglbd^3 d^{3/2}}\Bigr)\,,
                    \end{equation}
                    and the right-hand side is bounded by Eq.~\eqref{eq:expar}, so by Assumption~\ref{assume:fewheads} and the assumption that $\tempd\ll d$, Eq.~\eqref{eq:sufficient_penultimate} is satisfied.

                \item \underline{Eq.~\eqref{eq:sufficient_last}}:
                    The left-hand side is given by
                    \begin{equation}
                        \frac{m\norm{a}}{\underline{r}^2}\cdot (\sqrt{m}\underline{\xi}\cdot \prod_j \xi_j / \norm{v_j} + \overline{\xi}) \le (m^2\sqrt{\log k} + m^{3/2}\tempd)\cdot(\xi/\Siglbd)^{O(m)}\,.
                    \end{equation}
                    The assumed bound in Eq.~\eqref{eq:Deltabound} ensures that this is upper bounded by $\exp(-\norm{a}^2/4\underline{r}^2)$.
            \end{itemize}

            So by Theorem~\ref{thm:borovkov} applied to the random bitstring $\bX_{2,S^c}$, and using Eq.~\eqref{eq:detbound}, we conclude that
            \begin{equation}
                \Pr{x^\top(\Sig_i)_{:,S^c} \bX_{2,S} \in R} \ge \Omega(\xi/(ke^{\tempd^2})^{O(1/\Siglbd)})^m
            \end{equation}
            By combining this Eq.~\eqref{eq:xEXSc_a}, we conclude that in Case (A),
            \begin{equation}
                \Pr{\bX_{2:} \ \text{satisfies its bullet points}} \gtrsim \Omega(\xi/(ke^{\tempd^2})^{O(1/\Siglbd)})^m\,.
            \end{equation}

            \paragraph{Step 2: $\bX_{3:}$.} \hypertarget{A2}{}

            By the fact that we are in Case (A) and conditioning on $E_2$, we have 
            \begin{equation}
                \norm{(\spanerr'_S)^\top x}\gtrsim \norm{\spanerr'^\top x} \gtrsim \norm{\spanerr'}_F = \norm{\Sig_1}_F\,.
            \end{equation}
            By Theorem~\ref{thm:booleananti}, $|x^\top \spanerr'_S \bX_{3,S}| \gtrsim \norm{\Sig_1}_F$ with probability $\Omega(1)$. As in \eqref{eq:xEXSc_a} in the analysis for $\bX_2$ above, we can show in Case (A) for $\eta$ given by \eqref{eq:etadef} that
            \begin{equation}
                \Pr[\bX_{3,S^c}]{|x^\top \spanerr'_{S^c} \bX_{3, S^c}| \le \eta\, \norm{\Sig_1}_F} \ge 1 - O(\xi/k^{O(1/\Siglbd)})^m\,.\label{eq:xEXSc2}
            \end{equation} 
            If this event happens, then because $\eta = o(1)$ and because we are taking $\hulldifftemp$ such that $\epsilon^* / (\hulldifftemp(2\leftside + \xi\norm{\Sig_1}_F))$ is an arbitrarily large constant, we would conclude that the second bullet point for $\bX_{3:}$ above holds. As before, the constant factor in Eq.~\eqref{eq:etadef} will be chosen large enough that the failure probability of \eqref{eq:xEXSc2} is negligible relative to the lower bound we will show on the probability, relative to the remaining randomness of $\bX_{3,S^c}$, that the first bullet point for $\bX_{3:}$ holds.

            Henceforth condition on any $\bX_{3,S}$ such that $|x^\top \spanerr'_S \bX_{3,S}| \gtrsim \norm{\Sig_1}_F$. Also condition on any choice of $\bX_{2:}$ satisfying its respective bullet points above.

            We wish to apply Theorem~\ref{thm:borovkov} to the vectors $v_1 \triangleq x^\top (\Sig_1)_{:,S^c},\ldots, v_m \triangleq x^\top (\Sig_m)_{:, S^c}$ to bound the probability that
            \begin{equation}
                (\iprod{v_1, \bX_{2,S^c}}, \ldots, \iprod{v_m, \bX_{2,S^c}}) \in A
            \end{equation}
            for 
            \begin{equation}
                A \triangleq \prod^m_{i=1} [a_i + \xi_i] \ \ \ \text{for} \ \ \ a_i \triangleq x^\top \Sig_i \bX_{2:} - \leftside - \xi\norm{\Sig_1}_F \ \ \ \text{and} \ \ \ \xi_i \triangleq \xi\norm{\Sig_1}_F\,.
            \end{equation}
            In Theorem~\ref{thm:borovkov}, we have $\underline{\xi} = \overline{\xi} = \xi\norm{\Sig_1}_F$. As we are conditioning on $\bX_{2:}$ satisfying its respective bullet points, $|a_i| = \Theta(\norm{\Sig_1}_F \cdot \sqrt{\log k})$, so
            \begin{equation}
                \norm{a} = \Theta(\norm{\Sig_1}_F \cdot\sqrt{m\log k})\,. \label{eq:A2_abound}
            \end{equation}

            Eqs.~\eqref{eq:rhocor}, \eqref{eq:vbound}, and \eqref{eq:angles} are satisfied for $\rho = \Siglightrows\sqrt{\log d}$, $\underline{r} = \sqrt{(1 - 2\Siginc)\Siglbd} \cdot \norm{\Sig_1}_F$, $\overline{r} = \sqrt{1 + 2\Siginc}\cdot \norm{\Sig_1}_F$, and $\upkappa = \Siginc/\Siglbd$, exactly as in Step 1 above. Eq.~\eqref{eq:detbound} from Step 1 still holds, so by Eq.~\eqref{eq:A2_abound}, we obtain Eq.~\eqref{eq:expavva}.

            Note that all of the relevant parameters in Theorem~\ref{thm:borovkov} are the same up to constant factors, with the expectation that $\sqrt{\log k} + \tempd$ in the bound on $\norm{a}$ is replaced by $\sqrt{\log k}$. It is straightforward to check that in this setting, the remaining conditions in Eq.~\eqref{eq:sufficient_lam}-\eqref{eq:sufficient_last} of the Theorem are satisfied by minor modifications to the analysis from Step 1. We conclude that
            \begin{equation}
                \Pr{\bX_{3:} \ \text{satisfies its bullet points}} \gtrsim \Omega(\xi/k^{O(1/\Siglbd)})^m\,.
            \end{equation}

            \paragraph{Step 3: $\bX_{\ell:}$ for $\ell > 3$.}

            This step is straightforward as we can eschew the application of Theorem~\ref{thm:borovkov} in favor of a simple union bound. First by the same argument that led to Item~\ref{eq:Signorm}, we can show that $\norm{\Sig_i^\top x}^2 = (1 \pm 2\Siginc)\norm{\Sig_1}^2_F = \Theta(\norm{\Sig_1}^2_F)$ for all $i\in[m]$ with probability at least $1 - m\exp(-\Omega(\Siginc\reffsig))$, so we can additionally condition on this event in $x$. Then by Hoeffding's inequality,
            \begin{equation}
                \Pr[\bX_{\ell:}]{|x^\top \Sig_i \bX_{\ell:}| \ge \norm{\Sig_1}_F \sqrt{10\log k}} \ll 1/k
            \end{equation}
            So in Case (A),
            \begin{equation}
                \Pr{\bX_{\ell:} \ \text{satisfies its bullet point} \ \forall \ \ell > 3} \ge (1 - 1/k)^k \ge \Omega(1)\,.
            \end{equation}

            As $\bX_{2:}, \bX_{3:}, \ldots, \bX_{k:}$ are independent random vectors, by putting all of the steps above together, we conclude that for Case (A), the desired bound in Lemma~\ref{lem:main_sculpt} holds with probability at least $\Omega(\xi/k^{O(1/\Siglbd)})^m$.

        \subsubsection{Establishing sufficient conditions in Case (B)}
        \label{sec:applyintegrolocalB}

            In Case (B), we will take $\eta = 0$. This choice of $\eta$ will only be important in Step 2.

            \paragraph{Step 1: $\bX_{2:}$.} 
            \hypertarget{B1}{}

            First note that by symmetry, 
            \begin{equation}
                \mathbb{P}\Bigl[\mathrm{sgn}(x^\top \spanerr'_S \bX_{2,S}) \cdot \mathrm{sgn}\Bigl(\sum_i \alpha_i - 1\Bigr) \ge 0\Bigr] = 1/2 \label{eq:condition_sym12}
            \end{equation}
            so condition on this event.
            

            We next lower bound the probability with respect to the randomness in $\bX_{2,S^c}$, which is independent of the event Eq.~\eqref{eq:condition_sym12}, that the two bullet points for $\bX_{2:}$ hold. We will apply Theorem~\ref{thm:borovkov} to the vectors $v_1 \triangleq x^\top (\Sig_1)_{:, S^c}, \ldots, v_m \triangleq x^\top (\Sig_m)_{:, S^c}$, and $v_{m+1} \triangleq \mathrm{sgn}(\sum_i \alpha_i - 1)\cdot \frac{\norm{\Sig_1}_F}{\norm{x^\top \spanerr'_{S^c}}} \cdot x^\top \spanerr'_{S^c}$ to bound the probability that
            \begin{equation}
                (\iprod{v_1, \bX_{2,S^c}}, \ldots, \iprod{v_{m+1}, \bX_{2,S^c}}) \in A
            \end{equation}
            for
            \begin{equation}
                A \triangleq \prod^{m+1}_{i=1} [a_i + \xi_i]
            \end{equation}
            where
            \begin{equation}
                a_1,\ldots,a_m = \norm{\Sig_i}_F \sqrt{10\log k} + \tempd - x^\top (\Sig_i)_S \bX_{2,S}, a_{m+1} = 0
            \end{equation}
            and
            \begin{equation}
                \xi_1,\ldots,\xi_{m+1} = \xi\cdot\norm{\Sig_1}_F\,.
            \end{equation}
            As we are conditioning on the event of Eq.~\eqref{eq:condition_sym12}, this is a sufficient condition (by some margin) for the bullet points for $\bX_{2:}$ to hold.
            
            In Theorem~\ref{thm:borovkov}, we have $\underline{\xi} = \overline{\xi} = \xi\norm{\Sig_1}_F$. As in Case (A), recalling Eq.~\eqref{eq:junk1} from Corollary~\ref{cor:junk1} and the second bound in Eq.~\eqref{eq:sfab_want}, we have 
            \begin{equation}
                \norm{a} = \Theta(\sqrt{m}\cdot \norm{\Sig_1}_F \cdot (\sqrt{\log k} + \tempd))\,.
            \end{equation}

            Because we are in Case (B) and are conditioning on Item~\ref{eq:Sigdiffinfluence} from Lemma~\ref{lem:basic_highprob}, we can take $\rho$ in Theorem~\ref{thm:borovkov} to be $\max(\sqrt{d}/(\upnu\dee), \Siglightrows\sqrt{\log d})$. And by Item~\ref{eq:Signorm} from Lemma~\ref{lem:basic_highprob} and Item~\ref{eq:trivialEbound} and the fact that $\norm{\spanerr'}_F = \norm{\Sig_1}_F$, we can also take $\underline{r} = \sqrt{(1 - 2\Siginc)\Siglbd} \cdot \norm{\Sig_1}_F$ and $\overline{r} = \sqrt{1 + 2\Siginc} \cdot \norm{\Sig_1}_F$ in Eq.~\eqref{eq:vbound}. Note that we can take $\upkappa = \Siginc \exp(1/\upnu) / \Siglbd$ in Eq.~\eqref{eq:angles}. This follows by Assumption~\ref{assume:sig_orth}, the event in Eq.~\eqref{eq:lastrowcolA}, and the fact that in Case (B),
            \begin{equation}
                \norm{\spanerr'^\top_{S^c} x}^2 \ge \exp(-1/\upnu)\cdot \norm{\spanerr'^\top x}^2 \gtrsim \exp(-1/\upnu)\cdot\norm{\spanerr'}^2_F = \exp(-1/\upnu)\cdot \norm{\Sig_1}^2_F\,, \label{eq:Exbound}
            \end{equation}
            where the penultimate step follows as we are conditioning on event $E_2$. Observe that the condition $\upkappa \ll 1/m$ in Theorem~\ref{thm:borovkov} is satisfied by the first part of Eq.~\eqref{eq:sfab_want}.

            Additionally, letting $\bV$ in Theorem~\ref{thm:borovkov} be the matrix $\bV'$ defined in Lemma~\ref{lem:Vsigbound_B} whose columns consist of $v_1,\ldots,v_{m+1}$,
            \begin{equation}
                \Siglbd^m \cdot \Omega(\norm{\Sig_1}^2_F)^{m+1} \le \det(\bV'\bV'^\top) \le O(\norm{\Sig_1}^2_F)^{m+1}\,, \label{eq:detbound2}
            \end{equation}
            and, because $\norm{(\bV'\bV'^\top)^{-1}}_{\sf op} = \sigma_{\min}(\bV'\bV'^\top)^{-1} \lesssim \Siglbd^{-1}\norm{\Sig_1}^{-2}_F$,
            \begin{equation}
                \exp(-a^\top(\bV'\bV^\top)^{-1} a) \ge \exp(-O(\norm{a}^2\cdot \Siglbd^{-1} \norm{\Sig_1}^{-2}_F)) \ge k^{-O(m/\Siglbd)} \cdot \exp(-O(m\tempd^2/\Siglbd))\,.
            \end{equation}
            We verify the three conditions in Eqs.~\eqref{eq:sufficient_lam}, \eqref{eq:sufficient_penultimate}, \eqref{eq:sufficient_last}. First note that
            \begin{equation}
                \exp(-\Theta(\norm{a}^2 / \underline{r}^2)) = k^{-\Theta(m/\Siglbd)}\cdot \exp(-\Theta(m\tempd^2/\Siglbd))\,. \label{eq:expar2}
            \end{equation}

            \begin{itemize}[leftmargin=*]
                \item \underline{Eq.~\eqref{eq:sufficient_lam} and Eq.~\eqref{eq:sufficient_last}}: The relevant parameters involved are the same as from Case (A) Step 1, so we can verify Eq.~\eqref{eq:sufficient_lam} in an entirely identical fashion.

                \item \underline{Eq.~\eqref{eq:sufficient_penultimate}}: 
                    The left-hand side is given by
                    \begin{equation}
                        \frac{\rho^3 m^3 \overline{r}^3\norm{a}^3}{\underline{r}^6 d^{3/2}} = \wt{O}\Bigl(\frac{\max(\Siglightrows^3, d^{3/2} / (\upnu^3\dee^3))\cdot \poly(m) \tempd^3}{\Siglbd^3 d^{3/2}}\Bigr)\,,
                    \end{equation}
                    and the right-hand side is bounded by Eq.~\eqref{eq:expar}, so by Assumption~\ref{assume:fewheads} and the assumption that $\tempd \ll d$, Eq.~\eqref{eq:sufficient_penultimate} is satisfied.
            \end{itemize}

            So by Theorem~\ref{thm:borovkov} applied to the random bitstring $\bX_{2,S^c}$, and using Eq.~\eqref{eq:detbound2}, we conclude that
            \begin{equation}
                \Pr{x^\top(\Sig_i)_{:,S^c} \bX_{2,S^c} \in R} \ge \Omega(\xi/(ke^{\tempd^2})^{O(1/\Siglbd)})^m
            \end{equation}
            By combining this with the event of Eq.~\eqref{eq:condition_sym12}, we conclude that in Case (B),
            \begin{equation}
                \Pr{\bX_{2:} \ \text{satisfies its bullet points}} \gtrsim \Omega(\xi/(ke^{\tempd^2})^{O(1/\Siglbd)})^m\,.
            \end{equation}

            \paragraph{Step 2: $\bX_{3:}$}

            First note that by symmetry, 
            \begin{equation}
                \mathbb{P}\Bigl[\mathrm{sgn}(x^\top \spanerr'_S \bX_{2,S}) \cdot \mathrm{sgn}\Bigl(1 - \sum_i \alpha_i\Bigr) \ge 0\Bigr] = 1/2 \label{eq:condition_sym13}
            \end{equation}
            so condition on this event.


            We next lower bound the probability with respect to the randomness in $\bX_{3,S^c}$, independent of the event in Eq.~\eqref{eq:condition_sym13}, that the two bullet points for $\bX_{3:}$ hold. We will apply Theorem~\ref{thm:borovkov} to the vectors $v_1 \triangleq x^\top (\Sig_1)_{:, S^c}, \ldots, v_m \triangleq x^\top (\Sig_m)_{:, S^c}$, and $v_{m+1} \triangleq \mathrm{sgn}(1 - \sum_i \alpha_i)\cdot \frac{\norm{\Sig_1}_F}{\norm{x^\top \spanerr'_{S^c}}} \cdot x^\top \spanerr'_{S^c}$ to bound the probability that
            \begin{equation}
                (\iprod{v_1, \bX_{3,S^c}}, \ldots, \iprod{v_{m+1}, \bX_{3,S^c}}) \in A
            \end{equation}
            for
            \begin{equation}
                A \triangleq \prod^{m+1}_{i=1} [a_i + \xi_i]
            \end{equation}
            where
            \begin{equation}
                a_1,\ldots,a_m = x^\top \Sig_i \bX_{2:} - \leftside - \xi\norm{\Sig_1}_F\,,
            \end{equation}
            \begin{equation}
                a_{m+1} = \frac{\norm{\Sig_1}^2_F}{\norm{x^\top \spanerr'_{S^c}}} \cdot \frac{\hulldifftemp}{\epsilon^*}(2\leftside + \xi\norm{\Sig_1}_F)\,,
            \end{equation}
            and
            \begin{equation}
                \xi_1,\ldots,\xi_{m+1} = \xi \norm{\Sig_1}_F\,.
            \end{equation}
            In Theorem~\ref{thm:borovkov}, we have $\underline{\xi} = \overline{\xi} = \xi \norm{\Sig_1}_F$. As we are conditioning on $\bX_{2:}$ satisfying its respective bullet points, $|a_i| = \Theta(\norm{\Sig_1}_F \cdot \sqrt{\log k})$ for $i\in[m]$. Furthermore, recalling Eq.~\eqref{eq:Exbound}, 
            \begin{equation}
                |a_{m+1}| \lesssim \exp(1/2\upnu) \cdot \norm{\Sig_1}_F \cdot \frac{\hulldifftemp}{\epsilon^*}(2\leftside + \xi\norm{\Sig_1}_F)\,,
            \end{equation}
            we conclude that
            \begin{equation}
                \norm{\Sig_1}_F \cdot \sqrt{m\log k} \lesssim \norm{a} \lesssim \norm{\Sig_1}_F \cdot \Bigl[\sqrt{m\log k} + \exp(1/2\upnu)\cdot \frac{\hulldifftemp}{\epsilon^*}(2\leftside + \xi\norm{\Sig_1}_F)\Bigr]\,.  \label{eq:abound_pre}
            \end{equation}
            Now we make crucial use of the relation Eq.~\eqref{eq:cspecialrelation} (recall the definition of $\cspecial$ from Eq.~\eqref{eq:cspecialdef}), which implies that the upper bound on the right-hand side is simply $\norm{\Sig_1}_F \cdot O(\sqrt{m\log k})$, so we have
            \begin{equation}
                \norm{a} \asymp \norm{\Sig_1}_F \cdot \Theta(\sqrt{m\log k})\,. \label{eq:abound}
            \end{equation}

            As in Step 1 of Case (B), we can take $\rho$ in Theorem~\ref{thm:borovkov} to be $\max(\sqrt{d}/(\upnu\dee), \Siglightrows\sqrt{\log d})$, $\underline{r} = \sqrt{(1 - 2\Siginc)\Siglbd} \cdot \norm{\Sig_1}_F$ and $\overline{r} = \sqrt{1 + 2\Siginc} \cdot \norm{\Sig_1}_F$ in Eq.~\eqref{eq:vbound}, and $\upkappa = \Siginc \exp(1/\upnu) / \Siglbd$. The condition $\upkappa \ll 1/m$ in Theorem~\ref{thm:borovkov} is satisfied by the first part of Eq.~\eqref{eq:sfab_want}, as before.

            Additionally, letting $\bV$ in Theorem~\ref{thm:borovkov} be the matrix $\bV'$ defined in Lemma~\ref{lem:Vsigbound_B} whose columns consist of $v_1,\ldots,v_{m+1}$, Eq.~\eqref{eq:detbound2} still holds, so by Eq.~\eqref{eq:abound},
            \begin{equation}
                \exp(-a^\top(\bV'\bV^\top)^{-1} a) \ge \exp(-O(\norm{a}^2\cdot \Siglbd^{-1} \norm{\Sig_1}^{-2}_F)) \ge k^{-O(m/\Siglbd)}\,.
            \end{equation}
            We verify the three conditions in Eqs.~\eqref{eq:sufficient_lam}, \eqref{eq:sufficient_penultimate}, \eqref{eq:sufficient_last}. First note that
            \begin{equation}
                \exp(-\Theta(\norm{a}^2 / \underline{r}^2)) \ge k^{-O(m/\Siglbd)}\,. \label{eq:expar3}
            \end{equation}

            \begin{itemize}[leftmargin=*]
                \item \underline{Eq.~\eqref{eq:sufficient_lam}}:
                    Note that 
                    \begin{equation}
                        \sqrt{m}\norm{a}/\underline{r}^2 \gtrsim \frac{m\sqrt{\log k}}{\Siglbd\norm{\Sig_1}_F}\,.
                    \end{equation}
                    and 
                    \begin{equation}
                        \underline{\xi}^{-1} \exp(O(\norm{a}^2/\underline{r}^2))\prod_j \norm{v_j}/\xi_j \le \frac{k^{O(m/\Siglbd)}}{\norm{\Sig_1}_F} O(1/\xi)^{m+1}\,.
                    \end{equation}
                    So the left-hand side of Eq.~\eqref{eq:sufficient_lam} is bounded by Assumptino~\ref{assume:nonarithmetic} (see Section~\ref{sec:charfunction}).

                    On the other hand, the right-hand side of Eq.~\eqref{eq:sufficient_lam} is lower bounded by
                    \begin{equation}
                        \bigl(\norm{\Sig_1}_F\cdot k^{-O(m/\Siglbd)}\cdot \xi^{\Theta(m)}\bigr)^m\,.
                    \end{equation}

                \item \underline{Eq.~\eqref{eq:sufficient_penultimate}}: 
                    The left-hand side is given by
                    \begin{equation}
                        \frac{\rho^3 m^3 \overline{r}^3\norm{a}^3}{\underline{r}^6 d^{3/2}} = \wt{O}\Bigl(\frac{\max(\Siglightrows^3, d^{3/2} / (\upnu^3\dee^3))\cdot m^6}{\Siglbd^3 d^{3/2}}\Bigr)\,,
                    \end{equation}
                    and the right-hand side is bounded by Eq.~\eqref{eq:expar3}, so by Assumption~\ref{assume:fewheads}, Eq.~\eqref{eq:sufficient_penultimate} is satisfied.

                \item \underline{Eq.~\eqref{eq:sufficient_last}}:
                    The left-hand side is given by
                    \begin{equation}
                        \frac{m\norm{a}}{\underline{r}^2}\cdot (\sqrt{m}\underline{\xi}\cdot \prod_j \xi_j / \norm{v_j} + \overline{\xi}) \le (\xi/\Siglbd)^{O(m)}\,.
                    \end{equation}
                    The assumed bound in Eq.~\eqref{eq:Deltabound} ensures that this is upper bounded by $\exp(-\norm{a}^2/4\underline{r}^2)$.
            \end{itemize}

            So by Theorem~\ref{thm:borovkov} applied to the random bitstring $\bX_{3,S^c}$, and using Eq.~\eqref{eq:detbound2}, we conclude that
            \begin{equation}
                \Pr{x^\top(\Sig_i)_{:,S^c} \bX_{3,S^c} \in R} \ge \Omega(\xi)^m \cdot k^{-O(m/\Siglbd)}
            \end{equation}
            By combining this with Eq.~\eqref{eq:condition_sym13}, we conclude that in Case (B),
            \begin{equation}
                \Pr{\bX_{3:} \ \text{satisfies its bullet points}} \gtrsim \Omega(\xi)^m \cdot k^{-O(m/\Siglbd)}\,.
            \end{equation}

            \paragraph{Step 3: $\bX_{\ell:}$ for $\ell > 3$.} 

            The analysis here is unchanged from the corresponding analysis in Case (A). In Case (B), we still have that
            \begin{equation}
                \Pr{\bX_{\ell:} \ \text{satisfies its bullet point} \ \forall \ \ell > 3} \ge (1 - 1/k)^k \ge \Omega(1)\,.
            \end{equation}
            As $\bX_{2:}, \ldots, \bX_{k:}$ are independent random vectors, by putting all of the steps above together, we conclude that for Case (B), the bound in Lemma~\ref{lem:main_sculpt} holds with probability at least $\Omega(\xi)^m \cdot k^{-O(m/\Siglbd)}\cdot \exp(-O(\Siglbd^{-1})\cdot \exp(1/\upnu))$. But note that $k^{-O(m/\Siglbd)}\cdot \exp(-O(\Siglbd^{-1})\cdot \exp(1/\upnu))$ dominates $(ke^{\tempd^2})^{-O(m/\Siglbd)}$ from Step 1 of Case (B), so the proof of the Lemma is complete.

    \subsection{Properties of the approximate affine hull}
        \label{sec:affinehullprops}

        Define the convex body
        \begin{equation}
            K \triangleq \{\Sig: \norm{\Sig}_F \le 2\norm{\Sig_1}_F \ \text{and} \ \Sig \ \text{satisfies all constraints in} \ \calL\}\,, \label{eq:Kdef}
        \end{equation}
        where $\calL$ is the output of {\sc LPCertify}($\wh{\bW},\epsilon)$. Here we verify that $K$  has the following desirable properties:

        \begin{definition}\label{def:tighthull}
            Given parameters $\upepsilon,\upzeta > 0$, we say that a set $K\subseteq\R^{d\times d}$ is an \emph{$(\upepsilon,\upzeta)$-tight enclosure of $\Sig_1,\ldots,\Sig_m$} if the following properties are satisfied:
            \begin{enumerate}
                \item $K$ is convex and admits an efficient membership oracle.
                \item For every $\bM\in K$, $\norm{\bM}_F \le 2\norm{\Sig_1}_F$.
                \item For every $\bM\in K$, $\norm{\PiSig^\perp(\bM)}_F \le \upepsilon\cdot \norm{\bM}_F$.
                \item It contains the Frobenius norm balls $B(\Sig_1,\upzeta),\ldots, B(\Sig_m,\upzeta)$.
            \end{enumerate} 
        \end{definition}

        We begin by showing that the constraints in $\calL$ are always satisfied by points which are sufficiently close to the convex hull of $\Sig_1,\ldots,\Sig_m$:
        
        \begin{lemma}\label{lem:convexhullfeasible}
            For the parameter $\epsilon$ in Eq.~\eqref{eq:hatWbound}, let $\Sig$ be any point which is $\epsilon/(2\tempvar d)$-close in Frobenius distance to the \emph{convex} hull of $\Sig_1,\ldots,\Sig_m$. Let $\bX$ be any example encountered over the course of running {\sc LPCertify}($\wh{\bW}, \epsilon$) that satisfies the event $\calE$ in Eq.~\eqref{eq:goodEdef} and leads us to enter Step~\ref{step:addconstraint}. Then $\Sig$ satisfies the corresponding constraint that is added to $\calL$ in that step.
        \end{lemma}

        \begin{proof}
            Let $\alpha^*$ be the convex combination produced in Step~\ref{step:solvelinear}. Let $v^{(i)} \triangleq \softmax(\bX_{1:}\Sig_i\bX^\top)$. By the soundness guarantee in Lemma~\ref{lem:complete_sound}, we have that for all $i\in[m]$, $\norm{\alpha^* - v^{(i)}} \le \epsilon $. By Lemma~\ref{lem:logratio} applied to any $v = v^{(i)}$ and $v' = \alpha^*$, we conclude that $|s - \log(v^{(i)}_2/v^{(i)}_3)| \le 6\epsilon $.

            Note that
            \begin{equation}
                \log(v^{(i)}_2/v^{(i)}_3) = \bX_{1:}\Sig_i(\bX_{2:} - \bX_{3:})^\top\,,
            \end{equation}
            so by convexity, for any $\wt{\Sig}$ in the convex hull of $\Sig_1,\ldots,\Sig_m$ we have that
            \begin{equation}
                |\bX_{1:}\wt{\Sig}(\bX_{2:} - \bX_{3:})^\top  - s| \le 6\epsilon \,.
            \end{equation}
            Let $\wt{\Sig}$ be any point in the convex hull which is $\epsilon/(2\tempvar d)$-close to $\Sig$. Then
            \begin{equation}
                |\bX_{1:}(\wt{\Sig}- \Sig)(\bX_{2:} - \bX_{3:})^\top| \le 2d\norm{\wt{\Sig} - \Sig}_\op \le \epsilon \,,
            \end{equation}
            so the claim follows.
        \end{proof}

        The main step is to verify that by taking $T$, the number of examples drawn in {\sc LPCertify}, to be sufficiently large, we generate enough constraints in $\calL$ that $K$ will not contain any points which are far from the affine hull of $\Sig_1,\ldots,\Sig_m$.

        \begin{lemma}\label{lem:constraint_ruleout}
            There is an absolute constant $C > 0$ such that the following holds. Given $\Sig'\in\R^{d\times d}$, write it as $\Sig' = \Sig^\parallel + \Sig^\perp$ where $\Sig^\parallel$ is the projection of $\Sig'$ to $\affinehull$. Suppose $\Sig'$ satisfies $\norm{\Sig^\perp}_F \ge \epsilon^*$ for 
            \begin{equation}
                \epsilon^* \ge \max\Bigl(C\epsilon , m\log((ke^{\tempd^2})^{\Theta(1/\Siglbd)}/\xi)\cdot\Bigl(\frac{1}{2} + \frac{1}{\Siglbd}\Bigr)\cdot \xi\norm{\Sig_1}_F\Bigr) \label{eq:epsbound_useagain}
            \end{equation}
            for $\xi$ satisfying the bounds Eq.~\eqref{eq:Deltabound} and~\eqref{eq:Deltabound2} in Lemma~\ref{lem:main_sculpt} in addition to the condition $\xi \ll 1/\norm{\Sig_1}_F$ and the bound
            \begin{equation}
                \sqrt{m/\Wlbd}\cdot \Bigl(\frac{1}{2}\xi\norm{\Sig_1}_F + 2k e^{-\tempd\norm{\Sig_1}_F}\Bigr) \le \frac{\epsilon\sqrt{\Wlbd}}{5\tempvar }\,. \label{eq:xiprime_assume}
            \end{equation}
            Let $\bX\in\brc{\pm 1}^{k\times d}$ be an example satisfying the event $\calE$ defined in Eq.~\eqref{eq:goodEdef} and additionally satisfying the three bullet points in Lemma~\ref{lem:main_sculpt}.

            Then if $\bX$ is encountered in Step~\ref{step:drawexample} of Algorithm~\ref{alg:lp}, we will reach Step~\ref{step:addconstraint} and add a constraint which $\Sig'$ does not satisfy.
        \end{lemma}

        \begin{proof}
            Recall from the first and third bullet point that $\bX$ satisfies in Lemma~\ref{lem:main_sculpt} that for all $i\in[m]$,
            \begin{equation}
                \bX_{1:} \Sig_i(\bX_{2:} - \bX_{3:})^\top \in [\leftside, \leftside + \xi\norm{\Sig_1}_F] \qquad \text{and} \qquad \bX_{1:}\Sig_i(\bX_{2:} - \bX_{a:})^\top \ge \tempd\norm{\Sig_1}_F \ \text{for all} \ a\neq 2,3\,. \label{eq:bullets}
            \end{equation}
            Let $v^{(i)} \triangleq \softmax(\bX_{1:}\Sig_i\bX^\top)$. By Lemma~\ref{lem:softmaxtwo}, the latter condition in Eq.~\eqref{eq:bullets} ensures that
            \begin{equation}
                \Bigl\|v^{(i)} - \Bigl(\frac{e^{\beta_i}}{e^{\beta_i} + 1}\cdot e_2 + \frac{1}{e^{\beta_i} + 1}\cdot e_3\Bigr)\Bigr\|_1 \le 2k e^{-\tempd\norm{\Sig_1}_F} \ \ \ \text{for some} \ \ \ \beta_i \in [\leftside, \leftside+\xi\norm{\Sig_1}_F]\,.
            \end{equation}
            As
            \begin{equation}
                \Bigl|\frac{e^\leftside}{e^\leftside+1} - \frac{e^{\leftside+\xi\norm{\Sig_1}_F}}{e^{\leftside+\xi\norm{\Sig_1}_F} + 1}\Bigr| \le \frac{1}{4}\xi\norm{\Sig_1}_F\,,
            \end{equation}
            we conclude that there exists $\alpha\in\Delta^{k-1}$ such that $\norm{\alpha - v^{(i)}} \le \norm{\alpha - v^{(i)}}_1 \le \frac{1}{2}\xi\norm{\Sig_1}_F + 2k e^{-\tempd\norm{\Sig_1}_F}$ for all $i\in[m]$. By the completeness guarantee in Lemma~\ref{lem:complete_sound} and the assumption in Eq.~\eqref{eq:xiprime_assume}, this implies that
            \begin{equation}
                \norm{\alpha\bX\wh{\bW} - \bY_{1:}} \le \frac{5}{2}\sqrt{m/\Wlbd}\cdot \Bigl(\frac{1}{2}\xi\norm{\Sig_1}_F + 2k e^{-\tempd\norm{\Sig_1}_F}\Bigr) \le \frac{\epsilon\sqrt{\Wlbd}}{2\tempvar }\,.
            \end{equation}
            Furthermore, $\leftside$ is taken to be an arbitrarily small constant, and we are assuming that $\xi \ll 1/\norm{\Sig_1}_F$, so the former condition in Eq.~\eqref{eq:bullets} implies that $v^{(i)}_2, v^{(i)}_3 \ge 1/3$. We conclude that we would indeed reach Step~\ref{step:addconstraint} under this particular example $\bX$. 

            Next, we want to show that the quantity ``$\leftside$'' defined in Step~\ref{step:sdef} is sufficiently close to the quantity $\leftside$ from Lemma~\ref{lem:main_sculpt}. To avoid confusion, we will refer to the former as $s_{\sf alg}$. Note that by applying the soundness guarantee in Lemma~\ref{lem:complete_sound} to $\alpha^*$, we get that $\norm{\alpha^* - \softmax(\bX_{1:}\Sig_i\bX^\top)} \le \epsilon $ for all $i\in[m]$, and therefore $\norm{\alpha - \alpha^*} \le 2\epsilon $. By Lemma~\ref{lem:logratio}, we conclude that $|\leftside - s_{\sf alg}| \le 12\epsilon $.

            By the second bullet point in Lemma~\ref{lem:main_sculpt},
            \begin{equation}
                \bX_{1:}\Sig'(\bX_{2:} - \bX_{3:})^\top \not\in [\leftside - \Theta(\epsilon^*), \leftside + \xi\norm{\Sig_1}_F + \Theta(\epsilon^*)]\,. \label{eq:secondbullet}
            \end{equation}
            On the other hand, the constraint introduced in Step~\ref{step:addconstraint} of Algorithm~\ref{alg:lp} is of the form
            \begin{equation}
                \bX_{1:}\Sig(\bX_2 - \bX_{3:})^\top \in [s_{\sf alg} - 7\epsilon , s_{\sf alg} + 7\epsilon ]\,,
            \end{equation}
            a necessary condition for which, in light of our bound on $|\leftside - s_{\sf alg}|$, is that
            \begin{equation}
                \bX_{1:}\Sig(\bX_2 - \bX_{3:})^\top \in [s - 19\epsilon , s + 19\epsilon ]\,.
            \end{equation}
            So provided that $\epsilon^*$ in Eq.~\eqref{eq:secondbullet} is at least a sufficiently large constant multiple of $\epsilon $, we conclude that $\Sig'$ violates the constraint added in Step~\ref{step:addconstraint}, as claimed.
        \end{proof}

        Set 
        \begin{equation}
            \axwprob \triangleq O(\xi)^m\cdot (ke^{\upbeta^2})^{-\Omega(m/\Siglbd)} \cdot (d^2\log(d\norm{\Sig_1}_F / \epsilon^*) + \log(1/\delta))^{-1}\cdot \delta \label{eq:axwprob_def}
        \end{equation}
        With Lemma~\ref{lem:constraint_ruleout}, we can now use a standard epsilon-net argument now to conclude that the final convex body $K$ does not contain \emph{any} points $\Sig'$ which are far from $\affinehull$:

        \begin{corollary}\label{cor:farK_infeasible}
            Suppose $\epsilon^*, \xi$ satisfy the bounds in the hypothesis of Lemma~\ref{lem:constraint_ruleout}, and additionally 
            \begin{equation}
                \xi \gg (ke^{\tempd^2})^{-O(1/\Siglbd)} \cdot \Theta(\sqrt{mkd\Wlbd}/\epsilon)^k \exp(-O(\reffsig/m^2))\,.
            \end{equation}
            Let $K$ be as defined in Eq.~\eqref{eq:Kdef}. If the number of iterations $T$ in the main loop of Algorithm~\ref{alg:lp} satisfies
            \begin{equation}
                T \ge \Omega(1/\xi)^m\cdot (ke^{\upbeta^2})^{\Omega(m/\Siglbd)}\cdot (d^2\log(d\norm{\Sig_1}_F / \epsilon^*) +\log(1/\delta))\,, \label{eq:Tconstraint}
            \end{equation}
            then with probability at least $1 - \delta$, $K$ does not contain any $\Sig\in\R^{d\times d}$ which is at least $\epsilon^*$-far in Frobenius distance from $\affinehull$.
        \end{corollary}

        \begin{proof}
            Let $\calS$ be an $O(\epsilon^*/d)$-net over the set of $d\times d$ matrices of Frobenius norm at most $2\norm{\Sig_1}_F$ and which are at least $\epsilon^*$-far in Frobenius distance from $\affinehull$. Naively, we have the bound $|\calS| \le (d\norm{\Sig_1}_F / \epsilon^*)^{d^2}$.

            For every $\Sig'\in\calS$, by Lemma~\ref{lem:main_sculpt} and Lemma~\ref{lem:constraint_ruleout}, for $\bX\sim\brc{\pm 1}^{k\times d}$ the three bullet points in Lemma~\ref{lem:main_sculpt} and the event $\calE$ are satisfied with probability at least 
            \begin{equation}
                \Omega(\xi)^m\cdot (ke^{\upbeta^2})^{-O(m/\Siglbd)} - O(\sqrt{mkd\Wlbd}/\epsilon)^{km}\exp(-\Omega(\reffw/m)) - \axwprob \ge \Omega(\xi)^m\cdot (ke^{\upbeta^2})^{-O(m/\Siglbd)}\,,
            \end{equation}
            where the inequality follows by the bound on $\axwprob$ in Eq.~\eqref{eq:axwprob_def}, the assumed lower bound on $\xi$ in the hypothesis of the Lemma, and the lower bound on $\reffw$ in Assumption~\ref{assume:Weffr}. So for $\delta > 0$, if $T\ge \Omega(1/\xi)^{m}\cdot (ke^{\upbeta^2})^{\Omega(m/\Siglbd)}\cdot \log(|\calS|/\delta)$, then with probability at least $1 - \delta/|\calS|$, there is at least one sample $\bX$ among the iterations of Algorithm~\ref{alg:lp} for which a constraint is added to $\calL$ which is not satisfied by $\Sig'$. By a union bound over $|\calS|$, we conclude that with probability at least $1 - \delta$, every point in $\calS$ violates at least one of the constraints in $\calL$. Indeed, by the proof of Lemma~\ref{lem:constraint_ruleout}, the violations are by a margin of $\Omega(\epsilon^*)$. Now consider any $\Sig$ for which $\norm{\Sig}_F \le 2\norm{\Sig_1}_F$ and which is at least $\epsilon^*$-far in Frobenius distance from $\affinehull$, and let $\Sig'\in\calS$ be the closest point so that $\norm{\Sig - \Sig'}_F \le O(\epsilon^*/d)$. Then $\bX_{1:}(\Sig - \Sig')(\bX_{2:} - \bX_{3:})^\top \le O(\epsilon^*)$, so as long as the constant factor in the granularity $O(\epsilon^*/d)$ of $\calS$ is sufficiently small, this is smaller than the margin by which some constraint in $\calL$ is violated by $\Sig'$, so $\Sig$ also violates that constraint.
        \end{proof}

        \noindent We are now ready to conclude the two main guarantees for the section, namely that $K$ is a tight enclosure of $\Sig_1,\ldots,\Sig_m$, and that the point in $K$ with minimum Frobenius norm is close to a particular convex combination of the attention matrices.

        \begin{theorem}\label{thm:main_lp}
            Take any $\epsilon > 0$. Define the parameters
            \begin{equation}
                \tempd \asymp \frac{1}{\norm{\Sig_1}_F}\cdot \log\Bigl(\frac{k}{\epsilon\sqrt{\Wlbd}}\Bigr) \qquad \text{and} \qquad R \triangleq ke^{\tempd^2}\,.
            \end{equation}
            Note that by Assumption~\ref{assume:Q1notsmall}, $\tempd \lesssim 1$.
            Suppose $\xi > 0$ satisfies
            \begin{equation}
                R^{\Theta(1/\Siglbd)} / e^{\Siglbd/(m\Siginc)} \ll \xi \ll \min(R^{-\Theta(1/\Siglbd^2)},1/\norm{\Sig_1}_F) \label{eq:xiupper}
            \end{equation}
            Note that by Eq.~\eqref{eq:lambound} in Assumption~\ref{assume:sig_orth}, there is a nonempty range of $\xi$ satisfying both bounds.

            Suppose that $\wh{\bW} \in \R^{d\times d}$ is given by $\wh{\bW} = \sum^m_{i=1} \bW_i + \Delta$, where 
            \begin{equation}
                \norm{\Delta}_F \lesssim \bigl(m\Siglbd^{-1}\log(R) + m\log(1/\xi) + \log(d/\delta) + \log\log(\norm{\Sig_1}_F / \epsilon)\bigr)^{-1/2} \cdot \epsilon\,.
            \end{equation}
            Let $K$ be the convex body defined in Eq.~\eqref{eq:Kdef} as the result of running Algorithm~\ref{alg:lp} on input $\wh{\bW}$ for
            \begin{equation}
                T = \Theta(1/\xi)^m \cdot R^{\Theta(m/\Siglbd)}\cdot (d^2\log(d\norm{\Sig_1}_F/\epsilon) + \log(1/\delta))
            \end{equation}
            time and samples. Then with probability at least $1 - O(\delta)$,

            \begin{enumerate}[label=(\Roman*)]
                \item $K$ is an $(O(\epsilon^*), O(\epsilon/d))$-tight enclosure of $\Sig_1,\ldots,\Sig_m$ for 
                \begin{equation}
                    \epsilon^* = \max\Bigl(\epsilon , m\log((ke^{\tempd^2})^{\Theta(1/\Siglbd)}/\xi)\cdot\Bigl(\frac{1}{2} + \frac{1}{\Siglbd}\Bigr)\cdot \xi\norm{\Sig_1}_F\Bigr)\,.    
                \end{equation}
                \label{item:istight}
                \item Let
                    \begin{equation}
                        \Sig^* \triangleq \frac{1}{Z}\sum^m_{i=1} \frac{1}{\norm{\Sig_i}^2_F} \cdot \Sig_i  \ \ \ \text{for} \ \ \ Z \defeq \sum^m_{i=1} \frac{1}{\norm{\Sig_i}^2_F}\,,
                    \end{equation}
                    and let $\wh{\Sig}^*$ denote the minimum norm point in the convex body $K$ defined in Eq.~\eqref{eq:Kdef}. Then provided that $\epsilon \le O(\Siglbd\norm{\Sig_1}_F)$,
                    \begin{equation}
                        \norm{\Sig^* - \wh{\Sig}^*}_F \lesssim \sqrt{\frac{\epsilon\norm{\Sig_1}_F}{\Siglbd}} + \frac{\kappa^{1/2}m^{1/4}}{\Siglbd^{3/4}}\cdot \norm{\Sig_1}_F\,. \label{eq:closetoqstar}
                    \end{equation}
                    \label{item:convex}
            \end{enumerate} 
        \end{theorem}

        \begin{proof}
            (Proof of Item~\ref{item:istight}):
                The first two conditions in Definition~\ref{def:tighthull} are immediately satisfied by definition of $K$. For the third and fourth conditions, first note that because we have chosen $\axwprob$ in Eq.~\eqref{eq:axwprob_def} to be $O(\delta/T)$, by a union bound the probability that all of the $\bX$'s encountered in {\sc LPCertify} satisfy event $\calE$ is at least $1 - O(\delta)$. Hencefroth condition on this event.

                In that case, the fourth condition is satisfied because of Lemma~\ref{lem:convexhullfeasible}. For the third condition, we would like to apply Corollary~\ref{cor:farK_infeasible}. Observe that $\xi$ satisfies the bounds Eq.~\eqref{eq:Deltabound} and Eq.~\eqref{eq:Deltabound2} in addition to the condition that $\xi \ll 1/\norm{\Sig_1}_F$, and together with our choice of $\tempd$, it satisfies Eq.~\eqref{eq:xiprime_assume}. Additionally, by the rightmost upper bound on $\xi$ in Eq.~\eqref{eq:xiupper}, we additionally have that $\epsilon^*$ satisfies Eq.~\eqref{eq:epsbound}. The conditions for Lemma~\ref{lem:constraint_ruleout} and thus Corollary~\ref{cor:farK_infeasible} are met, so in the claimed number of iterations $T$, the resulting convex body $K$ will not contain any $\Sig$ which is at least $\epsilon^*$-far in Frobenius distance from $\affinehull$. Note that any $\Sig$ which is at least $\epsilon^*$ far from the \emph{span} of $\Sig_1,\ldots,\Sig_m$ will certainly be $\epsilon^*$-far from the affine hull, so the third condition of Definition~\ref{def:tighthull} is satisfied.

            (Proof of Item~\ref{item:convex}):
                We will prove that
                \begin{equation}
                    \norm{\Sig^* - \wh{\Sig}^*}_F \le \epsilon + O\Bigl(\sqrt{\frac{\epsilon\norm{\Sig_1}_F}{\Siglbd} + \Theta\Bigl(\frac{\kappa\sqrt{m}}{\Siglbd^{3/2}}\Bigr)\cdot\norm{\Sig_1}^2}\Bigr)\,. \label{eq:closetoqstar2}
                \end{equation}
                Note that the claimed bound follows as $\epsilon \lesssim \Siglbd\,\norm{\Sig_1}_F$.
                Let $\Sig$ denote the point in $\affinehull$ closest to $\wh{\Sig}^*$, and define $\eta \triangleq \frac{3\epsilon^*}{\Siglbd\norm{\Sig_1}_F}$. By Item~\ref{item:istight} of the Lemma,
                \begin{equation}
                    \norm{\wh{\Sig}^* - \Sig}_F \le \epsilon^*\,.
                \end{equation}
                Suppose to the contrary that Eq.~\eqref{eq:closetoqstar2} does not hold. Then by triangle inequality,
                \begin{equation}
                    \norm{\Sig^* - \Sig}^2_F > \eta\norm{\Sig_1}^2_F + \Theta\Bigl(\frac{\kappa\sqrt{m}}{\Siglbd^{3/2}}\Bigr)\cdot\norm{v_1}^2\,.
                \end{equation}
                But by the contrapositive of Lemma~\ref{lem:affinehullnorm} below applied to $\Sig_1,\ldots,\Sig_m$, this implies that
                \begin{equation}
                    \norm{\Sig}_F > \sqrt{1 + \eta}\,\norm{\Sig^*}_F \ge \norm{\Sig^*}_F + \frac{\eta}{3}\,\norm{\Sig^*}_F \ge \norm{\Sig^*}_F + \epsilon^*
                \end{equation}
                and thus by triangle inequality that
                \begin{equation}
                    \norm{\wh{\Sig}^*}_F > \norm{\Sig^*}_F\,.
                \end{equation}
                But by Lemma~\ref{lem:convexhullfeasible}, $\Sig^*\in K$, so this contradicts the minimality of the norm of $\wh{\Sig}^*$.
        \end{proof}

        \noindent The proof of Item~\ref{item:convex} of Theorem~\ref{thm:main_lp} relied on the following elementary bound:

        \begin{lemma}\label{lem:affinehullnorm}
            Given vectors $v_1,\ldots,v_m\in\R^D$ for which $|\iprod{v_i,v_j}| \le \kappa\norm{v_i}\norm{v_j}$ and $\norm{v_1}\ge \cdots \ge\norm{v_m} \ge \upalpha\norm{v_1}$ for $\upalpha \ge 2m\kappa$, let
            \begin{equation}
                v^* \triangleq \frac{1}{Z}\sum^m_{i=1} \frac{1}{\norm{v_i}^2} \cdot v_i \ \ \ \text{for} \ \ \ Z \triangleq \sum^m_{i=1} \frac{1}{\norm{v_i}^2}\,.
            \end{equation}
            Then for any $0 < \eta < 1$ and any point $v = \sum_i \lambda_i v_i$ for which $\sum_i \lambda_i = 1$ and $\norm{v}^2 \le (1 + \eta)\,\norm{v^*}^2$, we have that 
            \begin{equation}
                \norm{v - v^*}^2 \le \Bigl(\eta + O\Bigl(\frac{\kappa\sqrt{m}}{\alpha^{3/2}}\Bigr)\Bigr)\cdot \norm{v_1}^2
            \end{equation}
        \end{lemma}

        \begin{proof}
            Without loss of generality suppose that $\norm{v_1} \ge \cdots \ge \norm{v_m}$. 
            We have
            \begin{align}
                \MoveEqLeft \norm{v}^2 - \norm{v^*}^2 - \norm{v - v^*}^2 = \sum_{i,j} (\delta_i\lambda^*_j + \delta_j\lambda^*_j)\cdot \iprod{v_i, v_j} - \sum_i \delta_i^2 \norm{v_i}^2 \\
                &= \sum_i \delta_i \lambda^*_i\norm{v_i}^2 + \sum_{i\neq j} (\delta_i \lambda^*_j + \delta_j\lambda^*_j) \cdot \iprod{v_i,v_j} \\
                &\ge -\kappa\sum_{i\neq j} \Bigl(\frac{|\delta_i|}{Z\norm{v_j}^2} + \frac{|\delta_j|}{Z\norm{v_i}^2}\Bigr)\cdot \norm{v_i}\cdot \norm{v_j} \\
                &\ge -\frac{2\kappa}{Z}\Bigl(\sum_i |\delta_i|\cdot\norm{v_i}\Bigr)\Bigl(\sum_i \frac{1}{\norm{v_j}}\Bigr) \ge -\frac{2m\kappa\norm{\delta}_1}{Z\upalpha} \ge -\frac{2\kappa\norm{\delta}_1\cdot \norm{v_1}^2}{\alpha}\,.
            \end{align}
            As $\norm{v}^2 \le (1+\eta)\,\norm{v^*}^2$ by assumption, this implies that
            \begin{equation}
                \norm{v - v^*}^2 \le \eta\norm{v^*}^2 + \frac{2\kappa\norm{\delta}_1\cdot \norm{v_1}^2}{\alpha}\,. \label{eq:vvprime}
            \end{equation}
            It remains to bound $\norm{\vec{\delta}}_1$, for which we will use the assumption that $\norm{v}^2 \le 2\norm{v^*}^2$. Let $M \triangleq \sum_i v_i v_i^\top$. The minimum non-zero singular value of $M$ is the same as that of the Gram matrix $VV^\top$, where $V$ is the matrix whose rows are $v_1,\ldots,v_m$. The off-diagonal entries of $VV^\top$, by assumption, are bounded by $\kappa\norm{v_1}^2$, while the diagonal entries are lower bounded by $\upalpha\norm{v_1}^2$. By the Gershgorin circle theorem, this implies that $\sigma_{\min}(VV^\top) \ge (\upalpha - m\kappa)\norm{v_1}^2 \ge \frac{1}{2}\upalpha\norm{v_1}^2$.

            So $\frac{1}{2}\upalpha\norm{\vec{\lambda}}^2 \cdot \norm{v_1}^2 \le \norm{\sum_i \lambda_i v_i}^2 \le 4\norm{v^*}^2$, implying that $\norm{\vec{\lambda}} \lesssim \frac{\norm{v^*}}{\norm{v_1}\cdot \sqrt{\upalpha}} \le 1/\sqrt{\upalpha}$ and thus that 
            \begin{equation}
                \norm{\vec{\delta}}_1 \le \norm{\vec{\lambda}}_1 + \norm{\vec{\lambda^*}}_1 \le \sqrt{m}\norm{\vec{\lambda}}_2 + 1 \le \sqrt{m/\upalpha} + 1 \lesssim \sqrt{m/\upalpha}\,.
            \end{equation}
            Substituting this into Eq.~\eqref{eq:vvprime} yields the desired bound.

        \end{proof}



\section{Average attention matrix as a proxy}
\label{sec:proxy}

Suppose we have access to a matrix close to the one defined in Eq.~\eqref{eq:wtQdef}, concretely a matrix of the form
\begin{equation}
	\wt{\Sig} = \frac{1}{Z}\sum^m_{i=1} \frac{1}{\norm{\Sig_i}^2_F} \cdot \Sig_i + \calE  \ \ \ \text{for} \ \ \ Z \defeq \sum^m_{i=1} \frac{1}{\norm{\Sig_i}^2_F}
\end{equation}
for
\begin{equation}
	\|\calE \|_F \lesssim \upmu\,\norm{\Sig_1}_F\,.
\end{equation}
for some small $\upmu > 0$.

We will now show how to use $\wt{\Sig}$ to produce a very good estimate for $\sum_i \bW_i$, that is, one that which achieves arbitrarily small inverse polynomial error.

The idea will be to draw many samples $(\bX,\bY)$ until we find several for which the attention patterns $\softmax(\bX\Sig_i\bX^\top)_{1:}$ for all $i\in[m]$ are extremely close to $e_2$, the second standard basis vector. Then $\bY_1 \approx \bX_{2:}\sum_i \bW_i$, and we can use linear regression to find an approximation to $\sum_i \bW_i$. The key issue is that without $\wt{\Sig}$, using the methods in Section~\ref{sec:cert} we can at best certify for a given $\bX$ whether $\softmax(\bX\Sig_i\bX^\top)_{1:}$ is $\epsilon_{\sf 1}$-close to $e_2$ for $\epsilon_{\sf 1}$ given by the error incurred in the previous stages of the learning algorithm. To improve upon this error bound, we will use $\softmax(\bX\wt{\Sig}\bX^\top)$, to which we have exact access, as a proxy for each $\softmax(\bX\Sig_i\bX^\top)$. Our main claim in this section is the following:

\begin{lemma}\label{lem:proxy}
	Let $\omega >0$ be a parameter satisfying
	\begin{equation}
		1 \le \omega \ll \Siglbd^{7/4} \sqrt{\frac{m\Siginc\reffsig}{\log d}} \,. \label{eq:omegabounds}	
	\end{equation}
	Given access to a matrix $\wt{\Sig}$ for which
	\begin{equation}
		\Bigl\|\wt{\Sig} - \frac{1}{Z}\sum^m_{i=1} \frac{1}{\norm{\Sig_i}^2_F} \cdot \Sig_i\Bigr\| \le \upmu\,\norm{\Sig_1}_F  \ \ \ \mathrm{for} \ \ \ Z \defeq \sum^m_{i=1} \frac{1}{\norm{\Sig_i}^2_F},
	\end{equation}
	for \begin{equation}
		\upmu \ll \frac{\Siglbd^{5/2}}{\omega^3 m^3 \log^2 d}\,, \label{eq:muassume}
	\end{equation}
	there is a procedure ({\sc LeastSquaresRefine}($\wt{\Sig},\omega$)) that, with probability at least $1 - 1/\poly(d)$, draws $N = d^{O(\omega^2 m/ \Siglbd^4)}$ samples, runs some $\poly(N,d)$-time computation, and outputs $\wh{\bW}\in\R^{d\times d}$ for which
	\begin{equation}
		\Bigl\|\wh{\bW} - \sum_i \bW_i\Bigr\|_F \le 2mdk^{3/2} \exp(-\omega\sqrt{\log d}\norm{\Sig_1}_F + O(\omega^2))\,. \label{eq:refinedsumbound}
	\end{equation}
\end{lemma}

\begin{algorithm2e}
\DontPrintSemicolon
\caption{\textsc{LeastSquaresRefine}($\wt{\Sig}$)}
\label{alg:waittorefine}
	\KwIn{Estimate $\wt{\Sig}$ for convex combination of $\Sig_i$'s, parameter $\omega > 0$ satisfying Eq.~\eqref{eq:omegabounds}}
	\KwOut{Estimate for $\sum_i W_i$ satisfying Eq.~\eqref{eq:refinedsumbound}}
	\For{$i\in[N]$}{
		$C\gets \Theta(\omega \sqrt{m\log d}/\Siglbd^2)$\;
		Continue drawing samples until one encounters $(\bX,\bY)$ satisfying $\bX_{1:}\wt{\Sig}(\bX_{2:})^\top > C/\sqrt{Z}$.\;
		Denote this example by $(\bX^{(i)}, \bY^{(i)})$.\;
	}
	$\wh{\bW} \gets \min_{\wh{\bW}\in\R^{d\times d}} \Bigl(\sum^N_{i=1} \norm{\bY^{(i)}_{1:} - \bX^{(i)}_{2:}\wh{\bW}}^2\Bigr)^{1/2}$.\; \label{step:leastsquares}
	\Return{$\wh{\bW}$}.\;
\end{algorithm2e}

\subsection{Proxy is accurate}

Define \begin{equation}
	C \triangleq \gamma_1\omega \sqrt{m\log d} / \Siglbd^2 \qquad \sfa \triangleq \gamma_2 \omega \sqrt{\log(mkd)} \qquad \alpha \triangleq \Siglbd^2/4 \qquad \tau \triangleq (\projtrace + \sfa)\,\norm{\Sig_1}_F \qquad \beta \asymp \frac{m\Siginc}{\sqrt{\Siglbd}}\,, \label{eq:C_and_a}
\end{equation}
for sufficiently large absolute constants $\gamma_1 \ge \gamma_2 > 0$, so that $C^2 \ll \beta m \reffsig$ by Eq.~\eqref{eq:omegabounds}.
Consider
\begin{equation}
	\text{Event} \ \calA:  	\bX_{1:} \wt{\Sig} (\bX_{2:})^\top > C / \sqrt{Z} \,. \label{eq:Adef}
\end{equation}

\begin{lemma}\label{lem:condprobbound}
	For any $\gamma > 0$, there is a choice of the constants $\gamma_1, \gamma_2$ in Eq.~\eqref{eq:C_and_a} such that under event $\calA$ and for any $\omega > 1$, the conditional probability that
	\begin{equation}
		\norm{\softmax(\bX_{1:}\Sig_{i^*}\bX^\top) - e_2}_2 \le k\exp(-\omega\sqrt{\log d}\,\norm{\Sig_1}_F)
	\end{equation}
	is at least $1 - 1/d^\gamma$.
\end{lemma}

\noindent To prove Lemma~\ref{lem:condprobbound}, we bound the conditional probability that either of the following bad events occurs:
\begin{itemize}
	\item (Second entry) $|\bX_{1:} \Sig_{i^*} (\bX_{2:})^\top| < \alpha C / \sqrt{Z}$ for some $i^*\in[m]$.
	\item (Other entries) $\bX_{1:} \Sig_{i^*} (\bX_{a:})^\top > \tau$ for some $i^*\in[m]$ and $a\neq 2$.
\end{itemize}

Before we prove this in the subsequent sections, we verify that this would imply the claimed bound in Lemma~\ref{lem:condprobbound}. Indeed, note that by our choice of parameters,
\begin{equation}
	\frac{\alpha C}{\sqrt{Z}} - \tau = (\gamma_1\omega\sqrt{\log d} - \projtrace - \gamma_2\omega\sqrt{\log(mkd)})\,\norm{\Sig_1}_F \ge \gamma_3\omega\sqrt{\log d}\,\norm{\Sig_1}_F
\end{equation}
for some new absolute constant $\gamma_3 > 0$, where the last step follows by our assumption from Assumption~\ref{assume:trace} that $\projtrace \le \sqrt{\log d}$, that $\gamma_1$ is a constant which is sufficiently large relative to $\gamma_2$, and that $m,k\ll d$ so that $\log(mkd) \lesssim \log d$.

So if none of the above bad events happen, then for every $i^*\in[m]$, we have by Lemma~\ref{lem:softmaxmargin} that
\begin{equation}
	\norm{\softmax(\bX_{1:}\Sig_{i^*}(\bX_{2:})^\top) - e_2}_2 \le \frac{k - 1}{k - 1 + \exp(\gamma_3\omega\sqrt{\log d}\,\norm{\Sig_1}_F)}\,,
\end{equation}
which implies the claimed bound in Lemma~\ref{lem:condprobbound}, provided we take $\gamma_1,\gamma_2$ large enough that we can take $\gamma_3 = 1$.

\subsubsection{Second entry}

Note that
\begin{equation}
	\Pr{|\bX_{1:} \Sig_{i^*} (\bX_{2:})^\top| < \alpha C \mid \calA} = \Pr{\calA}^{-1} \cdot \Pr{|\bX_{1:} \Sig_{i^*} (\bX_{2:})^\top)| < \alpha C \ \wedge \ \calA} \le \Pr{\calA}^{-1} \cdot \Pr{\calA'}
\end{equation}
for 
\begin{equation}
	\text{Event} \ \calA': \bX_{1:}\Sig'(\bX_{2:})^\top > \frac{C}{\sqrt{Z}}\Bigl(1 - \frac{\alpha}{Z\norm{\Sig_{i^*}}^2_F}\Bigr) \ \ \ \text{for} \ \ \ \Sig'\triangleq \frac{1}{Z}\sum_{i\neq i^*} \frac{1}{\norm{\Sig_i}^2_F} \Sig_i + \calE
\end{equation}

\paragraph{Lower bound on $\mathbb{P}[\calA]$} 

We will use Theorem~\ref{thm:kolmogorov} to lower bound $\Pr{\calA}$:

\begin{lemma}\label{lem:PrAbound}
	For $\eta \asymp \frac{m\Siginc + \upmu\sqrt{m}}{\sqrt{\Siglbd}}$, we have \begin{equation}
		\Pr{\calA} \ge \exp\Bigl(-\frac{C^2}{2(1-\eta)^2}\Bigr)\,.
	\end{equation}
\end{lemma} 

\begin{proof}	
	We first establish that the vector $\bX_{1:}\wt{\Sig}$ has norm close to $\norm{\wt{\Sig}}_F$ and entries which are not too large. For the first part, we must lower bound $\norm{\wt{\Sig}}^2_F / \norm{\wt{\Sig}}^2_{\sf op}$. Note that
	\begin{align}
		\frac{1}{Z^2}\Bigl\|\sum^m_{i=1} \frac{1}{\norm{\Sig_i}^2_F}\cdot \Sig_i\Bigr\|^2_F &= \frac{1}{Z^2} \sum^m_{i=1} \frac{1}{\norm{\Sig_i}^2_F} \pm \frac{1}{Z^2} \sum_{i\neq i'} \frac{\Siginc}{\norm{\Sig_i}_F \cdot \norm{\Sig_{i'}}_F} \\
		&= \Bigl(1 \pm \frac{m\Siginc}{2\sqrt{\Siglbd}}\Bigr) \cdot \frac{1}{Z}\,.
	\end{align}
	As $\upmu\,\norm{\Sig_1}_F \le \upmu\sqrt{m/(\Siglbd Z)}$, by triangle inequality this implies that
	\begin{equation}
		\norm{\wt{\Sig}}_F = \Bigl(1 \pm O\Bigl(\frac{m\Siginc + \upmu\sqrt{m}}{\sqrt{\Siglbd}}\Bigr)\Bigr)\cdot\frac{1}{\sqrt{Z}} \,, \label{eq:wtQnormF}
	\end{equation}
	so in particular
	\begin{equation} \\
		\norm{\wt{\Sig}}_F \gtrsim 1/\sqrt{Z}
	\end{equation}
	as $\upmu \ll \sqrt{\Siglbd/m}$ by Eq.~\eqref{eq:muassume} and the assumption that $m\kappa \ll \Siglbd \le \sqrt{\Siglbd}$ in Assumption~\ref{assume:sig_orth}.

	Additionally, we have
	\begin{align}
		\norm{\wt{\Sig}}_{\sf op} \le |\iprod{\wt{\Sig}, \Sig_1 / \norm{\Sig_1}_{\sf tr}}| \le \Bigl(\frac{1 + m\kappa/\sqrt{\Siglbd}}{Z} + \upmu\Bigr) \norm{\Sig_1}^{-1}_{\sf tr} \lesssim \frac{1}{Z}\norm{\Sig_1}^{-1}_{\sf tr}\,,
	\end{align}
	where the last step follows by H\"{o}lder's inequality and Eq.~\eqref{eq:muassume}. We conclude that
	\begin{equation}
		\frac{\norm{\wt{\Sig}}^2_F}{\norm{\wt{\Sig}}^2_{\sf op}} \ge Z\,\norm{\Sig_1}^2_{\sf tr} \ge \frac{m\norm{\Sig_1}^2_{\sf tr}}{\norm{\Sig_1}_F^2}\ge m\reffsig\,.
	\end{equation}
	By Theorem~\ref{thm:hanson_wright}, we conclude that for $\beta$ defined in Eq.~\eqref{eq:C_and_a},
	we have
	\begin{equation}
		\norm{\bX_{1:}\wt{\Sig}}^2 = (1\pm \beta)\,\norm{\wt{\Sig}}^2_F
	\end{equation}
	with probability at least $1 - \exp(-\Omega(\beta m\reffsig))$ over the randomness of $\bX_{1:}$. So under this event, combining with Eq.~\eqref{eq:wtQnormF}, we have
	\begin{equation}
		 \Bigl|\norm{\bX_{1:}\wt{\Sig}} - \frac{1}{\sqrt{Z}} \Bigr| \le \frac{\eta}{\sqrt{Z}}\,, \label{eq:xQnorm2}
	\end{equation}
	where $\eta$ is defined in the statement of Lemma~\ref{lem:PrAbound}.

	Next, we show that the entries of $\bX_{1:}\wt{\Sig}$ are not too large. Take any $j\in[d]$ and consider the norm of the $j$-th column of $\wt{\Sig}$. We have
	\begin{equation}
		\norm{\wt{\Sig}_{:j}} \le \max_{i\in[m]} \norm{(\Sig_i)_{:j}} + \norm{\calE_{:j}} \le \Bigl(\frac{\Siglightrows}{\sqrt{d}} + \upmu\Bigr)\,\norm{\Sig_1}_F
	\end{equation}
	By Hoeffding's inequality, for any $j\in[d]$ we have
	\begin{equation}
		\mathrm{Pr}\Bigl[|(\bX_{1:}\wt{\Sig})_j| > t\Bigl(\frac{\Siglightrows}{\sqrt{d}} + \upmu\Bigr)\,\norm{\Sig_1}_F\Bigr] \le \exp(-t^2 / 2)\,,
	\end{equation}
	so by a union bound, we conclude that with arbitrarily small constant failure probability,
	\begin{equation}
		\norm{\bX_{1:}\wt{\Sig}}_\infty \le {O}\Bigl(\frac{\Siglightrows}{\sqrt{d}} + \upmu\Bigr)\,\norm{\Sig_1}_F \triangleq M\,. \label{eq:xQnorminf}
	\end{equation}
	Condition on $\bX_{1:}$ for which the events of Eq.~\eqref{eq:xQnorm2} and~\eqref{eq:xQnorminf} hold. We can then apply Theorem~\ref{thm:kolmogorov} to the random variables $(\bX_{1:}\wt{\Sig})_1\cdot \bX_{21}, \ldots, \ldots, (\bX_{1:}\wt{\Sig})_d\cdot \bX_{2d}$. We can take $\sigma$, $M$, and $t$ in the Theorem to be $\norm{\bX_{1:}\wt{\Sig}} \le \frac{1}{\sqrt{Z}}(1 + \eta)$, $M$ defined in Eq.~\eqref{eq:xQnorminf}, and $\frac{C}{1 - \eta}$ respectively. Then by Theorem~\ref{thm:kolmogorov}, recalling that $\Phi(z) \triangleq \E[g\sim\calN(0,1)]{g > z}$,
	\begin{align}
		\MoveEqLeft \Pr{\calA \mid \bX_{1:} \ \text{satisfies Eqs.~\eqref{eq:xQnorm2} and \eqref{eq:xQnorminf}}} \\
		&\ge \Bigl[1 - \Phi\Bigl(\frac{C}{1-\eta}\Bigr)\Bigr] \cdot \Bigl[1 - O\Bigl(\frac{MC\sqrt{Z}}{1 - \eta^2}\Bigr)\Bigr] \cdot \exp\Bigl(-O\Bigl(\frac{MC^3\sqrt{Z}}{(1 - \eta^2)(1 - \eta)^2}\Bigr)\Bigr) \label{eq:Alower} \\
		&\ge \Bigl[1 - \Phi\Bigl(\frac{C}{1 - \eta}\Bigr)\Bigr] \cdot \Bigl[1 - {O}\Bigl(\frac{\Siglightrows}{\sqrt{d}} + \upmu\Bigr)\cdot C\sqrt{m/\Siglbd}\Bigr]\cdot \exp\Bigl\{-{O}\Bigl(\frac{\Siglightrows}{\sqrt{d}} + \upmu\Bigr)\cdot C^3\sqrt{m/\Siglbd}\Bigr\} \\
		&\gtrsim 1 - \Phi\Bigl(\frac{C}{1 - \eta}\Bigr)\,, 
	\end{align}
	where in the last step we used that $\frac{\upsilon}{\sqrt{d}}, \upmu \ll \frac{1}{C^3\sqrt{m/\Siglbd}}$ by Eq.~\eqref{eq:muassume}.

	As Eqs.~\eqref{eq:xQnorm2} and~\eqref{eq:xQnorminf} are satisfied with probability at least $1 - \exp(-\Omega(\beta m \reffsig)) - o(1) = \Omega(1)$, we conclude that
	\begin{equation}
		\Pr{\calA} \gtrsim 1 - \Phi\Bigl(\frac{C}{1 - \eta}\Bigr)\ge \exp\Bigl(-\frac{C^2}{2(1-\eta)^2}\Bigr)
	\end{equation}
	as claimed.
\end{proof}

\paragraph{Upper bound on $\mathbb{P}[\calA']$} 

We will use Hoeffding's inequality to upper bound $\Pr{\calA'}$:
\begin{lemma}\label{lem:PrAprimebound}
	For $\eta' \asymp \frac{m\Siginc + \upmu\sqrt{m + 1/\Siglbd}}{\sqrt{\Siglbd}}$, we have
	\begin{equation}
		\Pr{\calA'} \le \exp\Bigl(-\frac{C^2(1+\Siglbd/2m)}{2(1 + \eta')^2}\Bigr) + \exp(-\Omega(\beta m \reffsig))\,,
	\end{equation}
	where $\beta$ is defined in Eq.~\eqref{eq:C_and_a}.
\end{lemma}

\begin{proof}
	As above, we first establish that the vector $\bX_{1:}\Sig'$ has norm close to $\norm{\Sig'}_F$ using Theorem~\ref{thm:hanson_wright}, which requires lower bounding $\norm{\Sig'}^2_F / \norm{\Sig'}^2_{\sf op}$.

	Note that
	\begin{align}
		\frac{1}{Z^2}\Bigl\|\sum_{i\neq i^*} \frac{1}{\norm{\Sig_i}^2_F}\cdot \Sig_i\Bigr\|^2_F &= \frac{1}{Z^2} \sum_{i: i\neq i^*} \frac{1}{\norm{\Sig_i}^2_F} \pm \frac{1}{Z^2} \sum_{i,i': i\neq i' \neq i^*} \frac{\Siginc}{\norm{\Sig_i}_F \cdot \norm{\Sig_{i'}}_F} \\
		&= \Bigl(1 \pm \frac{m\Siginc}{2\sqrt{\Siglbd}}\Bigr) \cdot \frac{1}{Z^2} \cdot \sum_{i: i\neq i^*} \frac{1}{\norm{\Sig_i}^2_F}\,.
	\end{align}
	As $\upmu\,\norm{\Sig_1}_F \le \upmu\sqrt{m/(\Siglbd Z)}$ and
	\begin{equation}
		\frac{1}{\sqrt{Z}} \le \frac{1}{Z} \cdot \sqrt{\sum_{i: i\neq i^*} \frac{1}{\norm{\Sig_i}^2_F}}\cdot \sqrt{1 + \frac{1}{(m - 1)\Siglbd}}\,,
	\end{equation}
	by triangle inequality we get that
	\begin{equation}
		\norm{\Sig'}_F = \Bigl(1 \pm O\Bigl(\frac{m\Siginc + \upmu\sqrt{m + 1/\Siglbd}}{\sqrt{\Siglbd}}\Bigr)\Bigr)\cdot\frac{1}{Z} \sqrt{\sum_{i: i\neq i^*} \frac{1}{\norm{\Sig_i}^2_F}} \,, \label{eq:QprimenormF}
	\end{equation}
	so in particular
	\begin{equation} \\
		\norm{\Sig'}_F \gtrsim \frac{1}{Z} \sqrt{\sum_{i: i\neq i^*} \frac{1}{\norm{\Sig_i}^2_F}}
	\end{equation}
	by the assumption that $m\Siginc \ll \Siglbd \le \sqrt{\Siglbd}$ in Assumption~\ref{assume:sig_orth} and the fact that $\upmu \ll \sqrt{\frac{\Siglbd}{m + 1/\Siglbd}}$ by Eq.~\eqref{eq:muassume}. 

	Next, we upper bound $\norm{\Sig'}_{\sf op}$. If $i^* = 1$, then define $j^* = 2$, otherwise, define $j^* = 1$. Then
	\begin{align}
		\norm{\Sig'}_{\sf op} \le |\iprod{\Sig', \Sig_{j^*} / \norm{\Sig_{j^*}}_{\sf tr}}| \le \Bigl(\frac{1 + m\kappa/\sqrt{\Siglbd}}{Z} + \upmu\Bigr)\norm{\Sig_{j^*}}^{-1}_{\sf tr} \lesssim \frac{1}{Z}\norm{\Sig_{j^*}}^{-1}_{\sf tr}\,,
	\end{align}
	where the last step follows by H\"{o}lder's inequality and Eq.~\eqref{eq:muassume}. We conclude that
	\begin{equation}
		\frac{\norm{\Sig'}^2_F}{\norm{\Sig'}^2_{\sf op}} \ge \frac{m\norm{\Sig_{j^*}}^2_{\sf tr}}{\norm{\Sig_{j^*}}_F^2}\ge m\reffsig\,.
	\end{equation}
	By Theorem~\ref{thm:hanson_wright}, we conclude that for $\beta$ as in Eq.~\eqref{eq:C_and_a},
	\begin{equation}
		\norm{\bX_{1:}\Sig'}^2 = (1\pm \beta)\,\norm{\Sig'}^2_F
	\end{equation}
	with probability at least $1 - \exp(-\Omega(\beta m\reffsig))$ over the randomness of $\bX_{1:}$. We will take $\beta \ll 1$, so under this event, combining with Eq.~\eqref{eq:wtQnormF}, we have
	\begin{equation}
		 \norm{\bX_{1:}\Sig'} = \frac{(1\pm \eta')}{Z}\sqrt{\sum_{i: i\neq i^*} \frac{1}{\norm{\Sig_i}^2_F}}\,, \label{eq:xQnorm2prime}
	\end{equation}
	where $\eta'$ is given in the statement of Lemma~\ref{lem:PrAprimebound}.

	Condition on $\bX_{1:}$ for which the event Eq.~\eqref{eq:xQnorm2prime} holds. By Hoeffding's inequality, we conclude that 
	\begin{equation}
		\Pr{\calA' \mid \bX_{1:} \ \text{satisfies Eq.~\eqref{eq:xQnorm2prime}}} \le \exp\Bigl\{-\frac{C^2}{2(1 + \eta')^2}\cdot Z \Bigl(\sum_{i: i\neq i^*} \frac{1}{\norm{\Sig_i}^2_F}\Bigr)^{-1} \Bigl(1 - \frac{\alpha}{Z\,\norm{\Sig_{i^*}}^2_F}\Big)\Bigr\}
	\end{equation}
	As Eq.~\eqref{eq:xQnorm2prime} is satisfied with probability at least $1 - \exp(-\Omega(\beta m \reffsig))$, and as
	\begin{equation}
	 	Z \ge \Bigl(1 + \frac{\Siglbd}{m-1}\Bigr)\cdot \sum_{i: i\neq i^*} \frac{1}{\norm{\Sig_i}^2_F}
	 \end{equation} 
	 and
	 \begin{equation}
	 	Z\norm{\Sig_{i^*}}^2_F \ge m\Siglbd
	 \end{equation}
	 by triangle inequality we conclude that
	\begin{align}
		\Pr{\calA'} &\le \exp\Bigl\{-\frac{C^2}{2(1 + \eta')^2} \Bigl(1+\frac{\Siglbd}{m-1}\Bigr) \Bigl(1 - \frac{\alpha}{m\Siglbd}\Bigr)^2\Bigr\} + \exp(-\Omega(\beta m \reffsig))\,\\
		&\le \exp\Bigl(-\frac{C^2(1+\Siglbd/2m)}{2(1 + \eta')^2}\Bigr) + \exp(-\Omega(\beta m \reffsig))\,,
	\end{align}
	where in the last step we used our choice of $\alpha = \Siglbd^2/4$.
\end{proof}

\paragraph{Bounding the ratio}

It remains to take the ratio between the bounds in Lemmas~\ref{lem:PrAbound} and~\ref{lem:PrAprimebound} to obtain an upper bound on $\Pr{|\bX_{1:}\Sig_{i^*}(\bX_{2:})^\top| < \alpha C \mid \calA}$. Note that by Eq.~\eqref{eq:lambound} and Eq.~\eqref{eq:muassume},
\begin{equation}
	\eta \lesssim \eta' \ll \Siglbd/m\,.\label{eq:etabound}
\end{equation}
We have
\begin{align}
	\Pr{\calA'\mid \calA} &\le \Pr{\calA}^{-1}\cdot \Pr{\calA'} \\
	&\lesssim \exp\Bigl\{-\frac{C^2}{2} \Bigr[\frac{1+\Siglbd/2m}{(1 + \eta')^2} - \frac{1}{(1-\eta)^2}\Bigr]\Bigr\} + \exp(-\Omega(\beta m \reffsig)) \\
	&\le \exp(-\Omega(C^2\Siglbd/m)) + \exp(-\Omega(\beta m\reffsig))\,,
\end{align}
where in the second step we used that $C^2 \ll \beta\reffsig m$ by our assumed upper bound on $\omega$.

\subsubsection{First entry}

For any $i^*\in[m]$, define
\begin{equation}
	\text{Event} \ \calB_{i^*}: |\bX_{1:}\Sig_{i^*}(\bX_{1:})^\top| > \tau\,.
\end{equation}
We have
\begin{equation}
	\Pr{\calB_{i^*} | \calA} = \Pr{\calA}^{-1} \cdot \Pr{\calA \wedge \calB_{i^*}}\,.
\end{equation}

The following was already proved in the first Item of Lemma~\ref{lem:basic_highprob}; we record it again here for convenience, using our current notation:

\begin{lemma}\label{lem:diagbound}
	For any $i^*\in[m]$, $\Pr{\calB_{i^*}} \le \exp(-\Omega(\sfa^2))$.
\end{lemma}


Take $\beta$ according to Eq.~\eqref{eq:C_and_a}, and define
\begin{equation}
 	\text{Event} \ \calC: \norm{\bX_{1:} \wt{\Sig}}^2 \le (1 + \beta)\norm{\wt{\Sig}}^2_F\,. \label{eq:eventC}
\end{equation} 
Note that for any $\bX_{1:} \in \calC$, by Theorem~\ref{thm:hanson_wright} we have
\begin{equation}
	\Pr[\bX_{2:}\sim\brc{\pm 1}^d]{\calA \mid \bX_{1:}} \le \exp\Bigl(-\frac{C^2}{2(1 + \eta)^2}\Bigr)\,, \label{eq:useHW}
\end{equation}
where $\eta$ is defined in Eq.~\eqref{eq:xQnorm2}. Furthermore, denoting by $\overline{\calC}$ the complement of the event $\calC$, recall that earlier we had shown that
\begin{equation}
	\Pr[\bX_{1:}\sim\brc{\pm 1}^d]{\overline{\calC}} \le \exp(-\Omega(-\beta m \reffsig))\,. \label{eq:Ccomplement}
\end{equation}

We have
\begin{equation}
	\Pr{\calA\wedge \calB_{i^*}} = \Pr{\calA \wedge \calB_{i^*} \wedge \calC} + \Pr{\calA \wedge \calB_{i^*}\wedge \overline{\calC}}\,. \label{eq:bayes_decompose}
\end{equation}

We can write the first term on the right-hand side of Eq.~\eqref{eq:bayes_decompose} as 
\begin{equation}
	\Pr{\calB_{i^*}\wedge \calC} \cdot \Pr{\calA \mid \calB_{i^*}\wedge  \calC} \le \Pr{\calB_{i^*}}\cdot \Pr{\calA \mid \calB_{i^*}\wedge \calC} \le \exp(-\Omega(\sfa^2))\cdot\exp\Bigl(-\frac{C^2}{2(1 + \eta)^2}\Bigr)\,. \label{eq:ABC}
\end{equation}
using Lemma~\ref{lem:diagbound} and Eq.~\eqref{eq:useHW}.

For the second term on the right-hand side of Eq.~\eqref{eq:bayes_decompose}, we have
\begin{equation}
	\Pr{\calA\wedge \calB_{i^*}\wedge \overline{\calC}} \le \Pr{\overline{\calC}} \le \exp(-\Omega(\beta m\reffsig))\,.\label{eq:ABnotC}
\end{equation}

Combining Lemma~\ref{lem:PrAbound} and Eqs.~\eqref{eq:bayes_decompose}, \eqref{eq:ABC}, \eqref{eq:ABnotC}, and recalling that $C^2 \ll \beta m \reffsig$ by our assumed upper bound on $\omega$, we conclude that
\begin{equation}
	\Pr{\calB_{i^*}\mid \calA} \le \exp(-\Omega(\sfa^2))\cdot \exp(O(C^2\eta)) + \exp(-\Omega(\beta m \reffsig))
\end{equation}

\subsubsection{Remaining entries}

The argument is analogous to the one in the previous section, except in place of $\calB_{i^*}$, we define for any $a > 2$:
\begin{equation}
 	\text{Event} \ \calD_{a,i^*}: |\bX_{1:} \Sig_{i^*} (\bX_{a:})^\top| > \tau\,.
\end{equation}
We have
\begin{equation}
	\Pr{\calD_{i^*}\mid \calA} = \Pr{\calA}^{-1} \cdot \Pr{\calA \wedge \calD_{a,i^*}}\,.
\end{equation}
In analogy to Lemma~\ref{lem:diagbound}, we first upper bound $\Pr{\calD_{a,i^*}}$:
\begin{lemma}
	$\Pr{\calD_{a,i^*}} \le \exp(-\Omega(\sfa^2)) + m\exp(-\Omega(m\reffsig))$.
\end{lemma}

\begin{proof}
	By Theorem~\ref{thm:hanson_wright}, we have that
	\begin{equation}
		\norm{\bX_{1:}\Sig_{i^*}}^2 \lesssim \norm{\Sig_{i^*}}^2_F \ \ \text{for all} \ i^*\in[m]\,. \label{eq:X1Qis}
	\end{equation}
	with probability at least $1 - m\exp(-\Omega(m \reffsig)) = \Omega(1)$ over the randomness of $\bX_{1:}$. Conditioned on $\bX_{1:}$ satisfying this event, by Hoeffding's inequality
	\begin{equation}
		\Pr[\bX_{a:}\sim\brc{\pm 1}^d]{\calD_{a,i^*}\mid \bX_{1:} \ \text{satisfies Eqs.} \ \ref{eq:X1Qis}} \le \exp(-\Omega(\tau^2 / \norm{\Sig_{i^*}}^2_F)) \le \exp(-\Omega(\projtrace^2 + \sfa^2)) \le \exp(-\Omega(\sfa^2))\,.
	\end{equation}	
	The lemma follows by a union bound.
\end{proof}
 
\noindent By replacing $\calB_{i^*}$ in the preceding section with $\calD_{a,i^*}$, we conclude by the exact same logic as above that
 \begin{align}
 	\Pr{\calD_{a,i^*}\mid \calA} &\le (\exp(-\Omega(\sfa^2)) + m\exp(-\Omega(m\reffsig))) \cdot \exp(O(C^2\eta)) + \exp(-\Omega(\beta m\reffsig)) \\
 	&\le \exp(-\Omega(\sfa^2))\cdot \exp(O(C^2\eta)) + \exp(-\Omega(\beta m \reffsig))\,.
 \end{align}

\subsubsection{Combining the bounds}

\begin{proof}[Proof of Lemma~\ref{lem:condprobbound}]
	We conclude that conditioned on $\calA$, the conditional probability that there is some $i^*\in[m]$ for which $|\bX_{1:}\Sig_{i^*}(\bX_{2:})^\top| < \alpha C/\sqrt{Z}$ or some $i^*\in[m]$ and $a\neq 2$ for which $\bX_{1:}\Sig_{i^*}(\bX_{a:})^\top > \tau$ is upper bounded
	\begin{align}
		\MoveEqLeft\Pr{\calA'\mid\calA} + \sum_{i^*} \Pr{\calB_{i^*}\mid\calA} + \sum_{i^*} \sum_{a\neq 2} \Pr{\calD_{a,i^*}\mid \calA} \\
		&\le \exp(-\Omega(C^2\Siglbd/m)) + mk\exp(-\Omega(\sfa^2))\cdot \exp(O(C^2\eta)) + mk\exp(-\Omega(\beta m\reffsig)) \\
		&\le \exp(-\Omega(C^2\Siglbd/m)) + mk\exp(-\Omega(\sfa^2))\cdot \exp(O(C^2\eta))\,.
	\end{align}
	By taking $C$ and $\sfa$ as in Eq.~\eqref{eq:C_and_a} and recalling Eq.~\eqref{eq:reffsig_vs_logd}, we conclude that this conditional probability is $\le 1/\poly(d)$, where the degree of the polynomial can be made arbitrarily large by taking the constants in Eq.~\eqref{eq:C_and_a} sufficiently large and noting that $\omega > 1$.
\end{proof}

\subsection{Distortion of distribution over \texorpdfstring{$\bX_{2:}$}{X2}}

We will eventually need to prove that if we sample many $\bX$'s conditioned on $\calA$, the second rows $\bX_{2:}$ of these samples are sufficiently ``diverse'' that we can obtain a refined estimate of $\sum_i W_i$ using linear regression.

To that end, here we establish that conditioned on $\calA$, the resulting distribution on $\bX_{2:}$ is only a mild distortion of the uniform distribution over $\brc{\pm 1}^d$. Formally,

\begin{lemma}\label{lem:ratiobound}
	There is a subset $\Omega\subseteq\brc{\pm 1}^d$ such that $|\Omega|/2^d \le \exp(-\Omega(\beta m \reffsig)) + \exp(-\Omega(m\log d / \Siglbd^4))$, and such that for any $x,x'\not\in \Omega$,
	\begin{equation}
		\frac{\Pr{\bX_{2:} = x\mid \calA}}{\Pr{\bX_{2:} = x'\mid \calA}} \le \exp(O(\omega^2))
	\end{equation}
\end{lemma}

\begin{proof}
	Note that because without conditioning on $\calA$, $\bX_{2:}$ is uniform, we have
	\begin{equation}
		\frac{\Pr{\bX_{2:} = x\mid \calA}}{\Pr{\bX_{2:} = x'\mid \calA}} = \frac{\Pr{\calA\mid \bX_{2:} = x}}{\Pr{\calA\mid \bX_{2:} = x'}}\,. \label{eq:uniformprior}
	\end{equation}

	The ingredients for lower and upper bounding the denominator and numerator respectively were already established in the proofs above.

	Recall that for any $\bX_{1:}$ that satisfies the event $\calC$ defined in Eq.~\eqref{eq:eventC}, we have that Eq.~\eqref{eq:useHW} holds. This logic still holds if we reverse the roles of $\bX_{1:}$ and $\bX_{2:}$, so we conclude that provided 
	\begin{equation}
		\norm{\wt{\Sig}^\top x}^2 \le (1 + \beta)\norm{\wt{\Sig}}^2_F\,, \label{eq:Aubd_conditions}
	\end{equation}
	we have that
	\begin{equation}
		\Pr{\calA\mid \bX_{2:} = x} \le \exp\Bigl(-\frac{C^2}{2(1+\eta)^2}\Bigr)\,. \label{eq:Ax}
	\end{equation}
	Recall from Eq.~\eqref{eq:Ccomplement} that Eq.~\eqref{eq:Aubd_conditions} holds with probability at least $1 - \exp(-\Omega(\beta m\reffsig))$.

	Likewise, recall from the logic of Eq.~\eqref{eq:Alower} that for any $\bX_{1:}$ that satisfies Eqs.~\eqref{eq:xQnorm2} and~\eqref{eq:xQnorminf}, we have $\Pr{\calA\mid \bX_{1:}} \gtrsim 1 - \Phi(\frac{C}{1 - \eta}) \ge \exp(-\frac{C^2}{2(1-\eta)^2})$. This still holds if we reverse the roles of $\bX_{1:}$ and $\bX_{2:}$. We will need a slight modification of the bound in Eq.~\eqref{eq:xQnorminf}, as that bound only held with large constant probability. Instead, suppose that $x'\in\brc{\pm 1}^d$ satisfied
	\begin{equation}
		\Bigl|\norm{\wt{\Sig}^\top x'} - \frac{1}{\sqrt{Z}}\Bigr| \le \frac{\eta}{\sqrt{Z}} \qquad \text{and} \qquad \norm{\wt{\Sig}^\top x'}_\infty \le O\Bigl(\frac{\Siglightrows}{\sqrt{d}} +\upmu\Bigr)\cdot t\cdot \norm{\Sig_1}_F\,. \label{eq:Albd_conditions}
	\end{equation}
	for $t = O(\sqrt{m\log d}/\Siglbd^2)$.
	Then we have that
	\begin{align}
		\Pr{\calA \mid \bX_{2:} = x'} &\ge \Bigl[1 - \Phi\Bigl(\frac{C}{1 - \eta}\Bigr)\Bigr] \cdot \Bigl[1 - {O}\Bigl(\frac{\Siglightrows}{\sqrt{d}} + \upmu\Bigr)\cdot Ct\sqrt{m/\Siglbd}\Bigr]\cdot \exp\Bigl\{-{O}\Bigl(\frac{\Siglightrows}{\sqrt{d}} + \upmu\Bigr)\cdot C^3 t\sqrt{m/\Siglbd}\Bigr\} \label{eq:pre_Axp} \\
		&\gtrsim \exp\Bigl(-\frac{C^2}{2(1-\eta)^2}\Bigr)\,, \label{eq:Axp}
	\end{align}
	where we used Eq.~\eqref{eq:muassume} in the last step.
	
	Finally, recall from the discussion above Eq.~\eqref{eq:xQnorm2} that the former condition in Eq.~\eqref{eq:Albd_conditions} holds with probability at least $1 - \exp(-\Omega(\beta m\reffsig))$, and from the discussion above Eq.~\eqref{eq:xQnorminf} that the latter condition holds with probability at least $1 - \exp(-\Omega(t^2))$.

	Substituting Eq.~\eqref{eq:Ax} and Eq.~\eqref{eq:Axp} into Eq.~\eqref{eq:uniformprior}, we conclude that provided $x,x'$ satisfy Eqs.~\eqref{eq:Aubd_conditions} and~\eqref{eq:Albd_conditions}, we have
	\begin{equation}
		\frac{\Pr{\bX_{2:} = x\mid \calA}}{\Pr{\bX_{2:} = x'\mid \calA}} \lesssim \exp(\Omega(C^2 \eta)) \le \exp(O(\omega^2))\,,
	\end{equation}
	where in the last step we used Eq.~\eqref{eq:lambound} and our choice of $C,\eta$.
\end{proof}

\newcommand{\thres}{\uptau}
\subsection{Well-conditioned linear system}

Next, we show that provided that if collect many samples from the distribution over $\bX_{2:}$ conditioned on $\calA$ holding, then the matrix whose rows consist of these samples is well-conditioned.

\begin{lemma}\label{eq:sigmin_B}
	Suppose we draw $N$ i.i.d. samples from the distribution over $\bX_{2:}$ conditioned on the event $\calA$. Call these samples $z_1,\ldots,z_N \in \brc{\pm 1}^d$ and let $B\in\brc{\pm 1}^{N\times d}$ denote the matrix whose rows are these vectors. Then for $N = \Omega(d(\omega^2 + \log d))$,
	\begin{equation}
		\Pr{\sigma_{\min}(B) \gtrsim \sqrt{N}\cdot \exp(-\Omega(\omega^2))} \ge 1 - d^{-\Omega(d)}\,. 
	\end{equation}
\end{lemma}

\begin{proof}
	Let $\Omega$ denote the set of bad points from Lemma~\ref{lem:ratiobound}. Note that
	\begin{equation}
		\Pr{\bX_{2:} \in \Omega\mid\calA} \le \Pr{\bX_{2:}\in \Omega} \cdot \Pr{\calA}^{-1} \le \exp\Bigl\{-\Omega(\beta m \reffsig \vee m\log(d)/ \Siglbd^4) + O(m\log (d) / \Siglbd^4)\Bigr\} \ll 1\,.
	\end{equation}
	By Lemma~\ref{lem:ratiobound}, we conclude that $\Pr{\bX_{2:} = x \mid \calA} \ge \exp(-O(\omega^2))/2^d$ for all $x\not\in\Omega$.

	For any $w\in\S^{d-1}$, let $A_w$ denote the set of points $x\in\brc{\pm 1}^d$ for which $|\iprod{w,x}| \ge 1$. By Theorem~\ref{thm:booleananti}, $|A_w|/2^d > 3/32$, so $|A_w \backslash \Omega|/2^d \ge \exp(-O(\omega^2))$. So by Lemma~\ref{lem:PrAbound}, $\Pr{\bX_{2:} \in A_w \backslash \Omega \mid\calA} \ge \exp(-O(\omega^2))$  and we conclude that
	\begin{equation}
		\Pr{\#\{i\in[N]: |\iprod{z_i,w}| \ge 1\} \ge N\cdot\exp(-O(\omega^2))} \ge 1 - \exp(-\Omega(N))\,.
	\end{equation}
	Note that under this event, $\norm{Bw}_2 \ge \sqrt{N}\cdot \exp(-O(\omega^2))$.

	Let $\calN$ denote an $O(\exp(-O(\omega^2))/\sqrt{d})$-net over $\S^{d-1}$ of size $(d\exp(\omega^2))^{O(d)}$. Suppose for every $\wh{w}\in\calN$, $\norm{B\wh{w}}_2 \ge \sqrt{N}\cdot \exp(-O(\omega^2))$. This happens with probability at least $1 - (d\exp(\omega^2))^{O(d)} \cdot \exp(-\Omega(N))$. Then for any $w\in\S^{d-1}$, if $\wh{w}$ is the closest element in $\calN$ to $w$, we have $|\norm{Bw}_2 - \norm{B\wh{w}}_2| \le \norm{B}_{\sf op} \cdot \norm{\wh{w} - w}_2 \le O(\sqrt{N})$, where we used the na\"{i}ve estimate $\norm{B}_{\sf op} \le \norm{B}_F = \sqrt{Nd}$. The claim follows by the assumed lower bound on $N$.
\end{proof}

\noindent We can now use this to conclude the proof of Lemma~\ref{lem:proxy}.

\begin{proof}
	Let $\gamma$ in Lemma~\ref{lem:condprobbound} be a constant much larger than $1$. Let $N = \Theta(d(\omega^2 + \log d))$ as in Lemma~\ref{eq:sigmin_B}. Suppose that $\bX\in\brc{\pm 1}^{k\times d}$ satisfies the event $\calA$. Then by Lemma~\ref{lem:condprobbound}, with probability at least $1 - 1/d^\gamma$ over such a $\bX$, we have that 
	\begin{equation}
		\norm{\softmax(\bX_{1:}\Sig_{i^*}\bX^\top) - e_2}_2 \le k\exp(-\omega\sqrt{\log d}\,\norm{\Sig_1}_F)\,.	
	\end{equation}
	Note that given a labeled example $(\bX,\bY)$ where $\bX$ satisfies the above bound and $\bY = F(\bX)$, we have
	\begin{align}
		\Bigl\|\bY_{1:} - \bX_{2:}\sum_i \bW_i\Bigr\| &\le \Bigl\|\sum_i (\softmax(\bX_{1:}\Sig_i \bX_{2:}^\top) - e_2)^\top \bX\bW_i\Bigr\| \\
		&\le\sqrt{k^3d}\exp(-\omega\sqrt{\log d}\norm{\Sig_1}_F)\sum_i \norm{\bW_i}_{F}\,, \label{eq:labelclose}
	\end{align}
	where we have chosen to very crudely bound $\norm{\bX\bW_i}_{\sf op} \le \sqrt{kd}\cdot\norm{\bW_i}_F$ because we can afford to do so by taking the constant $\omega$ sufficiently large.
	Letting $\bX^{(1)},\ldots, \bX^{(N)}$ denote independent samples from the distribution over $\bX$ conditioned on the event $\calA$, and let $\bY^{(i)} \triangleq F(\bX^{(i)})$ for $i\in[N]$.

	Note that getting enough samples from this conditional distribution with probability at least $1 - 1/\poly(d)$ only requires drawing $O(N\log(d)\Pr{\calA}^{-1}) = d^{\Omega(O(\omega^3 m/\Siglbd))}$ samples as claimed.

	Finally, consider the least-squares problem defined in Line~\ref{step:leastsquares} of Algorithm~\ref{alg:waittorefine}, which we reproduce here:
	\begin{equation}
		\min_{\wh{\bW}\in\R^{d\times d}} \Bigl(\sum^N_{i=1} \norm{\bY^{(i)}_{1:} - \bX^{(i)}_{2:}\wh{\bW}}^2\Bigr)^{1/2}\,.
	\end{equation}
	The probability that each of these $\bX^{(i)}$'s satisfies the bound in Eq.~\eqref{eq:labelclose} is at least $1 - N/d^\gamma$, and because $N \le \poly(d)$, by taking $\gamma$ large enough, this probability is at least $1 - 1/d^{\gamma'}$ for arbitrarily large constant $\gamma'$.

	The minimizing $\wh{\bW}$ in the above least-squares problem achieves objective at most that of $\sum_i \bW_i$, which is at most
	\begin{equation}
		\sqrt{k^3Nd}\exp(-\omega\sqrt{\log d}\norm{\Sig_1}_F)\sum_i \norm{\bW_i}_F
	\end{equation}
	By triangle inequality, this implies that
	\begin{equation}
		\Bigl(\sum^N_{i=1} \norm{\bX^{(i)}_{2:}(\wh{\bW} - \sum_i \bW_i)}^2\Bigr)^{1/2} \le 2\sqrt{k^3Nd}\exp(-\omega\sqrt{\log d}\norm{\Sig_1}_F)\sum_i \norm{\bW_i}_F \label{eq:X2Wdiff}
	\end{equation}
	Suppose to the contrary that $\norm{\wh{\bW} - \sum_i \bW_i}_F \ge 2mdk^{3/2} \exp(-\omega\sqrt{\log d}\norm{\Sig_1}_F + O(\omega^2)cdot m$. Then there is a column of $\wh{\bW} - \sum_i \bW_i$ whose $L_2$ norm is at least $1/\sqrt{d}$ times this, so if $B$ denotes the matrix whose rows are $\bX^{(1)}_{2:},\ldots,\bX^{(N)}_{2:}$, then by Lemma~\ref{eq:sigmin_B}, the left-hand side of Eq.~\eqref{eq:X2Wdiff} is lower bounded by 
	\begin{equation}
		2m\sqrt{k^3Nd}\exp(-\omega\sqrt{\log d}\norm{\Sig_1}_F)
	\end{equation}
	with probability at least $1 - d^{-\Omega(d)}$, contradicting the bound in Eq.~\eqref{eq:X2Wdiff}.
\end{proof}


\section{Extracting the span from the approximate affine hull}
\label{sec:brute}

In this section we show that given a sufficiently accurate approximation to the affine hull of the attention matrices, one can produce a sufficiently good estimate for their span.

In this section we show how to take the feasible set we have constructed in the previous section and extract an $\upepsilon$-net over a subspace close to the span of the attention matrices.

First, we define some parameters. Let $\upepsilon > 0$, to be tuned later, and let $\upupsilon_0 = 0$. Given $0 \le \ell < m$, define
\begin{equation}
    \upupsilon_{\ell+1} \triangleq 2d\sqrt{\frac{2m}{\lambda}}\Bigl(\sum^\ell_{a=1} \upupsilon_a + \upepsilon\Bigr)\,. \label{eq:upsilon_recur}
\end{equation}
Note that $\upupsilon_a$ is increasing in $a$, and
\begin{equation}
	\upupsilon_{\ell+1} = \Theta(d\sqrt{m^3/\lambda})^{\ell+1}\cdot \upepsilon\,. \label{eq:unroll_recurrence}
\end{equation}

\noindent Our main guarantee is the following:

\begin{lemma}\label{lem:main_net}
	Let $\upepsilon^* > 0$, and let
	\begin{equation}
		\upepsilon \triangleq \frac{\upepsilon^*}{2\sqrt{m} \,\norm{\Sig_1}_F\cdot (d\sqrt{m^3/\lambda})^m} \qquad \text{and} \qquad \upzeta \le \upepsilon\cdot \norm{\Sig_1}_F\,. \label{eq:epszeta_def}
	\end{equation}
	If $K$ is an $(\upepsilon,\upzeta)$-tight enclosure of $\Sig_1,\ldots,\Sig_m$, then there is a procedure ({\sc NetFromEnclosure}($K$)) that outputs a list $\mathcal{L}$ of 
	$N_{\sf net} \triangleq O(\norm{\Sig_1}_F / \upzeta)^{O(m^2)}\cdot O(1/\upepsilon^*)^{O(m^2)}$ matrices after running in time $\poly(d)\cdot N_{\sf net}$, such that for every $i\in[m]$, there exists some $\wh{\Sig}\in\mathcal{L}$ such that $\norm{\wh{\Sig} - \Sig_i}_F \le \upepsilon^*$.
\end{lemma}

\begin{algorithm2e}
\DontPrintSemicolon
\caption{\textsc{NetFromEnclosure}($K$)}
\label{alg:net}
	\KwIn{Membership oracle access to tight enclosure $K$ in the sense of Definition~\ref{def:tighthull}}
	\KwOut{List of matrices containing an approximation to $\Sig_i$ for every $i\in[m]$}
        $\calL \gets \emptyset$\;
        $\calC \gets \emptyset$\;
        \For{$\ell \in[m]$}{
            $\calC \gets \textsc{AccumulateMatrices}(\calC)$\; \label{step:call_accumulate}
        }
        Let $\calS$ be an $\upepsilon^*/\sqrt{m}$-net over the set of vectors in $\R^m$ of norm at most $O(\norm{\Sig_1}_F)$\; \label{eq:lam_net}
        \For{$(\bM^{(1)},\ldots,\bM^{(m)})\in\calC$}{
        	\For{$\vec{\lambda} \in \calS$}{
            	Add to $\calL$ the matrix $\sum^m_{j=1} \lambda_j \bM^{(j)}$\;
        	}
        }
        \Return{$\calL$}
\end{algorithm2e} 

\begin{algorithm2e}
\DontPrintSemicolon
\caption{\textsc{AccumulateMatrices}($\calC$)}
\label{alg:accumulate}
	\KwIn{Set $\calC$ of orthonormal collections of matrices $\bM^{(1)},\ldots,\bM^{(\ell)}$}
	\KwOut{Set $\calC'$ of orthonormal collections of matrices $\bM^{(1)},\ldots,\bM^{(\ell+1)}$}
		$\calC' \gets \emptyset$\;
        Let $\brc{\upupsilon_a}$ be defined by Eq.~\eqref{eq:upsilon_recur}.\;
        \For{$(\bM^{(1)},\ldots,\bM^{(\ell)})\in \calC$ \label{step:mainforloop}}{
            $\calL \gets \emptyset$\;
    		Let $\calS$ be a $\upzeta$-net over the set of matrices of Frobenius norm at most $\norm{\Sig_1}_F$ in span of $\bM^{(1)},\ldots,\bM^{(\ell)}$\;
            Let $H$ be the subspace orthogonal to the span of $\bM^{(1)},\ldots,\bM^{(\ell)}$\;
            \For{$s\in\brc{\pm 1}$, $i,j\in[d]$, and $\bA\in\calS$}{
                Define convex body $K^{\bA}_{ij,s} \triangleq (K - \bA) \cap H \cap \brc{\bB\in\R^{d\times d}: s\cdot \bB_{ij} \ge \sqrt{\lambda}/(d\sqrt{2m})}$\;
                Query the membership oracle for $K^{\bA}_{ij,s}$ and add the output, if any, to $\calL$\;
            }
            \For{$\bM\in\calL$}{
                Add to $\calC'$ the tuple $(\bM^{(1)},\ldots,\bM^{(\ell)},\bM/\norm{\bM}_F)$\;
            }
        }
        \Return{$\calC'$}
\end{algorithm2e}

\noindent The main step in the proof of Lemma~\ref{lem:main_net} is the following which ensures that we can sequentially construct an orthonormal collection of $m$ matrices each of which is close to the span of $\Sig_1,\ldots,\Sig_m$.

\begin{lemma}\label{lem:outputlist}
    Suppose $K\subset\R^{d\times d}$ is an $(\upepsilon,\upzeta)$-tight enclosure of $\Sig_1,\ldots,\Sig_m$ in the sense of Definition~\ref{def:tighthull}, where $\upepsilon, \upzeta$ are given in Eq.~\eqref{eq:epszeta_def}.
    Suppose we are also given orthogonal matrices $\bM^{(1)},\ldots,\bM^{(\ell)}\in\R^{d\times d}$ of unit Frobenius norm such that $\norm{\PiSig^\perp(\bM^{(a)})}_F \le \upupsilon_a$. 
    If $\ell < m$, then for
    \begin{equation}
        N = \Theta(\norm{\Sig_1}_F / \upzeta)^\ell\,,
    \end{equation}
    there is an algorithm that makes at most $N$ membership oracle queries to $K$ and outputs a list of at most $N$ matrices which contains a matrix $\bM^{(\ell+1)}$ of unit Frobenius norm which is orthogonal to $\bM^{(1)},\ldots,\bM^{(\ell)}$ and $\upupsilon_{\ell+1}$-close to the span of $\Sig_1,\ldots,\Sig_m$.
\end{lemma}

\noindent Before turning to the proof of Lemma~\ref{lem:outputlist}, we note that we only ensure that at each step there \emph{exists} a matrix with the desired properties that we can add to the collection. Nevertheless, because there are at most $N$ matrices in the list constructed at each step, we can simply generate all $N^m$ possible collections of matrices obtained from picking a particular matrix from the list at each step and validate the accuracy of each on held-out test data. This step is standard.

Returning to the proof of Lemma~\ref{lem:outputlist}, we will need the following basic linear-algebraic fact.

\begin{lemma}\label{lem:interesting}
    Let $V\subset \R^D$ be any subspace of dimension at most $m - 1$, and let $v_1,\ldots,v_m\in \R^D$ be vectors satisfying $|\iprod{v_i,v_j}|\le \kappa\norm{v_i}\norm{v_j}$ and $\norm{v_1}^2 \ge \norm{v_2}^2 \ge \cdots \ge \norm{v_m}^2 \ge \Siglbd\norm{v_1}^2$ for $\kappa < \frac{\Siglbd}{2m^2}$.
    
    If $\Pi^\perp$ denotes the projection to the orthogonal complement of $V$, then there is a vector $v$ in the convex hull of $v_1,\ldots,v_m$ such that $\norm{\Pi^\perp v} \ge \sqrt{\Siglbd/2m}\cdot \norm{v_1}$.
\end{lemma}

\begin{proof}
    We will show the stronger statement that there exists $i\in[m]$ such that $\norm{\Pi^\perp v_i} \ge \sqrt{\Siglbd/2m}\cdot\norm{v_1}$. Suppose to the contrary, in which case $\sum^m_{i=1} \norm{\Pi^\perp v_i}^2 < \frac{\Siglbd}{2}\cdot \norm{v_1}^2$. Letting $\Pi \triangleq \Id - \Pi^\perp$, we see that this would imply that 
    \begin{equation}
        \sum^m_{i=1} \norm{\Pi v_i}^2 > \sum^m_{i=1} \norm{v_i}^2 - \frac{\Siglbd}{2}\norm{v_1}^2\,. \label{eq:bigenergy}
    \end{equation}
    On the other hand, note that
    \begin{equation}
        \Bigl\|\sum_i v_i v_i^\top - \diag(\norm{v_1}^2,\ldots,\norm{v_i}^2)\Bigr\|_\op \le m\kappa\norm{v_1}^2\,,
    \end{equation}
    which implies that for any $(m-1)$-dimensional projector $\Pi'$, 
    \begin{equation}
        \Bigl\langle\Pi, \sum_i v_i v_i^\top\Bigr\rangle \le \sum^{m-1}_{i=1} \norm{v_i}^2 + m^2\kappa\norm{v_1}^2 \le \sum^m_{i=1} \norm{v_i}^2 - (\Siglbd - m^2\kappa)\norm{v_1}^2 \le \sum^m_{i=1}\norm{v_i}^2 - \frac{\Siglbd}{2}\norm{v_1}^2\,,\label{eq:smallenergy}
    \end{equation}
    a contradiction of~\eqref{eq:bigenergy}.
\end{proof}

\noindent In light of Lemma~\ref{lem:interesting}, define
\begin{equation}
	\tempa \triangleq \sqrt{\Siglbd/2m}\cdot \norm{\Sig_1}\,. \label{eq:tempadef}
\end{equation}

\begin{proof}[Proof of Lemma~\ref{lem:outputlist}]
    We say that a matrix $\bM$ is \emph{interesting} if it lies in the convex hull of $\Sig_1,\ldots,\Sig_m$ and satisfies $\norm{\bM^\perp}_F \ge \tempa$, where $\bM^\perp$ denotes the projection of $\bM$ to the orthogonal complement of the span of $\bM^{(1)},\ldots,\bM^{(\ell)}$. Also let $\bM^\parallel$ denote the projection of $\bM$ to the span of $\bM^{(1)},\ldots,\bM^{(\ell)}$. By Lemma~\ref{lem:interesting}, there always exists an interesting matrix $\bM$, provided $\ell < m$.

    We show how to use membership oracle access to $K$ to obtain a list of matrices which contains a matrix $\bM^*$ close to $\bM^\perp$. Formally, the list will contain some $\bM^*$ with the following properties:
    \begin{itemize}
        \item There exist an interesting $\bM$, a matrix $\bM'\in K$, and a matrix $\wt{\bM}^\parallel \in \mathrm{span}(\bM^{(1)},\ldots,\bM^{(\ell)})$ satisfying $\norm{\wt{\bM}^\parallel - \bM^\parallel}_F \le \frac{\upupsilon_{\ell+1}\tempa}{2d}$, such that $\bM^* = \bM' - \wt{\bM}^\parallel$
        \item $\bM^*$ is orthogonal to $\bM^{(1)},\ldots,\bM^{(\ell)}$
        \item $\norm{\bM^*}_F \ge \tempa/d$.
    \end{itemize}
    Before proving this, we first show how to obtain $\bM^{(\ell+1)}$ from any such $\bM^* = \bM' - \wt{\bM}^\parallel$ in the list: simply define $\bM^{(\ell+1)} = \bM^* / \norm{\bM^*}_F$.
    
    We need to bound the distance of $\bM^{(\ell+1)}$ to the span of $\Sig_1,\ldots,\Sig_m$. Concretely, we need to show that $\bM^*$ is $\upupsilon_{\ell+1}\,\norm{\bM^*}_F$-close to the span. Because $\norm{\wt{\bM}^\parallel - \bM^\parallel}_F \le \frac{\upupsilon_{\ell+1}\tempa}{2d} \le \frac{\upupsilon_{\ell+1}}{2}\norm{\bM^*}_F$, it suffices to show that $\bM' - \bM^\parallel$ is $\frac{\upupsilon_{\ell+1}}{2}\,\norm{\bM^*}_F$-close.

    Write every $\bM^{(a)}$ as $\bN^{(a)} + \calE^{(a)}$, where $\bN^{(a)}$ is in the span of $\Sig_1,\ldots,\Sig_m$ and $\calE^{(a)}$ is orthogonal to $\Sig_1,\ldots,\Sig_m$ and has Frobenius norm at most $\upupsilon_a$ by hypothesis. Then 
    \begin{equation}
        \bM^\parallel = \sum^\ell_{a=1} \iprod{\bM^{(a)}, \bM} \bN^{(a)} + \sum^\ell_{a=1}\iprod{\bM^{(a)}, \bM} \calE^{(a)}\,,
    \end{equation}
    so because $|\iprod{\bM^{(a)}, \bM}| \le \norm{\bM}_F \le \norm{\Sig_1}_F$, its distance to the span of $\Sig_1,\ldots,\Sig_m$ is 
    \begin{equation}
        \norm{\Sig_1}_F \sum^\ell_{a=1} \upupsilon_a\,.
    \end{equation}
    On the other hand, because $\bM'\in K$, $\norm{\PiSig^\perp(\bM')}_F \le \upepsilon\,\norm{\bM'}_F \le 2\upepsilon\,\norm{\Sig_1}_F$. So $\bM' - \bM^\parallel$ is $\norm{\Sig_1}_F \cdot (\sum^\ell_{a = 1}\upupsilon_a + 2\upepsilon) \le \frac{\upupsilon_{\ell+1}}{2}\,\norm{\bM^*}_F$-close to the span of $\Sig_1,\ldots,\Sig_m$ as claimed.

    It remains to show how to produce the aforementioned list of matrices. Let $\calS$ be a $\upzeta$-net over the set of $d\times d$ matrices of Frobenius norm at most $\norm{\Sig_1}_F$ in the span of $\bM^{(1)},\ldots,\bM^{(\ell)}$. For any $\bA \in \calS$, define the convex body
    \begin{equation}
        K^{\bA} \triangleq (K - \bA) \cap H,
    \end{equation}
    where $H$ denotes the subspace orthogonal to the span of $\bM^{(1)},\ldots,\bM^{(\ell)}$. Then for every $i,j\in[d]$ and $s\in\brc{\pm 1}$, we query the membership oracle for
    \begin{equation}
        K^{\bA}_{ij,s}\triangleq K^\bA \cap \brc{\bB \in \R^{d\times d}: s\cdot \bB_{ij} \ge \tempa/d}\
    \end{equation}
    and add the output to our list. Note that the resulting list has size $2|\calS|d^2 \le O(\norm{\Sig_1}_F / \upzeta)^\ell$ as claimed.

    We now analyze the properties of this list. Take any interesting $\bM$. Then $\calS$ contains a matrix $\wt{\bM}^\parallel\in \mathrm{span}(\bM^{(1)},\ldots,\bM^{(\ell)})$ satisfying $\norm{\wt{\bM}^\parallel - \bM^\parallel}_F \le \frac{\upupsilon_{\ell+1}\tempa}{2d}$. Consider the convex body
    \begin{equation}
        K^* \triangleq (K - \wt{\bM}^\parallel) \cap H\,,
    \end{equation}
    We claim that $\bM^\perp \in K^*$. For this, it suffices to show that $\bM^\perp + \wt{\bM}^\parallel \in K$. This follows from the fact that $K$ contains the $\upzeta$-neighborhood of the convex hull of $\Sig_1,\ldots,\Sig_m$, and
    \begin{equation}
        \norm{\bM^\perp + \wt{\bM}^\parallel - \bM}_F = \norm{\bM^\parallel - \wt{\bM}^\parallel}_F \le \upzeta \le \frac{\upupsilon_{\ell+1}\tempa}{2d}\,,
    \end{equation}
    where in the penultimate step we used the assumption on $\wt{\bM}^\parallel$, and in the last step we used the form of $\upzeta$ in Eq.~\eqref{eq:epszeta_def}.
    
    Additionally, one of the entries of $\bM^\perp$ has magnitude at least $\tempa/d$, so define for any $i,j\in[d]$ and $s\in\brc{\pm 1}$ the convex body
    \begin{equation}
        K^*_{ij,s} \triangleq K^* \cap \brc{\bA \in \R^{d\times d}: s\cdot \bA_{ij} \ge \tempa / d}\,.
    \end{equation}
    Because at least one entry of $\bM^\perp$ exceeds $\tempa/d$ in magnitude, at least one of $K^*_{ij,s}$ contains $\bM^\perp$ and in particular is nonempty. Conversely, any matrix in $K^*_{ij,s}$ trivially has Frobenius norm at least $\tempa/d$. So the matrix returned by querying the membership oracle for any $K^*_{ij,s}$ that contains $\bM^\perp$ satisfies the three bullet points above, completing the proof.
\end{proof}

\noindent By repeatedly applying Lemma~\ref{lem:outputlist}, we can (nondeterministically) construct an orthonormal collection of $m$ matrices which are all close to the span of $\Sig_1,\ldots,\Sig_m$. We now show how to use these matrices to produce a list of matrices that is guaranteed to contain an approximation for each of $\Sig_1,\ldots,\Sig_m$:

\begin{lemma}\label{lem:net_combos}
	For $\upupsilon \le 1/\sqrt{2m}$, let $\bM_1,\ldots,\bM_m \in \R^{d\times d}$ be an orthonormal collection of matrices satisfying $\norm{\PiSig^\perp(\bM_i)}_F \le \upupsilon$ for all $i\in[m]$. Then for every $i\in[m]$, there is some $\vec{\lambda}\in\R^m$ with $\norm{\vec{\lambda}}_2 \lesssim \norm{\Sig_1}_F$ such that
	\begin{equation}
		\Bigl\|\sum^m_{j=1} \lambda_j \bM_j - \Sig_i\Bigr\|_F \lesssim \upupsilon\sqrt{m}\cdot \norm{\Sig_1}_F  \,.
	\end{equation}
	In particular, if $\calS$ is a $\delta$-net over the set of vectors in $\R^m$ of norm at most $O(\norm{\Sig_1}_F)$, then $\calS$ contains a vector $\wh{\lambda}$ such that
	\begin{equation}
        \Bigl\|\sum^m_{j=1} \wh{\lambda}_j \bM_j - \Sig_i\Bigr\|_F \le (\upupsilon\,\norm{\Sig_1}_F + \delta)\sqrt{m}\,.
    \end{equation}
\end{lemma}

\begin{proof}
	We can write every $\bM_i$ as $\sum^m_{j=1} \mu_{ij} \Sig_j + \calE_i$ for $\calE_i$ orthogonal to $\Sig_1,\ldots,\Sig_m$ and of Frobenius norm at most $\upupsilon$. Let $\bA\in\R^{m\times m}$ denote the matrix whose $i$-th row consists of $\mu_{i1},\ldots,\mu_{im}$ so that if the $i$-th row of $\bA^{-1}$ consists of some $\lambda_{i1},\ldots,\lambda_{im}$, then
    \begin{equation}
        \Sig_i = \sum^m_{j=1} \lambda_{ij} (\bM_j - \calE_j)\,. \label{eq:SigME}
    \end{equation}
    We would like to show that $\bA$ is well-conditioned. Let $\bB\in\R^{m\times d^2}$ denote the matrix whose $i$-th row is the vectorization of $\bM_i - \calE_i$, and let $\bC\in\R^{m\times d^2}$ denote the matrix whose $i$-th row is the vectorization of $\Sig_i$. Then $\bB = \bA \bC$, so in particular
    \begin{equation}
        \bB\bB^\top = \bA \bC \bC^\top \bA^\top\,. \label{eq:BAC}
    \end{equation}
    Note that $(\bC \bC^\top)_{ii} = \norm{\Sig_i}^2_F \ge \Siglbd\norm{\Sig_1}^2_F$ for all $i\in[m]$, and for all $i\neq j$, $|(\bC\bC)^\top_{ij}| \le \kappa\norm{\Sig_i}_F\norm{\Sig_j}_F \le \kappa\norm{\Sig_1}^2_F$, so by \eqref{eq:lambound} and the Gershgorin circle theorem, \begin{equation}
        \bC\bC^\top \preceq 2\norm{\Sig_1}^2_F\,. \label{eq:CC}    
    \end{equation}
    Define $\bB'\in\R^{m\times d^2}$ the matrix whose $i$-th row is the vectorization of $\bM_i$. Because $\bM_1,\ldots,\bM_m$ are orthonormal, $\sigma_{\min}(\bB') = 1$. As $\norm{\bB' - \bB}^2_\op \le \norm{\bB' - \bB}^2_F \le \sum^m_{i=1} \norm{\calE_i}^2_F \le m\upupsilon^2\le 1/2$, we conclude that 
    \begin{equation}
        \bB\bB^\top \succeq \frac{1}{4}\Id\,. \label{eq:BB}
    \end{equation}
    Combining Eqs.~\eqref{eq:BAC}, \eqref{eq:CC}, and \eqref{eq:BB}, we conclude that
    \begin{equation}
        \bA\bA^\top \succeq \frac{1}{8}\norm{\Sig_1}^{-2}_F \Id\,.
    \end{equation}
    In particular, $\norm{\bA^{-1}}_{\op} \lesssim \norm{\Sig_1}_F$. The lemma follows by \eqref{eq:SigME}.
\end{proof}

We are now ready to prove the main result of this section.

\begin{proof}[Proof of Lemma~\ref{lem:main_net}]
    Note that by our choice of $\upepsilon$ in Lemma~\ref{lem:main_net} and Eq.~\eqref{eq:unroll_recurrence},
    \begin{equation}
    	\upupsilon_1\le \cdots \le \upupsilon_m \le \frac{\upepsilon^*}{2\sqrt{m}\,\norm{\Sig_1}_F}\,.
    \end{equation}
    Suppose over the course of running {\sc AccumulateMatrices} on some set $\calC$, one encounters a tuple $(\bM^{(1)},\ldots,\bM^{(\ell)})\in\calC$ in Line~\ref{step:mainforloop} of Algorithm~\ref{alg:accumulate} which satisfies the hypotheses of Lemma~\ref{lem:outputlist}, i.e. $\norm{\PiSig^\perp(\bM^{(a)})}_F \le \upupsilon_a$ for all $1 \le a \le \ell$. Then we know that the output $\calC'$ of {\sc AccumulateMatrices} contains a tuple $(\bM^{(1)},\ldots,\bM^{(\ell+1)})$ such that all the matrices in this tuple are orthonormal, and furthermore $\bM^{(\ell+1)}$ is $\upupsilon_{\ell+1}$-close to the span of $\Sig_1,\ldots,\Sig_m$. Furthermore, with each call to {\sc AccumulateMatrices} in Step~\ref{step:call_accumulate} of Algorithm~\ref{alg:net} increases the size of the collection $\calC$ by a factor of $\Theta(\norm{\Sig_1}_F / \upzeta)^\ell$.
    So the final size of $\calC$ is $O(\norm{\Sig_1}_F / \upzeta)^{O(m^2)}$,
    and the total number of membership oracle queries made in order to produce the final $\calC$ is linear in the size of $|\calC|$.

    Proceeding inductively, we know from the above that the final $\calC$ contains a tuple $(\bM^{(1)},\ldots,\bM^{(m)})$ of orthonormal matrices that are all $\frac{\upepsilon^*}{\sqrt{m}\,\norm{\Sig_1}_F}$-close to the span of $\Sig_1,\ldots,\Sig_m$. So by Lemma~\ref{lem:net_combos}, for the $\upepsilon^*/\sqrt{m}$-net $\calS$ defined in Step~\ref{eq:lam_net} of Algorithm~\ref{alg:net}, there exists for every $i\in[m]$ there exists $\wh{\lambda}$ such that $\sum^m_{j=1} \wh{\lambda}_j \bM^{(j)}$ is $\upepsilon^*$-close to $\Sig_i$. Note that $\calS$ has size at most $O(\norm{\Sig_1}_F\sqrt{m} / \upepsilon^*)^{O(m)}$, so the size of the output of {\sc NetFromEnclosure}($K$) is at most $|\calC|\cdot|\calS| \le O(\norm{\Sig_1}_F / \upzeta)^{O(m^2)}\cdot O(1/\upepsilon^*)^{O(m^2)}$.
\end{proof}


\newcommand{\Dsmax}{\calD_{\sf att}}
\section{Solve for projection matrices}
\label{sec:linreg}

Finally, we show that given sufficiently good estimates for $\Sig_1,\ldots,\Sig_m$, there is a simple procedure for producing estimates for $\bW_1,\ldots,\bW_m$ using linear regression.

\begin{lemma}\label{lem:linreg}
	Let $\epsilon, \delta > 0$. Suppose that $\wh{\Sig}_1,\ldots,\wh{\Sig}_m$ satisfy
	\begin{equation}
		\norm{\wh{\Sig}_i - \Sig_i}_F \le \epsilon
	\end{equation}
	Then Algorithm~\ref{alg:linreg} draws $\poly(m,k,d)\cdot\sqrt{\log(1/\delta)}$ samples and produces $\wh{\bW}_1,\ldots,\wh{\bW}_m$ for which, for the estimator 
	\begin{equation}
		\wh{F}(\bX) \triangleq \sum^m_{i=1} \softmax(\bX\wh{\Sig}_i\bX^\top)\bX\wh{\bW}^\top\,,
	\end{equation}
	we have
	\begin{equation}
		\E{(\wh{F}(\bX) - F(\bX))^2} \lesssim \epsilon^2 d^4 k^3 m\norm{\bW_1}^2_F
	\end{equation}
	with probability at least $1 - \delta$.
\end{lemma}

\begin{algorithm2e}
\DontPrintSemicolon
\caption{\textsc{EstimateValueMatrices}($\brc{\wh{\Sig}_i}$)}
\label{alg:linreg}
	\KwIn{Estimates $\wh{\Sig}_i$ for the attention matrices}
	\KwOut{Estimates $\wh{\bW}_i$ for the projection matrices}
		$N\gets \poly(m,d,k)\cdot\sqrt{\log(1/\delta)}$.\;
		Draw random examples $(\bX^{(1)},\bY^{(1)}),\ldots,(\bX^{(N)},\bY^{(N)})$.\;
		For each $i\in[N]$, define $\bZ^{(i)}$ to be the matrix $\bZ$ given by example $\bX^{(i)}$ according to Eq.~\eqref{eq:bZdef}.\;
		\For{$s \in[d]$}{
			$\wh{w}\gets \arg\min_{\norm{w}\le\sqrt{m}} \frac{1}{N}\sum_{a=1} \norm{\bZ_a w - \bY_{:s}}^2$.\;
			Set the $s$-th column of the output $\wh{\bW}_i$ to be the $i$-th block of $d$ entries in $\wh{w}$.\;
		}
		\Return{$\brc{\wh{\bW}_i}$}
\end{algorithm2e}

\noindent We want to show generalization bounds for predicting any single column $\bY_{:i}$. Here we focus on $i = 1$; our proof immediately generalizes to predicting other $i$, and we only choose $i = 1$ for notational clarity. The following formalizes the linear regression we wish to solve, and in particular the distribution over covariates:

\begin{definition}\label{def:D}
	Let $\Dsmax$ denote the distribution over $\R^{k\times md}\times\R^k$ of pairs $(\bZ,y)$ given by sampling $\bX\sim\brc{\pm 1}^{k\times d}$ and forming $y = \bY_{:1} = F(\bX)_{:1}$ and each row of $\bZ$ consists of $m$ blocks of $d$ coordinates defined as follows. Within the $r$-th row and in the $i$-th block, the $j$-th coordinate of $\bZ$ is given by
	\begin{equation}
		\overline{\bZ}_{r;i,j} \triangleq \sum^k_{\ell = 1} \softmax(\bX_{r:}\wh{\Sig}_i(\bX_{\ell:})^\top)\bX_{\ell j}\,. \label{eq:bZdef}
	\end{equation}
\end{definition}

Observe that if we had $\Sig_i = \wh{\Sig}_i$, then for $(\bZ,y)\sim \Dsmax$, we would have
\begin{equation}
	y = \bZ w^*\,,	
\end{equation}
where $w^*\in\R^{md}$ also consists of $m$ blocks of $d$ coordinates, where within the $i$-th block, the $j$-th coordinate of $w^*$ is
\begin{equation}
	w^*_{i,j} \triangleq (\bW_i)_{j,1}\,. \label{eq:truew}
\end{equation}
Note that $\norm{w^*} \le \sqrt{m}\norm{\bW_1}_F = \sqrt{m}$.

If $\wh{\Sig}_i$'s only approximate $\Sig_i$'s, we can still set up a linear regression problem
\begin{equation}
 	\wh{w} \triangleq \arg\min_{w\in\R^{md}: \norm{w} \le \sqrt{m}} \frac{1}{N}\sum_{a=1} \norm{\bZ_a w - y_a}^2 \label{eq:wleastsquares}
\end{equation}
given $N$ i.i.d. samples $(\bZ_a,y_a)$ from $\Dsmax$. If we define our estimator for $\bY_{:1}$ to be 
\begin{equation}
	\wh{f}(\bX) \triangleq \sum_i \softmax(\bX\wh{\Sig}_i\bX^\top)\bX\wh{w},\,,
\end{equation} then the test loss $\E{\norm{\bZ w - y}^2}$ for our linear regression problem is precisely the test loss $\E{(\wh{f}(\bX) - \bY_{:1})^2}$ in predicting $\bY_{:1}$.

Here we record some basic observations about the regression problem, namely that $w^*$ achieves small error \emph{pointwise} over the domain of $\Dsmax$ (Lemma~\ref{lem:pointwise}), and that the covariates and labels in this problem are bounded (Lemma~\ref{lem:boundedreg}).

\begin{lemma}\label{lem:pointwise}
	Let $\epsilon > 0$. If $\norm{\wh{\Sig}_i - \Sig_i}_F \le \epsilon$ for all $i\in[m]$, then if $w^*$ is defined as in Eq.~\eqref{eq:truew}, we have that
	\begin{equation}
		\norm{\bZ w^* - y}^2 \le \epsilon^2 d^3 k^3 m \norm{\bW_1}^2_F\,.
	\end{equation}
	for any $(\bZ,y)$ in the support of $\Dsmax$.
\end{lemma}

\begin{proof}
	For any $\bX\in\brc{\pm 1}^{k\times d}$ and $i\in[m], r\in[k]$, note that by Lipschitzness of softmax,
	\begin{equation}
		\norm{\softmax(\bX_{r:}\Sig_i\bX^\top) - \softmax(\bX_{r:}\wh{\Sig}_i\bX^\top)} \le \norm{\bX_{r:}(\Sig_i - \wh{\Sig}_i)\bX^\top} \le \epsilon d\sqrt{k}\,.
	\end{equation}
	So
	\begin{equation}
		|\iprod{x,w^*} - y| = \bigl|(\softmax(\bX_{r:}\Sig_i\bX^\top) - \softmax(\bX_{r:}\wh{\Sig}_i\bX^\top))\bX w^*\bigr| \le \epsilon d\sqrt{k}\norm{\bX w^*} \le \epsilon d^{3/2} k \norm{w^*}\,.
	\end{equation}
	The proof is complete upon summing over $r$.
\end{proof}

\begin{lemma}\label{lem:boundedreg}
	For any $(\bZ,y)$ in the support of $\Dsmax$, we have that $\norm{\bZ}_F \le \sqrt{md}k$ and $|y| \le \sqrt{mdk}$
\end{lemma}

\begin{proof}
	For any $\bX\in\Xdom$, note that by convexity,
	\begin{equation}
		|\bZ_{r;i,j}| \le \norm{\bX_{:j}} = \sqrt{k}\,,
	\end{equation}
	so $\norm{\bZ}_F \le \sqrt{md}k$. Additionally,
	\begin{equation}
	 	\norm{F(\bX)_{:1}} = \norm{\softmax(\bX\Sig\bX^\top)\bX w^*} \le \sqrt{dk}\norm{w^*} \le \sqrt{mdk}\,.
	 \end{equation} 
\end{proof}

The proof of Lemma~\ref{lem:linreg} then follows from standard results on generalization:

\begin{theorem}\label{thm:gen}
	For $\calD$ a distribution over $\calX\times\calY$ and $\ell:\calY\times\calY\to\R$ a loss function that is $L$-Lipschitz in its first argument and uniformly bounded above by $c$. Let $\calF$ be a class of functions $\calX\to\calY$ such that for any $f\in\calF$ and pairs $(x_1,y_1),\ldots,(x_N,y_N)$ drawn independently from $\calD$, with probability at least $1 - \delta$, \begin{equation}
		\E[(x,y)\sim\calD]*{\ell(f(x),y)} \le \frac{1}{N}\sum_a \ell(f(x_a),y_a) + 4L\cdot \calR_N(\calF) + 2c\cdot \sqrt{\frac{\log(1/\delta)}{2N}},
	\end{equation} where $\calR_N(\calF)$ denotes the Rademacher complexity of $\calF$.
\end{theorem}

\begin{theorem}\label{lem:rademacher}
	If $\calX$ is a set of vectors $x$ satisfying $\norm{x} \le R$, and $\calF$ is a set of linear functions $\iprod{w,\cdot}$ on $\calX$ for $\norm{w} \le W$, then $\calR_N(\calF) \le XW/\sqrt{N}$.
\end{theorem}

\begin{proof}[Proof of Lemma~\ref{lem:linreg}]
	We wish to apply Theorem~\ref{thm:gen} to $\calD = \Dsmax$ and $\ell(y,y') \triangleq \norm{y - y^2}$  which is $L = O(\sqrt{mdk})$-Lipschitz in its first argument and uniformly bounded above by $c = O(mdk\norm{\bW_1}^2_F)$, by the second part of Lemma~\ref{lem:boundedreg}. Define
	\begin{equation}
		\calF \triangleq \Bigl\{\R^{md} \ni x\mapsto \iprod{w,x}: \norm{w}\le \sqrt{m}\}
	\end{equation}
	The following standard bound allows us to control the Rademacher complexity of $\calF$:

	We can take $R = \sqrt{md}k$ and $W = \sqrt{m}$ in Lemma~\ref{lem:rademacher} by the first part of Lemma~\ref{lem:boundedreg}. Applying Theorem~\ref{thm:gen}, taking $N = \poly(m,d,k)\sqrt{\log(1/\delta)}$, and noting that the empirical loss of $\wh{w}$ is at most that of $w^*$ which is at most $\epsilon^2 d^3 k^3 m\norm{\bW_1}^2_F$ by Lemma~\ref{lem:pointwise}, we conclude that the solution $\wh{w}$ to the linear regression achieves test loss $O(\epsilon^2 d^3 k^3 m\norm{\bW_1}^2_F)$. This is the test loss for predicting $\bY_{:1}$, and by repeating the above argument for all columns of $\bY$, we incur an additional factor of $d$.	
\end{proof}


\section{Putting the pieces together}
\label{sec:validate}

In this section we conclude the proof of our main result by combining the guarantees from the preceding sections. Recall that our algorithm operates in six phases; the following is paraphrased from Section~\ref{sec:overview}:
\begin{enumerate}
	\item {\bf Crude estimation of projection matrix sum}: We use the matrix $\E{\bX^\top \J \bY}$ to obtain a nontrivial approximation $\wh{\bW}$ to $\sum_i \bW_i$.
	\item {\bf Sculpting crude affine hull}: We use the LP-based certification procedure of Algorithm~\ref{alg:lp} to produce a convex body which is a nontrivially tight enclosure $K$ of $\Sig_1,\ldots,\Sig_m$, so that the minimum-norm point in $K$ is nontrivially close to a certain convex $\wh{\Sig}$ combination of $\Sig_i$'s that places similar mass on each $\Sig_i$
	\item {\bf Refining estimate for projection matrix sum}: We use $\wh{\Sig}$ as a proxy to detect when a the attention patterns induced by a given example $\bX$ are all extremely close to the same standard basis vector. We use such examples to construct a least-squares problem (Algorithm~\ref{alg:waittorefine}) to significantly refine our estimate for the projection matrix sum.
	\item {\bf Rerun sculpting algorithm}: Now that we have a much better estimate for $\sum_i W_i$, we can rerun Algorithm~\ref{alg:lp} from Step 2 to produce a new convex body $K^*$ which is a much tighter enclosure of $\Sig_1,\ldots,\Sig_m$.
	\item {\bf Extracting the span of the attention matrices from the convex body}: We use membership oracle access to $K^*$ to estimate the linear span of $\Sig_1,\ldots,\Sig_m$ and construct an epsilon-net over this span (Algorithm~\ref{alg:net}).
	\item {\bf Solve for projection matrices}: For each $m$-tuple $\brc{\wh{\Sig}_i}$ of elements from the epsilon-net, we run linear regression (Algorithm~\ref{alg:linreg}) to produce estimates $\brc{\wh{\bW}_i}$ for the projection matrices. We evaluate each of the resulting estimates $\brc{(\wh{\Sig}_i, \wh{\bW}_i)}$ on a validation set to identify one with test loss $(dk)^{-\Omega(m)}$.
\end{enumerate}

We can now state and prove our main result:

\begin{proof}[Proof of Theorem~\ref{thm:main}]
    Let $\delta > 0$ be a failure probability parameter. By Theorem~\ref{thm:sumvalue}, $\frac{1}{k}\E{\bX^\top\J\bY}$ is $\epsilon_{\sf 1}$-close in Frobenius norm to $\sum^m_{i=1}\bW_i$ for
	\begin{equation}
		\epsilon_{\sf 1} \triangleq \wt{\Theta}\Bigl(\frac{mk^5}{\reffsig^{1/12} \wedge (d^C / \upsilon)}\Bigr)\,.
	\end{equation}
	We can estimate $\sum^m_{i=1}\bW_i$ by drawing $N$ examples and forming an empirical estimate for $\frac{1}{k}\E{\bX^\top\J\bY}$. By standard matrix concentration, $\poly(d,1/\epsilon_{\sf 1})\cdot\sqrt{\log(1/\delta)}$ samples are sufficient for this empirical estimate $\wh{\bW}$ to be $\epsilon_{\sf 1}$-close in Frobenius norm to $\sum_i \bW_i$ with probability at least $1 - \delta$.

    We can then continue to the second stage in which we apply {\sc LPCertify}($\wh{\bW}$) (Algorithm~\ref{alg:lp}) to produce a convex body approximating the affine hull of $\Sig_i$'s. Take $\epsilon$ in Theorem~\ref{thm:main_lp} to be
	\begin{equation}
		\epsilon_{\sf 1}\cdot \wt{O}\Bigl(m\Siglbd^{-1}\log k+ \log(d/\delta) + \log\log(\norm{\Sig_1}_F)\Bigr)^{1/2}
	\end{equation}
	and take $\xi$ in Theorem~\ref{thm:main_lp} to be 
	\begin{equation}
		\xi \triangleq R^{\Theta(1/\Siglbd)} \cdot \max(e^{\Siglbd/(m\Siginc)}, \epsilon\log(m\norm{\Sig_1}_F)) \label{eq:xideffinal}	
	\end{equation}
	where $R = ke^{\upbeta^2} \asymp k$. Note that $\log(1/\xi) = O(\Siglbd^{-1}\log k + \log(1/\epsilon) + \log\log(m\norm{\Sig_1}_F)$, and $\log(1/\xi) \gg \log d$ by Eq.~\eqref{eq:lambound}.
    Then by Part (II) of Theorem~\ref{thm:main_lp}, the minimum norm point $\wh{\Sig}^*$ in the convex body satisfies $\norm{\Sig^* - \wh{\Sig}^*}_F \lesssim \epsilon_{\sf 2}$, where
	\begin{equation}
		\epsilon_{\sf 2} \triangleq \wt{\Theta}\Bigl(\frac{\kappa^{1/2}m^{1/4}}{\Siglbd^{3/4}}\cdot\norm{\Sig_1}_F + \sqrt{\frac{\norm{\Sig_1}_F}{\Siglbd}}\cdot \epsilon^{1/2}_{\sf 1} \cdot \wt{O}\Bigl(m\Siglbd^{-1}\log k+ \log(d/\delta) + \log\log(\norm{\Sig_1}_F)\Bigr)^{1/4}\Bigr)\,.
	\end{equation}
	Furthermore, Algorithm~\ref{alg:lp} runs in time $(k^{O(1/\Siglbd)}\epsilon\log(m\norm{\Sig_1}_F))^m \cdot \wt{O}(d^2\log(\norm{\Sig_1}_F) + \log(1/\delta))$

    The first term in $\epsilon_{\sf 2}$ is the dominant one, as $\epsilon_1$ scales inverse polynomially in $\reffsig$, and in Assumption~\ref{assume:effranksig} we assumed that $\reffsig$ was at least polylogarithmic in $d$ for sufficiently large degree $c$. By Eq.~\eqref{eq:lambound}, we have that $\epsilon_{\sf 2} \ll \Siglbd^{5/2}/(m^3\log^2 d)\cdot \norm{\Sig_1}_F$. We can now continue to the third stage. Take $\omega$ in Lemma~\ref{lem:proxy} to be an arbitrarily large constant multiple of $m$, and the matrix $\wh{\Sig}^*$ therefore satisfies the hypotheses of Lemma~\ref{lem:proxy}, so we conclude that {\sc LeastSquaresRefine}($\wh{\Sig}^*$) (Algorithm~\ref{alg:waittorefine}) returns a refined matrix $\wh{\bW}$ which is $(kd)^{-cm}$-close to $\sum_i \bW_i$ for any constant $c > 0$.

    We then turn to the fourth stage and apply Theorem~\ref{thm:main_lp} and Algorithm~\ref{alg:lp} one more time to this new $\wh{\bW}^*$. This time, $\wh{\bW}^*$ is accurate enough that we will use Part (I). Take $\xi$ in Theorem~\ref{thm:main_lp} as in Eq.~\eqref{eq:xideffinal}, noting that now $\xi \le kd^{-cm}$ for any constant $c > 0$, but this time we take $\epsilon$ to be
	\begin{equation}
		kd^{-cm} \cdot \wt{O}\Bigl(m\Siglbd^{-1}\log k+ \log(d/\delta) + \log\log(\norm{\Sig_1}_F)\Bigr)^{1/2} \le (kd)^{-c'm}
	\end{equation}
	for some other arbitrarily large constant $0 < c' < c$.
	Note that because we assumed $\norm{\Sig_1}_F \le (kd)^{O(m)}$, we satisfy the assumption $\xi \ll 1/\norm{\Sig_1}_F$.
	Then by Part (I) of Theorem~\ref{thm:main_lp}, we conclude that the new convex body $K$ produced by Algorithm~\ref{alg:lp} is a $((kd)^{-\Omega(m)}, (kd)^{-\Omega(m)})$-tight enclosure of $\Sig_1,\ldots,\Sig_m$. Furthermore, Algorithm~\ref{alg:lp} runs in time $(dk)^{O(m^2)}\cdot \log(1/\delta)$.

	In the fifth stage of the algorithm, we run {\sc NetFromEnclosure($K$)} (Algorithm~\ref{alg:net}), and by Lemma~\ref{lem:main_net}, we produce a list of $(dk)^{O(m^3)}$ matrices after running in time $(dk)^{O(m^3)}$, such that for every $i\in[m]$, there exists some matrix in the list which is $(dk)^{-\Omega(m)}$ close to $\Sig_i$.

	In the sixth and final stage of the algorithm, we enumerate over all possible $m$-tuples $\brc{\wh{\Sig}_i}$ of elements from this list and run {\sc EstimateValueMatrices}($\brc{\wh{\Sig}_i}$) (Algorithm~\ref{alg:linreg}) on each tuple. This algorithm draws $\poly(m,k,d)\sqrt{\log(1/\delta)}$ samples and runs in polynomial time. Each run results in a tuple $\brc{\wh{\bW}_i}$, and for at least one of the choices of $\brc{\wh{\Sig}_i}$, namely the one where each $\wh{\Sig}_i$ is $(dk)^{-\Omega(m)}$ close to $\Sig_i$, the resulting tuple $\brc{\wh{\bW}_i}$ achieves test loss $(dk)^{-\Omega(m)}$ with probability at least $1 - \delta$.

    Finally, by evaluating each of the $(dk)^{O(m^3)}$ estimates $(\brc{\wh{\Sig}_i}, \brc{\wh{\bW}_i})$ on a held-out test set of size $(dk)^{O(m^3)}$, we can identify one that achieves test loss $(dk)^{-\Omega(m)}$, thus completing the proof.
\end{proof}

\begin{remark}
    Here we discuss the extent to which we can improve our error guarantee beyond $(dk)^{-\Omega(m)}$. 
    \begin{enumerate}[leftmargin=*]
        \item \underline{Application of Lemma~\ref{lem:proxy}:} First note that our application of Lemma~\ref{lem:proxy} was rather lossy. We took $\omega$ therein to be an arbitrarily large constant multiple of $m$, thus resulting in a $(kd)^{-\Omega(m)}$-close estimate of $\sum_i \bW_i$ which was passed in to the fourth stage of the algorithm. In actuality, we could have taken $\omega$ as large as $\min(\norm{\bQ_1}_F\sqrt{\log d}, \Siglbd^{7/4}\sqrt{\frac{m\Siginc\reffsig}{\log d}})$, thus yielding an bound on $\Bigl\|\wh{\bW} - \sum_i \bW_i\Bigr\|_F$ that can be as small
        \begin{equation}
            2mdk^{3/2}\exp\bigl(-\Omega\bigl(\min(\norm{\bQ_1}^2_F \log d, \Siglbd^{7/4}\sqrt{m\kappa\reffsig}\norm{\bQ_1}_F)\bigr)\bigr)\,. \label{eq:bottom1}
        \end{equation}
        Recalling that we are assuming a lower bound of $\norm{\bQ_1} \gtrsim \log\Bigl(kd/\sqrt{\Wlbd}\Bigr) = \Theta(\log d)$ and that $\reffsig$ is at least polylogarithmic in $d$, we can actually estimate $\sum_i \bW_i$ to error as low as inverse quasi-polynomial in $d$, and even better if $\norm{\bQ_1}$ or $\reffsig$ are larger.
    
        What is the runtime cost of this improved error bound for estimating $\sum_i \bW_i$ in the third stage of the algorithm? Recall that the runtime of the algorithm used in Lemma~\ref{lem:proxy} is $d^{O(\omega^2 m/\Siglbd^4)}$, so for any target error $\epsilon'$ which is at least Eq.~\eqref{eq:bottom1}, by taking $\omega$ scaling with $\log(1/\epsilon) / (\sqrt{\log d}\cdot \norm{\bQ_1}_F)$, we obtain a runtime exponentially in $\frac{\log^2(1/\epsilon')}{\norm{\bQ_1}^2_F \cdot \log d}$\--- note that for $\epsilon' = (dk)^{-\Omega(m)}$, this exponential dependence is $O(m)$ as claimed in the analysis above. 
    
        \vspace{0.5em}
    
        \item \underline{Second application of Theorem~\ref{thm:main_lp}:} Intuitively, the error incurred by the convex body produced by Algorithm~\ref{alg:lp} in the fourth stage of the algorithm scales linearly in the error in estimating $\sum_i \bW_i$. However, note that the final error $\epsilon^*$ in Theorem~\ref{thm:main_lp} scales linearly with the maximum of the error $\epsilon'$ from the previous stage (denoted by $\norm{\Delta}_F$ in the theorem statement) and the quantity
        \begin{equation}
            O\Bigl(m\log((ke^{\tempd^2})^{\Theta(1/\Siglbd)}/\xi)\cdot\Bigl(\frac{1}{2} + \frac{1}{\Siglbd}\Bigr)\cdot \xi\norm{\bQ_1}_F\Bigr) \label{eq:epsstarbound}
        \end{equation}
        for a parameter $\xi$ which satisfies Eq.~\eqref{eq:xiupper}. 
        
        Unfortunately, $\xi$ is assumed to be lower bounded by $(ke^{\tempd^2})^{\Theta(1/\Siglbd)}/e^{\Siglbd/(m\Siginc)}$. Previously when we were targeting a final error bound of $(dk)^{-\Omega(m)}$, our choice of $\Siginc$ in Eq.~\eqref{eq:lambound} of Assumption~\ref{assume:sig_orth} was small enough that this lower bound on $\xi$ is dominated by $(dk)^{-\Omega(m)}$. This means that the best possible error we can hope for for the convex body produced by Algorithm~\ref{alg:lp} in the fourth stage is given by Eq.~\eqref{eq:epsstarbound} with $\xi$ taken to be $(ke^{\tempd^2})^{\Theta(1/\Siglbd)}/e^{\Siglbd/(m\Siginc)}$. By the choice of $\kappa$ in Eq.~\eqref{eq:lambound}, this best possible error ultimately scales inverse quasi-polynomially with $d$. As with the discussion about Lemma~\ref{lem:proxy} above, we can hope for even lower error if we assume a smaller value of $\kappa$, which ultimately corresponds to a stronger assumption on the incoherence of the attention matrices. 
        
        What is the runtime cost of this improved error bound in the fouth stage? By Theorem~\ref{thm:main_lp}, the algorithm for producing the convex body in the fourth stage has runtime scaling with $(1/\xi)^m$, so if we want to estimate the convex body to error $\epsilon''$, the runtime must scale with $(1/\epsilon'')^m$ \--- note that for $\epsilon'' = (dk)^{-\Omega(m)}$, we obtain the $(dk)^{\Omega(m^2)}$ scaling claimed in the analysis above.
    
        \vspace{0.3em}
        
        \item \underline{Achieving arbitrarily small error?} The fact that our analysis can only guarantee test loss up to a certain value is a byproduct of the discreteness of the distribution over $\bX$. This discreteness makes it impossible to observe arbitrarily extreme tail events and thus obtain arbitrarily close estimates for the ground truth parameters. Concretely, in the proof and algorithm for Lemma~\ref{lem:proxy}, we rely on seeing examples which induce approximately $1$-sparse attention patterns, but because the number of possible $\bX$'s is finite, the level of approximate 1-sparsity is bottlenecked at some nonzero quantity. Similarly, in the proof and algorithm for Theorem~\ref{thm:main_lp}, we rely on seeing examples which induce approximately $2$-sparse attention patterns, and the same issue applies. Neither of these issues manifests in the case where the $\bX$'s are sampled from some continuous distribution, e.g. Gaussian, and in that case it should be possible to achieve arbitrarily small error with our techniques. It is an interesting open question to obtain such a guarantee for discrete $\bX$.
    \end{enumerate}
\end{remark}


\section{Computational lower bound}
\label{sec:lowerbound}

In this section we prove computational lower bounds suggesting that exponential dependence in the runtime on the number of heads may be necessary in the worst case. One lower bound is cryptographic in nature, based on a variant of the \emph{learning with errors} assumption (see Conjectures~\ref{conj:lwr_pre} and~\ref{conj:lwr}), and the other is a \emph{statistical query} (SQ) lower bound, the definition of which we briefly recall:

\begin{definition}\label{def:sq}
    Let $f: \brc{\pm 1}^m\to\brc{\pm 1}$, and let $\calD$ be a distribution over $\brc{\pm 1}^m$. For tolerance parameter $\tau > 0$, the $\mathrm{STAT}(\tau)$ oracle answers any query $h: \brc{\pm 1}^m \to [0,1]$ with a value $v$ such that $|\mathbb{E}_{x\sim\calD}[h(x)] - v| \le \tau$.
\end{definition}

\noindent Our main results in this section are the following lower bounds:

\begin{theorem}\label{thm:main_lbd}
    Let $\calC$ denote the class of multi-head attention layers $F: \brc{\pm 1}^{2\times d}$ on two tokens with $m$ heads and attention/projection matrices of norm at most $\poly(d)$. Then:
    \begin{enumerate}
        \item Under the hypothesis that \emph{learning with rounding with secret leakage} does not admit a polynomial-time algorithm (see Cconjecture~\ref{conj:lwr}), there is no polynomial-time algorithm for PAC learning $\calC$ over the uniform distribution over $\brc{\pm 1}^{2\times d}$.
        \item If $\tau = d^{-O(m)}$, then any SQ algorithm for PAC learning $\calC$ requires $d^{\Omega(m)}$ queries to $\mathrm{STAT}(\tau)$.
    \end{enumerate} 
\end{theorem}

\noindent The ingredient common to the proofs of both parts of Theorem~\ref{thm:main_lbd} is a construction that allows us to exactly implement any function $f:\brc{\pm 1}^{2d}\to\brc{\pm 1}$ of the form
\begin{equation}
    f(z_1,z_2) \triangleq h\Bigl(\Bigl\langle\frac{1}{2}\cdot \vec{1}_S, z_1 - z_2\Bigr\rangle\Bigr)\,,
\end{equation}
where $z_1,z_2 \in \brc{\pm 1}^d$, $\vec{1}_S \in \brc{0,1}^d$ is the indicator vector for a subset $S\subseteq[d]$, and $h: \mathbb{Z}\to\brc{0,1}$ is arbitrary (see Lemma~\ref{lem:reduction_sample} below), using a small multi-head attention layer. This turns out to be somewhat delicate to show given that we are trying to interpolate Boolean-valued functions with softmaxes.

\subsection{An attention gadget}

Define the activation function
\begin{equation}
    \phi(z) \triangleq z\cdot \tanh(z/2) 
\end{equation}

\begin{proposition}\label{prop:phiclose}
    $\phi(0) = 0$. Furthermore, for any $z\in\R$,
    \begin{equation}
        |\phi(z) - |z|| \le e^{-z/3}\,.
    \end{equation}
\end{proposition}

\begin{proof}
    The first part is immediate. For the second, note that $\phi$ is symmetric so it suffices to consider positive $z$. We have
    \begin{equation}
        |\phi(z) - z| = z\cdot \Bigl(1 - \frac{e^z - 1}{e^z + 1}\Bigr) = \frac{2z}{e^z + 1}\,.
    \end{equation}
    The claim then follows from the elementary inequality $\frac{2z}{e^z + 1} \le e^{-z/3}$.
\end{proof}

\noindent Our motivation for defining $\phi(\cdot)$ is the following gadget construction: 

\begin{proposition}\label{prop:Gtauv}
    Given $\bX\in\brc{\pm 1}^{2\times d}$, denote the rows of $\bX$ by $(a_1,z_1)$ and $(a_2,z_2)$ respectively for $a_1,a_2\in\brc{\pm 1}$ and $z_1, z_2\in\brc{\pm 1}^{d - 1}$. If $w = (\tau, v)$, $w' = (-\tau,v)$, $\Sig = e_1\cdot w^\top$, and $\Sig' = e_1\cdot w'^\top$ for $\tau\in\R$ and $v\in\R^{d-1}$, then for $s\in\brc{\pm 1}$ define
    \begin{multline}
        G^s_{\tau,v}(\bX) \triangleq \softmax(\bX \Sig \bX^\top)\bX w - \softmax(-\bX \Sig \bX^\top) \bX w + s\cdot(\softmax(\bX \Sig' \bX^\top)\bX w' - \softmax(-\bX \Sig' \bX^\top) \bX w') \\
        = (a_1, a_2) \cdot \Bigl\{\phi\bigl((a_1 - a_2)\tau + \iprod{v,z_1 - z_2}\bigr) + s\cdot\phi\bigl((a_2 - a_1)\tau + \iprod{v,z_1 - z_2}\bigr)\Bigr\}\,.
    \end{multline}
\end{proposition}



\subsection{Existence of an interpolation}

\noindent We will show that any even or odd function which depends on the Hamming weight of a substring of the input can be implemented as a linear combination of neurons with the activation function $\phi$. The following is the central technical step in the proof of Theorem~\ref{thm:main_lbd}:

\begin{lemma}\label{lem:main_lbd_step}
    Let $d, M\in\mathbb{N}$ and $w\in\mathbb{Z}^d$ with $\norm{w}_\infty \le M$. Let $h: \brc{-d,-d+1,\ldots,d-1,d}\to \brc{\pm 1}$ be any even (resp. odd) function. Then for $m = 2dM + 2$, there exist $\tau_1,\ldots,\tau_m \ge 0$, coefficients $\lambda_1,\ldots,\lambda_m\in \R$, and $v\in \R^d$ such that
    \begin{enumerate}
        \item $|\lambda_i| \lesssim dM$ for all $i\in[m]$
        \item $\sum^d_{i=1} \lambda_i = 0$
        \item $\norm{v} = \Theta(M\sqrt{d}\log d)$.
        \item For all $x\in\brc{-1,0,1}^d$,
        \begin{equation}
            \sum^m_{i=1} \lambda_i \, \phi(\tau_i + \iprod{v,x}) + \lambda_i \, \phi(-\tau_i + \iprod{v,x}) = h(\iprod{w, x})
        \end{equation} if $h$ is even, or
        \begin{equation}
            \sum^m_{i=1} \lambda_i \, \phi(\tau_i + \iprod{v,x}) - \lambda_i \, \phi(-\tau_i + \iprod{v,x}) = h(\iprod{w, x})
        \end{equation} if $h$ is odd.
    \end{enumerate}
\end{lemma}

\begin{proof}
    First note that it suffices to show that there exist $\lambda_1,\ldots,\lambda_m, \tau_1,\ldots, \tau_m, v$ such that Items 1-3 hold, and further
    \begin{equation}
        \sum^m_{i=1} \lambda_i \, \phi(-\tau_i + \iprod{v, x}) = \frac{1}{2}h(\iprod{w,x})\,.
    \end{equation}
    The reason is as follows. If $h$ is even, then replacing $x$ with $-x$ above and using the fact that $\phi$ is also even, we would get
    \begin{equation}
        \sum^m_{i=1} \lambda_i \, \phi(\tau_i + \iprod{v, x}) = \frac{1}{2}h(\iprod{w,x})\,,
    \end{equation}
    so by adding the two equalities we would get the desired claim. On the other hand, if $h$ is odd, then replacing $x$ with $-x$ in the first equality and using the fact that $\phi$ is even and $h$ is odd, we would get
    \begin{equation}
        \sum^m_{i=1} \lambda_i \, \phi(\tau_i + \iprod{v, x}) = -\frac{1}{2}h(\iprod{w,x})\,,
    \end{equation}
    so by subtracting the first and third equalities we would get the desired claim.

    We will take $v = \rho\cdot w$ for some large constant $\rho > 0$, $m = 2dM + 2$, and
    \begin{equation}
        (\tau_1,\ldots,\tau_{2dM+2}) = \rho\cdot (-dM,\ldots,dM,dM+1)\,.
    \end{equation}
    Note that for any string $x\in\brc{-1,0,1}^d$, the quantity $\iprod{v,x}$ ranges over $\rho\cdot \brc{-dM,\ldots,dM}$, so we just need tod find coefficients $\lambda_1,\ldots,\lambda_{2dM+2}$ for which Items 1-3 hold and 
    \begin{equation}
        \sum^{2dM+2}_{i=1} \lambda_i \, \phi(\rho\cdot (dM + 1 - i + \ell)) = \frac{1}{2} h(\ell) \ \forall \ \ell\in\brc{-dM,\ldots,dM}\,.
    \end{equation}
    Reindexing via $\mu_i \triangleq \lambda_{dM + 1 - i}$ for $i\in\brc{-dM,\ldots,dM,dM+1}$, we see that the above is equivalent to
    \begin{equation}
        \sum^{dM+1}_{i=-dM} \mu_i \, \phi(\rho\cdot (i - \ell)) = \frac{1}{2}h(\ell) \ \forall \ \ell\in\brc{-dM,\ldots,dM}\,. \label{eq:linearsys}
    \end{equation}
    Consider the matrix $\bM\in\R^{(2dM+2)\times (2dM+2)}$ whose last row is the all-1's vector and whose subsequent rows are $\brc{\phi(\rho \cdot (i - \ell))}_{-dM \le i \le dM+1}$ for $-dM \le \ell \le dM$. Letting $\vec{h} \in \R^{2dM+2}$ denote the vector whose last entry is $0$ and whose remaining entries consist of $\brc{\frac{1}{2} h(\ell): -dM \le \ell \le dM}$, and letting $\vec{\mu}$ denote the vector with entries $\mu_{-dM},\ldots,\mu_{dM+1}$, we can rewrite \eqref{eq:linearsys} as 
    \begin{equation}
        \bM \vec{\mu} = \vec{h}\,. \label{eq:Mmusys}
    \end{equation}
    Let $\wh{\bM}\in\R^{(2dM+2)\times (2dM+2)}$ denote the matrix whose last row is the all-1's vector and whose subsequent rows are $\brc{\rho\cdot |i - \ell|}_{-dM\le i \le dM+1}$ for $-dM \le \ell \le dM$.
    
    By Proposition~\ref{prop:phiclose}, $\norm{\bM - \wh{\bM}}_{\max} \le e^{-\rho/3}$, so if we take $\rho = \Theta(\log(1/\epsilon))$ for $\epsilon > 0$, then $\norm{\bM - \wh{\bM}}_\op \le \epsilon$. 

    Define the distance matrix $\bD \in \R^{(2dM+2)\times (2dM+2)}$ by $\bD_{i,j} = |i - j|$ for $i,j\in[2dM+2]$ and note that the first $2dM+1$ rows of $\rho\cdot \bD$ are identical to the first $2dM+1$ rows of $\wh{\bM}$. Using Lemma~\ref{lem:grahamlovasz} below, we have
    \begin{equation}
        \one^\top \bD^{-1} = \Bigl(\frac{1}{2dM+1},0,\ldots,0,\frac{1}{2dM+1}\Bigr)\,.
    \end{equation}
    In other words, we have the identity
    \begin{equation}
        \wh{\bM} = \bA \cdot \bD \ \text{for} \ \bA \triangleq \begin{pmatrix}
            \rho\cdot \Id_{2dM+1} & \vec{0}_{2dM+1} \\
            \frac{1}{2dM+1} e^\top_1 & \frac{1}{2dM+1}
        \end{pmatrix}\,.\label{eq:MAD}
    \end{equation}
    The singular values of $\bA$ consist of $\frac{1}{(2dM+1)^2} + \rho^2/2 - \sqrt{\frac{1}{(2dM+1)^4} + \rho^4/4}$, $\frac{1}{(2dM+1)^2} + \rho^2/2 + \sqrt{\frac{1}{(2dM+1)^4} + \rho^4/4}$, and $2dM$ copies of $\rho$, so $\sigma_{\min}(\bA) \gtrsim \rho/dM$. By combining this, \eqref{eq:MAD}, and the second part of Lemma~\ref{lem:grahamlovasz}, we conclude that $\sigma_{\min}(\wh{\bM}) \gtrsim \rho/dM$. So if we take $\epsilon = 1/\poly(dM)$ so that $\rho = \Theta(\log dM)$ (thus verifying Item 3), then we conclude that $\sigma_{\min}(\bM) \gtrsim \rho/dM$ and in particular $\bM$ is invertible.

    We can thus take $\vec{\mu}$ in \eqref{eq:Mmusys} to be $\bM^{-1}\vec{h}$, which has $L_\infty$ norm at most of order $(dM / \rho)\cdot \norm{\vec{h}}_\infty \lesssim dM/\log(dM)$, completing the proof of the lemma.
\end{proof}

\noindent In the above proof, we used the following classical result:

\begin{lemma}[Lemma 1 in \cite{graham1978distance}]\label{lem:grahamlovasz}
    Let $\bD\in \R^{m\times m}$ be the distance matrix of the path graph on $m$ vertices for $m > 2$, that is, $\bD_{ij} = |i - j|$ for all $i,j\in[m]$. Then 
    \begin{equation}
        (\bD)^{-1} = \begin{pmatrix}
            \frac{2-m}{2m-2} & \frac{1}{2} & 0 & \cdots & 0 & \frac{1}{2m-2} \\
            \frac{1}{2} & -1 & \frac{1}{2} & 0 & \cdots & 0 \\
            0 & \frac{1}{2} & \ddots & \ddots & \ddots & \vdots \\
            \vdots & \ddots & \ddots & \ddots & \frac{1}{2} & 0 \\
            0 & \cdots & 0 & \frac{1}{2} & -1 & \frac{1}{2} \\
            \frac{1}{2m-2} & 0 & \cdots & 0 & \frac{1}{2} & \frac{2-m}{2m-2}\,.
        \end{pmatrix}
    \end{equation}
    Additionally, we have
    \begin{equation}
        \sigma_{\min}(\bD) = \Theta(1) 
    \end{equation}
\end{lemma}

\begin{proof}
    The form of $\bD^{-1}$ follows immediately by specializing \cite[Lemma 1]{graham1978distance} to the path graph. The bound on $\sigma_{\min}(\bD)$ is immediate from the tridiagonal structure of $\bD^{-1}$. 
\end{proof}

\subsection{Reduction from functions of projections}
\label{sec:reductiontoprojfunctions}

\noindent We now show how to convert random example access to any function $f:\brc{\pm 1}^{2d}\to \brc{\pm 1}$ of the form $f(z_1,z_2) = h(\iprod{\frac{1}{2}\cdot w, z_1 - z_2})$ for an arbitrary Boolean-valued function $h$ and vector $w\in\brc{-M,\ldots,M}^d$, into random example access to a multi-head attention layer with $\Theta(|S|)$ heads and token size $2$.

\begin{lemma}\label{lem:reduction_sample}
    Suppose for any $d\in\mathbb{N}$ there is an algorithm that, given random example access to any multi-head attention layer $F: \brc{\pm 1}^{2\times d} \to \brc{\pm 1}^{2\times d}$ with $m$ heads with attention and projection matrices of norm at most $\poly(d)$, draws $N(m,d)$ examples and in time $T(m,d)$ outputs a hypothesis $\wh{F}$ satisfying $\E{\norm{F(\bX) - \wh{F}(\bX)}^2_F} \le \epsilon^2$ with high probability over the randomness of the examples.

    Then given random example access to any function $f:\brc{\pm 1}^{2d}\to \brc{\pm 1}$ of the form $f(z_1,z_2) = h(\iprod{\frac{1}{2}\cdot w, z_1 - z_2})$ for any Boolean-valued function $h$ and $w\in\brc{-M,\ldots,M}^d$ with $M = \poly(d)$, there is an algorithm that draws 
    $O(N(\poly(d), d+1))$ examples and in time 
    $O(T(\poly(d), d+1))$ outputs a hypothesis $\wh{f}$ satisfying $\E{(f(x) - \wh{f}(x))^2} \le 2\epsilon^2$.
\end{lemma}

\begin{proof}
    Let $S\subseteq[d]$. Suppose we have access to random examples $(x,y)$ where $x = (z_1,z_2) \sim\brc{\pm 1}^{2d}$ and $y = f(z_1,z_2)$ for $f: \brc{\pm 1}^{2d}\to\brc{\pm 1}$ given by 
    \begin{equation}
        f(z_1,z_2) \triangleq h\Bigl(\Bigl\langle\frac{1}{2}\cdot w,z_1-z_2\Bigr\rangle\Bigr)    
    \end{equation}
    for Boolean-valued function $h$. We show how to produce random examples labeled by a multi-head attention layer consisting of 
    $O(Md)$ heads. 
    Because any $h$ can be written as the sum of an even and odd function, it suffices to prove this for the special cases that $h$ is even or odd.

    
    Sample $a_1,a_2\in\brc{\pm 1}$ independently at random and form $\bX\in\brc{\pm 1}^{2\times d}$ with the first row given by $(a_1,z_1)$ and the second row given by $(a_2,z_2)$. If $a_1 = a_2$, define $Y = \vec{0}_{2\times d}$, otherwise define $Y = y\cdot (a_1,a_2)\cdot e_1^\top$.

    Let $\tau_1,\ldots,\tau_m, \lambda_1,\ldots,\lambda_m, v$ be the parameters guaranteed by Lemma~\ref{lem:main_lbd_step}, and consider the function:
    \begin{equation}
        F(\bX) \triangleq \sum^m_{i=1} \lambda_i G^s_{\tau_i, v}(\bX) e_1^\top\,,
    \end{equation}
    where $s = +1$ if $h$ is even and $s = -1$ if $h$ is odd.
    By Proposition~\ref{prop:Gtauv} and Item 4 of Lemma~\ref{lem:main_lbd_step}, when $a_1 \neq a_2$ then we have $F(\bX) = h(\iprod{\frac{1}{2}\cdot w,z_1-z_2}) = Y$ when $h$ is either even or odd. By Proposition~\ref{prop:Gtauv} and Item 2 of Lemma~\ref{lem:main_lbd_step}, when $a_1 = a_2$ and $h$ is even, we have $F(\bX) = (2\sum_i \lambda_i) \cdot \phi(\iprod{v,z_1 - z_2}) = 0 = Y$. When $a_1 = a_2$ and $h$ is odd, then we also have $F(\bX) = 0$, because $G^{-1}_{\tau,v}(\bX) = 0$.
    
    Recall from the proof of Lemma~\ref{lem:main_lbd_step} that $m = 2dM+2$, $(\tau_1,\ldots,\tau_m) = \rho\cdot (-dM,\ldots,dM,dM+1)$, and $v = \rho\cdot w$ for $\rho = \Theta(\log dM)$. So $F$ is an $(8dM+8)$-head attention layer with rank-1 attention matrices $\pm e_1 \cdot (\tau_i, v)^\top, \pm e_1 \cdot (-\tau_i, v)^\top$ of norm $\Theta(\sqrt{d}\log dM)$ and rank-1 projection matrices $\pm \lambda_i (\tau_i, v)\cdot e_1^\top$ and $\pm \lambda_i (-\tau_i, v)\cdot e_1^\top$ of norm $\Theta(d^{3/2}M\log dM)$.

    Now suppose there is a learning algorithm that, given random examples labeled by $F$, produces a hypothesis $\wh{F}: \brc{\pm 1}^{2\times d}\to\R^{2\times d}$ such that with high probability over the examples, $\norm{\norm{F(\bX) - \wh{F}(\bX)}^2} \le \epsilon^2$. We may assume without loss of generality that $\wh{F}$ is such that under any $\bX = ((a_1,z_1),(a_2,z_2))$, it outputs $0_{2\times d}$ when $a_1 = a_2$, and furthermore, $\wh{F}((1,z_1),(-1,z_2)) = \wh{F}((-1,z_1),(1,z_2))$ (otherwise, we can modify $\wh{F}$ to enforce these constraints without increasing the test loss).
    
    Now consider the function $\wh{f}: \brc{\pm 1}^{2d}\to\brc{\pm 1}$ defined as follows. For $(z_1,z_2)\in\brc{\pm 1}^{2d}$, let $\wh{f}(z_1,z_2)$ be the top left entry of $\wh{F}((1,z_1),(-1,z_2))$. Then
    \begin{align}
        \E{(f(z_1,z_2) - \wh{f}(z_1,z_2))^2} &\le \E{\norm{F(\bX) - \wh{F}(\bX)}^2_F \mid (a_1, a_2) = (1,-1)} \\
        &= \E{\norm{F(\bX) - \wh{F}(\bX)}^2_F \mid a_1 \neq a_2} \\
        &= 2\epsilon^2\,,
    \end{align}
    where the first step follows by the fact that error on the top left entry of the output of the multi-head attention layer is upper bounded by the overall Frobenius norm, the second step follows by the fact that $\wh{F}((1,z_1),(-1,z_2)) = \wh{F}((-1,z_1),(1,z_2))$, and the last step follows by the fact that $a_1 \neq a_2$ with probability $1/2$ and $\wh{F}(\bX) = F(\bX) = \vec{0}_{2\times d}$ when $a_1 = a_2$.
\end{proof}

\subsection{Cryptographic lower bound}

Here we state the cryptographic conjecture under which we prove hardness. The conjecture is slightly non-standard, so to provide context we first state the ``standard'' version of the conjecture:

\begin{definition}[Learning with rounding]\label{def:lwr}
    For positive integers $p < q$, security parameter $n$, and secret vector $w\in\mathbb{Z}^n_q$, define the function $f_w: \mathbb{Z}^n_q \to \mathbb{Z}_p$ by
    \begin{equation}
        f_w(x) \triangleq \lfloor \langle w,x\rangle \rceil_p \triangleq \lfloor \frac{p}{q}(\langle w,x\rangle \ \mathrm{mod} \ {q})\rceil\,,
    \end{equation}
    where $\lfloor z\rceil$ denotes the nearest integer to $z$. The \emph{learning with rounding} (LWR) problem~\cite{banerjee2012pseudorandom}, denoted $\mathsf{LWR}_{n,p,q}$, is the following. A secret vector $w$ is drawn uniformly at random from $\mathbb{Z}^n_q$, one receives as input $\poly(n)$ labeled examples of the form $(x,y)$ where $x\sim\mathbb{Z}^n_q$, and the goal is to distinguish between the following two scenarios:
    \begin{enumerate}
        \item The labels $y$ are independent random draws from the uniform distribution over $\mathbb{Z}_p$, or
        \item The labels $y$ are given by $f_w(x)$.
    \end{enumerate}
    This can be thought of as a noiseless version of the well-known \emph{learning with errors} (LWE) problem~\cite{regev2009lattices}.
\end{definition}

\begin{conjecture}[Hardness of LWR with polynomial modulus is hard]\label{conj:lwr_pre}
    When $q = \poly(n)$, there is no $\poly(n)$-time algorithm for $\mathsf{LWR}_{2,q,n}$.
\end{conjecture}

\noindent LWR with polynomial modulus is conjectured to be as hard as worst-case lattice problems~\cite{banerjee2012pseudorandom} and underlies several leading post-quantum cryptographic proposals, e.g.~\cite{d2018saber}, though it remains open to establish a direct reduction from the more standard LWE assumption~\cite{bogdanov2015hardness}. Conjecture~\ref{conj:lwr_pre} was the basis for the recent cryptographic hardness result of~\cite{chen2022hardness} for PAC learning two-hidden-layer feed-forward networks over the Gaussian distribution.

In this work, we will need a slightly stronger cryptographic assumption. The reason is that our construction in Section~\ref{sec:reductiontoprojfunctions} can only capture LWR functions $f_w$ for specific choices of $w$, namely the ones given in the following lemma:

\begin{lemma}\label{lem:implement_round}
    Let $q$ be a power of $2$. Given $x\in\mathbb{Z}^{2d}_q$, express it as $(z_1[x], z_2[x])\in\{\pm 1\}^{2d\log_2 q}$ in the natural way, where we represent any $t\in\brc{0,\ldots,q-1}$ as $s\in\{\pm 1\}^{\log_2 q}$ via
    \begin{equation}
        t = \sum^{\log_2(q) - 1}_{i = 0} 2^i\cdot \frac{s_i + 1}{2}\,.
    \end{equation}

    Let $w = (w_1,w_2)\in\mathbb{Z}^{2d}_q$ be such that $w_1 = -w_2$. Then there exists $S\subseteq[d]$ and odd function $h: \mathbb{Z}\to\brc{0,1}$ and for any $x\in\mathbb{Z}^{2d}_q$, we have
    \begin{equation}
        f_w(x) = h(\iprod{\frac{1}{2}\cdot w', z_1[x] - z_2[x]})
    \end{equation}
    for $w'\in\brc{-q,\ldots,q}^{d\log_2 q}$.
\end{lemma}

\begin{proof}
    Write $x = (x_1,x_2)$ for $w_1,w_2, x_1, x_2\in\mathbb{Z}^{d}_q$. Note that $f_w(x)$ only depends on the quantity $\langle w,x\rangle$. Regarding the entries of $w,x$ as elements of $\brc{0,\ldots,q-1}\subset\Z$, we have
    \begin{align}
        \langle w, x\rangle &= \sum^d_{j=1} (w_1)_j (x_1)_j + (w_2)_j (x_2)_j\\
        &= \sum^d_{j=1} \sum^{\log_2(q)-1}_{i=0} \Bigl((w_1)_j \cdot 2^i \cdot \frac{(z_1[x])_{j,i} + 1}{2} + (w_2)_j \cdot 2^i \cdot \frac{(z_2[x])_{j,i} + 1}{2} \Bigr) \\
        &= \sum^d_{j=1} \sum^{\log_2(q) - 1}_{i=0} \frac{1}{2}(w_1)_j \cdot 2^i \cdot (z_1[x] - z_2[x])_{j,i}\,,  
    \end{align}  
    where $z_1[x]_{j,i}$ denotes the $i$-th coordinate of the $j$-th block of $\log_2 q$ coordinates of $z_1[x]$, and similarly for $z_2[x]_{j,i}$, and where in the last step we used the assumption that $w_1 = -w_2$. Letting $w'\in\mathbb{Z}^{d\log_2 q}$ denote the vector whose $i$-th coordinate of the $j$-th block of $\log_2 q$ coordinates is given by $(w_1)_j \cdot 2^i$. Then the above implies that $\langle w,x\rangle = \langle \frac{1}{2}\cdot w', z_1[x] - z_2[x]\rangle$, and $f_w(x)$ is a Boolean-valued function of this as claimed. Finally, note that the coordinates of $w'$ are bounded in magnitude by $q$, in fact by $q/2$, though this constant factor difference is immaterial to us.
\end{proof}

\noindent Because we can only implement certain LWR functions $f_w$ in the above construction, we need a version of Definition~\ref{def:lwr} and Conjecture~\ref{conj:lwr_pre} which imposes restrictions on $w$.

\begin{definition}[LWR with secret leakage]\label{def:LWR_with_secret}
    Let $n,p,q,$ be as in Definition~\ref{def:lwr}, and assume $q$ is a power of $2$. Given integer $k\le n$, define the \emph{learning with rounding problem with $k$ bits of secret leakage}, denoted $\mathsf{LWR}_{p,q,n}[k]$ to be a version of the distinguishing problem where the secret vector $w$ is not drawn uniformly at random from $\mathbb{Z}^n_q$ but instead from any distribution over $\mathbb{Z}^n_q$ with \emph{min-entropy} at least $n\log_2(q) - k$.\footnote{Recall that a discrete random variable supported over a domain $\Omega$ of size $2^m$ has min-entropy $\ell$ if no element of $\Omega$ has probability mass greater than $1/2^\ell$}
\end{definition}

\begin{conjecture}\label{conj:lwr}
    When $q = \poly(n)$ and $k = n/2$, there is no $\poly(n)$-time algorithm for $\mathsf{LWR}_{2,q,n}[k]$.
\end{conjecture}

\noindent This version of LWR was studied in~\cite{alwen2013learning}, who showed that this conjecture holds in a weaker regime (superpolynomial modulus).

We are now ready to prove our main cryptographic lower bound:

\begin{proof}[Proof of Part 1 of Theorem~\ref{thm:main_lbd}]
    Let $n = 2d$. Consider the distribution over secret vectors $w = (w_1,w_2) \in \R^{2d}$ where $w_1 = -w_2 \in\mathbb{Z}^d_q$. By Lemma~\ref{lem:implement_round}, the rounding function $f_w(x)$ for $\mathsf{LWR}_{2,q,2d}$ can be implemented as $h(\langle \frac{1}{2}\cdot w', z_1[x] - z_2[x])$, where $(z_1[x], z_2[x])\in\brc{\pm 1}^{2d\log_2 q}$ is the encoding of $x$ defined in the Lemma. Note that this distribution over secret vectors is uniform over a subset of $\mathbb{Z}^{2d}_q$ of size $q^d$, so the min-entropy is at least $d\log_2(q)$. The resulting distinguishing problem is thus an instance of learning with rounding with $d\log_2(q)$ bits of secret leakage. So by Lemma~\ref{lem:reduction_sample}, a polynomial-time algorithm for PAC learning multi-head attention layers $F: \brc{\pm 1}^{2\times d}\to\brc{\pm 1}^{2\times d}$ with $\poly(m)$ heads and attention and projection matrices of norm at most $\poly(d)$ would imply a polynomial-time algorithm for the distinguishing problem in Def~\ref{def:LWR_with_secret}. Concretely, Lemma~\ref{lem:reduction_sample} lets us convert the dataset given by the LWR with secret leakage problem into a dataset for learning multi-head attention over the uniform distribution over $\brc{\pm 1}^{2\times d}$. If the test error by the PAC learning algorithm is nearly trivial, then the original LWR instance must have been given by random labels, and conversely. By estimating the test error on training data, one can use this observation to get a procedure that violates Conjecture~\ref{conj:lwr}.
\end{proof}

\subsection{Statistical query lower bound}

\begin{proof}[Proof of Part 2 of Theorem~\ref{thm:main_lbd}]
    Consider the even function $h(x) \triangleq (-1)^{\bone{x \ \text{is even}}}$. Note that for $z = (z_1,z_2)\in\{\pm 1\}^{2d}$ and $S\subseteq[d]$ of size $m$, if $T\subseteq[2d]$ denotes the subset which contains $i$ and $i + d$ for all $i\in S$, then the quantity $\frac{1}{2}\iprod{\one_S, z_1 - z_2}$ is even-valued if and only if the Fourier basis function $z_T$ evaluates to $1$. So $h(\frac{1}{2}\iprod{\one_S,z_1-z_2})$ exactly computes the parity function on $\brc{\pm 1}^{2d}$ corresponding to the subset $T$.

    The set of parity functions over subsets $T = \brc{i: i\in S} \cup \brc{i + d: i\in S}$ for $S\subseteq[d]$ of size $O(m)$ has statistical dimension $d^{\Theta(m)}$: for any $T,T'$ of this form, $z_T$ and $z_{T'}$ are orthogonal with respect to the uniform measure on the cube. This immediately implies that any SQ algorithm for learning parities of this form over $\brc{\pm 1}^{2d}$ requires either $d^{\Omega(m)}$ queries or $d^{-\Omega(m)}$ tolerance (see e.g. Theorem 2 in~\cite{szorenyi2009characterizing}).

    The proof is complete upon observing that the reduction in Lemma~\ref{lem:reduction_sample} converted random examples labeled by $f(z_1,z_2) = h(\iprod{\frac{1}{2}\cdot\one_S, z_1-z_2})$ \emph{exactly} into random examples labeled by a multi-head attention layer with $\Theta(m)$ heads and token size 2. The same construction thus also implies that one can convert \emph{SQ} access to the former into SQ access to the latter.
\end{proof}

\begin{remark}
    Note that the problem instance constructed in the proof of the SQ lower bound is only hard for SQ algorithms. If one allows non-SQ algorithms, then because the construction exactly implements a parity function, one can PAC learn by running Gaussian elimination. That said, the SQ part of Theorem~\ref{thm:main_lbd} still implies that a large family of techniques like method of moments, noisy gradient descent, PCA, etc. are insufficient to improve upon the $d^{\Theta(m)}$ scaling. On the other hand, the first part of Theorem~\ref{thm:main_lbd} is stronger in the sense that it rules out arbitrary polynomial-time algorithms for learning multi-head attention layers with polynomially many heads and polynomially bounded attention/projection matrices, but weaker in the sense that it does not prove the same fine-grained exponential dependence on $m$ that the SQ lower bound does. Note that this exact tradeoff between SQ and cryptographic lower bounds for learning neural networks is also present in the prior work of~\cite{chen2022hardness} on feed-forward networks.
\end{remark}

\paragraph{Acknowledgments.} We thank Surbhi Goel and Adam Klivans for helpful feedback and discussions during the preparation of this paper.

\bibliographystyle{alpha}
\bibliography{biblio}

\appendix

\newpage 

\paragraph{Roadmap.} In Appendix~\ref{app:HW}, we compile some concentration inequalities over the product of the uniform distribution over the cube with the uniform distribution over the slice of some other cube. In Appendix~\ref{app:integrolocal}, we give a self-contained proof of a quantitative version of the ``integro-local central limit theorem.'' In Appendix~\ref{app:arithmetic}, we prove that the non-arithmeticity condition in Assumption~\ref{assume:nonarithmetic} holds in a natural smoothed analysis setting and connect this assumption to the characteristic function bound of Eq.~\eqref{eq:sufficient_lam} in Theorem~\ref{thm:borovkov}.


\section{Polynomial concentration for product of cube and slice}
\label{app:HW}


In this section we prove Lemma~\ref{lem:slice_lip_linear} and Theorem~\ref{thm:main_HW}, i.e. concentration of linear and quadratic forms over $\pi^s_{\mu,\nu;d_1,d_2,d_3}$. These bounds follow entirely from standard techniques in the literature but, to our knowledge, was not available in an off-the-shelf form. As such, we provide a self-contained proof here.

\subsection{Preliminaries}

Let $P$ be the transition matrix of a reversible Markov chain over $\brc{\pm 1}^d$ with stationary distribution $\pi$. Given $x\in\brc{\pm 1}^d$, we use the notation $y\sim x$ denote a sample obtained by taking one step of the Markov chain starting at $x$.

Given functions $f,g: \brc{\pm 1}^d\to\R$, define
\begin{equation}
    \Ent_\pi(f) \triangleq \E[\pi]{f\log f} - \E[\pi]{f} \log \E[\pi]{f} \qquad \calE(f,g) \triangleq \frac{1}{2}\,\E[x\sim\pi, y\sim x]{(f(x) - f(y))(g(x) - g(y))}\,.
\end{equation}

\begin{definition}
    We say that $\pi$ satisfies a \emph{modified log-Sobolev inequality (MLSI) with constant $\rho$} if for all functions $f: \brc{\pm 1}^d\to\R_{>0}$,
    \begin{equation}
        \Ent_\pi(f) \le \rho\cdot \calE(f,\log f)\,.
    \end{equation}
\end{definition}

\noindent Subadditivity of entropy implies the following standard property:

\begin{fact}[Tensorization]\label{lem:tensorize}
    Let $P_1, P_2, P_3$ be transition matrices of reversible Markov chains over $\brc{\pm 1}^{d_1}, \brc{\pm 1}^{d_2}, \brc{\pm 1}^{d_3}$ with stationary distributions $\pi_1,\pi_2,\pi_3$ satisfying MLSI's with constants $\rho_1,\rho_2,\rho_3$ respectively. Given $0\le \lambda_1,\lambda_2,\lambda_3 \le 1$ satisfying $\lambda_1 + \lambda_2 + \lambda_3 = 1$, consider the Markov chain over $\brc{\pm 1}^{d_1 + d_2 + d_3}$ with transition matrix $P \triangleq \lambda_1 \cdot P_1 \otimes\Id_{d_2} \otimes \Id_{d_3} + \lambda_2 \cdot \Id_{d_1}\otimes P_2\otimes \Id_{d_3} + \lambda_3 \cdot \Id_{d_1}\otimes \Id_{d_2}\otimes P_3$. The stationary distribution for $P$ is $\pi_1 \otimes \pi_2\otimes p_3$ and satisfies an MLSI with constant $\max(\rho_1 / \lambda_1, \rho_2 / \lambda_2, \rho_3 / \lambda_3)$.
\end{fact}

\noindent With tensorization we can readily verify the following standard fact:

\begin{lemma}\label{lem:mlsi_cube}
    Let $P_{\sf cube}$ be the transition matrix associated to the Markov chain over $\brc{\pm 1}^d$ which at every step picks a random $i\in[d]$ and flips the $i$-th bit with probability $1/2$. Its stationary distribution is the uniform distribution $\unif{d}$, which satisfies an MLSI with constant $\Theta(d)$.
\end{lemma}

\begin{lemma}[Theorem 2 in \cite{gao2003exponential}]\label{lem:mlsi_slice}
    Let $P^\mu_{\sf slice}$ be the transition matrix associated to the Markov chain over $\slice{d}{\mu}$ which at every step picks a random pair of indices $1 \le i < j \le d$ and flips bits $i$ and $j$. Its stationary distribution is $\sldist{d}{\mu}$, which satisfies an MLSI with constant $\Theta(d)$.
\end{lemma}

\noindent Combining Lemmas~\ref{lem:tensorize},~\ref{lem:mlsi_cube},~\ref{lem:mlsi_slice}, we obtain the following bound for $\unif{d_1} \otimes \slice{d_2}{\mu} \otimes \slice{d_2}{\nu}$:

\begin{corollary}\label{cor:chain_on_product}
    Given probability vector $\lambda = (\lambda_1,\lambda_2,\lambda_3)$, let $P_\lambda$ denote the transition matrix associated to the Markov chain over $\brc{\pm 1}^{d_1 + d_2+d_3}$ which at every step, with probability $\lambda_1$, applies one step of the Markov chain associated to $P_{\sf cube}$ to the first $d_1$ bits or, with probability $\lambda_2$, applies one step of the Markov chain associated to $P^\mu_{\sf slice}$ to the next $d_2$ bits or, with probability $\lambda_3$, applies one step of the Markov chain associated to $P^\nu_{\sf slice}$ to the last $d_3$ bits . Equivalently, $P_\lambda \triangleq \lambda_1\cdot P_{\sf cube}\otimes \Id_{d_2} \otimes \Id_{d_3} + \lambda_2 \Id_{d_1} \otimes P^\mu_{\sf slice}\otimes \Id_{d_3} + \lambda_3 \Id_{d_1}\otimes \Id_{d_2}\otimes P^\nu_{\sf slice}$. Then the stationary distribution of $P_\lambda$ is $\unif{d_1}\otimes \slice{d_2}{\mu}\otimes \slice{d_3}{\nu}$, which satisfies an MLSI with constant $\Theta(\max(\frac{d_1}{\lambda_1}, \frac{d_2}{\lambda_2}, \frac{d_3}{\lambda_3}))$.
\end{corollary}

\begin{definition}
    Define the \emph{difference operator} $\Gamma$ so that given $g:\Omega\to\R$ and $x\in\Omega$, 
    \begin{equation}
        \Gamma(g)(x)^2 \triangleq d\cdot \E[y\sim x]{\max(0,g(x) - g(y))^2}\,.
    \end{equation}
\end{definition}

\begin{fact}
    If $\pi$ satisfies an MLSI with constant $\rho\cdot d$, then for any $f:\Omega\to\R_{>0}$,
    \begin{equation}
        \Ent_\pi(f) \le 2\rho \cdot \E[\pi]{\Gamma(\log f)^2\cdot f}\,. \label{eq:mlsi_consequence}
    \end{equation}
\end{fact}

\begin{proof}
    Note that for $a,b > 0$, we have the elementary inequality $(a - b)(e^a - e^b) \le \max(0,a - b)^2(e^a + e^b)$, so
    \begin{align}
        \calE(f,\log f) &\le \frac{1}{2}\E[x,y]{(\log f(x) - \log f(y))^2(f(x) + f(y))} = \E[x,y]{\max(0,\log f(x) - \log f(y))^2 f(x)}
    \end{align}
    as claimed.
\end{proof}

\subsection{Convex concentration}

Here we verify that $\pi^s_{\mu,\nu;d_1,d_2,d_3}\triangleq \unif{d_1}\otimes \sldist{d_2}{\mu}\otimes \sldist{d_3}{\nu}$ satisfies an analogue of Talagrand's convex distance inequality. Such a result was already known in the special case where either exactly one of $d_1,d_2,d_3$ is nonzero \cite{talagrand1995concentration,boucheron2003concentration,boucheron2009concentration,paulin2014convex}. We closely follow the entropy method of \cite{boucheron2003concentration}, which was extended to the slice in \cite{sambale2021modified}. Let $d\triangleq d_1 + d_2 + d_3$. For convenience, in this section we will refer to the domain $\cube{d_1}\times \sldist{d_2}{\mu}\times \sldist{d_3}{\nu}$ as $\Omega$, and the distribution $\pi^s_{\mu,\nu;d_1,d_2,d_3}$ as $\pi$. For
\begin{equation}
    \lambda_i = d_i / d\,,
\end{equation} 
denote the transition matrix $P_\lambda$ defined in Corollary~\ref{cor:chain_on_product} by $P$. Note that by Corollary~\ref{cor:chain_on_product}, $\pi$ satisfies an MLSI with constant $d$.

Given $x\in\Omega$, let $N(x)$ denote the set of neighbors of $x$ under the Markov chain associated to $P$.

\begin{definition}[Convex distance]
    Given $x\in\Omega$ and $A\subseteq\Omega$, define
    \begin{equation}
        \tal(x,A) \triangleq \max_{\alpha\in\S^{d-1}} \min_{y\in A} \sum^d_{i=1} \alpha_i \cdot \bone{x_i \neq x'_i}\,.
    \end{equation}
\end{definition}

\begin{lemma}\label{lem:selfbounded}
    Let $f: \Omega\to\R$ be a nonnegative function such that for all $x\in\Omega$,
    \begin{equation}
        \Gamma(f)(x)^2 \lesssim f(x) \label{eq:self}
    \end{equation}
    and furthermore $|f(x) - f(y)| \le 1$ for all $y\in N(x)$. Then
    for all $0\le t \le \E[\pi]{f}$,
    \begin{equation}
        \Pr[\pi]{f \le \E[\pi]{f} - t} \le \exp(-\Omega(t^2/\E[\pi]{f}))\,.
    \end{equation}
\end{lemma}

\begin{proof}
    Corollary~\ref{cor:chain_on_product} applied to the function $e^{-\beta f}$ for any $0 \le \beta \le 1$ implies that
    \begin{align}
        \MoveEqLeft\Ent_\pi(e^{-\beta f}) \\
        &\lesssim \beta\,\mathbb{E}_{x\sim\unif{d_1}, x'\sim\slice{d_2}{\mu}, x''\sim\slice{d_3}{\nu}}\Bigl[\sum^{d_1}_{i=1} (f(x,x',x'') - f(x^{[i]},x',x''))(e^{-\beta f(x^{[i]}, x', x'')} - e^{-\beta f(x,x',x'')}) \\
        &\qquad\qquad\qquad  + \frac{1}{d_2}\,\sum_{1 \le i < j \le d_2} (f(x,x',x'') - f(x,\tau_{ij}(x'),x''))(e^{-\beta f(x,\tau_{ij}(x'),x'')} - e^{-\beta f(x,x',x'')})\Bigr] \\
        &\qquad\qquad\qquad  + \frac{1}{d_3}\,\sum_{1 \le i < j \le d_3} (f(x,x',x'') - f(x,x',\tau_{ij}(x'')))(e^{-\beta f(x,x'\tau_{ij}(x''))} - e^{-\beta f(x,x',x'')})\Bigr] \\
        &= 2\beta\,\mathbb{E}_{x,x',x''}\Bigl[\sum^{d_1}_{i=1} \max(0,f(x,x',x'') - f(x^{[i]},x',x''))(e^{-\beta f(x^{[i]}, x', x'')} - e^{-\beta f(x,x',x'')}) \\
        &\qquad + \frac{1}{d_2}\,\sum_{1 \le i < j \le d_2} \max(0,f(x,x',x'') - f(x,\tau_{ij}(x'),x''))(e^{-\beta f(x,\tau_{ij}(x'),x'')} - e^{-\beta f(x,x',x'')})\Bigr] \\
        &\qquad + \frac{1}{d_3}\,\sum_{1 \le i < j \le d_3} \max(0,f(x,x',x'') - f(x,x',\tau_{ij}(x'')))(e^{-\beta f(x,x',\tau_{ij}(x''))} - e^{-\beta f(x,x',x'')})\Bigr]\,. \label{eq:entbound}
    \end{align}
    As $(x^{[i]}, x',x'') \in N((x,x',x''))$, we can use the assumption that $|f(x^{[i]},x',x'') - f(x,x',x'')| \le 1$ and the elementary inequality $e^z \le 2\max(0,z)$ for $z\in[0,1]$ to conclude that
    \begin{align}
        e^{-\beta f(x^{[i]},x',x'')} - e^{-\beta f(x,x',x'')} &\le e^{-\beta f(x,x',x'')} \cdot (e^{-\beta (f(x^{[i]},x'') - f(x,x''))} - 1) \\
        &\le 2\beta \,e^{-\beta f(x,x',x'')}\,\max(0,f(x,x',x'') - f(x^{[i]},x',x''))\,.\label{eq:taylor1}
    \end{align}
    Analogously, we have that
    \begin{equation}
        e^{-\beta f(x,\tau_{ij}(x'),x'')} - e^{-\beta f(x,x',x'')} \le 2\beta \,e^{-\beta f(x,x',x'')}\,\max(0,f(x,x',x'') - f(x,\tau_{ij}(x'),x''))\,,\label{eq:taylor2}
    \end{equation}
    and likewise for $e^{-\beta f(x,x',\tau_{ij}(x''))} - e^{-\beta f(x,x',x'')}$
    Substituting these into \eqref{eq:entbound} and recalling \eqref{eq:self}, we get
    \begin{equation}
        \Ent_\pi(e^{-\beta f}) \lesssim \beta^2 \E[x,x',x'']{e^{-\beta f(x,x',x'')} f(x,x',x'')} \le \beta^2 \,\E[\pi]{e^{-\beta f}}\cdot \E[\pi]{f}\,, \label{eq:usechebyshev}
    \end{equation}
    where the last step follows by Chebyshev's association inequality. Setting $h(\beta) \triangleq \E[\pi]{e^{-\beta f}}$, observe that 
    \begin{equation}
        \frac{\partial}{\partial \beta} \frac{\log h(\beta)}{\beta} = \frac{1}{\beta^2}\Bigl(\frac{\beta h'(\beta)}{h(\beta)} - \log h(\beta)\Bigr) = \frac{1}{\beta^2}\Bigl(\frac{\E[\pi]{f e^{-\beta f}}}{\E[\pi]{e^{-\beta f}}} - \log h(\beta)\Bigr) = \frac{1}{\beta^2} \frac{\Ent_\pi(e^{-\beta f})}{\E[\pi]{e^{-\beta f}}}\,,
    \end{equation}
    so \eqref{eq:usechebyshev} implies that for all $0 \le \beta \le 1$,
    \begin{equation}
        \frac{\partial}{\partial \beta}\frac{\log h(\beta)}{\eta} \lesssim \E[\pi]{f}\,.
    \end{equation}
    Integrating, taking exponentials on both sides, and rearranging shows that for any $0\le \beta \le 1$,
    \begin{equation}
        \E[\pi]{\exp(\beta(\E{f} - f))} \le \exp(O(\beta^2 \E{f}))\,,
    \end{equation}
    so the lemma follows by Markov's upon taking $\beta = \Theta(t / \E{f})$.
\end{proof}

\begin{lemma}\label{lem:conv_dist_selfbounded}
    For any $A\subseteq\Omega$, the function $f(x) \triangleq c\cdot\tal(x, A)^2$ satisfies the hypotheses of Lemma~\ref{lem:selfbounded} for some absolute constant $c > 0$.
\end{lemma}

\begin{proof}
    We first prove the second hypothesis in Lemma~\ref{lem:selfbounded} holds. By Sion's minimax theorem,
    \begin{equation}
        \tal(x,A) = \min_\nu \max_{\alpha\in\S^{d-1}} \sum^d_{i=1} \alpha_i \Pr[x'\sim\nu]{x_i \neq x'_i}\,, \label{eq:sion}
    \end{equation}
    where the minimum is over all probability measures over $\Omega$. By Cauchy-Schwarz, we have
    \begin{equation}
        \tal(x,A)^2 \le \min_\nu \sum^d_{i=1} \Pr[x'\sim\nu]{x_i \neq x'_i}^2\,.
    \end{equation}
    Now take any $y\in N(x)$. Because $y$ and $x$ differ on at most two bits, we see that for any $\nu$,
    \begin{equation}
        \sum^d_{i=1} \Pr[x'\sim\nu]{x_i \neq x'_i}^2 - \Pr[x'\sim\nu]{y_i \neq y'_i}^2 \le 2\,,
    \end{equation}
    so it suffices to take $c \le 1/2$ in the definition of $f(x)$ to ensure that $|f(y) - f(x)| \le 1$.

    It remains to verify the first hypothesis in Lemma~\ref{lem:selfbounded}. Take any $(x,x',x'')\in\Omega$, where $x\in\cube{d_1}$, $x'\in\slice{d_2}{\mu}$, $x''\in\slice{d_3}{\nu}$, and let $\wt{\nu}, \wt{\alpha}$ denote the parameters under which the value $\tal(x,A)$ is attained. For any $(y,y',y'')\in N((x,x',x''))$, let $\wh{\nu}_{y,y',y''}$ denote the minimizer of $\inf_\nu \sum^d_{i=1} \wt{\alpha}_i \Pr[(z,z',z'')\sim\nu]{(z,z',z'') \neq (y,y',y'')}$. Then
    \begin{align}
        \MoveEqLeft
        \Gamma(\tal(\cdot, A))(x,x',x'')^2 \\
        &\le \sum^{d_1}_{i=1} \max\Bigl(0,\sum^d_{k=1} \wt{\alpha}_k\Bigl\{\Pr[\wh{\nu}_{x^{[i]},x',x''}]{(z,z',z'') \neq (x,x',x'')} - \Pr[\wh{\nu}_{x^{[i]},x',x''}]{(z,z',z'') \neq (x^{[i]},x',x'')}\Bigr\}\Bigr)^2 \\
        &\qquad + \frac{1}{d_2} \sum_{1 \le i < j \le d_2} \max\Bigl(0,\sum^d_{k=1} \wt{\alpha}_k\Bigl\{\Pr[\wh{\nu}_{x,\tau_{ij}(x'),x''}]{(z,z',z'') \neq (x,x',x'')} - \Pr[\wh{\nu}_{x,\tau_{ij}(x'),x''}]{(z,z',z'') \neq (x,\tau_{ij}(x'),x'')}\Bigr\}\Bigr)^2 \\
        &\qquad + \frac{1}{d_3} \sum_{1 \le i < j \le d_3} \max\Bigl(0,\sum^d_{k=1} \wt{\alpha}_k\Bigl\{\Pr[\wh{\nu}_{x,x',\tau_{ij}(x'')}]{(z,z',z'') \neq (x,x',x'')} - \Pr[\wh{\nu}_{x,x',\tau_{ij}(x'')}]{(z,z',z'') \neq (x,x',\tau_{ij}(x''))}\Bigr\}\Bigr)^2 \\
        &\le \sum^{d_1}_{i=1} \wt{\alpha}_i^2 + \frac{1}{d_2} \sum_{1\le i < j \le d_2} (\wt{\alpha}_i^2 + \wt{\alpha}_j^2) + \frac{1}{d_3} \sum_{1\le i < j \le d_3} (\wt{\alpha}_i^2 + \wt{\alpha}_j^2) \lesssim 1\,. \label{eq:talsmall}
    \end{align}
    Because $\tal(\cdot,A)$ is nonnegative, we conclude that
    \begin{equation}
        \Gamma(f)(\cdot)^2 \lesssim \tal(\cdot,A)^2 \,\Gamma(\tal(\cdot,A))^2 \lesssim \tal(\cdot,A)^2\,,
    \end{equation}
    so the first part of the hypothesis of Lemma~\ref{lem:selfbounded} holds as desired.
\end{proof}

\noindent We will also use the following consequence of $\pi$ satisfying an MLSI:

\begin{lemma}[Eq. (2.4) in \cite{bobkov1999exponential}]\label{lem:bg}
    If $\pi$ satisfies \eqref{eq:mlsi_consequence}, then for any $g: \Omega\to\R_{>0}$ which satisfies $\Gamma(g^2) \le 2g\, \Gamma(g)$ and $\Gamma(g) \le 1$, we have for all $0 \le \beta < (8\rho)^{-1})$ that
    \begin{equation}
        \E[\pi]{\exp(\beta g^2)} \le \exp\Bigl(\frac{\beta}{1 - 8\rho\beta}\,\E[\pi]{g^2}\Bigr)\,.
    \end{equation}
\end{lemma}

\noindent We are now ready to prove a version of Talagrand's convex distance inequality for the distribution $\pi$:

\begin{corollary}\label{cor:talagrand}
    There is an absolute constant $b > 0$ such that 
    \begin{equation}
        \Pr[\pi]{x\in A}\cdot \E[\pi]{\exp(b\cdot \tal(\cdot,A)^2)} \le 1
    \end{equation}
    for every $A\subseteq \Omega$.
\end{corollary}

\begin{proof}
    Note that $\tal(x,A) = 0$ if and only if $x\in A$. So substituting $f(x) \triangleq c\cdot \tal(x,A)^2$ from Lemma~\ref{lem:conv_dist_selfbounded} into Lemma~\ref{lem:selfbounded} and taking $t = \E{f}$, we find that 
    \begin{equation}
        \Pr[\pi]{x\in A}\cdot \exp(c'\,\E[\pi]{\tal(\cdot, A)^2}) \le 1 \label{eq:pre_tal_ineq1}
    \end{equation}
    for some absolute constant $c' > 0$.
    
    Note that $\Gamma$ satisfies $\Gamma(g^2) \le 2g \Gamma(g)$ for all positive functions $g$, and furthermore $\Gamma(\tal(\cdot,A)) \lesssim 1$ by \eqref{eq:talsmall}. So for $g = c''\cdot \tal(x,A)$ for $c'' > 0$ sufficiently small, we can apply Lemma~\ref{lem:bg} to get
    \begin{equation}
        \E[\pi]{\exp(\beta c''^2 \tal(\cdot, A)^2)} \le \exp\Bigl(\frac{\beta}{1 - 8\rho\beta}\,\E[\pi]{c''^2 \tal(\cdot,A)^2}\Bigr) \label{eq:pre_tal_ineq2}
    \end{equation}
    for all $0 \le \beta < (8\rho)^{-1}$, where $\rho\cdot d$ is the MLSI constant in Corollary~\ref{cor:chain_on_product}. If we take $\beta$ to be the absolute constant which solves $\frac{\beta}{1 - 8\rho\beta}\cdot c''^2 = c'$, then by combining \eqref{eq:pre_tal_ineq1} and \eqref{eq:pre_tal_ineq2} we obtain the desired inequality.
\end{proof}

\noindent We will now use Corollary~\ref{cor:talagrand} to show that $\pi$ satisfies the following property:

\begin{definition}
    We say that a distribution $D$ over $\R^d$ satisfies the \emph{convex concentration property with constant $K > 0$} if for all convex, $1$-Lipschitz functions $f: \R^d\to\R$, we have that $\mathbb{E}_D |f| < \infty$ and, for all $t > 0$,
    \begin{equation}
        \Pr[D]{|f - \E{f}| > t} \le 2\exp(-t^2/K^2)\,.
    \end{equation}
\end{definition}

\begin{lemma}\label{lem:pi_conv_conc}
    The distribution $\pi$ satisfies the convex concentration property with constant $K = \Theta(1)$.
\end{lemma}

\begin{proof}
    Let $f:\R^d\to\R$ be any convex, $1$-Lipschitz function. Take the set $A$ in Corollary~\ref{cor:talagrand} to be $\brc{x\in \Omega: f(x) \le s}$ for $s$ to be chosen later. Corollary~\ref{cor:talagrand} combined with Markov's inequality tells us that
    \begin{equation}
        \Pr[\pi]{f(x) \le s}\cdot \Pr[\pi]{\tal(\cdot, A) > t} \le \exp(-\Omega(t^2))\,.
    \end{equation}
    The set $A$ is convex as $f$ is convex, so by \cite[Lemma 7.11]{boucheron2013concentration}, $\tal(x,A)$ upper bounds the Euclidean distance from $x$ to the closest point in $A$. So because $f$ is $1$-Lipschitz,
    \begin{equation}
        f(x) \le s + \tal(x,A),
    \end{equation}
    meaning that $f(x) \ge s + t$ implies $\tal(x,A) > t$. Therefore,
    \begin{equation}
        \Pr[\pi]{f(x) \le s}\cdot \Pr[\pi]{f(x) \ge s + t} \le \exp(-\Omega(t^2))\,.
    \end{equation}
    By taking $s = \med_\pi(f)$ and $s = \med_\pi(f) - t$, we conclude that
    \begin{equation}
        \Pr[\pi]{|f - \med_\pi(f)| > t} \le 4\exp(-\Omega(t^2))\,. \label{eq:medtail}
    \end{equation}
    As the left-hand side is always upper bounded by $1$, we can assume that the $\Omega(t^2)$ quantity in the exponent is at least $\ln 4$, in which case we can upper bound \eqref{eq:medtail} by $2\exp(-\Omega(t^2))$ with a larger smaller constant factor in the exponent. We conclude the proof of the lemma by invoking \cite[Lemma 3.2]{adamczak2014note} to go from concentration around the median to concentration around the mean.
\end{proof}

\noindent Convex concentration immediately implies Lemma~\ref{lem:slice_lip_linear}, i.e. concentration of linear forms.

\subsection{A Hanson-Wright inequality for \texorpdfstring{$\pi$}{pi}}

Having established the convex concentration property for $\pi$, we can now invoke the following result of \cite{adamczak2014note}:

\begin{theorem}\label{thm:adamczak}
    If $D$ is a distribution over $\R^d$ satisfying the convex concentration property with constant $K$, and if $\E[x\sim D]{x} = 0$, then for any $\bA\in \R^{d\times d}$ we have
    \begin{equation}
        \Pr{|x^\top\bA x - \E{x^\top \bA x}| \ge t} \le 2\exp\Bigl(-c\, \min\Bigl(\frac{t}{K^2\norm{\bA}_{\sf op}}, \frac{t^2}{2K^4\norm{\bA}_F^2}\Bigr)
    \end{equation}
    for some absolute constant $c > 0$.
\end{theorem}

\noindent As the mean of $\pi$ is nonzero, we will need to apply Theorem~\ref{thm:adamczak} to an appropriate shift of $\pi$ to obtain Theorem~\ref{thm:main_HW}, i.e. the analogue of Hanson-Wright for $\pi$ which we restate below for convenience:

\HW*

\begin{proof}
    The mean of $\pi$ is the vector $v \triangleq (0,\ldots,0,\mu,\ldots,\mu,\nu,\ldots,\nu)$ consisting of $d_2$ copies of $\mu$ and $d_3$ copies of $\nu$.
    Then by Theorem~\ref{thm:adamczak} and Lemma~\ref{lem:pi_conv_conc},
    \begin{equation}
        \Pr{|(x - v)^\top \bA (x - v) - \E{(x - v)^\top \bA (x - v)}| \ge t} \le 2\exp\Bigl(-c'\, \min\Bigl(\frac{t}{\norm{\bA}_{\sf op}}, \frac{t^2}{\norm{\bA}_F^2}\Bigr)
    \end{equation}
    for some absolute constant $c' > 0$.
    Note that
    \begin{equation}
        (x - v)^\top \bA (x - v) = x^\top \bA x - x^\top \bA v - v^\top \bA x + v^\top \bA v
    \end{equation}
    has expectation $\E{x^\top \bA x} - v^\top \bA v$. The function $x\mapsto x^\top \bA v + v^\top \bA x$ is $\norm{(\bA + \bA^\top) v}$-Lipschitz and convex, so by Lemma~\ref{lem:slice_lip_linear}
    \begin{align}
        \Pr{|x^\top \bA v + v^\top \bA x - 2v^\top \bA v| > t} &\le 2\exp(-\Omega(t^2 / \norm{(\bA + \bA^\top) v}^2)) \\
        &\le 2\exp\Bigl(-\Omega\Bigl(\frac{t^2}{(\mu^2 d_2 + \nu^2 d_3)\,\norm{\bA}^2_{\sf op}}\Bigr)\Bigr) \le 2\exp\Bigl(-\Omega\Bigl(\frac{t^2}{\norm{\bA}^2_F}\Bigr)\Bigr)\,,
    \end{align}
    where in the last step we used that $\mu^2 d_2 \lesssim 1$ and $\nu^2 d_3\lesssim 1$ by assumption and upper bounded operator norm by Frobenius norm. The first part of the theorem then follows by triangle inequality.
\end{proof}

\section{Integro-local CLT}
\label{app:integrolocal}

In this section, we prove a lower bound on the probability that a vector given by projecting a random bitstring along some ``regular'' directions lies inside a prescribed small box. Given a random vector $X$, we denote its characteristic function by $\phi_X$, that is,
\begin{equation}
    \phi_X(\lambda) \triangleq \E{e^{\iu \iprod{\lambda,X}}}\,.
\end{equation}

\subsection{Proof of Borovkov's bound}
\label{sec:borovkov}

We will use a multi-dimensional integro-local central limit theorem due to~\cite{borovkov2017generalization}. That result is stated asymptotically, so here we give a non-asymptotic version.

Let $X_1,\ldots,X_d$ be independent mean-zero random vectors in $\R^m$, and define $S\triangleq \sum_i X_i$. For every $i\in[d]$, define $Q^{(i)} \triangleq \E{X_iX_i^\top}$ and $Q \triangleq \sum_i Q^{(i)} = \E{SS^\top}$. Let $R = [a_1,a_1+\Delta_1] \times \cdots \times [a_m,a_m+\Delta_m]$. Define
\begin{equation}
    \underline{\Delta} \triangleq \min_i \Delta_i\,.
\end{equation}

\begin{theorem}\label{thm:actual_borovkov}
    Let $\epsilon, \eta > 0$. Suppose there is $\tau > 0$ such that $\norm{X_i} \le \tau$ for all $i$ with probability one. For any $r_1 \le r_2$, define
    \begin{equation}
        \Lambda(r_1,r_2) \triangleq \det(Q)^{1/2} \cdot \sup_{r_1\le \norm{\lambda} \le r_2}|\phi_S(\lambda)|\,.
    \end{equation}
    Suppose
    \begin{equation}
        \max(\tau^2 m, \delta^2 m^3, c^2 m^2 \overline{\Delta}^2)\cdot \log(1/\epsilon) \le \sigma_{\min}(Q) / 10\,. \label{eq:small}
    \end{equation}
    Then
    \begin{equation}
        \Pr{S\in R} \ge 0.9^m \,\Bigl(\prod_j \Delta_j\Bigr) \cdot \frac{1}{\det(2\pi Q)^{1/2}}\cdot \Bigl\{\exp\Bigl(-\frac{1}{2}a^\top Q^{-1} a\Bigr) - \calE\Bigr\} - \eta
    \end{equation}
    \begin{equation}
        \Pr{S\in R} \le 1.1^m\, \Bigl(\prod_j \Delta_j\Bigr) \cdot \frac{1}{\det(2\pi Q)^{1/2}}\cdot \Bigl\{\exp\Bigl(-\frac{1}{2}a^\top Q^{-1} a\Bigr) + \calE\Bigr\} + \eta
    \end{equation}
    for
    \begin{multline}
        \calE \lesssim \exp(-\Omega(m))\cdot \biggl\{\Lambda\biggl(\sqrt{\frac{m\log(1/\epsilon)}{\sigma_{\min}(Q)}}, O\Bigl(\frac{\sqrt{m}}{\eta\underline{\Delta}}\Bigr)\biggr) \cdot \Bigl(\frac{1}{\eta\underline{\Delta}}\Bigr)^m \\
        + \epsilon + \frac{\tau^3(m\log(1/\epsilon))^{3/2}}{\sigma_{\min}(Q)^{3/2}} + m\sqrt{\frac{\log(1/\epsilon)}{\sigma_{\min}(Q)}}\cdot \max(\sqrt{m}\eta\underline{\Delta}, \overline{\Delta}/c)\biggr\}\,. \label{eq:calEdef}
    \end{multline}
    where $\overline{\Delta} \triangleq \max_j \Delta_j$ and $\underline{\Delta} \triangleq \min_j \Delta_j$.
\end{theorem}

We will first bound the probability that a ``smoothing'' of $S$ lies in $R$. Define
\begin{equation}
    \wt{S} = S + \delta\cdot \zeta
\end{equation}
for $\delta>0$ a constant to be tuned later and $\zeta$ a random vector with characteristic function
\begin{equation}
    \phi_\zeta(\lambda) = \prod^m_{j=1} \max(0, 1 - |\lambda_j|)\,.
\end{equation}
Let $R_0$ denote the box $[0,\Delta_1]\times\cdots\times[0,\Delta_m]$. Then $\Pr{\wt{S} \in R}$ is simply the volume of $R_0$ times the density at $a = (a_1,\ldots,a_m)$ of the convolution of the distribution of $\wt{S}$ with $\mathrm{unif}(R_0)$ (the uniform distribution over $R_0$). The Fourier transform of this convolution is the pointwise product of the Fourier transforms of the constituent distributions, so by the Fourier inversion formula,
\begin{equation}
    \Pr{\wt{S} \in R} = \frac{\prod_j \Delta_j}{(2\pi)^m} \int_{\R^m} e^{-i\langle \lambda, a\rangle} \phi_S(\lambda) \cdot \phi_\zeta(\delta \lambda) \cdot \phi_{\mathrm{unif}(R_0)}(\lambda)\, \D \lambda\,. \label{eq:main_integral}
\end{equation}

We split this integral into two parts: the one for which $\norm{\lambda} \ge \gamma$ for some $\gamma > 0$ to be specified later, and the complement.

\subsubsection{Large $\lambda$}

By our choice of $\zeta$, the integrand in Eq.~\eqref{eq:main_integral} vanishes for any $\lambda$ satisfying $|\lambda_j| > 1/\delta$ for some $j\in[m]$. Therefore,
\begin{equation}
    \Bigl|\int_{\norm{\lambda} \ge \gamma} e^{-i\langle \lambda, a\rangle} \phi_S(\lambda) \cdot \phi_\zeta(\delta \lambda) \cdot \phi_{\mathrm{unif}(R_0)}(\lambda)\, \D \lambda\Bigr| \le (1/\delta)^m \cdot \sup_{\gamma \le \norm{\lambda} \le \sqrt{m}/\delta} |\phi_S(\lambda)|\,.
\end{equation}

\subsubsection{Small $\lambda$}

Recalling that the distribution of $S$ is itself a convolution, we have $\phi_S(\lambda) = \prod^d_{j=1} \phi_{X_j}(\lambda)$. We first bound $\phi_{X_j}(\lambda) - 1$. Observe that because $\E{X_k} = 0$,
\begin{equation}
    \phi_{X_j}(\lambda) - 1 + \frac{1}{2}\lambda^\top Q^{(j)} \lambda = \E{e^{\iu \iprod{\lambda, X_j}} - 1 - i\iprod{\lambda, X_j} + \frac{1}{2}\iprod{\lambda, X_j}^2}\,.
\end{equation}
We use the elementary inequality $|e^{\iu z} - 1 - iz + z^2/2| \le |z^3|/6$ to get
\begin{align}
    \mathbb{E}\Bigl[\bigl|e^{\iu \iprod{\lambda, X_j}} - 1 - i\iprod{\lambda, X_j} + \frac{1}{2}\iprod{\lambda, X_j}^2\bigr|\Bigr] &\lesssim \mathbb{E}\bigl[|\iprod{\lambda, X_j}|^3\bigr] \le \gamma\tau \cdot \E{\iprod{\lambda,X_j}^2} = \gamma\tau \cdot \lambda^\top Q^{(j)}\lambda \triangleq \epsilon_{j,\lambda} \,.
\end{align}


We conclude that
\begin{equation}
    \phi_{X_j}(\lambda) - 1 =  -\frac{1}{2}\lambda^\top Q^{(j)}\lambda \pm \epsilon_{j,\lambda} \,. 
\end{equation}
We will choose $\gamma$ small enough and $\tau$ large enough that as long as $\norm{\lambda} < \gamma$,
\begin{equation}
    \frac{1}{2}|\lambda^\top Q^{(j)} \lambda| \le 1/10 \ \ \text{and} \ \ \epsilon_{j,\lambda} \le 1/100\,, \label{eq:iou1}
\end{equation}
at which point we can conclude, using the elementary inequality $|\log(1 + x + \xi) - (x + \xi)| \le 6|\xi|$ for all $|x|\le 1/10$ and $|\xi| \le 1/100$, that
\begin{equation}
    \log(\phi_{X_j}(\lambda)) = -\frac{1}{2}\lambda^\top Q^{(j)} \lambda\cdot (1 \pm 6\epsilon_{j,\lambda})\,,
\end{equation}
and thus that for $\epsilon' \triangleq 6\sup_{j\in[d], \norm{\lambda} < \gamma} \epsilon_{j,\lambda}$,
\begin{equation}
    \phi_S(\lambda) = \exp\Bigl(-\frac{1}{2}\lambda^\top Q\lambda \cdot (1 \pm \epsilon')\Bigr)\,.
\end{equation}
Finally, we argue that $|\phi_{\zeta}(\delta\lambda)|$ and $|\phi_{\mathrm{unif}(R_0)}(\lambda)$ are both close to $1$. 
If we take $\epsilon'' > 0$ given by
\begin{equation}
    \epsilon'' \triangleq \gamma \cdot \max(\delta m, \sqrt{m}\overline{\Delta}/c)\,,
\end{equation}
for sufficiently small constant $0 < c < 1$, then $\norm{\lambda} < \gamma$ implies that
\begin{equation}
    |\phi_{\zeta}(\delta\lambda) - 1| \le |(1 - \epsilon''/m)^m - 1| \le \epsilon''
\end{equation}
and that
\begin{equation}
    |\phi_{\mathrm{unif}(R_0)} - 1| \le \sup_{x\in R_0} |e^{\iu \iprod{\lambda,x}} - 1| \lesssim \gamma\sup_{x\in R_0}\norm{x} \le c\epsilon''\,.
\end{equation}
So provided $\gamma,\tau$ are chosen such that Eqs.~\eqref{eq:iou1} is satisfied and $\epsilon'' \le 1$, then we have
\begin{align}
    \MoveEqLeft\int_{\norm{\lambda} < \gamma} e^{-i\iprod{\lambda,a}} \phi_S(\lambda)\cdot \phi_\zeta(\delta\lambda) \cdot \phi_{\mathrm{unif}(R_0)}(\lambda)\,\D \lambda \\
    &= \int_{\norm{\lambda}<\gamma} e^{-i\iprod{\lambda,a}} \exp\Bigl(-\frac{1}{2}\lambda^\top Q\lambda \cdot (1 \pm \epsilon')\Bigr)\cdot (1 \pm 3\epsilon'') \,\D \lambda
\end{align}

We can use the elementary inequality $|e^{-z^2/2} - e^{-z^2/2\cdot (1 + c)}| \le |c| e^{-c^2/3}$ for all $|c| \le 1/4$ to get
\begin{align}
    \MoveEqLeft\int_{\norm{\lambda} < \gamma} e^{\iu \iprod{\lambda, a}}\cdot \exp\Bigl(-\frac{1}{2}\lambda^\top Q \lambda\cdot (1 \pm \epsilon')\Bigr)\cdot (1 \pm 3\epsilon'') \, \D \lambda  \\
    &=  \int_{\norm{\lambda} < \gamma} e^{\iu \iprod{\lambda, a}}\cdot \exp\Bigl(-\frac{1}{2}\lambda^\top Q \lambda\Bigr)\, \D \lambda \pm O(\epsilon' + \epsilon'') \int_{\R^m} \exp\Bigl(-\frac{1}{3}\lambda^\top Q \lambda\Bigr)\,\D \lambda \\
    &= \int_{\norm{\lambda} < \gamma} e^{\iu \iprod{\lambda,a}}\cdot \exp\Bigl(-\frac{1}{2}\lambda^\top Q \lambda\Bigr) \pm O(\epsilon' + \epsilon'') \cdot \det(3\pi Q^{-1}  / 2)^{1/2}
\end{align}
where we used that the volume of the unit ball in $m$ dimensions is $O(1)$. 

It will be convenient to upper bound this by an integral over all of $\R^m$. Note that the integral above, with the domain of integration replaced with $\brc{\lambda: \norm{\lambda} \ge \gamma}$, has integrand bounded in magnitude by $\exp(-\frac{1}{2}\lambda^\top Q \lambda)$, and the integral of this over $\norm{\lambda} \ge \gamma$ is at most $\det(2\pi Q^{-1})^{1/2}\cdot \Pr[g\sim\calN(0,Q^{-1})]{\norm{g} \ge \gamma} \lesssim \det(2\pi Q^{-1})^{1/2} \cdot \exp(-\gamma^2\sigma_{\min}(Q)/2m)$, so the above is bounded by
\begin{equation}
    \int_{\R^m} e^{\iu \iprod{\lambda,a}} \cdot \exp\Bigl(-\frac{1}{2}\lambda^\top Q \lambda\Bigr) \pm O(\epsilon' + \epsilon'')\cdot \det(3\pi Q^{-1}/2)^{1/2} \pm O\bigl(\det(2\pi Q^{-1})^{1/2} \cdot \exp(-\gamma^2\sigma_{\min}(Q)/2m)\bigr)\,.
\end{equation}
We apply a change of variable: let $\lambda = \lambda' Q^{-1/2}$ and $a = a' Q^{1/2}$. Then we can rewrite the above as
\begin{align}
    \MoveEqLeft\frac{1}{\det(Q)^{1/2}}\int_{\R^m} e^{-i\iprod{\lambda',a'}} \exp\Bigl(-\frac{1}{2}\norm{\lambda'}^2\Bigr)\,\D \lambda' \\ 
    &\pm O\Bigl((\epsilon' + \epsilon'')\cdot \det(3\pi Q^{-1}/2)^{1/2} \pm \det(2\pi Q^{-1})^{1/2} \cdot \exp(-\gamma^2\sigma_{\min}(Q)/2m)\Bigr)\,.
\end{align}
Finally, note that
\begin{equation}
    \frac{1}{\det(Q)^{1/2}}\int_{\R^m} e^{-i\iprod{\lambda',a'}}\exp\Bigl(-\frac{1}{2}\norm{\lambda'}^2\Bigr)\, \D \lambda' = \frac{(2\pi)^{m/2}}{\det(Q)^{1/2}}  \exp\Bigl(-\frac{1}{2}a^\top Q^{-1} a\Bigr)
\end{equation}

\subsubsection{Combining the bounds and setting parameters}

We conclude that
\begin{equation}
    \frac{\Pr{\wt{S}\in R}\cdot \det(Q)^{1/2}}{\prod_j \Delta_j} = \frac{1}{(2\pi)^{m/2}} \exp\Bigl(-\frac{1}{2}a^\top Q^{-1} a\Bigr) \pm O(\calE)
\end{equation}
where
\begin{equation}
    \calE \triangleq  \det(Q^{1/2}/2\pi\delta)\cdot \sup_{\gamma\le \norm{\lambda} \le \sqrt{m}/\delta}|\phi_S(\lambda)| + (\epsilon'+\epsilon'')\cdot (8\pi/3)^{-m/2} + (2\pi)^{-m/2}\cdot \exp(-\gamma^2\sigma_{\min}(Q)/2m)\,. \label{eq:Edef}
\end{equation}

It remains to set $\gamma$. For any $\epsilon > 0$, we can take 
\begin{equation}
    \gamma \triangleq \sqrt{\frac{m\log(1/\epsilon)}{\sigma_{\min}(Q)}}\,. \label{eq:gamdef}
\end{equation}
In this case, Eq.~\eqref{eq:small} ensures that Eq.~\eqref{eq:iou1} holds, $\epsilon'' \le 1$, and 
\begin{align}
    \calE \le \exp(-\Omega(m))\cdot \Bigl(\Lambda(\gamma,\sqrt{m}/\delta) \cdot (1/\delta)^m + \epsilon + \frac{\tau^3 (m\log(1/\epsilon))^{3/2}}{\sigma_{\min}(Q)^{3/2}} + \frac{m\sqrt{\log(1/\epsilon)}}{\sqrt{\sigma_{\min}(Q)}}\cdot \max(\delta\sqrt{m}, \overline{\Delta}/c)\Bigr)\,.
\end{align}



\subsubsection{From $\wt{S}$ to $S$}

It remains to relate $\Pr{\wt{S}\in R}$ back to $\Pr{S\in R}$. For this, we will need the following lemma bounding the tails of the random vector $\zeta$. Note that the coordinates of $\zeta$ are independent and identically distributed. In particular, each $\zeta_j$ has density at $x\in\R$ given by $\frac{1}{2\pi} \int \max(0,1-|\lambda|) e^{-i\iprod{\lambda,x}}\, \D \lambda = \frac{1 - 1\cos(x)}{\pi x^2}$. We can thus conclude the following:

\begin{lemma}
    For any $\eta > 0$ and $j\in[d]$, $\Pr{|\zeta_j| > 1/\eta} \lesssim \eta$.
\end{lemma}

\begin{proof}
    This follows by integrating the density, which we can pointwise upper bound by $O(1/x^2)$.
\end{proof}

\begin{corollary}
    Let $\overline{R}$ and $\underline{R}$ denote the sets $[a_1 - \delta/\eta, a_1 + \Delta_1 + \delta/\eta]\times \cdots \times [a_m - \delta/\eta, a_m + \Delta_m + \delta/\eta]$ and $[a_1 +\delta/\eta, a_1 + \Delta_1 - \delta/\eta]\times \cdots \times [a_m + \delta/\eta, a_m + \Delta_m - \delta/\eta]$. Then
    \begin{equation}
        \Pr{\wt{S}\in \underline{R}} - \eta \le \Pr{S\in R} \le \Pr{\wt{S} \in \overline{R}} + \eta\,.
    \end{equation}
\end{corollary}

\noindent Theorem~\ref{thm:actual_borovkov} follows by taking $\delta = 0.1\epsilon\underline{\Delta}$ in the Corollary.

\subsection{Applying Borovkov's bound}


\begin{theorem}\label{thm:borovkov}
    
    Let $R$ be any product of intervals $R = [a_1,a_1 + \Delta_1]\times\cdots\times[a_m,a_m+\Delta_m]$ for $\Delta_1,\ldots,\Delta_m > 0$. Denote $a = (a_1,\ldots,a_m)$, and let $\underline{\Delta} = \min_i \Delta_i$ and $\overline{\Delta} = \max_i \Delta_i$.

    Let $\rho, \upkappa, \underline{r}, \overline{r} > 0$, and let $v_1,\ldots,v_m \in \R^d$ be vectors that satisfy
    \begin{equation}
        \norm{v_j}_\infty \le \frac{\rho}{\sqrt{d}}\norm{v_j}_2 \label{eq:rhocor}
    \end{equation}
    \begin{equation}
        \underline{r} \le \norm{v_j} \le \overline{r}\,. \label{eq:vbound}
    \end{equation}
    \begin{equation}
        \frac{|\iprod{v_i,v_j}|}{\norm{v_i} \cdot \norm{v_j}} \le \upkappa \label{eq:angles}
    \end{equation}
    for $0 \le \upkappa \ll 1/m$.
    Denote by $\bV\in\R^{m\times d}$ the matrix whose rows consist of $v_1,\ldots,v_m$ and define
    \begin{equation}
        \Lambda(r_1,r_2) \triangleq \det(\bV\bV^\top)^{1/2} \cdot \sup_{r_1 \le \norm{\lambda}\le r_2} |\phi_{\bV x}(\lambda)|\,.
    \end{equation}
    Suppose additionally the following three conditions hold:
    \begin{equation}
        \Lambda\Bigl(\Omega(\sqrt{m}\norm{a}/\underline{r}^2), \underline{\Delta}^{-1} \exp(O(\norm{a}^2/\underline{r}^2))  \prod_j \norm{v_j}/\Delta_j\Bigr) \lesssim \underline{\Delta}^m  \exp(-\Theta(m^2 + m\norm{a}^2/\underline{r}^2)) \prod_j (\Delta_j / \norm{v_j})^m\,, \label{eq:sufficient_lam}
    \end{equation}
    \begin{equation}
        \frac{\rho^3 m^3\overline{r}^3 \norm{a}^3}{\underline{r}^6 d^{3/2}} \lesssim \exp(-\norm{a}^2 / 4\underline{r}^2) \,, \label{eq:sufficient_penultimate}
    \end{equation}
    \begin{equation}
        \frac{m\norm{a}}{\underline{r}^2}\cdot (\sqrt{m}\underline{\Delta}\cdot \prod_j \Delta_j / \norm{v_j} + \overline{\Delta}) \lesssim \exp(-\norm{a}^2 / 4\underline{r}^2)\,. \label{eq:sufficient_last}
    \end{equation}
    Then for $x\sim\brc{\pm 1}^d$ and $g\sim\calN(0,\Id)$, we have that
    \begin{equation}
        \Pr{\bV x \in R} \gtrsim 0.9^m \Bigl(\prod_j \Delta_j\Bigr)\cdot \frac{1}{\det(2\pi \bV \bV ^\top)^{1/2}}\cdot \exp\Bigl(-\frac{1}{2}a^\top (\bV \bV ^\top)^{-1} a\Bigr)
    \end{equation}
    \begin{equation}
        \Pr{\bV x \in R} \lesssim 1.1^m \Bigl(\prod_j \Delta_j\Bigr)\cdot \frac{1}{\det(2\pi \bV \bV ^\top)^{1/2}}\cdot \exp\Bigl(-\frac{1}{2}a^\top (\bV \bV ^\top)^{-1} a\Bigr)\,.
    \end{equation}
\end{theorem}

\begin{proof}
    For $i\in[d]$, let $\bV ^i\in\R^m$ denote $i$-th column of $\bV $. Let $x$ denote a random element of $\brc{\pm 1}^d$. In the notation of Section~\ref{sec:borovkov}, consider the following random variables. For every $i\in[d]$, define
    \begin{equation}
        X_i \triangleq \bV ^i \cdot x_i\,,
    \end{equation}
    so that $S \triangleq \sum_i X_i = \bV  x \in \R^m$ satisfies
    \begin{equation}
        S_j = \iprod{v_j, x}\,.
    \end{equation}
    Note that the covariance $Q$ in Theorem~\ref{thm:actual_borovkov} can be taken to be $Q = \sum_i \bV^i(\bV^i)^\top = \bV\bV^\top$. By Eq.~\eqref{eq:angles}, for every distinct $j,j'\in[m]$, we have $|Q_{jj'}| \le \upkappa \sqrt{Q_{jj} Q_{j'j'}}$. By Fact~\ref{fact:condnumber_angle}, this implies that $\sigma_{\min}(Q) \ge (1 - \upkappa m)\underline{r}^2$ and $\det(Q) \ge (1 - \upkappa m)^m \prod_j \norm{v_j}^2$. In particular,
    \begin{equation}
        a^\top Q^{-1} a \le \norm{a}^2 / \sigma_{\min}(Q) = (1 - \upkappa m)^{-1} \norm{a}^2 / \underline{r}^2 \le \norm{a}^2 / 2\underline{r}^2\,.
    \end{equation}
    We can take $\tau$ in Theorem~\ref{thm:actual_borovkov} to be 
    \begin{equation}
        \tau = \max_i \norm{\bV^i} \le \frac{\rho}{\sqrt{d}}\norm{\bV}_F \le \rho\sqrt{m/d}\cdot \overline{r}\,.
    \end{equation}
    So that $\calE$ in Eq.~\eqref{eq:calEdef} is of the same order as $\exp(-a^\top Q^{-1} a/2)$, take 
    \begin{equation}
        \epsilon \asymp \exp(-\norm{a}^2 / 4\underline{r}^2)
    \end{equation}
    with sufficiently small constant factor. Likewise, take 
    \begin{equation}
        \eta \asymp 0.9^m \Bigl(\prod_j \Delta_j\Bigr) \cdot \frac{1}{(2\pi)^{m/2} (1 - \upkappa m)^{m/2} \prod_j \norm{v_j}}\cdot \exp\Bigl(-\frac{1}{2}a^\top Q^{-1} a\Bigr)\,.
    \end{equation}
    Note that
    \begin{equation}
        \exp(-\Theta(m + \norm{a}^2 / \underline{r}^2)) \le \frac{\eta}{\prod_j \Delta_j / \norm{v_j}} \le 1 \,.
    \end{equation}
    Provided that 
    \begin{equation}
        \Lambda\bigl(\Omega({\sqrt{m}\norm{a}}/{\underline{r}^2}), O({\sqrt{m}}/{\eta\underline{\Delta}})\bigr) \lesssim \epsilon\cdot (\eta\underline{\Delta})^m\,, \label{eq:lamless}
    \end{equation}
    the contribution of $\Lambda(\cdot,\cdot)$ to $\calE$ in Eq.~\eqref{eq:calEdef} is dominated by $\epsilon$. Note that 
    \begin{equation}
        \sqrt{m}/\eta\underline{\Delta} \lesssim \underline{\Delta}^{-1} \cdot \exp(O(\norm{a}^2/\underline{r}^2))\cdot \prod_j \norm{v_j} / \Delta_j
    \end{equation}
    and $\Lambda(\cdot,\cdot)$ is clearly non-decreasing in the second argument. So a sufficient condition for Eq.~\eqref{eq:lamless} to hold is that Eq.~\eqref{eq:sufficient_lam} in the hypothesis of Theorem~\ref{thm:borovkov} holds.
    
    
    Provided that
    \begin{equation}
        \frac{\tau^3 m^{3/2} \norm{a}^3 / \underline{r}^3}{\sigma_{\min}(Q)^{3/2}} \lesssim \exp(-\norm{a}^2/4\underline{r}^2)\,,
    \end{equation}
    then the contribution of the penultimate term in the definition of $\calE$ in Eq.~\eqref{eq:calEdef} to $\calE$ is of order $\epsilon$. Because $\sigma_{\min}(Q) \gtrsim \underline{r}^2$ and $\tau \le \rho\sqrt{m/d}\cdot \overline{r}$, a sufficient condition for this is that Eq.~\eqref{eq:sufficient_penultimate} in the hypothesis of Theorem~\ref{thm:borovkov} holds.
    
    Similarly, provided that
    \begin{equation}
        m\sqrt{\frac{\log(1/\epsilon)}{\sigma_{\min}(Q)}}\cdot (\sqrt{m}\eta\underline{\Delta} + \overline{\Delta}) \lesssim \epsilon\,,
    \end{equation}
    then the contribution of the final term in the definition of $\calE$ in Eq.~\eqref{eq:calEdef} to $\calE$ is of order $\epsilon$.
    Because $\sigma_{\min}(Q) \gtrsim \underline{r}^2$ and $\eta \le \prod_j \Delta_j / \norm{v_j}$, using the definition of $\epsilon$ we conclude that a sufficient condition for this is that Eq.~\eqref{eq:sufficient_last} in the hypothesis of Theorem~\ref{thm:borovkov} holds.
    
    So $\calE$ in Theorem~\ref{thm:actual_borovkov} is bounded by $O(\epsilon \cdot \exp(-\Omega(m))) \ll \exp(-a^\top Q^{-1} a / 2)$, and the bound in the Theorem thus yields the claimed bound in Theorem~\ref{thm:borovkov}.
\end{proof}

\noindent We used the following elementary bound in the above proof:

\begin{fact}\label{fact:condnumber_angle}
    For $m\in\mathbb{N}$, let $\upkappa \le 1/m$. If a collection of vectors $v_1,\ldots,v_m\in\R^d$ satisfies $\frac{|\iprod{v_i,v_j}|}{\norm{v_i}\cdot\norm{v_j}} \le \upkappa$ for all distinct $i,j$, and $r_1 \le \norm{v_i} \le r_2$ for all $i$, then if $V\in\R^{m\times d}$ denotes the matrix whose rows consist of $v_1,\ldots,v_m$, we have
    \begin{equation}
        \det(VV^\top) \ge (1 - \upkappa m)^m \cdot \prod^m_{i=1} \norm{v_i}^2\,,
    \end{equation}
    Furthermore, all eigenvalues of $VV^\top$ lie in the interval $[(1 - \upkappa m)r_1^2, (1 + \upkappa m) r_2^2]$.
\end{fact}

\begin{proof}
    Let $D = \diag(1/\norm{v_1},\ldots,1/\norm{v_m})$. Then $DVV^\top D$ has diagonal entries equal to 1, and off-diagonal entries bounded in magnitude by $\upkappa$. By Gershgorin's disk theorem, all the eigenvalues of $DVV^\top D$ lie in $[1- \upkappa m, 1 + \upkappa m]$, so the second part of the claim follows. Furthermore, by multiplicativity of the determinant, $\det(VV^\top) = \det(DVV^\top D) \cdot \prod^m_{i=1} \norm{v_i}^2 \ge (1 - \upkappa m)^m \cdot \prod^m_{i=1} \norm{v_i}^2$.
\end{proof}

\section{Non-arithmeticity}
    \label{app:arithmetic}

    \subsection{Smoothed matrices are non-arithmetic}

        Here we justify Assumption~\ref{assume:nonarithmetic} by showing it holds in a smoothed analysis setting.

        \begin{lemma}\label{lem:singlecol}
            Consider any $x\in\brc{\pm 1}^d$ and $c\in \R$ bounded away from zero. Suppose $v\in\R^d$ 
            is drawn from a smoothed distribution, that is, $v = v' + \frac{\sigma}{\sqrt{d}}\gamma_j$ for some deterministic vector $v'$, and $\gamma \sim \calN(0,\Id)$. Then if $c\sigma = \Omega(1/\sqrt{d})$, we have
            \begin{equation}
                \Pr{|\cos(c\iprod{x,v}) - 1| \ge \Omega(1/\sqrt{d})} \gtrsim \Omega(1)\,.
            \end{equation}
        \end{lemma}

        \begin{proof}
            It suffices to show that $\mathrm{dist}(c\iprod{x, v}, 2\pi\mathbb{Z})\ge\Omega(1/\sqrt{d})$ with probability $\Omega(1)$. Note that $c\iprod{x, v} = c\iprod{x,v'} + c\iprod{x,\gamma}$ is an independent sample from some $\calN(\mu,c^2\sigma^2)$. Note that $\mathrm{dist}(c\iprod{x, v}, 2\pi\mathbb{Z})$ is identical in distribution to the random variable $\mathrm{dist}(\zeta, 2\pi\mathbb{Z})$ where $\zeta$ is an independent sample from $\calN((\mu \ \mathrm{mod} \ \pi), c^2\sigma^2)$. Fact~\ref{fact:elem_gaussian} implies the claimed bound.
        \end{proof}

        \begin{fact}\label{fact:elem_gaussian}
            For any $\mu\in[0,2\pi)$ and $\tau = \Omega(1/\sqrt{d})$, $\Pr[\zeta\sim\calN(\mu,\tau^2)]{\mathrm{dist}(\zeta,2\pi\mathbb{Z}) \ge \Omega(1/\sqrt{d})} \ge \Omega(1)$
        \end{fact}

        \noindent We can use Lemma~\ref{lem:singlecol} to conclude the following:

        \begin{lemma}\label{lem:fixedx}
            Let $x\in\brc{\pm 1}^d$. Suppose $\Sig_1,\ldots,\Sig_m$ are generated from a smoothed distribution, that is, there exist deterministic matrix $\Sig'_1,\ldots,\Sig'_m$ such that every entry of every $\Sig_i$ is generated by perturbing the corresponding entry of $\Sig'_i$ by a Gaussian with variance $\sigma^2/d$. Suppose the columns of $\Sig'_1,\ldots,\Sig'_m$ have norm at most $\alpha$ for $\alpha \gg \sigma$. Then for any $r,R$ for which $r\sigma \ge \Omega(\sqrt{m/d})$, we have that with probability at least $1 - (Rm\alpha d^2 e^{O(\sqrt{d})})^m\cdot \exp(-\Omega(d))$ over the randomness of $\Sig_i$, for all $T$ of size at least $(1 - o(1))d$,
            \begin{equation}
                \sup_{r\le \norm{\lambda} \le R} \prod^m_{i=1} \prod_{j\in T} |\cos(\lambda_i \iprod{x, (\Sig_i)_{:j}})| \le e^{-\Theta(\sqrt{d})} \label{eq:phi_split}
            \end{equation}
        \end{lemma}

        \begin{proof}
            Denote the expression in the supremum by $\phi_T(\lambda)$. Note that $\phi_T(\lambda)$ is $L$-Lipschitz for 
            \begin{equation}
                L \triangleq \sum^m_{i=1} \sum^d_{j=1} \iprod{x, (\Sig_i)_{:j}} \lesssim m\alpha d^2\,.    
            \end{equation}
            For $\eta = e^{-\Theta(\sqrt{d})}$, let $\mathcal{S}$ be an $\eta / 2m\alpha d^2$-net over the set of $\lambda$ for which $r \le \norm{\lambda} \le R$; note that we can take $|\mathcal{S}| \le (Rm\alpha d^2/\eta)^m$. If $\sup_{\lambda\in\mathcal{S}} \phi_S(\lambda) \le \eta/2$,  $\sup_{r \le \norm{\lambda} \le R} \phi_T(\lambda) \le \eta$. It thus remains to establish the former.

            For any $\lambda\in \mathcal{S}$, let $i\in[m]$ denote the index for which $|\lambda_i| \ge r/\sqrt{m}$. Because we are assuming $r\sigma/\sqrt{m} \ge \Omega(1/\sqrt{d})$, we conclude by Lemma~\ref{lem:singlecol} and standard binomial tail bounds that with probability $1 - \exp(-\Omega(d))$ over the randomness of $\Sig_i$, there exist $\Omega(d)$ indices $j\in[d]$ for which $|\cos(\lambda_i \iprod{x, (\Sig_i)_{:j}}) - 1| \ge \Omega(1/\sqrt{d})$, in which case for any $T$ of size at least $(1 - o(1))d$, we must have $\phi_T(\lambda) \le e^{-\Theta(\sqrt{d})}$. By a union bound over $\mathcal{S}$, we conclude the proof of the lemma.
        \end{proof}

        \noindent Lemma~\ref{lem:fixedx} shows that for any fixed $x$, with high probability over $\Sig_i$'s, the inequality in Assumption~\ref{assume:nonarithmetic}. We can then reverse quantifiers and conclude that with high probability over $\Sig_i$'s, there is a large fraction of $x$'s for which the inequality in Assumption~\ref{assume:nonarithmetic} holds.

    \subsection{Relating Assumption~\ref{assume:nonarithmetic} to the characteristic function}
        \label{sec:charfunction}

        Given $x\in\brc{\pm 1}^d$ and $T\subseteq[d]$ (playing the role of $S^c$ in the proof of Lemma~\ref{lem:main_sculpt}), let $\bV \in \R^{m\times d}$ denote the matrix whose rows consist of $x^\top (\Sig_i)_{:,T}$ for $i\in[m]$. Given $y\in\brc{\pm 1}^d$ and $\lambda\in\R^m$, we can write
        \begin{equation}
            \iprod{\lambda, \bV y} = \sum^m_{i=1} \lambda_i x^\top (\Sig_i)_{:,T} y = \sum^m_{i=1} \sum_{j\in T} \iprod{x,(\Sig_i)_{:,j}} y_j\,,
        \end{equation}
        so noting that $\E{e^{\iu a\cdot y}} = \cos(a)$ for any $a\in\R$, we conclude that the characteristic function of $\bV y$ is given by
        \begin{equation}
            \E[y\sim\brc{\pm 1}^d]{e^{\iu\iprod{\lambda,\bV y}}} = \prod^m_{i=1} \prod_{j\in T} \cos(\lambda_i \iprod{x,(\Sig_i)_{:,j}})\,.
        \end{equation}
        In other words, the condition Eq.~\eqref{eq:sufficient_lam} in Theorem~\ref{thm:borovkov}, in our applications thereof, is satisfied provided Assumption~\ref{assume:nonarithmetic}.

\end{document}